\documentclass{article}

\usepackage[margin=1in]{geometry}

\usepackage{amssymb}

\usepackage{booktabs} % For formal tables
\usepackage[ruled]{algorithm2e} % For algorithms

\SetAlFnt{\small}
\SetAlCapFnt{\small}
\SetAlCapNameFnt{\small}
\SetAlCapHSkip{0pt}
\IncMargin{-\parindent}

\usepackage{natbib}

\setcitestyle{authoryear}
\makeatletter
\let\algorithm\@undefined
\let\listofalgorithms\@undefined
\makeatother

\usepackage[utf8]{inputenc}

\usepackage{amsmath,amsfonts,amsthm}

\usepackage{amssymb}

\usepackage{mathtools}

\usepackage{thmtools}
\usepackage{thm-restate}

\usepackage{bm}
\usepackage{bbm}

\usepackage{xcolor}

\usepackage{tikz}
\usetikzlibrary{arrows,arrows.meta}
\usetikzlibrary{positioning,fit}
\usetikzlibrary{patterns}
\usetikzlibrary{decorations.pathmorphing,decorations.text,shapes.geometric}
\usetikzlibrary{fadings}
\usepackage{pgfplots}
\usepgfplotslibrary{fillbetween}

\usepackage{nicefrac}

\usepackage{multicol}

\usepackage{relsize}

\usepackage{commath}

\usepackage{enumitem}

\usepackage{upgreek}

\usepackage{yfonts}

\usepackage{pifont}

\usepackage[integrals]{wasysym}

\usepackage{microtype}

\usepackage{multirow}

\usepackage{subcaption}

\usepackage{hyperref} %After MOST packages.

\usepackage[noabbrev]{cleveref}

\usepackage[leftcaption]{sidecap}

\usepackage{algorithm} %hyperref -> algorithm
\usepackage{algpseudocode}

\usepackage{geometry}

\DeclareMathOperator*{\argmin}{argmin}

\DeclareMathOperator{\Diam}{Diam}
\newcommand{\grad}{\nabla}

\DeclarePairedDelimiter{\ceil}{\lceil}{\rceil}
\DeclarePairedDelimiter{\floor}{\lfloor}{\rfloor}

\newcommand{\1}{\mathbbm{1}}
\newcommand{\R}{\mathbb{R}}
\newcommand{\RNN}{\R_{0+}}
\newcommand{\RPos}{\R_{+}}
\newcommand{\RExt}{\bar{\R}}
\newcommand{\N}{\mathbb{N}}

\newcommand{\Ball}{\mathcal{B}}

\DeclareMathOperator{\sgn}{sgn}

\DeclareMathOperator*{\Var}{\mathbb{V}}

\DeclareMathOperator*{\Expect}{\mathbb{E}}
\DeclareMathOperator*{\EExpect}{\hat{\Expect}}
\DeclareMathOperator*{\Prob}{\mathbb{P}}

\DeclareMathOperator{\Support}{Support}

\newcommand{\x}{\vec{x}}

\newcommand{\w}{\bm{w}} %SVM Weights
\newcommand{\lv}{\mathcal{S}}
\newcommand{\epsv}{\bm{\varepsilon}}
\newcommand{\wv}{\bm{w}}

\newcommand{\distributed}{\thicksim}

\newcommand{\X}{\mathcal{X}}
\newcommand{\Y}{\mathcal{Y}}

\newcommand{\HC}{\mathcal{H}}

\newcommand{\Risk}{\mathrm{R}}
\newcommand{\ERisk}{\hat{\Risk}}

\newcommand{\Utility}{\mathrm{U}}

\newcommand{\Population}{\Omega}
\newcommand{\PopItem}{\omega}

\newcommand{\Mean}{\mathrm{M}}

\newcommand{\MeanAs}{\Mean^{\textsc{as}}}

\newcommand{\Malfare}{\vphantom{\mathrm{W}}\raisebox{1.53ex}{\rotatebox{180}{\ensuremath{\mathrm{W}}}}}

\newcommand{\Welfare}{\mathrm{W}}

\DeclareMathOperator{\IE}{I}
\DeclareMathOperator{\ATK}{Atk}
\DeclareMathOperator{\GEI}{GEI}
\DeclareMathOperator{\Theil}{T}

\newcommand{\ProbDist}{\mathcal{D}}

\newcommand{\frange}{r}

\newcommand{\loss}{\ell}

\newcommand{\LossFunction}{\ell}

\DeclareFontShape{OMX}{cmex}{m}{b}{<-> cmexb10}{}
\SetSymbolFont{largesymbols}{bold}{OMX}{cmex}{m}{b}

\providecommand{\LandauO}{\bm{\mathrm{O}}} %\usefont{OMS}{cmsy}{m}{n}O

\providecommand{\LandauTheta}{\bm{\Uptheta}}
\providecommand{\LandauOmega}{\bm{\Upomega}}

\providecommand{\LandauOTilde}{\smash{\bm{\widetilde{\mathrm{O}}}}\vphantom{\LandauO}}

\newcommand{\todo}[1]{\textcolor{red!50!black}{ToDo: #1}}
\newcommand{\draftnote}[1]{\textcolor{orange!50!black}{Note: #1}}
\newcommand{\cyrus}[1]{\textcolor{green!50!black}{Cyrus: #1}}

\theoremstyle{definition}

\newtheorem{definition}{Definition}[section]

\newtheorem{property}[definition]{Property}

\newtheorem{lemma}[definition]{Lemma}
\newtheorem{theorem}[definition]{Theorem}
\newtheorem{corollary}[definition]{Corollary}

\newtheorem{conjecture}[definition]{Conjecture}
\newtheorem{consequence}[definition]{Consequence}

\newtheorem{example}[definition]{Example}
\newtheorem{observation}[definition]{Observation}

\newcommand{\Input}{\textbf{Input}: }
\newcommand{\Output}{\textbf{Output}: }

\newcommand{\st}{s.t.}
\newcommand{\wrt}{w.r.t.}

\newcommand{\whp}{w.h.p.}

\DeclareMathOperator{\mix}{mix} %TODO

\newcommand{\NGroups}{g}%TODO: use g?
\newcommand{\DSeq}{d}%TODO: use n?

\DeclareMathOperator{\Poly}{Poly}

\newcommand{\PACAlgo}{\mathcal{A}}
\newcommand{\UCAlgo}{\mathcal{A}_{\mathrm{UC}}}

\DeclareMathOperator{\PAC}{PAC}
\DeclareMathOperator{\FPAC}{FPAC}

\DeclareMathOperator{\UC}{UC}

\DeclareMathOperator{\Proj}{Proj}

\newcommand{\Realizable}{\mathrm{Rlz}}
\newcommand{\Agnostic}{\mathrm{Agn}}

\newcommand{\SampleComplexity}{\mathrm{m}}
\newcommand{\TimeComplexity}{\mathrm{t}}

\DeclareMathOperator{\VC}{VC}

\DeclareMathOperator{\PDim}{PD}
\DeclareMathOperator{\FSDim}{FS}

\DeclareMathOperator{\Covering}{\mathcal{N}}
\DeclareMathOperator{\Cover}{\mathcal{C}}
\DeclareMathOperator{\ECover}{\hat{\Cover}}

\newcommand{\cmark}{\ding{51}}
\newcommand{\xmark}{\ding{55}}

\newcommand{\Rade}{\textgoth{R}}
\newcommand{\ERade}{\hat{\Rade}}

\hypersetup{bookmarks=true,pdfstartview={FitH},pdfauthor={Cyrus Cousins},pdftitle={An Axiomatic Theory of Provably-Fair Welfare-Centric Machine Learning},pdfsubject={Machine Learning, Fair Machine Learning, Cardinal Welfare Theory},pdfkeywords={Machine Learning, Fair Machine Learning, Axiomatic Welfare Theory, Cardinal Welfare Theory, PAC Learning, Concentration Inequalities, Empirical Process Theory, Uniform Convergence, Rademacher Averages, Statistical Learning Theory, Computational Learning Theory, Algorithmic Fairness},pdfdisplaydoctitle=true}

\Crefname{figure}{Figure}{Figures}

\author{Cyrus Cousins}
\title{An Axiomatic Theory of Provably-Fair \\ Welfare-Centric Machine Learning} %Fairness with Population Means: Malfare and Welfare}

\newif\ifshowpac
\showpactrue

\newif\ifthesis
\thesisfalse

\newif\ifshowtodos
\showtodosfalse

\newif\iftitletoc
\titletoctrue

\iftitletoc
\usepackage[nottoc]{tocbibind} %TOC bib
\fi

\ifshowtodos
\else
\renewcommand{\todo}[1]{}
\renewcommand{\draftnote}[1]{}
\renewcommand{\cyrus}[1]{}
\fi

\providecommand{\keywords}[1]{\bigskip\centering\textbf{Keywords:} {\small\it #1}\\}

\begin{document}

%\begin{titlepage}

\maketitle

\begin{abstract}

We address an inherent difficulty in welfare-theoretic fair machine learning by proposing an equivalently axiomatically-justified alternative and studying the resulting computational and statistical learning questions.
Welfare metrics quantify \emph{overall wellbeing} across a population %consisting
of one or more groups, and welfare-based objectives and constraints have recently been proposed to incentivize \emph{fair machine learning methods} to produce satisfactory solutions that consider the diverse needs of multiple groups.
%We address an inherent difficulty in this setting;
Unfortunately, many machine-learning problems are more naturally cast as \emph{loss minimization} tasks, rather than \emph{utility maximization}, which complicates direct application of welfare-centric methods to fair machine learning.
In this work, we define a complementary measure, % to \emph{welfare}, 
termed \emph{malfare}, measuring overall societal harm (rather than wellbeing), with axiomatic justification via the standard axioms of cardinal welfare.

We then cast fair machine learning as \emph{malfare minimization} over the \emph{risk values} (expected losses) of each group.
Surprisingly, %except in two trivial cases, 
the axioms of cardinal welfare (malfare) dictate that this is not equivalent to simply defining utility as negative loss.
%By flipping signs in this manner we made directly quantify overall societal harm well in much the same way welfare quantifies societal well-being. Previous welfare based methods have generally defined welfare in an ad hoc manner and we show that in particular seemingly intuitive definitions of welfare in terms of Plus lead to absurdities and or non-trivial and uninterpretable hyperparameters. Furthermore we show that Ma Fair shares the same axiomatic justification as welfare and maybe efficiently estimated from small samples.
Building upon these concepts, we define %develop a notion of generic fair machine learning, termed 
\emph{fair-PAC learning}, where a fair-PAC learner is an algorithm that learns an $\varepsilon$-$\delta$ malfare-optimal model with bounded sample complexity, for \emph{any data distribution}, and %$finite training samples may guarantee high-probability malfare bounds, 
for \emph{any} (axiomatically justified) malfare concept. % satisfying the appropriate fairness axioms.
Finally, we show broad conditions under which, with appropriate modifications, %many
 standard PAC-learners %(e.g., linear classifiers) 
may be converted to fair-PAC learners.
\if 0
and argue that fair-PAC-learners are intuitive and easy to use, as one must only select a %n (axiomatically justified) 
malfare concept (encoding their desired fairness concept), hypothesis class, and confidence guarantees, % $\epsilon$-$\delta$ (probabilistic guarantees), 
and then one receives a provably %probabilistically
 near-optimal model.
\fi
This places fair-PAC learning on firm theoretical ground, as it yields %statistical, and in some cases computational, 
\emph{statistical} and \emph{computational} efficiency guarantees for many well-studied machine-learning models, and is also practically relevant, as it democratizes fair machine learning by providing concrete training algorithms and rigorous generalization guarantees for these models.


\end{abstract}

\iftitletoc %Separete title / TOC / body
\vfill
\pagebreak[4]

\small
\tableofcontents

%\end{titlepage}

\vfill\pagebreak[4]
\else
\pagebreak[1]
\fi

%\subimport{paper/}{content}

\section{Introduction}

\todo{
LEAD WITH examples: sth like ``trained to optimize a single metric of performance. As a result, the decisions made by such algorithms can
have unintended adverse side effects: profit-maximizing loans can have detrimental effects on borrowers
(Skiba and Tobacman, 2009) and fake news can undermine democratic institutions (Persily, 2017).'' FACIAL RECOGNITION
then expand this:
}

It is now well-understood that contemporary machine learning systems for %tasks like
facial recognition~\citep{buolamwini2018gender,cook2019demographic,cavazos2020accuracy}, medical settings \citep{mac2002problem,ashraf2018learning}, and many others exhibit \emph{differential accuracy} across gender, race, and other protected-group membership.
This %immediately yields
causes \emph{accessibility issues} to users of such systems, and can lead to \emph{direct discrimination}, e.g., facial recognition in policing yields disproportionate false-arrest rates, and machine learning in medical technology yields disproportionate %ly bad
 health outcomes, thus exacerbating existing structural and societal inequalities impacting many minority groups.\todo{citations for outcome claims}
%When these decisions are 
In welfare-centric machine learning methods, both \emph{accuracy} and \emph{fairness} are encoded in a single \emph{welfare function} defined on a \emph{collection of subpopulations}.
Welfare is then directly optimized~\citep{rolf2020balancing}\todo{check this.} or constrained~\citep{speicher2018unified,heidari2018fairness} to promote \emph{fair learning} across \emph{all groups}.  %[todo: examples / how welfare encodes tradeoff]. %here both accuracy and fairness are encoded in a single objective, rather than split across the objective and constraints). %; over a group is
%which has presented a direct contrast 2 demographic parity as welfare is the economic science of measuring societal well-being and with deeper axiomatic Theory Fair machine learning promoting welfare has a deeper history and justification.
%
This addresses differential performance and bias issues across groups by ensuring that (1), each group is \emph{seen} and \emph{considered} during training, and (2), an outcome is incentivized that is desirable overall, ideally according to some mutually-agreed-upon welfare function.
Unfortunately, welfare based metrics require a notion of (positive) utility, and we argue that this is not natural to many machine learning tasks, %, particularly in supervised and unsupervised learning, which are generally cast as (negatively connoted)
where we instead \emph{minimize} some \emph{negatively connoted} risk value (expected loss). % (e.g., in decision-theory or with proper scoring rules). % or scoring tasks (as is standard in decision theory and scoring rules).
%We show that when lost functions are the natural measure studied historically in decision Theory and scoring rules regression analysis maximum likelihood and Bayesian methods a notion of lost may be viewed as an ocean of negative utility and we Define the concept of welfare which measures overall population loss rather than utility. 
We thus define a complementary measure to \emph{welfare}, 
termed \emph{malfare}, measuring societal harm (rather than wellbeing).
In particular, malfare arises naturally %as a fairness target, 
when one applies the standard \emph{axioms of cardinal welfare} (with appropriate modifications) to \emph{risk}, rather than \emph{utility}.  %  with axiomatic justification via the standard axioms of cardinal welfare.
With this framework, we cast fair machine learning as a direct \emph{malfare minimization} problem over the \emph{risk values} of each \emph{group}. %, where a group’s malfare is their risk (expected loss).

\draftnote{Comparisons:
Under our axiomatization, which we argue is quite natural, all welfare and malfare functions are \emph{power means} (see~\cref{sec:pop-mean:pow-mean}), which leads to convenient statistical and computational properties.
We contrast the power-mean with the family that arises under the \emph{additive separability} axiom, which is isomorphic under comparison, but is less interpretable, and due to statistical instability, is more difficult to estimate.
We also contrast with the Gini social welfare function (see~\cref{?}), which contains the utilitarian and egalitarian welfare %and malfare
functions, but the remaining cases are not equivalent.
}

%The axioms of cardinal welfare (malfare) ensure that 
Perhaps surprisingly, defining and minimizing a \emph{malfare function} is \emph{not equivalent} to defining and maximizing some \emph{welfare function} while taking utility to be negative loss (except in the trivial cases of egalitarian and utilitarian malfare).
%Welfare and malfare are not equivalent,
This is essentially because nearly every function satisfying the standard axioms of cardinal welfare requires \emph{nonnegative} inputs, and it is not in general possible to contort a loss function into a utility function while satisfying this requirement.
%For example, while the 0-1 loss function $\loss_{01}$, which assigns loss $0$ to correct classifications, and $1$ to incorrect classifications, could
For example, while minimizing the 0-1 loss, which simply counts the number of mistakes a classifier makes, is isomorphic to maximizing the 1-0 gain, which counts number of correct classifications, minimizing some \emph{malfare function} defined on 0-1 loss over groups \emph{is not} in general equivalent to maximizing any \emph{welfare function} defined on 1-0 gain.
More strikingly, for learning problems with unbounded loss functions (i.e., absolute or square error in regression problems, or cross entropy in logistic regression), it is in general not even possible to define a complementary nonnegative gain function without changing the optimal solution. %, \emph{even for a single group}.
\draftnote{Comparisons: Not true for Gini?}

%By flipping signs in this manner we made directly quantify overall societal harm well in much the same way welfare quantifies societal well-being. Previous welfare based methods have generally defined welfare in an ad hoc manner and we show that in particular seemingly intuitive definitions of welfare in terms of Plus lead to absurdities and or non-trivial and uninterpretable hyperparameters. Furthermore we show that Ma Fair shares the same axiomatic justification as welfare and maybe efficiently estimated from small samples.

\ifshowpac

Building upon these concepts, we develop a \emph{mathematically precise} concept of generic fair machine learning, termed 
%\emph{fair probably-approximately-correct (FPAC) learning},
%where a class $\HC$ is FPAC-learnable if there exists an \emph{algorithm} that can \emph{probably} learn an \emph{approximately correct} $\hat{h} \in \HC$, where correctness is ...
\emph{fair probably-approximately-correct} (FPAC) learning, wherein a model class is FPAC-learnable if an $\varepsilon$-$\delta$ malfare-optimal model can be learned %(\whp)
 with \emph{uniformly-bounded sample complexity}, \wrt\ any \emph{fair malfare concept} and per-group \emph{instance distributions}.
In other words, it must be possible to learn a model that, with probability at least $1 - \delta$, has $\varepsilon$-additively optimal malfare, from a \emph{finite sample} whose size depends only on $\varepsilon$, $\delta$, the \emph{group count}, %number of groups $\NGroups$,
and the \emph{model class}, \emph{but not} on the \emph{instance distributions}, nor on the \emph{malfare concept}.
%where the goal is to learn models for which finite training samples may guarantee high-probability malfare bounds, for \emph{any} (axiomatically justified) malfare concept. % satisfying the appropriate fairness axioms.
This definition extends Valiant's~%\citeyear{valiant1984theory}
(\citeyear{valiant1984theory}) %classic
 PAC-learning formalization of machine learning beyond a single group, and we show that, with appropriate modifications, many (standard) PAC-learners may be converted to FPAC learners.
We argue that FPAC-learners are intuitive and easy to use, as one must only select a %n (axiomatically justified) 
malfare concept (encoding their desired fairness concept), model class, and error tolerance, %confidence guarantees, % $\varepsilon$-$\delta$ (probabilistic guarantees), 
and then one receives a provably $\varepsilon$-$\delta$ optimal model. %[TODO: DEFINE PAC?]
Crucially, the class of ``fair malfare concepts'' considered in FPAC learning is not arbitrary, but rather arises from our natural axiomatization, and thus should contain every fair malfare objective that one would want to minimize.
\draftnote{Discuss: why additive, not multiplicative.}

The \emph{uniformly-bounded sample complexity} requirement of FPAC-learnability is substantially stronger than classical concepts of statistical estimability. % like \emph{consistency}
In particular, although \emph{consistent} %\footnote{Note thate the same is not true of \emph{unbiased} estimators; in general we must work with biased estimators of welfare and malfare} %TODO: we do; but is it necessary?
 estimators of (dis)utility values generally imply consistent estimators of welfare or malfare functions, we show that the \emph{rate} at which a consistent estimator converges, and thus \emph{sample complexity}, is strongly impacted by the choice of welfare or malfare function, as well as the instance distributions.
%Consequently, it is possible for a class to not be FPAC-learnable, even if one has \emph{consistent estimators} for all relevant quantities (i.e., the per-group risk values of each model in the class).
Consequently, a class may not be FPAC-learnable, even if there exist consistent estimators for per-group risk values for each model in the class.
This is essentially due to the order of existential quantifiers: \emph{uniform} sample complexity requires a convergence rate to hold \emph{uniformly} over a family of related estimation tasks.
Despite this difficulty, 
%In particular, 
we show via a constructive polynomial reduction that %the worst-case sample-complexity of 
\emph{realizable FPAC-learning} and \emph{realizable PAC-learning} are equivalent, %, where with $\NGroups$ groups, sample complexity is linear in $\NGroups$.
and furthermore, we show, non-constructively, that for learning problems where PAC-learnability implies uniform convergence, it is equivalent to FPAC-learnability.
We also show that when training is possible via \emph{convex optimization}, or by efficient-enumeration of an \emph{approximate cover} of the space of models, then training $\varepsilon$-$\delta$ malfare-optimal models, like risk-optimal models, requires polynomial time.

We argue that our axiomatization of malfare is quite natural, and the resulting family of malfare functions %that arise
admits uniform sample-complexity guarantees.
% is no more than a factor $\NGroups$  if [X], then [fair PAC learnable]
%It remains an open question whether there are model classes that can be PAC-learned, but not FPAC-learned, with \emph{bounded sample complexity}, or \emph{in polynomial time}.
\Cref{sec:comparisons:as} explores the alternative \emph{additive separability} axiom, under which the resulting welfare and malfare families are isomorphic to ours under comparison, %the comparison operator,
however \emph{uniform} sample complexity bounds are unsatisfying (and often impossible), % for this malfare family, 
essentially because additive-error guarantees %are not particularly 
are less meaningful, as the scale, and even %disutility
the units, of additively separable malfare functions vary wildly across the %malfare
 family. %scale across the family.
Our alternative axiomatization essentially nonlinearly normalizes this variation in scale, and also standardizes malfare units to match disutility units; under it, uniform sample complexity guarantees for malfare are possible and meaningful. 
It should be noted that uniform sample complexity bounds for \emph{welfare functions} are generally impossible, %the resulting welfare functions, 
due to the statistical instability of estimating some welfare functions, such as the \emph{geometric mean} (or \emph{Nash social welfare}), thus we argue that, compared to welfare maximization, malfare minimization is not only often more natural, but also more statistically tractable.
\draftnote{More detail cut below:}

\fi

 %guarantee that % the model is fair not just on the training set, but also over similarly distributed populations (i.e., it does not \emph{overfit to fairness}). 
%Ultimately, the value of our work is that the 
% with other Concepts in general are not equivalent as example?
%This contrasts the recently proposed Seldonian-learner model, which ...

\subsection{Related Work}
\label{sec:related}

\draftnote{
TODO THESIS transition, where is this?
}
\draftnote{Contrast with fair allocation work.}

Constraint-based notions of algorithmic fairness \citep{dwork2012fairness} have risen to prominence in %seen increasing attention in the 
fair machine learning, % world,
with the potential to ensure demographic-parity (e.g., equality of opportunity, equality of outcome, or equalized odds), %and other such desiderata).
thus correcting for some forms of data or algorithmic bias. 
%$supposedly desirable quantities. 
While noble in intent and intuitive by design, fairness by demographic-parity constraints has several prominent flaws: most notably, several popular parity constraints are mutually unsatisfiable~\citep{kleinberg2017inherent}, and their constraint-based formulation inherently puts \emph{accuracy} and \emph{fairness} at odds, where additional \emph{tolerance parameters} are required to strike a balance between the two.
Furthermore, recent works \citep{hu2020fair,kasy2021fairness} have shown that \emph{welfare} and even \emph{disadvantaged group utility} can decrease even as fairness constraints are tightened, calling into question whether demographic parity constraints are even beneficial to those they purport to aid.
%based fairness has encountered several issues as with many such methods it's not clear which notion of demographic parity to consider to be fair and the situation is made worse by the fact that many have been shown to be mutually incompatible. 
%Furthermore, 
\draftnote{Is this cut good:
Without deeper axiomatic underpinnings, it is unclear why %we
one should choose one notion of fairness over another, and it is unsatisfying that we do not have a way of comparing and selecting between them without appealing to informal arguments.
}

%Recently, 
Perhaps in response to these issues, some recent work has trended toward welfare-based fairness-concepts, % have risen to prominence in machine learning, %directly %contrasting the balance between loss minimization o
%With the potential
wherein both \emph{accuracy} and \emph{fairness} are encoded in a \emph{welfare function} defined on a \emph{group of subpopulations}.
Welfare is then directly optimized~\citep{hu2020fair,rolf2020balancing,siddique2020learning}\todo{check this} or constrained~\citep{speicher2018unified,heidari2018fairness} to promote \emph{fair learning} across \emph{all groups}.  %[todo: examples / how welfare encodes tradeoff]. %here both accuracy and fairness are encoded in a single objective, rather than split across the objective and constraints). %; over a group is
%which has presented a direct contrast 2 demographic parity as welfare is the economic science of measuring societal well-being and with deeper axiomatic Theory Fair machine learning promoting welfare has a deeper history and justification.
%
Perhaps the most similar to our work is a method of \citet{hu2020fair}, wherein they \emph{directly maximize} empirical welfare over linear (halfspace) classifiers; however as with other previous works, an appropriate utility function must be selected. %, which we avoid by instead using malfare.
We argue that \emph{empirical welfare maximization} is an effective strategy when a %an appropriate and natural 
measure of \emph{utility} is available, but in machine learning contexts, there is no ``correct'' or clearly neutral way to convert loss to utility.\todo{ad hoc methods}
Our strategy avoids this issue by working directly in terms of malfare and risk.\todo{again; on same axiomatization}

\todo{OLD:
The most poignant contrast to existing work we can make is to the \emph{Seldonian learner}~\citep{thomas2019preventing} framework, which can be thought of as extending PAC-learning to learning problems with both \emph{arbitrary constraints} and \emph{arbitrary nonlinear objectives}.
We argue that this generality is harmful to the utility of the concept as a mathematical or practical object, as nearly any machine learning problem can be posed as a constrained nonlinear optimization task.
The utility in FPAC learning is that it is sophisticated enough to handle fairness issues, with a particular axiomatically justified objective, but remains simple enough to study as a mathematical object; in particular reductions between various PAC and FPAC learnable classes are of great value in understanding FPAC learning, which would not be possible in a more general framework.
}

The above works, and even their criticisms, largely focus on fairness concepts in-and-of-themselves, and sparsely treat the issue of showing that a given fairness concept \emph{generalizes} from \emph{training} to \emph{underlying task}. %the \emph{training data} to the \emph{underlying distribution}.
The history of machine learning is fraught with %examples of 
the consequences of ignoring overfitting (as after all, it is human nature to perceive patterns, even where none exist), and we argue they are particularly dire in fairness sensitive settings.
%Overfitting in risk-minimization is now well-studied and well-understood, from VC theory \citep{} to, more recently, sophisticated data-dependent uniform convergence methods \citep{}, characterizing sufficient sample sizes, and bounding the ...... can determine how much concern one should have with those two overfitting.
%In this paper, 
We argue that %not only do the usual issues with a models \emph{overall accuracy} appear, but rather 
\emph{overfitting to fairness} is manifest not only in \emph{generalization error}, but also as models \emph{appearing fair} in training, but failing to be so on the underlying task. %on the training data, but failing to be so on the underlying distribution.
This can mean fairness constraints are satisfied in the training set but violated on the underlying distribution, or that a model %may be badly overfit on 
overfits to \emph{small} or \emph{poorly studied} groups (for which a dearth of data may be available).\draftnote{TODO example?}
More complicated issues may arise; with data-dependent constraints, the feasible model space is data-dependent, and thus learning may exhibit instability, sample complexity depends on these constraints in complicated ways, and in some cases it may not even be possible to satisfy all constraints. % be impossible to satisfy the constraints (on either the training data or the underlying distribution).

\Citet{rothblum2018probably} argue that the \emph{individual-level} metric-fair constraints of \citet{dwork2012fairness} can't be expected to generalize, so they introduce a relaxed notion for which they can show generalizability.\draftnote{TODO criticise individual level fairness}
%The most poignant contrast to existing work we can make is to
\Citet{thomas2019preventing} make similar criticisms, and introduce the \emph{Seldonian learner} %~\citep{thomas2019preventing} 
framework, which can be thought of as extending PAC-learning to learning problems with both \emph{arbitrary constraints} and \emph{arbitrary nonlinear objectives}.
While very useful from a practical perspective to codify the desiderata of fair learning algorithms, the authors investigate individual Seldonian learners of interest, rather than studying the class of Seldonian learners as a \emph{mathematical object}. %and furthermore, 
Such study is difficult, due to the extreme generality of the class,\footnote{
The Seldonian learner concept generalizes earlier fair-learnability concepts, such as \emph{probably approximately correct and fair learning} for approximate metric-fairness \citep{rothblum2018probably}, as well as the standard PAC concept, and indeed, the FPAC concept presented here.
} and also due to the difficulty of bounding sample complexity for \emph{constrained objectives}.\footnote{
Note that the sample complexity of determining whether constraints are even feasible is, in general, unbounded.
}
%In contrast, our % is one of the main goals with the 
\draftnote{Is this cut good:
Our 
FPAC-learning framework was formulated %specifically
 with mathematical study of computational and sample complexity in mind, and this narrower scope is key to showing reductions between various PAC and FPAC learnable classes, which are of great value in understanding FPAC learning.
}

In contrast to the above methods, the FPAC-learning framework considers optimizing a single (unconstrained) cardinal malfare objective. % over a fixed {hypothesis class}.
No fairness tolerance parameters, demographic parity constraints, or explicit utility function definitions are required, and, although nonlinear, all fair malfare objectives, \emph{unlike some fair welfare objectives}, are Lipschitz continuous.
This simplicity also leads naturally to straightforward \emph{statistical analysis} and \emph{generalization guarantees} for malfare objectives, and such generalization guarantees are particularly significant, as with malfare, they control for overfitting of both accuracy and fairness.
Consequently, in many cases, the \emph{sample complexity} (statistical hardness), and often the \emph{computational complexity} (algorithmic hardness) of training malfare-optimal models is comparable to standard (fairness-agnostic) machine-learning methods.

\todo{Old text:
Say: Consequently, in many cases, the \emph{sample complexity} (statistical hardness), and often the \emph{computational complexity} (algorithmic hardness) of training malfare-optimal models is comparable to standard (fairness-agnostic) machine-learning methods.
This directly contrasts demographic-parity models, where guaranteeing $\varepsilon$-optimality \emph{while respecting constraints} often has unbounded sample complexity.
}

\subsection{Contributions}
\label{sec:intro:contrib}
%TODO concl welfare non-eq.

\ifthesis{}%
TODO is this a good transition?
\else{}%
This manuscript is split into two main parts; we first define malfare and derive its properties in \cref{sec:pop-mean,sec:comparisons,sec:stat}, and subsequently we define and explore FPAC learning (and
learnability) in \cref{sec:pac,sec:ftfsl,sec:ccl}.
\fi{}%
We briefly summarize our contributions as follows.
\begin{enumerate}[wide, labelwidth=0pt, labelindent=0pt]\setlength{\itemsep}{3pt}\setlength{\parskip}{0pt}
%
\if 0
\item We introduce the \emph{unit scale} axiom $\Mean(\bm{1}; \wv) = 1$ and \emph{multiplicative linearity} axiom $\alpha\Mean(\lv; \wv) = \alpha\Mean(\lv; \wv)$, which together ensure that the \emph{units} and \emph{scale} of aggregator functions and individual values are directly comparable.
Under the conditions of the Debreu-Gorman theorem, these imply that all welfare (malfare) functions are \emph{power means}.
% the sense of \emph{direct comparison} between populations welfare and utility,  (ensuring welfare is invariant under change-of-units) welfare axioms, justified via arguments regarding 
\fi
%
\item We derive in \cref{sec:pop-mean} the \emph{malfare concept}, extending welfare to measure \emph{negatively-connoted} sentiments, and show that \emph{malfare-minimization} naturally generalizes \emph{risk-minimization} to produce \emph{fairness-sensitive} machine-learning objectives that consider multiple protected groups.
%\item We show that \emph{power-mean 
%
\item We show in \cref{sec:stat} that in many cases, while empirical estimates of welfare and malfare are \emph{statistically biased}, they are \emph{consistent}, and malfare may be \emph{sharply estimated} using %standard
 finite-sample concentration-of-measure bounds.
\item In \cref{sec:comparisons}, we examine the decisions made in \cref{sec:pop-mean}, and explore what would change under alternative axioms and other counterfactuals. %axiomatizations and with alternative decisions.
We also contrast malfare minimization with welfare maximization, and relate both to fairness constraints on \emph{inequality indices}.
This section contextualizes the work as a whole, but may be skipped without impeding %the readability
understanding of the sequel. %subsequent . %remainder of th.
\ifshowpac
\item \Cref{sec:pac} extends PAC-learning to fair-PAC (FPAC) learning, where we consider minimization not only of \emph{risk} (expected loss) objectives, but also of \emph{malfare} objectives.
Both PAC and FPAC learning are parameterized by a \emph{learning task} (model space and loss function), and we explore the rich learnability-hierarchy under variations of these concepts.
In particular, we show that
\begin{enumerate}%[wide, labelwidth=0pt, labelindent=0pt]\setlength{\itemsep}{3pt}\setlength{\parskip}{0pt}
\item for many loss functions, PAC and FPAC learning are \emph{statistically equivalent} (i.e., PAC-learnability implies FPAC-learnability) in \cref{sec:ftfsl}; and
\item standard convexity and coverability conditions sufficient for PAC-learnability are also sufficient for FPAC-learnability in \cref{sec:ccl}.
%Valiant's classic formulation of cl formalize \emph{fair learning
%
\end{enumerate}
While we explore the basic relationships between various learnability classes, many open questions remain, and we hope future work will further characterize these practically interesting and theoretically deep problems.
For brevity, %the 
longer, %and 
more technical proofs are presented in the appendix.
\fi
\end{enumerate}

\draftnote{Bibliography: JSTOR entries? Better source on ``Fairness, equality, and power in algorithmic decision making.'' Proper Russian / German text? Better theil citation? Caps in ``Learnability and the Vapnik-Chervonenkis dimension?''}

\todo{improve stat estimability; mcd? bias? sbf for Bennett? ebennett? rade?}

\todo{John Kleinberg; mechanism design; sendal talk?  Other Hoda paper; fix description Quote at 49:25 https://www.youtube.com/watch?v=rp965fnd3qE Predicting wrong thing at 1:02:00, cost vs health, arrest vs guilt}

%\section{Quantifying Population-Level Sentiment}
%\section{Quantifying Population-Level Sentiment with Aggregator Functions}
\section{Aggregating Sentiment within Populations}
\label{sec:pop-mean}

\todo{Measure theoretic treatment: \url{https://guilhermejacob.github.io/context/3-8-generalized-entropy-and-decomposition-svygei-svygeidec.html}}

\todo{$\lv$ to $\bm{S}$ for \emph{sentiment}?}

\draftnote{par / def are redundant}

A generic \emph{aggregator function} function $\Mean(\lv; \wv)$ quantifies some \emph{sentiment value} $\lv$ \emph{in aggregate} across a population $\Population$ weighted by $\wv$.
In particular, $\lv: \Population \to \RNN$ describes the \emph{values} over which we aggregate, and $\wv$, a probability measure over $\Population$, describes their \emph{weights}.
We assume throughout the \emph{nondegeneracy condition} that $\Support(\wv) = \Population$; this ensures no part of the population is ignored, and simplifies the algebra and presentation.
We also often assume $\abs{\Population} > 1$, and usually $\Population$ is %discrete.
finite, in which case $\lv$ and $\wv$ may be represented as a \emph{sentiment vector} and \emph{probability vector}, respectively.

When $\lv$ measures a \emph{desirable quantity}, generally termed \emph{utility}, the aggregator function is a measure of \emph{cardinal welfare} \citep{moulin2004fair}, and thus quantifies overall \emph{wellbeing}.
We also consider the inverse-notion, that of overall \emph{illbeing}, termed \emph{malfare}, in terms of an \emph{undesirable} $\lv$, generally \emph{loss} or \emph{risk}, which naturally extends the concept.
We show an equivalent \emph{axiomatic justification} for malfare, and argue that its use is more natural in many situations, particularly when considering or optimizing \emph{loss functions} in machine learning.

\draftnote{More convenient to have two nondegeneracy conditions?:
For convenience, we assume two \emph{nondegeneracy conditions} throughout:
\begin{enumerate}
\item $\wv$ has full support; and %i.e., $\not \exists \Population' \subseteq \Population$ s.t.\ $\Population'$ is a nonempty open set and $\wv(\Population') = 0$; and
\item $\Population$ has multiplicity; i.e., its cardinality obeys $\abs{\Population} > 1$.
\end{enumerate}
When the first is violated it is trivially fixed by defining a new population Population Prime equals support population and analyzing this subpopulation. The second is also not particularly limiting as Singleton populations are of very little interest and indeed we see the only axiomatically Justified aggregator function Shelby the identity function.
}

\draftnote{More convenient to have three nondegeneracy conditions?:
We assume for convenience throughout that three nondegeneracy conditions hold:
\begin{enumerate}
\item $\wv$ has full support; and % i.e., $\not \exists \Population' \subseteq \Population$ s.t.\ $\Population'$ is a nonempty open set and $\wv(\Population') = 0$; and
\item $\lv$ is mutually distinct, i.e., $\not \exists \PopItem, \PopItem' \in \Population$ s.t.\ $\lv(\PopItem) = \lv(\PopItem')$;
\item $\Population$ has multiplicity; i.e., its cardinality obeys $\abs{\Population} > 1$.
\end{enumerate}
}

%We now define Malfare and show various properties.

\begin{definition}[Aggregator Functions: Welfare and Malfare]
\label{def:pop-mean}
An \emph{aggregator function} function $\Mean(\lv; \wv)$ measures the \emph{overall sentiment} of population $\Population$, measured by \emph{sentiment function} $\lv: \Population \to \RNN$, weighted by \emph{probability measure} $\wv$ over $\Population$ (with full support).
%In particular, $\lv: \Population \to \RNN$ is a \emph{measurable function}, and $\wv$ is a probability measure over $\Population$.
If $\lv$ denotes a desirable quantity (e.g., utility), we call $\Mean(\lv; \wv)$ a \emph{welfare function}, written $\Welfare(\lv; \wv)$, and inversely, if it is undesirable (e.g., disutility, loss, or risk), we call $\Mean(\lv; \wv)$ a \emph{malfare function}, written $\Malfare(\lv; \wv)$.

\todo{factoring of some sort?  Must be $F(\Expect_{\PopItem \distributed \wv}(f(\lv(\PopItem)))$?}

\todo{unweighted vectors?}

\end{definition}

For now, think of the term \emph{aggregator function} as signifying that an entire population, with diverse and subjective desiderata, is considered and summarized, as opposed to an individual's objective viewpoint (sentiment value).
%As we introduce axioms and show consequent properties, the appropriateness of the term shall become more apparent.
Note that we use the term \emph{sentiment} to refer to $\lv$ with neutral connotation, but when discussing welfare or malfare, we often refer to $\lv$ as \emph{utility} or \emph{risk}, respectively, as in these cases, $\lv$ describes a well-understood pre\"existing concept.
%Aggregator functions are written $\Mean(\lv; \wv)$; when $\lv$ is positively connoted, we generally express as a \emph{welfare function} $\Welfare(\lv; \wv)$, and for negative $\lv$, we express it as a \emph{malfare function} $\Malfare(\lv; \wv)$.
Coarsely speaking, the three notions are identical, all being functions of the form\footnote{Ideally, aggregator functions would have domain $\RNN = [0, \infty)$ (the nonnegative reals), rather than $\RExt = [-\infty, \infty]$ (the extended reals), to match that of the sentiment value function, but infinite and/or negative aggregates are sometimes required, particularly in the \emph{additively separable form} (see~\cref{sec:comparisons:as}).\draftnote{Is $+\infty$ ever possible?}} $(\Population \to \RNN) \times \textsc{Measure}(\Population, 1) \to \RExt$, however, we shall see that
 in order to promote fairness, the %desirable
  axioms of malfare and welfare functions differ slightly.
The notation reflects this; $\Mean(\lv; \wv)$ is an $\Mean$ for \emph{mean}, whereas $\Welfare(\lv; \wv)$ is a $\Welfare$ for \emph{welfare}, and $\Malfare(\lv; \wv)$ is an $\Malfare$ (\emph{inverted} $\Welfare$), to emphasize its inverted nature.

Often we are interested in \emph{unweighted} aggregator functions of finite discrete populations, where the \emph{sentiment function} may be represented as a \emph{sentiment vector} $\lv \in \RNN^{\NGroups}$.
Unweighted aggregators may then be defined in terms of weighted aggregators as
\[
\Mean(\lv) \doteq \Mean\left( i \mapsto \lv_{i}; i \mapsto \smash{\mathsmaller{\frac{1}{\NGroups}}} \right) \enspace,
\]
abusing notation to concisely express the \emph{uniform measure}.
%\cyrus{Rational weights are ``free'', because you can duplicate loss entries, but I think in general we need weighted malfare?}
Indeed, it may seem antithetical to fairness to allow for weights in malfare and welfare definitions; consider however that weights can represent \emph{differential population sizes}, and thus ensure that the welfare or malfare of \emph{weight-preserving decompositions} of groups into subgroups with equal risk or utility remains constant.
\todo{weighing rights of humans, machines, animals, etc; dangerous but perhaps necessary?}
\todo{Limit form criticism of utilitarian / egal?}

\todo{talk egal somewhere?}

\begin{example}[Utilitarian Welfare]
\label{ex:util}
\if 0 %Cut numbers from example. TODO: use in thesis?
Suppose a 3-group population $\Population \doteq \{ \PopItem_{1}, \PopItem_{2}, \PopItem_{3} \}$ consisting of three groups, encompassing %$\bm{m} = (30, 30, 40)$\
$30\%$, $30\%,$ and $40\%$ of the total population, respectively.
We now take \emph{weights measure} $\wv$ \st{} $\wv(\PopItem_{i}) = (0.3, 0.3, 0.4)_{i}$, representing the relative sizes of each group in the population.\todo{$\wv(\PopItem_{i})$ is an abuse.}
\fi
%
%Now
Suppose individuals reside in some space $\X$, where \emph{distributions} $\ProbDist_{1:\NGroups}$ over domain $\X$ describe the distribution over individuals in \emph{each group}.
Suppose also %, where $\ProbDist_{i}$ represent the space of individuals in group $\PopItem_{i}$ 
\emph{utility function} $\Utility(x): \X \to \RNN$, describing the \emph{level of satisfaction} of an individual, \wrt{}, e.g., some  %particular
 \emph{situation}, \emph{allocation}, or \emph{classifier}.
%Given \emph{distributions} $\ProbDist_{1:3}$ over domain $\X$, representing the space of individuals, where $\ProbDist_{i}$ represent the space of individuals in group $\PopItem_{i}$, 
We now take the \emph{sentiment function} to be the \emph{arithmetic mean utility} (per-group), i.e.,
\[
\lv(\PopItem_{i}) 
  \doteq \Expect_{x \distributed \ProbDist_{i}}[\Utility(x)]
  = \Expect_{\ProbDist_{i}}[\Utility]
  \enspace.
\]
%Similarly, we take \emph{weights measure} proportional to population size, i.e., $\wv(\Population') \doteq \sum_{\PopItem \in \Population'} \dots$.
%
Now, given a \emph{weights vector} $\wv$, describing the \emph{relative frequencies} of membership in each of the $\NGroups$ groups, we define the \emph{utilitarian welfare} as
\[
\Welfare_{1}(\lv; \wv) \doteq \smash{\sum_{i=1}^{\NGroups}} \wv(\PopItem_{i}) \lv(\PopItem_{i}) = \Expect_{\PopItem \distributed \wv}[\lv(\PopItem)] = \Expect_{\wv}[\lv] \enspace.
\]
\end{example}

Of course, in statistical, sampling, and machine learning contexts, $\ProbDist_{1:\NGroups}$ and $\wv$ may be unknown, so we now discuss an \emph{empirical analog} of utilitarian welfare.
\Cref{sec:stat} is then devoted to showing how and when empirical aggregator functions well-approximate their true counterparts.\todo{mention bias?}

\begin{example}[Empirical Utilitarian Welfare]
\label{ex:eutil}
Now suppose $\ProbDist_{1:\NGroups}$ are unknown, but instead, we are given a \emph{sample} $\bm{x}_{1:\NGroups,1:m} \in \X^{\NGroups \times m}$, where $\bm{x}_{i,1:m} \distributed \ProbDist_{i}$. %\in \X^{m}$ 
We define the \emph{empirical analog} of the utilitarian welfare %as in \cref{ex:util}, instead taking
as
\[
\hat{\lv}(\PopItem_{i})
  \doteq \EExpect_{x \in \bm{x}_{i}}[\Utility(x)] \ \ \ \& \ \ \ \hat{\Welfare}_{1}(\hat{\lv}, \wv) \doteq \Expect_{\wv}[\hat{\lv}] \enspace.
\]
Similarly, %if $\Utility$ is unknown, we may instead assume we may sample ... 
if $\wv$ is unknown, but we may sample from some $\ProbDist$ over $\Population \times \X$, we can use \emph{empirical frequencies} $\hat{\wv}$ in place of \emph{true frequencies} $\wv$, and define $\hat{\lv}(\PopItem_{i})$ as \emph{conditional averages} over the subsample associated with group $i$.
%\todo{note; can also estimate $\wv$}
\end{example}

\todo{measurability of $\lv$ required (why); generally do not worry about it; usually care about discrete.  Continuous useful for density estimation; priors; unknown, etc.  Other applications; allocate radio / cell towers uniformly over sphere, etc.}

\subsection{Axioms of Cardinal Welfare and Malfare}
\label{sec:pop-mean:axioms}

\draftnote{Measurability assumptions}

\draftnote{Do we need restriction on symmetry? Measure-preserving permutations?} % must preserve measure $\pi$ s.t.\ $\int_{\Population} \, \mathrm{d}\wv(\PopItem)??? = !??$ Or $(\pi \circ \wv)(\Population) = \wv(\Population) = 1$!

In %the next
this section, we describe various desiderata for aggregator functions, and in particular for fair malfare and welfare functions.
We shall see that the utilitarian welfare is the only aggregator function that is both a fair malfare and welfare function (due to the opposite sense of utility and disutility, \emph{egalitarian} welfare and malfare are analogous, but do not share a functional form, being the \emph{minimum} or \emph{maximum} sentiment value, respectively).
In general, with our axioms, all aggregator functions belong to the single-parameter \emph{power-mean} family (\cref{sec:pop-mean:pow-mean}), but if an alternative, \emph{additive separability} axiom is instead taken, we get a similar family %, discussed in 
(\cref{sec:comparisons:as}). %, and the remaining members of the family are all either fair welfare or fair malfare functions.
The axioms are generally referred to as the axioms of cardinal welfare, though nearly all work equally well as malfare axioms.
Typically, they are stated for \emph{positive}, \emph{unweighted}, and \emph{finite} populations, rather than non-negative, weighted, and measurable populations, but the technical impact of this distinction is quite minor.\footnote{
In particular, allowing $\lv$ to attain $0$ values (by including the appropriate limit sequences)  %analytic continuation
 may violate \emph{strict monotonicity}, so the condition is relaxed around $0$. % (e.g., Nash social welfare or harmonic welfare TODO also limitin gcase of egalitarian welfare).
%The definition of symmetry changes, as permutation over uncountably infinite sets must be restricted so as to be measured preserving.
}

\todo{Axioms 1-4 are standard \citep{sen1977weights,roberts1980interpersonal}, see, e.g., \citep{moulin2004fair} for further reading. TODO they must be earlier; D-G thm is earlier.}

\todo{Normally stated for unweighted.}

\todo{Split into ``classic'' and ``novel''?}
%\begin{definition}[Axioms of Cardinal Welfare (Weighted)]
\begin{definition}[Axioms of Cardinal Welfare and Malfare]
\label{def:cardinal-axioms}
We define the \emph{aggregator function axioms} for aggregator function $\Mean(\lv; \wv)$ below.
For each item, assume (if necessary) that the axiom applies $\forall \lv, \lv' \in \Population \to \RNN$, % $\forall$ permutations $\pi$ over $\Population$, 
scalars $\alpha, \beta \in \RNN$, % \geq 0$, 
and probability measures $\wv$ over $\Population$. %, \dots
\todo{fix $\Population$ notation.}
\todo{Assumptions?}
\begin{enumerate}[wide, labelwidth=0pt, labelindent=0pt]\setlength{\itemsep}{3pt}\setlength{\parskip}{0pt}
%\item \label{def:cardinal-axioms:mono} (Weak) Monotonicity: $\forall \epsv: \Population \to \RNN: \Mean(\lv; \wv) \leq \Mean(\lv + \epsv; \wv)$. %(resp.~\emph{strict monotonicity} with $<$, $\epsv \neq \bm{0}$)

\item \label{def:cardinal-axioms:mono}\label{def:cardinal-axioms:dgfirst}\label{def:cardinal-axioms:gmeanfirst} (Strict) Monotonicity: If $0 \not\in \lv(\Population)$, then $\forall \epsv: \Population \to \RNN$ s.t.\ $\int\limits_{\wv} \epsv(\PopItem) \, \mathrm{d}(\PopItem) > 0$: $\Mean(\lv; \wv) < \Mean(\lv + \epsv; \wv)$. %(resp.~\emph{strict monotonicity} with $<$, $\epsv \neq \bm{0}$)

\todo{Weaker, but still sufficient? notion of monotonicity: $\forall \epsv \text{ s.t. } \Prob(\wv \in \{ p \in \Population \text{ s.t. } \epsv(p) < 0 \}) = 0 \implies \Mean(\lv; \wv) \leq \Mean(\lv + \epsv; \wv)$}
\item \label{def:cardinal-axioms:symm} Symmetry: $\forall$
%measure-preserving
permutations $\pi$ over $\Population$: %(i.e., $(\wv \circ \pi) (\Population) = (\wv \circ \pi) (\Population): 
$\Mean(\lv; \wv) = \Mean(\pi(\lv); \pi(\wv))$.

\item \label{def:cardinal-axioms:cont}\label{def:cardinal-axioms:gmeanlast}
Continuity: %$\forall \lv: \, 
%$\{\lv' \mid \Mean(\lv'; \wv) \leq \Mean(\lv; \wv)\}$ is a closed set.
$\{\lv' \mid \Mean(\lv'; \wv) \leq \Mean(\lv; \wv)\}$ and $\{\lv' \mid \Mean(\lv'; \wv) \geq \Mean(\lv; \wv)\}$ are closed sets.

%?   ``3.  Continuity: for every profile v, the set of profiles weakly better than v and the set of profiles weakly worse than v are closed sets.''

\item \label{def:cardinal-axioms:ioua} Independence of Unconcerned Agents (IOUA): Suppose %$\alpha, \beta \in \RNN$, and take 
subpopulation $\Population' \subseteq \Population$.  Then

\ifthesis
{\scalebox{0.98}{{
$\displaystyle
\Mean\!\left(\begin{cases} \PopItem \in \Population': & \!\!\! \alpha \\ \PopItem \not \in \Population': & \!\!\! \lv(\PopItem) \\ \end{cases}\!; \wv\right) \leq \Mean\!\left(\begin{cases} \PopItem \in \Population': & \!\!\! \alpha \\ \PopItem \not \in \Population': & \!\!\! \lv'(\PopItem) \\ \end{cases}\!; \wv\right)
\implies
\Mean\!\left(\begin{cases} \PopItem \in \Population': & \!\!\! \beta \\ \PopItem \not \in \Population': & \!\!\! \lv(\PopItem) \\ \end{cases}\!; \wv \right) \leq \Mean\!\left(\begin{cases} \PopItem \in \Population': & \!\!\! \beta \\ \PopItem \not \in \Population': & \!\!\! \lv'(\PopItem) \\ \end{cases} \!; \wv \right) \enspace.
$
}}}
\else
{\scalebox{0.95}{{
$\displaystyle
\hspace{-0.45cm}\Mean\!\left(\!\begin{cases}\PopItem \in \Population': & \!\!\!\! \alpha \\ \PopItem \not \in \Population': & \!\!\!\! \lv(\PopItem) \\ \end{cases}\!\!; \wv\right) \leq \Mean\!\left(\!\begin{cases}\PopItem \in \Population': & \!\!\!\! \alpha \\ \PopItem \not \in \Population': & \!\!\!\! \lv'(\PopItem) \\ \end{cases}\!\!; \wv\right)
\! \implies \!\!
\Mean\!\left(\!\begin{cases}\PopItem \in \Population': & \!\!\!\! \beta \\ \PopItem \not \in \Population': & \!\!\!\! \lv(\PopItem) \\ \end{cases}\!\!; \wv \right) \leq \Mean\!\left(\!\begin{cases}\PopItem \in \Population': & \!\!\!\! \beta \\ \PopItem \not \in \Population': & \!\!\!\! \lv'(\PopItem) \\ \end{cases} \!\!; \wv \right) \enspace.
$
}}}
\fi

\item \label{def:cardinal-axioms:iocs}\label{def:cardinal-axioms:dglast} Independence of Common Scale (IOCS): %Suppose scalar
%Let $\alpha \geq 0$.  Then 
$\Mean(\lv; \wv) \leq \Mean(\lv'; \wv) \implies \Mean(\alpha\lv; \wv) \leq \Mean(\alpha\lv'; \wv)$.

\todo{Was: Suppose $\alpha > 1$.  Then $\Mean(\alpha \lv; \wv) \leq \Mean(\lv; \wv)$.  This is wrong?}

\item \label{def:cardinal-axioms:mult} Multiplicative Linearity: %Suppose $\alpha \geq 0$.  Then 
$\Mean(\alpha \lv; \wv) = \alpha \Mean(\lv; \wv)$.

\item \label{def:cardinal-axioms:unit} {Unit Scale}: $\Mean(\bm{1}; \wv) = \Mean(\PopItem \mapsto 1; \wv) = 1$. %$\Mean(\bm{1}; \wv) = \Mean(1, \dots, 1; \wv) = 1$.

\item \label{def:cardinal-axioms:pd} Pigou-Dalton Transfer Principle:
%\todo{``Batch representation:'' needed for measure-theoretic forzmulation?}
Suppose %$\lv, \lv'$ s.t. 
$\mu = \Expect_{\wv}[\lv] = \Expect_{\wv}[\lv']$, and $\forall \PopItem \in \Population: \abs{\mu - \lv'(\PopItem)} \leq \abs{\mu - \lv(\PopItem)}$.
Then $\Welfare(\lv'; \wv) \geq \Welfare(\lv; \wv)$.

\item \label{def:cardinal-axioms:apd} Anti Pigou-Dalton Transfer Principle: Suppose %the hypothesis of
as in axiom \ref{def:cardinal-axioms:pd}, %but instead require the conclusion 
and conclude $\Malfare(\lv'; \wv) \leq \Malfare(\lv; \wv)$.

\todo{weak and strong variants.}

\todo{Stronger than needed; sets of measure 0 could violate condition without issue.}
%Suppose $\lv, \lv', \lv''$ s.t. $\forall p \in \Population: \, \lv(p) \leq \lv''(p) \leq \lv'(p)$, and also $\forall p \in \Population: \, \lv(p) + \lv'(p) = 2\lv'(p)$.
%Then $\Mean(\lv'; \wv) \leq \Mean(\lv; \wv)$

\if 0
\todo{It's stronger than needed; measure theoretic formulation looks more like
$\forall p \subseteq \Population'$, (need to involve $\wv$:) $\lv(p) \leq \lv''(p) \leq \lv'(p)$ and $\lv(p) + \lv'(p) = 2\lv''(p)$.
Furthermore, suppose $\forall p \subsetq \Population \setminus \Population'$, \dots}
\fi

%\item Sub-averaging $\abs{\Mean(\lv) - \Mean(\lv + \epsv)} \leq \frac{\lVert \epsv \rVert_{1}}{\NGroups}$ Is Lipschitz enough?  This doesn't hold for egal.  Should it be $\bm{\varepsilon}_{\infty}$?

%\item \label{def:cardinal-axioms:contraction} Contraction $\abs{\Mean(\lv) - \Mean(\lv + \epsv)} \leq \norm{\epsv}_{\infty}$ \cyrus{Do we need this axiom? I think it's implied by 1-4+7}

%\item \label{def:cardinal-axioms:injustice} Injustice Principle: inverse of \cref{def:cardinal-axioms:pd}.

\end{enumerate}

\end{definition}

We take a moment to comment on each of these axioms, to preview their purpose and assure the reader of their necessity.
Axioms \ref{def:cardinal-axioms:dgfirst}-\ref{def:cardinal-axioms:dglast} are the standard \emph{axioms of cardinal welfarism} (\ref{def:cardinal-axioms:mono}-\ref{def:cardinal-axioms:ioua} are discussed by \citet{sen1977weights,roberts1980interpersonal}, and \ref{def:cardinal-axioms:dglast} by \citet{debreu1959topological,gorman1968structure}).
%\todo{say more about 1-4.} Together, they ..., analagous to \emph{natural sufficient statistics} of the exponential family. %This is done later now.
Together, they imply (via the Debreu-Gorman theorem) that any aggregator function can be decomposed as a \emph{monotonic function} of a \emph{sum} (over groups) of \emph{logarithm} or \emph{power} functions\todo{ref DG thm}.
Axiom \ref{def:cardinal-axioms:mult} is a natural and useful property, and ensures that \emph{dimensional analysis} on aggregator functions is possible; in particular, the \emph{units} of aggregator functions match those of sentiment values.
Note that axiom~\ref{def:cardinal-axioms:mult} implies axiom~\ref{def:cardinal-axioms:iocs}, and it is thus a simple strengthening of a traditional cardinal welfare axiom.
%This axiom also ensures that \emph{units} of aggregator functions preserve the \emph{units} of $\lv$, making dimensional analysis greatly more convenient; 
We will also see that it is essential to show convenient \emph{statistical} and \emph{learnability} properties.
Axiom~\ref{def:cardinal-axioms:unit} furthers this theme, as it ensures that not only do \emph{units} of aggregates match those of $\lv$, but \emph{scale} does as well (making comparisons like ``$\lv_{i}$ is \emph{above} the welfare (of the population)'' meaningful), and also enabling comparison \emph{across populations}, in the sense that comparing \emph{averages} is more meaningful than \emph{sums}.
Finally, axiom~\ref{def:cardinal-axioms:pd} (the \emph{Pigou-Dalton transfer principle}. %\citep[see][]{pigou1912wealth,dalton1920measurement}) 
see \citet{pigou1912wealth,dalton1920measurement})
is also standard in cardinal welfare theory as it ensures fairness, in the sense that welfare is higher when utility values are more uniform, i.e., incentivizing \emph{equitable redistribution of ``wealth''} in welfare.
Its antithesis, axiom~\ref{def:cardinal-axioms:apd}, encourages the opposite; in the context of welfare, this perversely incentivizes an expansion of inequality, but for malfare, which we generally wish to \emph{minimize}, the opposite occurs, thus this axiom characterizes \emph{fairness} %in the context of 
for \emph{malfare}.\todo{Discuss novelty}

Axioms \ref{def:cardinal-axioms:mult} \& \ref{def:cardinal-axioms:unit} are novel to this work, and are key in strengthening the Debreu-Gorman theorem to ensure that all welfare and malfare functions are \emph{power means} in the sequel.
Axiom~\ref{def:cardinal-axioms:apd} is also novel, as it is necessary to flip the inequality of axiom \ref{def:cardinal-axioms:pd} when the sense of the aggregator function is inverted from welfare to malfare; in particular, the semantic meaning shifts from requiring that ``redistribution of utility is \emph{desirable}'' to ``redistribution of disutility is \emph{not undesirable}.''

%
%The reader is invited to consult, e.g., \citet{moulin2004fair}, for further information on axiomatic cardinal welfare theory.

\todo{fwd ref figure?}

\todo{Contrast with standard additivity properties from IOCS?}

\todo{GEI / Theil indices citations}

\draftnote{Look into ``A Moral Framework for Understanding of Fair ML
through Economic Models of Equality of Opportunity''}

\subsection{The Power Mean}
\label{sec:pop-mean:pow-mean}

\todo{See %\url{https://books.google.com.na/books?id=GybdNuNsarIC&printsec=frontcover&source=gbs_ge_summary_r&cad=0#v=onepage&q&f=false}
\citep{bullen2013handbook} pp 216-275 for proper measure-theoretic treatment.}

We now define the $p$-\emph{power mean}\footnote{The $p$-power-mean is referred to by some authors as the \emph{generalized mean} or H\"older mean, and is itself a generalization of the \emph{Pythagorean} (arithmetic, geometric, and harmonic) means.} $\Mean_{p}(\cdot)$, for any $p \in \RExt$, %\R \cup \pm \infty$,
and the weighted $p$-power-mean $\Mean_{p}(\cdot; \cdot)$, which we shall use to quantify both malfare and welfare.
We shall see that power means exhibit many convenient properties (\cref{thm:pow-mean-prop}), and arise often (\cref{thm:pop-mean-prop}) when analyzing aggregator functions obeying the various axioms of \cref{def:cardinal-axioms}. %, and as we shall see in , are a particularly important class of aggregator function.

\todo{Cut unweighted?}

\begin{definition}[Power-Mean Welfare and Malfare]
\label{def:pmean}
Suppose $p \in \RExt$. %\R \cup \pm \infty$.
We first define the \emph{unweighted power-mean} of sentiment vector $\lv \in \RNN^{\NGroups}$ as
%\Mean_{p}(\lv) \doteq \sqrt[p]{\frac{1}{\NGroups}\sum_{i=1}^{\NGroups} \lv^{p}_{i}}
\[
\Mean_{p}(\lv) \doteq
  \begin{cases}
  p \in \R \setminus \{0\} & \displaystyle \sqrt[p]{\frac{1}{\NGroups}\sum_{i=1}^{\NGroups} \lv^{p}_{i}} \\
  p = -\infty & \displaystyle \min_{i \in 1, \dots, \NGroups} \lv_{i} \\
  p = 0 & \displaystyle \sqrt[\NGroups]{\prod_{i=1}^{\NGroups} \lv_{i}} = \exp\left( \frac{1}{\NGroups}\sum_{i=1}^{\NGroups} \ln(\lv_{i}) \right) \\
  p = \infty & \displaystyle \max_{i \in 1, \dots, \NGroups} \lv_{i} \enspace. \\
  \end{cases}
\]

We now define the \emph{weighted power-mean}, given \emph{sentiment value function} $\lv: \Population \to \RNN$ and \emph{probability measure} $\wv$ over $\Population$, as
\[
\Mean_{p}(\lv; \wv) \doteq
  \begin{cases}
  p \in \R \setminus \{0\} & \displaystyle \sqrt[p]{\vphantom{\frac{1}{2}}\smash{\int\limits_{\wv}} \lv^{p}(\PopItem) \, \mathrm{d}(\PopItem) } = \sqrt[p]{\Expect_{\PopItem \distributed \wv}[\lv^{p}(\PopItem)]} \\[0.25cm]
  %p = -\infty & \displaystyle \inf_{\PopItem \in \Support(\wv)} \lv(\PopItem) \\
  p = -\infty & \displaystyle \inf_{\PopItem \in \Population} \lv(\PopItem) \\
  p = 0 & \displaystyle \exp\left( \int\limits_{\wv} \ln \lv(\PopItem) \, \mathrm{d}(\PopItem) \right) = \exp\left( \Expect_{\PopItem \distributed \wv}[\ln \lv(\PopItem)] \right) \\
  %p = \infty & \displaystyle \sup_{\PopItem \in \Support(\wv)} \lv(\PopItem) \\
  p = \infty & \displaystyle \sup_{\PopItem \in \Population} \lv(\PopItem) \enspace. \\
  \end{cases}
\]
In both the weighted and unweighted cases, $p \in \{-\infty, 0, \infty\}$ resolve as their (unique) limits, and for all $p \in \R$, %note that
 \emph{power means} are special cases of the (weighted) \emph{generalized $f$-mean} (a.k.a.\ the $f$-mean or Kolmogorov mean), defined for strictly monotonic $f$ as
\[
\Mean_{f}(\lv; \wv) \doteq f^{-1}\smash{\Bigl(\Expect_{\PopItem \distributed \wv}\bigl[ (f \circ \lv) (\PopItem) \bigr] \Bigr)} \enspace.\todo{citation; reorder? inverse symbol?}
\]
Note also that, as always, we assume nondegeneracy condition $\Support(\wv) = \Population$; otherwise the $p \in \pm \infty$ cases would need to restrict their attention to $\Support(\wv)$, rather than all of $\Population$.
Finally, note that, if care is not taken, $\Mean_{p}(\lv; \wv)$ is $p \leq 0$ is undefined when some $\lv(\PopItem) = 0$, as this creates $\log(0)$ or $\frac{1}{0}$ expressions.
We resolve this issue by taking the above definitions %to hold only 
for positive-valued $\lv$, and extending to the general case by taking
\[
%\Mean_{p}(\lv; \wv) = 
\lim_{\varepsilon \to 0^{+}} \Mean_{p}(\lv + \varepsilon; \wv) \enspace.
\]
\end{definition}

%\paragraph{pmean properties}

%\begin{theorem}[Properties of the Power-Mean]
\begin{restatable}[Properties of the Power-Mean]{theorem}{thmpowmeanprop}
\label{thm:pow-mean-prop}
Suppose $\lv, \lv'$ are sentiment functions in $\Population \to \RNN$, and $\wv$ is a probability measure over %some space
 $\Population$. % (i.e., $\Population = \{1, \dots, \NGroups\}$, $\wv = \{1, \dots, \NGroups\} \mapsto \frac{1}{\NGroups}$ in the discrete uniform case).
The following then hold.
\begin{enumerate}[wide, labelwidth=0pt, labelindent=0pt]\setlength{\itemsep}{3pt}\setlength{\parskip}{0pt}
\item \label{thm:pow-mean-prop:mono} Monotonicity: $\Mean_{p}(\lv; \wv)$ is weakly-monotonically-increasing in $p$, and strictly so if %there exist 
$\exists \, \PopItem, \PopItem' \in \Population$ s.t.\ $\lv(\PopItem) \neq \lv(\PopItem')$. %$\lv$ attains distinct $a, b \in \R$ with nonnegligible probability.
\item \label{thm:pow-mean-prop:subadditivity} Subadditivity: $\forall p \geq 1: \, \Mean_{p}(\lv + \lv'; \wv) \leq \Mean_{p}(\lv; \wv) + \Mean_{p}(\lv'; \wv)$. %\todo{If we want $\epsv$ to be allowed to have negative terms, we need $\abs{\epsv}$ or something?}

%Not really: \todo{Prop 1 of \url{https://citeseerx.ist.psu.edu/viewdoc/download?doi=10.1.1.156.2801&rep=rep1&type=pdf}: Lipschitz and convexity related.  Maybe we need $p < 0$?  Coro 3 for p<1?}
\item \label{thm:pow-mean-prop:contraction} Contraction: $\forall p \geq 1: \, \abs{ \Mean_{p}(\lv; \wv) - \Mean_{p}(\lv'; \wv) } \leq \Mean_{p}(\abs{\lv - \lv'}; \wv) \leq \norm{\lv - \lv'}_{\infty}$.

\todo{show analogous for $p < 1$}\todo{Check abs.}
\item \label{thm:pow-mean-prop:curvature} Curvature: $\Mean_{p}(\lv; \wv)$ is concave in $\lv$ for $p \in [-\infty, 1]$ and convex for $p \in [1, \infty]$.

\draftnote{For next version:
\item \label{thm:pow-mean-prop:curvature-wv} Curvature: $\Mean_{p}(\lv; \wv)$ is convex in $\wv$ for $p \in [-\infty, 1]$ and concave for $p \in [1, \infty]$.
Proof of this opposite curvature in $\wv$: linear interpolation under $\sqrt[p]{\cdot}$ operator? handle inf cases, true even for these degenerates.
}

%\item \label{thm:pow-mean-prop:contraction} Contraction: $\Mean(\lv + \epsv) \in \Mean(\lv) \pm \norm{\epsv}_{\infty}$\todo{unweighted} \todo{how different from 1st contraction property?}
%\item \label{thm:pow-mean-prop:lipschitz} Lipschitz-Continuity: \todo{unweighted}
\end{enumerate}
\end{restatable}
%\end{theorem}

%\todo{Can we handle negative losses?  Maybe just for $p \in \{-\infty, 0, \infty\}$, like welfare?  Is there a way to ``fix'' both, sacrificing convexity / concavity for having domain $\R$?}

\subsection{Properties of Welfare and Malfare Functions}
\label{sec:pop-mean:prop}

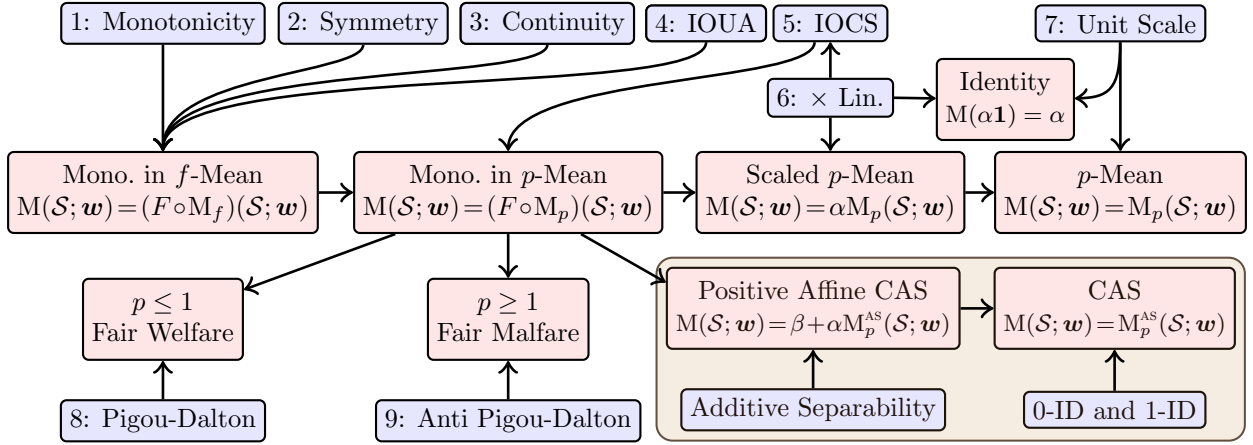
\begin{figure}

\ifthesis
\def\xs{3.1}
\def\ys{2.2}
TODO YSCALE
\else
\def\xs{3.06}
\def\ys{2.25}
\fi

\if 0
\if 0
\let\xs{2.6}
\let\ys{2.0}
\else
\let\xs{2.52}
\let\ys{1.95}
\fi
\fi

\begin{centering}
\begin{tikzpicture}[
    xscale=\xs,
    yscale=\ys,
    axs/.style={rectangle,draw,thick,rounded corners=0.5ex,fill=blue!10!white},
    props/.style={rectangle,draw,thick,rounded corners=0.5ex,fill=red!10!white},
    %arrs/.style={line width=0.5mm, -{Stealth[length=4mm, open]}},
    %arrs/.style={line width=0.4mm, -{Stealth[length=2.5mm, open]}},
    arrs/.style={line width=0.35mm, -{Classical TikZ Rightarrow[length=1.7mm,width=2.65mm,line width=0.35mm]}},
    %imp/.style={line width=0.4mm,-implies,double equal sign distance},
    %hatchfill/.style={opacity=0.4,pattern=hatch},
    grouping/.style={rectangle,draw,opacity=0.8,thick,rounded corners=1.4ex,inner sep=0.8ex}, %,postaction={hatchfill}
  ]

\node[axs] (a1) at (0, 0) {1: Monotonicity};
\node[axs] (a2) at (0.87, 0) {2: Symmetry};
\node[axs] (a3) at (1.67, 0) {3: Continuity};
\node[axs] (a4) at (2.36, 0) {4: IOUA}; %{4: Independence of Unconcerned Agents};
\node[axs] (a5) at (2.9, 0) {5: IOCS};

\node[axs] (a6) at (2.9, -0.44) 
  {6: $\times$ Lin.}; 
  %{6: $\times$ Linearity}; 
  %{6: Mult. Linearity};
\node[axs] (a7) at (4.16, 0) {7: Unit Scale};

%\node at (4, -0.25) {\rotatebox{90}{$\implies$}};
%\draw (a6) -- (a5);

\draw[arrs] (a6) -- (a5);

\node[props] (mgm) at (0, -1) {%
  \begin{tabular}{@{}c@{}}
  Mono.\ in $f$-Mean \\
  %$F\bigl(\Expect[f(u)]\bigr)$ \\
  %$(F \circ \Mean_{f})(\lv; \wv)$ \\
  $\Mean(\lv; \wv) \! = \! (F \! \circ \! \Mean_{f})(\lv; \wv)$ \\
  \end{tabular}};
\node[props] (mpm) at (1.5, -1) {
  \begin{tabular}{@{}c@{}}
  Mono.\ in $p$-Mean \\
  %$F\bigl(\Expect[f_{p}(u)]\bigr)$ \\
  %$(F \circ \Mean_{p})(\lv; \wv)$ \\
  $\Mean(\lv; \wv) \! = \! (F \! \circ \! \Mean_{p})(\lv; \wv)$ \\
  \end{tabular}};
\node[props] (lpm) at (2.9, -1) {
  \begin{tabular}{@{}c@{}}
  Scaled $p$-Mean \\
  %$\alpha \Mean_{p}(\lv; \wv)$ \\
  $\Mean(\lv; \wv) \! = \! \alpha \Mean_{p}(\lv; \wv)$ \\
  \end{tabular}};
\node[props] (pm) at (4.16, -1) {
  \begin{tabular}{@{}c@{}}
  $p$-Mean \\
  %$\Mean_{p}(\lv; \wv)$ \\
  $\Mean(\lv; \wv) \! = \! \Mean_{p}(\lv; \wv)$ \\
  \end{tabular}};

(0.87, 0) {2: Symmetry};
\node[axs] (a3) at (1.67, 0) {3: Continuity};
\node[axs] (a4) at (2.36, 0) {4: IOUA};

\draw[arrs] (a1) -- (mgm);
\iffalse
\draw[arrs] (a2) .. controls +(0, -0.25) and (0.1, -0.5) .. (mgm);
\draw[arrs] (a3) .. controls +(0, -0.25) and (0.2, -0.5) .. (mgm);
\draw[arrs] (a4) .. controls +(0, -0.25) and (0.3, -0.5) .. (mgm);
\else
\draw[arrs] (a2) .. controls +(0, -0.3) and (0, -0.3) .. (mgm);
\draw[arrs] (a3) .. controls +(0, -0.35) and (0, -0.23) .. (mgm);
\draw[arrs] (a4) .. controls +(0, -0.45) and (0, -0.22) .. (mgm);
\fi
%\draw[arrs] (a2.269) -- (mgm.70);
%\draw[arrs] (a3.265) -- (mgm.50);
%\draw[arrs] (a4.260) -- (mgm.30);

\draw[arrs] (mgm) -- (mpm);
\draw[arrs] (a5.230) .. controls +(0, -0.25) and (1.5, -0.3) .. (mpm.90);

\draw[arrs] (mpm) -- (lpm);
\draw[arrs] (a6) -- (lpm);

\draw[arrs] (lpm) -- (pm);
%\draw[arrs] (a7) edge[bend right=35] (pm);
\draw[arrs] (a7) -- (pm);

\iftrue %ID property:
\node[props] (id) at
    (3.66, -0.45) {
    %(4.5, -1.75) {}
  \begin{tabular}{@{}c@{}}
  Identity \\
  %$\Mean(x \mapsto \alpha; \wv) = \alpha$ \\
  \small $\Mean(\alpha\bm{1}) = \alpha$ \\
  \end{tabular}};

%\draw[arrs] (a6) edge[bend right=30] (id);
%\draw[arrs] (a7) edge[bend left=20] (id);
\draw[arrs] (a6.0) -- (id);%(a6.0) -- (id);
%\draw[arrs] (a7.269) .. controls (4.16, -0.44) .. (id);

\draw[arrs] (a7.269) .. controls (4.16, -0.39) and (4.11, -0.44) .. (id);
%\draw[arrs,tension=0.75] (a7.269) .. controls (4.16, -0.44) .. (id);
%TODO: .. controls +(\angle:1cm) and +(-1,0) .. (2.5,0);
\fi

\node[axs] (a8) at (0, -2.35) {8: Pigou-Dalton};
\node[axs] (a9) at (1.5, -2.35) {9: Anti Pigou-Dalton};

\node[props] (fwf) at (0, -1.75) {\begin{tabular}{@{}c@{}} $p \leq 1$ \\ Fair Welfare \\ \end{tabular}};
\node[props] (fmf) at (1.5, -1.75) {\begin{tabular}{@{}c@{}} $p \geq 1$ \\ Fair Malfare \\ \end{tabular}};

%\node[props] (util) at (0.5, -4) {Linearity (Utilitarianism)};

\draw[arrs] (mpm) -- (fwf);
\draw[arrs] (a8) -- (fwf);
%\draw[arrs] (mpm) edge[bend left=44] (fmf);
\draw[arrs] (mpm) -- (fmf);
\draw[arrs] (a9) -- (fmf);

%\draw[arrs] (fwf) -- (util);
%\draw[arrs] (fmf) -- (util);

%\draw[arrs] (util) -- (a6);

%\draw[arrs] (unit) -- (util2);

%Additively Separable Form:
\begin{scope}[xshift=-0.025cm,yshift=0.065cm]

\node[axs] (asep) at (2.85, -2.35) {Additive Separability};
\node[axs] (asep-ids) at (4.16, -2.35) {0-ID and 1-ID};

\node[props] (pacas) at (2.85, -1.75) {\begin{tabular}{@{}c@{}} Positive Affine CAS \\ \small $\Mean(\lv; \wv) \! = \! \beta \! + \! \alpha \MeanAs_{p}(\lv; \wv)$ \\ \end{tabular}};

\node[props] (cas) at (4.16, -1.75) {\begin{tabular}{@{}c@{}} CAS \\ \small $\Mean(\lv; \wv) \! = \! \MeanAs_{p}(\lv; \wv)$ \\ \end{tabular}};

\begin{scope}[blend mode=lighten]%[blend mode=screen]
\node[grouping,fit={ (asep) (asep-ids) (pacas) (cas) },draw=brown!30!black,fill=brown,fill opacity=0.15] (asgroup) {}; %TODO pattern color=brown
%\node[above = -0.6cm of asgroup] (astitle) {\bf Additively Separable};
\end{scope}

\end{scope}

\draw[arrs] (asep) -- (pacas);
\draw[arrs] (mpm) -- (pacas.170);

\draw[arrs] (pacas) -- (cas);
\draw[arrs] (asep-ids) -- (cas);

\if 0 %FPAC learning:

\node[axs] (a10) at (3, -2.5) {
    \begin{tabular}{@{}l@{}}
    10: Bounded $\lv$ \\
    %$\exists \frange \in \R \text{ s.t. } \lv \in \Population \to [0, \frange]$ \\
    %$\exists \frange \in \RNN : \, \lv \in \Population \to [0, \frange]$ \\
    $\lv \in \Population \to [0, \frange]$ \\
    \end{tabular}
  };
\node[axs] (a11) at (4, -2.5) {\begin{tabular}{@{}l@{}} 11: Uniform \\ Convergence \\ \end{tabular}};
\node[axs] (a12) at (5, -2.5) {\begin{tabular}{@{}l@{}} 12: Convex \\ Optimization \\ \end{tabular}};

%\node[axs] (a10) at (3, -2.5) {10: Bounded $\lv$}; %{10: Boundedness};
%\node[axs] (a11) at (4, -2.5) {11: Unif. Cvgc.}; %{Uniform Convergence};
%\node[axs] (a12) at (5, -2.5) {12: Cvx. Opt.}; %{Convex Optimization};

\node[props] (fl) at (4, -1.75) {
  \begin{tabular}{@{}c@{}}
  Fair-PAC-Learnability \\
  %(Empirical Malfare Minimization) \\
  $\exists$ poly.{} \emph{sufficient} sample size $m^{*}(\mathcal{\F}, \varepsilon, \delta, n)$, $\bot p$ \\
  \end{tabular}};

\draw[arrs] (a10) -- (fl);
\draw[arrs] (a11) -- (fl);
\draw[arrs] (a12) -- (fl);
\draw[arrs] (fmf) -- (fl);
\draw[arrs] (pm) -- (fl);
\fi

\end{tikzpicture} \\
\end{centering}

\todo{Less abbrev? thm numbers?  PD TP!}
\todo{Figure out utilitarian back-arrow. None, because PD / APD can still mono transform?}
\todo{Can some of these be bijections?}
\todo{Add interp coros to props? (units, comparisons).}
\todo{Add stat estimation to fig?}
%\todo{fair PAC $p \ge 1$ arrow makes no sense; we assume that anyway!}
%\vspace{-0.4cm}

\caption{Relationships between aggregator function axioms and properties.
Assumptions and axioms are shown in pastel blue, and properties shown in pastel red.
%Equivalent properties hold for unweighted aggregator functions. % functions.
%\cyrus{Remove FPAC stuff, shrink it.}
These results are stated as \cref{thm:pop-mean-prop}, except for the \emph{additive separability} results (brown box), which are derived in \cref{sec:comparisons:as}.
%and related axioms are boxed to emphasize their comparative rule in this work
\todo{Identity double implication is confusing.  More curved arrows.}
}
\label{fig:ax-prop}
\end{figure}

\draftnote{TODO: EQUATION NUMBERING!}

We now show that the axioms of \cref{def:cardinal-axioms} are sufficient to characterize many properties of welfare and malfare functions.
\begin{restatable}[Aggregator Function Properties]{theorem}{thmpopmeanprop}
\label{thm:pop-mean-prop}
Suppose aggregator function $\Mean(\lv; \wv)$. % meets the Malfare Axioms of \cref{def:cardinal-axioms}.
%Then (subsets of) the aggregator function axioms (see \cref{def:cardinal-axioms}) imply the following properties.
If $\Mean(\cdot; \cdot)$ satisfies (subsets of) the aggregator function axioms (see \cref{def:cardinal-axioms}), we have that $\Mean(\cdot; \cdot)$ exhibits the following properties.
For each, assume arbitrary sentiment-value function $\lv: \Population \to \RNN$ and weights measure $\wv$ over $\Population$.
The following then hold.
\begin{enumerate}[wide, labelwidth=0pt, labelindent=0pt]\setlength{\itemsep}{3pt}\setlength{\parskip}{0pt}
\item \label{thm:pop-mean-prop:id} \emph{Identity}: Axioms \ref{def:cardinal-axioms:mult}-\ref{def:cardinal-axioms:unit} imply %linearity+unit scale implies
$\Mean(\PopItem \mapsto \alpha; \wv) = \alpha$.
\todo{forall $\alpha$?}

\item \label{thm:pop-mean-prop:lin-fact-f} \emph{Linear Factorization}: Axioms~\ref{def:cardinal-axioms:gmeanfirst}-\ref{def:cardinal-axioms:gmeanlast} imply strictly-monotonically-increasing continuous $F,f: \R \to \R, $ s.t.
%\vspace{-0.1cm}
\[
\Mean(\lv; \wv) = F\left( \int\limits_{\wv} f(\lv(\PopItem) ) \, \mathrm{d}(\PopItem) \right) = F\left( \Expect_{\PopItem \distributed \wv} \bigl[ f(\lv(\PopItem)) \bigr] \right)
\enspace.
\]
\draftnote{Shouldn't it be $\RExt$?  Sometimes want $\RNN$ vs $\R$?}

\item \label{thm:pop-mean-prop:lin-fact} \emph{Debreu-Gorman}: Axioms~\ref{def:cardinal-axioms:dgfirst}-\ref{def:cardinal-axioms:dglast} imply that, for some $p \in \R$,
$f(x) = f_{p}(x) \doteq
\begin{cases}
\ p = 0 & \ln(x) \\[-0.08cm]
\ p \neq 0 & \sgn(p)x^{p} \\[-0.1cm]
\end{cases}
\enspace.
$
\if 0
, where
$
\begin{cases}
\ p = 0 & f_{0}(x) \doteq \ln(x) \\[-0.08cm]
%\ p \neq 0 & f_{p}(x) \doteq \sgn(p)x^{p} \\[-0.1cm]
\ p \neq 0 & f_{p}(x) \doteq \sgn(p)x^{p} \\[-0.1cm]
%p > 0 & f_{p}(x) \doteq x^{p} \\[-0.08cm]
%p = 0 & f_{0}(x) \doteq \ln(x) \\[-0.08cm]
%p < 0 & f_{p}(x) \doteq -x^{p} \\%[-0.1cm]
\end{cases}
\enspace.
$
\fi
\if 0 %Old, full statement
$\exists p \in \R$, strictly-monotonically-increasing continuous $F: \R \to \RNN$ s.t.
\vspace{-0.1cm}
\[
\Mean(\lv; \wv) = F\left( \int\limits_{\wv} f_{p}(\lv(\PopItem) ) \, \mathrm{d}(\PopItem) \right) = F\left( \Expect_{\PopItem \distributed \wv} \bigl[ f_{p}(\lv(\PopItem)) \bigr] \right) \enspace,
\quad \text{with} \ 
\begin{cases}
\ p = 0 & f_{0}(x) \doteq \ln(x) \\[-0.08cm]
%\ p \neq 0 & f_{p}(x) \doteq \sgn(p)x^{p} \\[-0.1cm]
\ p \neq 0 & f_{p}(x) \doteq \sgn(p)x^{p} \\[-0.1cm]
%p > 0 & f_{p}(x) \doteq x^{p} \\[-0.08cm]
%p = 0 & f_{0}(x) \doteq \ln(x) \\[-0.08cm]
%p < 0 & f_{p}(x) \doteq -x^{p} \\%[-0.1cm]
\end{cases}
\enspace.
\]
\fi
\item \label{thm:pop-mean-prop:pmean} \emph{Power Mean}: Axioms~\ref{def:cardinal-axioms:mono}-\ref{def:cardinal-axioms:unit} imply $F(x) = f_{p}^{-1}(x)$, thus $\Mean(\lv; \wv) = \Mean_{p}(\lv; \wv)$. %$\exists p \in \R$ \st{} $\Mean(\lv; \wv) = \Mean_{p}(\lv; \wv)$ (i.e., $\Mean(\cdot; \cdot)$ is a \emph{weighted power-mean}). % with finite $p$).
\todo{Restore old split}
\if 0
Axioms 1-4+6 imply $\Mean(\lv) = c\Mean_{p}(\lv)$ for some $p \in \R$, $c \in \RNN$.
Axiom \ref{def:cardinal-axioms:unit} further implies $c = 1$.
\fi

%\todo{contraction implies $\alpha \leq 1$?}

%False conjecture:
%\todo{I think it can also be $\Mean(\lv) = \Mean_{p}(\alpha\lv + \beta\1) - \beta $}
%\todo{I think linear offsets $\Welfare(\lv + \alpha\bm{1}) - \alpha)$ also meet the axioms, so it's not ``just'' the vanilla power means.}

\item \label{thm:pop-mean-prop:pmean-fair} \emph{Fair Welfare}: Axioms~\ref{def:cardinal-axioms:dgfirst}-\ref{def:cardinal-axioms:dglast} and \ref{def:cardinal-axioms:pd} imply $p \in (-\infty, 1]$.
\todo{could combine PD and APD?}

\item \label{thm:pop-mean-prop:pmean-unfair} \emph{Fair Malfare}: Axioms~\ref{def:cardinal-axioms:dgfirst}-\ref{def:cardinal-axioms:dglast} and \ref{def:cardinal-axioms:apd} imply $p \in [1, \infty)$.

\draftnote{Concatenation axioms commented:}
\if 0
\item \emph{Concatenation}: $\Mean(\lv \circ \lv') \in \operatorname{CH}\bigl(\Mean(\lv), \Mean(\lv')\bigr)$ \cyrus{No idea how to prove this.  Requires symmetry, some definition changes? IOUA?}

\item \emph{Concatenation II}: Axioms 1-4 imply $\Mean(\lv \circ \lv', \w \circ \w') = F\left( \frac{\norm{\w}_{1}}{\norm{\w + \w'}_{1}}F^{-1} \circ \Mean(\lv, \w') + \frac{\norm{\w'}_{1}}{\norm{\w + \w'}_{1}}F^{-1}\Mean(\lv', \w')\right)$;

Adding axiom \cref{def:cardinal-axioms:unit} implies \todo{handle $p=0$ case} $\Mean(\lv \circ \lv', \w \circ \w') = \sqrt[p]{\frac{\norm{\w}_{1}}{\norm{\w + \w'}_{1}}\Mean^{p}(\lv, \w')) + \frac{\norm{\w'}_{1}}{\norm{\w + \w'}_{1}}\Mean^{p}(\lv', \w')}$ \todo{use subscripts 1 \& 2 instead of primes?  Write this with new weight notation.}

\fi

\end{enumerate}
\end{restatable}
%\end{theorem}

Taken together, the items of \cref{thm:pop-mean-prop} tell us that the mild conditions of axioms~\ref{def:cardinal-axioms:dgfirst}-\ref{def:cardinal-axioms:dglast} (generally assumed for welfare), along with the \emph{multiplicative linearity} axiom (\ref{def:cardinal-axioms:mult}), 
imply that welfare and utility, or malfare and loss, are measured in the \emph{same units} (e.g., \emph{nats} or \emph{bits} for \emph{cross-entropy loss}, square-$\mathcal{Y}$-units for \emph{square error}, or \emph{dollars} for \emph{income utility}).
Furthermore, the entirely milquetoast \emph{unit scale} axiom (\ref{def:cardinal-axioms:unit}) implies that sentiment values and aggregator functions have the same \emph{scale}, imbuing meaning to comparisons like ``the risk of group $i$ is above (or below) the population malfare.''
Finally, as far as fairness goes, the Pigou-Dalton transfer principle (axiom~\ref{def:cardinal-axioms:pd}) leads to the conclusion that $p \in [-\infty, 1)$ incentivize redistribution of utility from better-off groups to worse-off groups, and similarly, the corresponding principle for malfare (axiom~\ref{def:cardinal-axioms:apd}) yields the conclusion that  $p \in (1, \infty]$ incentivize redistribution of harm\footnote{
Note that, mathematically speaking, it is entirely \emph{valid} to quantify welfare with $p>1$ or malfare with $p<1$, and indeed such characterizations may arise in the analysis of unfair systems; however we generally advocate against \emph{intentionally creating} such unfair systems.}
from worse-off groups to better-off groups.

We may also conclude that the power-mean is effectively the only reasonable family of welfare or malfare functions.
Even without axioms~\ref{def:cardinal-axioms:mult}-\ref{def:cardinal-axioms:unit}, axioms ~\ref{def:cardinal-axioms:dgfirst}-\ref{def:cardinal-axioms:dglast} imply (via the Debreu-Gorman theorem) that all aggregator functions are still \emph{monotonic transformations} of power-means.
These and other results relating various aggregator functions to the relevant axioms are summarized in \cref{fig:ax-prop}.

\section{Statistical Estimation of Welfare and Malfare Values}
\label{sec:stat}

%[Hoeffding vs Bennett: McDiarmid / Expo ES? Rade?]

\if 0 % OLD

We first illustrate the ease with which $p$-power-means can be estimated, in contrast to the standard additively separable welfare formulations.
Perhaps surprisingly, we find that the plug-in estimators (i.e., the %obvious
 empirical welfare and malfare) are \emph{biased estimators}, yet they admit much sharper finite-sample tail bounds than the corresponding estimators for additively separable welfare formulations, which are unbiased.
 \fi

We now show that for countable populations, consistent estimators for sentiment values imply consistent estimators for aggregator functions (via the plugin estimator).
Despite this promising first step, in general, aggregator functions don't preserve \emph{unbiasedness} or even \emph{asymptotic unbiasedness} of sentiment value estimators, and furthermore, the \emph{rate of convergence} of consistent estimators to the true aggregator function depends intricately on the aggregator function in question.
%and the rates of convergence of aggregator function estimators are in general quite complicated.
The following lemma requires only the \emph{monotonicity axiom}, and allows us to bound \emph{aggregator functions} in terms of \emph{estimated sentiment values}.

\begin{lemma}[Statistical Estimation]%  (Range Only)]
\label{lemma:stat-est}
Suppose probability distribution $\ProbDist$ over $\X$, sample $\bm{x} \distributed \ProbDist^{m}$, and some function $f: \X \times \Population \to \RNN$.
Let \emph{sentiment value} function $\lv(\PopItem) \doteq \Expect_{x \distributed \ProbDist}[f(x; \PopItem)]$, and \emph{empirical sentiment value estimate} $\hat{\lv}(\PopItem) \doteq \EExpect_{x \in \bm{x}}[f(x; \PopItem)]$.
\if 0
Let \emph{sentiment value} $\lv(\PopItem)$ and \emph{empirical sentiment value} $\hat{\lv}(\PopItem)$ functions be %defined as
\[
\lv(\PopItem) \doteq \Expect_{x \distributed \ProbDist}[f(x; \PopItem)] \enspace, \quad \text{and} \ \hat{\lv}(\PopItem) \doteq \EExpect_{x \in \bm{x}}[f(x; \PopItem)] \enspace.
\]
\fi
%If it holds for some $\epsv \succ \lv$ %\bm{0}$ %that with probability at least $1 - \delta$ over choice of $\bm{x}$
If it holds for some $\epsv \succ \bm{0}$ that, with probability at least $1 - \delta$ over choice of $\bm{x}$,
$\forall \PopItem \in \Population: \, \hat{\lv}(\PopItem) - \epsv(\PopItem) \leq \lv(\PopItem) \leq \hat{\lv}(\PopItem) + \epsv(\PopItem)$, then with said probability,
for all aggregator functions $\Mean(\cdot; \cdot)$ obeying the \emph{monotonicity axiom} (\cref{def:cardinal-axioms}~\cref{def:cardinal-axioms:mono}) and {weights measures} $\wv$ over $\Population$, we have that
\[
\Mean_{p}(\bm{0} \vee (\hat{\lv} - \epsv); \wv)
  \leq \Mean_{p}(\lv; \wv)
  \leq \Mean_{p}(\hat{\lv} + \epsv; \wv)
  \enspace,
\]
where $\bm{a} \vee \bm{b}$ denotes the (elementwise) maximum.

\end{lemma}
\begin{proof}
This result follows from the assumption, and the \emph{monotonicity} axiom (i.e., adding/subtracting $\epsv$ can not decrease/increase the aggregate, %aggregator function, 
respectively).
The minimum with $0$ on the LHS is \emph{valid} simply because, by definition, sentiment values are nonnegative, and is \emph{necessary}, since $\Mean_{p}(\cdot; \cdot)$ is in general undefined on negative sentiment values.
\todo{2-sided variant?}
\end{proof}

\if 0 estimators do not imply unbiased estimators for aggregator functions and even a synthetically unbiased is not sufficient. Consider that we must estimate square of the mean rather than mean of the square.

For the additively separable form however an unbiased estimator of ex-2 the p is not difficult to obtain consider for instance that it is easy to estimate e x squared + v x and z equals x minus y. This does not have eaten by an unbiased estimator for the power mean family since the application of the P fruit this should be familiar to the reader at the same reasons the standard deviation but not the variance it is difficult to obtain an unbiased estimator for.
\fi

%Note that this result implies consistency
The principal question we are interested in however is not merely \emph{whether} an estimator is {consistent}, but rather \emph{how rapidly} it converges to the %aggregator function estimates converge to 
true aggregator function.
In particular, an $\varepsilon$-$\delta$ additive-error guarantee allows us to solve for the sample complexity of estimating a particular aggregator function to within $\varepsilon$-$\delta$ error.  
Furthermore, we are interested in \emph{uniform} sample complexity bounds, which need to hold \emph{uniformly} over a family of probability distribution and aggregator functions.
This is even trickier then showing simple single-function sample complexity bounds, because it can be the case that while any individual function in the family admits a sample complexity bound, the entire family has unbounded sample complexity.\footnote{This is essentially due to the order of existential quantifiers: each quantity in the family may admit bounded sample complexity, even while the entire family has unbounded sample complexity.\draftnote{Should we reference Bernoulli $p$ geometric welfare example here?}}%
The following result shows such a uniform guarantee for fair malfare functions, by %on risk values when loss ranges are bounded.
%We now reify this result, 
applying the well-known \citet{hoeffding1963probability} %Hoeffding
and \citet{bennett1962probability} %Bennett
bounds to show concentration, and derive an explicit form for $\epsv$.

\begin{restatable}[Statistical Estimation with Hoeffding and Bennett Bounds]{corollary}{corostatest}% (Range Only)]
%\begin{corollary}
\label{coro:stat-est}
Suppose fair power-mean malfare $\Malfare(\cdot; \cdot)$ (i.e., $p \geq 1$), discrete \emph{weights measure} $\wv$ over $\NGroups$ groups, \emph{probability distributions} $\ProbDist_{1:\NGroups}$, \emph{samples} $\bm{x}_{i} \distributed \ProbDist_{i}^{m}$, and \emph{loss function} $\LossFunction: \X \to [0, \frange]$ %, $\lv \in [0, \frange]^{\NGroups}$ 
s.t. $\lv_{i} = \Expect_{\ProbDist_{i}}[\LossFunction]$ and
%$\hat{\lv}_{i} \doteq \frac{1}{m}\sum_{j=1}^{m} \LossFunction(\bm{x}_{i,j})$.
%$\hat{\lv}_{i} \doteq \EExpect_{x \in \bm{x}_{i}}[\LossFunction(x)]$.
$\hat{\lv}_{i} \doteq \EExpect_{\bm{x}_{i}}[\LossFunction]$.
Then, with probability at least $1 - \delta$ over choice of $\bm{x}$,
\[
\abs{ \Mean_{p}(\lv; \wv) - \Mean_{p}(\hat{\lv}; \wv) } \leq \frange\sqrt{\frac{\ln \frac{2\NGroups}{\delta}}{2m}} \enspace.
  %Bennett:
  %? + \sqrt{\frac{2 v? \ln \frac{n}{\delta}}{m}}
\]
Alternatively, again with probability at least $1 - \delta$ over choice of $\bm{x}$, we have
\[
\abs{ \Mean_{p}(\lv; \wv) - \Mean_{p}(\hat{\lv}; \wv) } \leq \frac{\frange\ln \frac{2\NGroups}{\delta}}{3m} + \max_{i \in 1, \dots, \NGroups} \sqrt{\frac{2\Var_{\ProbDist_{i}}[\LossFunction] \ln \frac{2\NGroups}{\delta}}{m}} \enspace.
  %Bennett:
  %? + \sqrt{\frac{2 v? \ln \frac{n}{\delta}}{m}}
\]
%\end{corollary}
\end{restatable}

\Cref{coro:stat-est} follows directly from \cref{lemma:stat-est}, with Hoeffding and Bennett inequalities applied to derive $\epsv$ bounds, and similar results are immediately possible with arbitrary concentration inequalities.
In particular, %we can show
similar \emph{data-dependent} bounds may be shown, e.g., with \emph{empirical Bennett bounds}\todo{cite}, removing dependence on \emph{a priori known variance}.
Furthermore, while such bounds may be used for {evaluating} the welfare or malfare of a \emph{particular} classifier or mechanism (through $\lv$ and $\hat{\lv}$), in machine-learning contexts, $\lv$ may be a function of some model, so we must consider the entire \emph{space of possible models}, represented by some \emph{hypothesis class} $\HC$.
%The above methods immediately extend to .
Via the union bound, \cref{coro:stat-est} is sufficient for \emph{learning} over \emph{finite} $\HC$, as the exponential tail bounds allow $\HC$ %the family
to grow \emph{exponentially}, at \emph{linear} cost to sample complexity.
%Also a
As in standard uniform convergence analysis (generally discussed in the context of \emph{empirical risk minimization}), we can easily handle infinite hypothesis classes, and obtain much sharper bounds by considering \emph{data-dependent uniform-convergence bounds} over the family, e.g., with Rademacher averages \citep{bartlett2002rademacher}, \emph{localized Rademacher averages} \citep{bartlett2005local}, or \emph{empirically-centralized Rademacher averages} \citep{cousins2020sharp}.
% and appropriate concentration-of-measure bounds.

\if 0
\ifthesis
Furthermore, we may apply \cref{thm:firstbound} to bound $\epsv$ with an \emph{empirically centralized Rademacher average}, either elementwise (over groups) with a \emph{union bound}, or jointly, if the loss functions for each group may be combined into a single function family.
\todo{backref rade bounds}
\else
Furthermore, we may bound $\epsv$ %(per-group)
 over the entire model space with \emph{data-dependent uniform convergence techniques}, such as \emph{Rademacher averages} \citep{bartlett2002rademacher}, \emph{localized Rademacher averages} \citep{bartlett2005local}, or \emph{centralized Rademacher averages} \citep{cousins2020sharp}.
\fi
\fi

\subsection{The Empirical Malfare Minimization Principle}
\label{sec:stat:emm}
In learning contexts, minimizing the \emph{malfare} among all groups generalizes minimizing \emph{risk} of a single group.
These statistical estimation bounds immediately imply that the \emph{empirical malfare-optimal} solution is a reasonable proxy for the true malfare-optimal solution, as we now formalize.\todo{Cite ERM, formalize and reference EMM.}
\Cref{fig:linear-emm} illustrates empirical malfare minimization in action with a \emph{linear classifier} %in $\R^{2}$ over
on two groups.
\begin{definition}[The Empirical Malfare Minimization (EMM) Principle]
\label{def:emm}
Suppose hypothesis class $\HC \subseteq \X \to \Y$, training samples $\bm{z}_{1:\NGroups}$ drawn from distributions $\ProbDist_{1:\NGroups}$ over $\X \times \Y$, loss function $\LossFunction: \Y \times \Y \to \RNN$, malfare function $\Malfare$, and group weights $\wv$.
The \emph{empirical malfare minimizer} is then defined as
\[
\hat{h} \doteq %\argmin_{h \in \HC} \hat{\Malfare}(h; \bm{z}_{1:\NGroups}, \Utility) = 
\argmin_{h \in \HC} \Malfare\left(i \mapsto \ERisk(h; \LossFunction, \bm{z}_{i}); \wv \right) \enspace,
%\EExpect_{(x, y) \in \bm{z}_{i}}[\LossFunction(y, h(x))]
\]
and the EMM principle states that $\hat{h}$ is a reasonable proxy for the \emph{true malfare minimizer}
\[
h^{*} \doteq \argmin_{h \in \HC} \Malfare \left(i \mapsto \Risk(h; \LossFunction, \ProbDist_{i}); \wv \right) \enspace.
\]
\end{definition}

\begin{SCfigure}
%\begin{figure}

\newcommand{\gone}{\textbf{\large\sun}}
\newcommand{\gtwo}{\textbf{\large\leftmoon\!}}

\tikzfading[name=fade out,
inner color=transparent!0,
outer color=transparent!100]

\tikzfading[name=fade out soft,
inner color=transparent!25,
outer color=transparent!100]

{
%\centering
%\vspace{-0.3cm}
\hspace{-0.8cm}\begin{tikzpicture}[
    xscale=1.075,
    yscale=0.9,
    fadenode/.style={fill=white,circle,path fading=fade out soft,inner sep=2pt},
    groupplane/.style={draw=purple!50!black,thick,dashed,line cap=round},
    sharedplane/.style={draw=purple!50!black,thick,line cap=round},
    planearrow/.style={thick,style=-{Triangle[length=4.5pt,width=3pt]}}, %-{Classical TikZ Rightarrow[length=1.7mm,width=2.65mm,line width=0.35mm]}
  ]

%\draw[draw=black,use as bounding box] (-3, -1.5) rectangle (6, 3.5);
%\draw[draw=black] (-3, -1.5) rectangle (8, 3.5);
%\clip (-3, -1.44) rectangle (8, 3.4);

%Legend:

\node[rectangle,draw=black] at (4.15, 2.2) {%
\!\!\!\!\begin{tabular}{cl}
\textcolor{red}{$\blacksquare$} & Class A\\ %+ \\ %$A$ \\
\textcolor{blue}{$\blacksquare$} & Class B \\ %- \\ % $B$ \\
%\onslide<2,4->
\gone & Group $1$ \\
%\onslide<3,4->
\gtwo & Group $2$ \\
\textcolor{purple!50!black}{$\bm{\brokenvert}$} & ERM $h^{*}$ \\ % Hyperplanes \\
\textcolor{purple!50!black}{$\bm{|}$} & EMM $h^{*}$ \\ %Hyperplane \\ %Linear Separator \\
\end{tabular}\!\!\!%
};

%\draw[purple] (-0.1, -1) -- (0.5, 3);

%Hyperplanes:

%\onslide<2,4->
{
\begin{scope}[rotate=-25]
%\draw[red,planearrow] (0, 1.1) -- (-0.25, 1.1);
%\draw[blue,planearrow] (0, 1.1) -- (0.25, 1.1);
\draw[groupplane] (0, -1.5) -- (0, 3.7);
%\draw[purple,thick,dashed] (0, -2) -- (0, 2);
\node at (0, -1.7) {\color{purple!50!black}\bf\gone};
\node at (0, 3.9) {\color{purple!50!black}\bf\gone};
\end{scope}
}

%\onslide<3,4->
{
\begin{scope}[rotate=47.5]
%\draw[purple,thick,dashed] (4.3, -4) -- (4.3, 4);
%\draw[purple,thick,dashed] (2.15, -2) -- (2.15, 2);
%\draw[red,planearrow] (2.15, 0) -- (2.5, 0);
%\draw[blue,planearrow] (2.15, 0) -- (1.9, 0);
\draw[groupplane] (2.15, -2.7) -- (2.15, 2.7);
\node at (2.15, -2.9) {\color{purple!50!black}\bf\gtwo};
\node at (2.15, 2.9) {\color{purple!50!black}\bf\gtwo};
\end{scope}
}

\if 0
\tikzfading[name=fade right,left color=transparent!0,right color=transparent!100]
\tikzfading[name=fade left,left color=transparent!100,right color=transparent!0]
\fi

%\onslide<5->
{
\begin{scope}[rotate=18]
%\draw[purple,thick] (2.2, -4) -- (2.2, 6);
%\draw[purple,thick] (1.1, -2) -- (1.1, 3);
%\draw[purple,thick] (1.1, -4) -- (1.1, 4);
%\fill[red,path fading=fade left,fading transform={rotate=18}] (1.0,-2.1) rectangle (1.1, 3.5);
%\fill[blue,path fading=fade right,fading transform={rotate=18}] (1.1,-2.1) rectangle (1.2, 3.5);
%\draw[postaction={}red,fading=west] (1.0,-2.1) rectangle (1.1, 3.5);
\draw[sharedplane] (1.1, -2.1) -- (1.1, 3.5);
%\draw[purple!50!black,postaction={decorate,decoration={text along path,text={group 1+2}}}] (1.1, -4) -- (1.1, 4);
%\draw[purple!50!black,postaction={decorate,decoration={text along path,text={testestest $\gone\gtwo$ testestest}}}] (1.1, -4) -- (1.1, 4);
%\draw[purple!50!black,postaction={decorate,decoration={text along path,text={testestest \gone \gtwo \gone \gtwo \gone \gtwo \gone \gtwo \gone \gtwo \gone \gtwo \gone \gtwo \gone \gtwo testestest}}}] (1.1, -4) -- (1.1, 4);
%\node at (1.1, -2.5) {\color{purple}\bf\gone\gtwo};
\end{scope}
}

%\onslide<2,4->
{
\node[fadenode] at (0.333, 1.69) {\textcolor{red}{\gone}};
\node[fadenode] at (-0.1, 0.1) {\textcolor{red}{\gone}};
\node[fadenode] at (-1.01, 1) {\textcolor{red}{\gone}};
\node[fadenode] at (-2.2, 3) {\textcolor{red}{\gone}};
\node[fadenode] at (-1.875, -0.5) {\textcolor{red}{\gone}};

\node[fadenode] at (0, -0.456) {\textcolor{blue}{\gone}};
\node[fadenode] at (1, 1.1) {\textcolor{blue}{\gone}};
\node[fadenode] at (2.3, 2.1) {\textcolor{blue}{\gone}};
\node[fadenode] at (3.1, -0.5) {\textcolor{blue}{\gone}};
}

%\onslide<3,4->
{
\node[fadenode] at (-0.69, 0.5) {\textcolor{red}{\gtwo}};
\node[fadenode] at (-2.08, 1.87) {\textcolor{red}{\gtwo}};
\node[fadenode] at (0.146, -1.24) {\textcolor{red}{\gtwo}};
\node[fadenode] at (1.824, 0.863) {\textcolor{red}{\gtwo}};

\node[fadenode] at (0.42, 2.81) {\textcolor{blue}{\gtwo}};
\node[fadenode] at (2.1, 3.08) {\textcolor{blue}{\gtwo}};
\node[fadenode] at (2.75, 0.6) {\textcolor{blue}{\gtwo}};
}

%Errors:
%\onslide<5->
\begin{scope}[blend mode=screen]
{
\node[draw,circle,black,inner sep=0.25em] at (1.824, 0.863) {\textcolor{red}{\gtwo}}; %thick,dotted,
\node[draw,circle,black,inner sep=0.25em] at (0, -0.456) {\textcolor{blue}{\gone}};
}
\end{scope}

\end{tikzpicture}%\hspace{-0.2cm}\null %\\
}
\caption{
Empirical malfare minimization on a \emph{linear classifier} family in $\R^{2}$ (with affine offset) over two groups.
Note that classification is realizable for both groups \emph{individually}, in the sense that both are linearly separable, thus there exists a 0-risk classifier for each, though \emph{jointly}, they are not realizable.
Risk-optimal classifiers are shown for both groups (dashed lines), as is a malfare-optimal classifier (solid line).
Note that exactly which classifier is optimal depends on the weighting and malfare metric, but the selected malfare-minimizer compromises fairly in the sense that each group suffers %only
 one error (circled).
}

\label{fig:linear-emm}
\end{SCfigure}
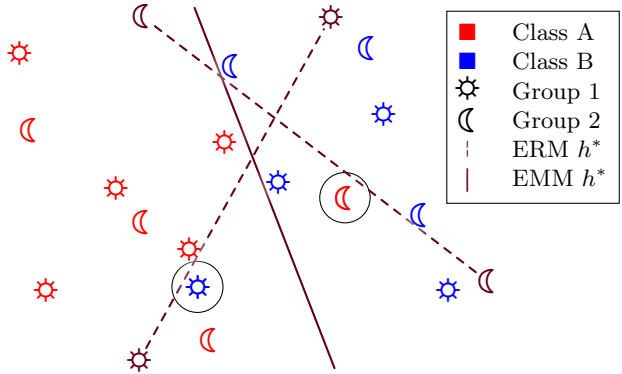
%\end{figure}

\draftnote{Contrast with ``General Fair Empirical Risk Minimization''}

\if 0

%, and similarly with uniform convergence and Rademacher averages, to show such bounds over a \emph{class} of sentiment vectors (i.e., each describing how a classifier impacts all groups).%  dependence on $n$.
In all cases, these bounds provide further justification for using the \emph{power-mean} over the \emph{additively separable} form, as it is substantially harder to achieve comparable bounds on
%\[
%\Mean_{p}^{p}(\lv; \wv) \enspace,
%\]
$\Mean_{p}^{p}(\lv; \wv)$,
due to the increased difficulty of controlling the \emph{range} and \emph{Lipschitz constant} of these quantities.
Of course, as power-mean bounds imply bounds on the additively-separable form (and vice-versa), we recommend working with power means, and then converting back to the additively separable form (if so desired).

%Variance-sensitive bounds are particularly irksome, as the variance proxies of interest are $\Var[\lv^{p}](\PopItem)$, rather than simply $\Var[\lv(\PopItem)]$.

\todo{It's about balancing interpretability and estimability; I think we're better on both counts.  Changing $p$ doesn't change interpretation (as much). Plot of normalized vs denorm?}

\fi

%\paragraph{Experimental Validation}
\subsection{Experimental Validation of Empirical Malfare Minimization}
\label{sec:stat:exper}

\begin{SCfigure}[1.5]

%\null

\vspace{-0.65cm} \\

\mbox{\null\hspace{-0.25cm}\includegraphics[width=0.53\textwidth,trim={0.49cm 0.2cm 1.5cm 0.64cm},clip
]{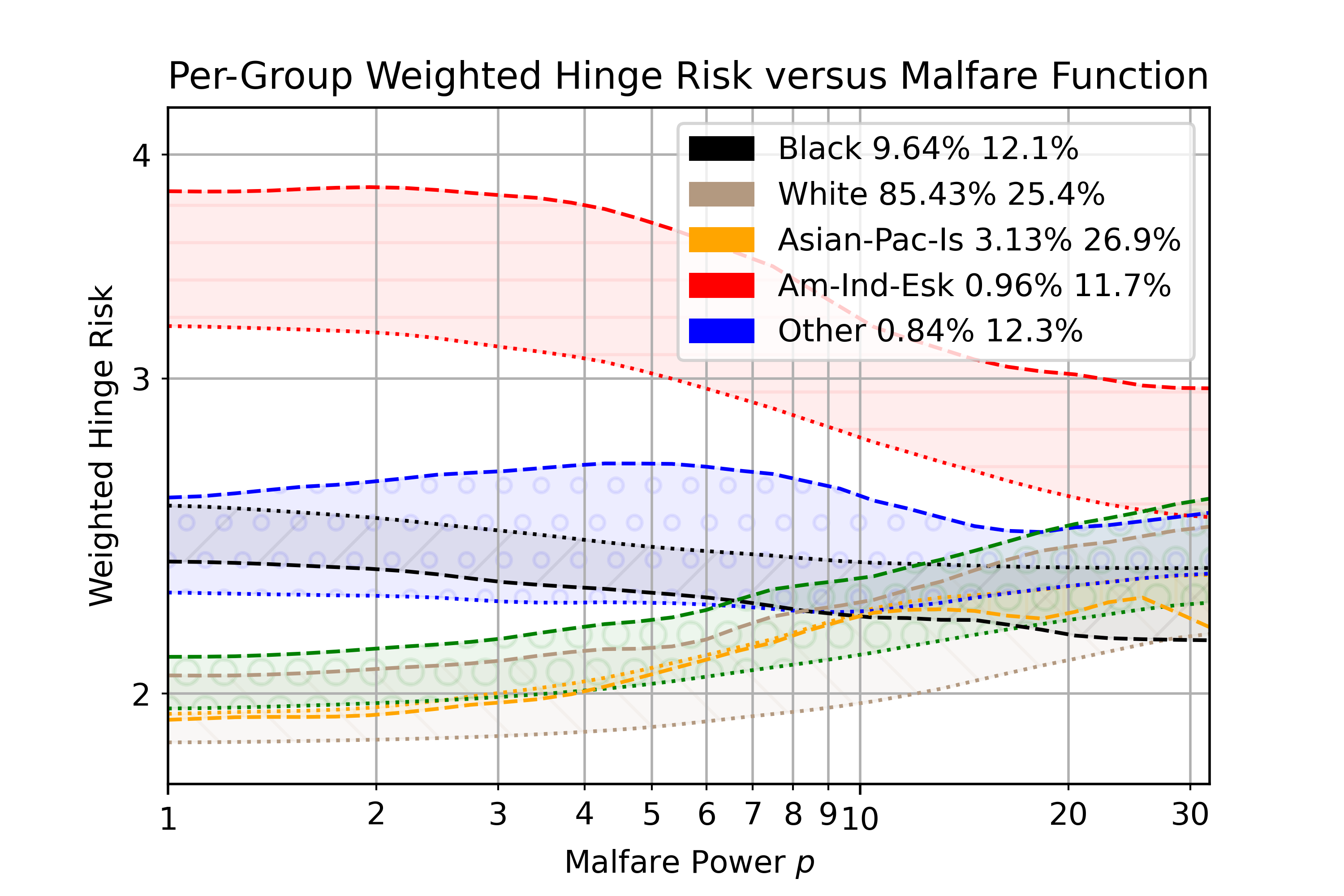}\hspace{-0.1cm}\null}

\vspace{-0.15cm}

\caption{%\small
%SC side caption test
We minimize malfare on a \emph{weighted hinge-loss} SVM, with $\NGroups = 5$ %ethnoracial protected 
racial groups, listed in the figure legend with \emph{group weight} $\wv_{i}$ (population frequency) and \emph{class bias} $\bm{b}_{i}$ (proportion with %yearly
 income $\geq \$$50,000 \emph{per annum}).
Due to existing societal inequity, %groups report vastly different 
\emph{class imbalance} varies widely by group, so we weight %each group's empirical 
all risk values as $\smash{\mathsmaller{\frac{1}{\bm{b}_{i}}}} \smash{\ERisk}(h; \LossFunction_{\mathrm{hinge}}, \bm{z}_{i})$.
We %then 
report %the
 per-group training %(dot) 
(dotted) %$\small\cdots$
 and test %(dash) 
 (dashed) %$\tiny---$
  hinge risk, %with their gap shaded, % and hatched, 
  along with the $\Malfare_{p}(\cdot; \wv)$ value (green) of the EMM solution %$\hat{h}$ optimized for the $\Malfare_{p}(\cdot; \wv)$ objective, 
%\vspace{-0.17cm}
\\[0.2cm]
$\displaystyle
\null \ \ \ \ \ \hat{h} \doteq {\argmin_{h \in \HC}} \, \Malfare_{p} \! \left( i \mapsto \mathsmaller{\frac{1}{\bm{b}_{i}}} \ERisk(h; \LossFunction_{\mathrm{hinge}}, \bm{z}_{i}); \wv \right) \enspace,
%\vspace{-0.15cm}
$
\\[0.1cm]
  as a function of $p \in [1, 32]$.
%A full %Detailed
% description of the experimental setup is provided in \cref{sec:appx:exper:setup}.%the appendix.
The experimental setup is fully detailed in \cref{sec:appx:exper:setup}.
%TODO BEIGE WHITE, AVERAGE PURPLE? THINNER LINES? CITE NEW KLEINBERG
%
%The gap between training and test risk and malfare statistics is shaded and hatched.
\if 0
\[
\Malfare_{p}\left( i \mapsto \frac{1}{\bm{b}_{i}} \Risk(h; \LossFunction_{\mathrm{hinge}}, \bm{z}_{i}); \wv \right)
\]
\fi
\if 0
with income over 50K it's buried the stew to class in balance the classification problem is much easier for the disadvantaged groups. To correct this we wait the hinge loss for each group by one over the proportion of the group but high-income
\fi
}
\label{fig:exp-adult}
\end{SCfigure}

\Cref{fig:exp-adult} presents a brief experiment on the lauded \texttt{adult} dataset, %[TODO citation], 
where the task is to predict whether income is above or below $\$50$k/year. %\$50,000 per annum.
%We train a regularized linear SVM to minimize malfare on per-group weighted hinge-risks. %, using \emph{race} designations for group membership.
%In this experiment, 
We train $\Malfare_{p}(\cdot; \wv)$-minimizing SVM, and 
 find significant variation in model performance (as measured by risk)
  between groups;
 in general, we observe that the classifier is most accurate for the 
 %generally low risk for the 
 \emph{white} and \emph{Asian-Pacific-Islander} % 
groups, and %high risk 
generally less accurate for the \emph{native American} and \emph{other} groups.
%
%The model trained with $p=1$ is simply 
The $p=1$ model is a standard weighted SVM, with poor performance for small and traditionally marginalized groups, as expected in an $85.43\%$ majority-white population. %but as $p$ increases, we observe interesting fairness tradeoffs. %see the model make interesting fairness tradeoffs.
As $p$ increases (towards egalitarianism), we observe interesting fairness tradeoffs; training malfare %and weighted average risks 
increases monotonically, and in general (but not monotonically\footnote{Note that for continuous loss functions and $g=2$ groups, group training risks are monotonic in $p$, as seen in the supplementary \emph{gender-group} experiments.}), \emph{white} and \emph{Asian} training risks increase, as the remaining risks decrease, and greater equity is achieved.
%In general, with increasing $p$ model performance degrades 
%for the  groups, and improves for the remaining groups, however, while the \emph{malfare} and \emph{weighted average} training risks of the EMM model increase monotonically in $p$, individual group risks are not necessarily monotonic.
At first, %for $p \in [2, 5]$, we mostly see
most improvement is in the relatively-large ($9.64\%$), high-risk Black group, %, which has substantially higher risk, and encnompasses $9.64\%$ of the population.
%For $p \in [5, 20]$, 
but for larger $p$, the much smaller ($0.96\%$), but even higher-risk, native American group %t becomes optimal to substantially improve performance for the native American group, which has substantially higher risk, but is only $0.96\%$ of the population.
%The highest risk population
%sees the most improvement.
sharply improves.
%improves most.

%Although b
Both training and test performance generally improve for high-risk groups, but significant overfitting occurs in small groups and malfare. %, and for the malfare itself.
This is unsurprising, as although SVM generalization error %of the SVM
 is well-understoond \citep[see][Chapter~26]{shalev2014understanding}, %such
 bounds are generally vaccuous for tiny subpopulations of $\approx 400$ individuals. % the groups containing only $\approx 400$ instances.
In general, %we see more overfitting as $p$ increases,
overfitting increases with $p$, due to higher relative importance of small high-risk groups on $\smash{\hat{h}}$. %the EMM-optimal model 
This experiment validates %the malfare objective 
EMM as a fair-learning technique, with the capacity to specify tradeoffs between %performance of
 majority and marginalized groups, %and it also highlights the phenomenon of 
while demonstrating \emph{overfitting to fairness}, which we formally treat in the sequel.
%TODO phrasing.
We observe similar fairness tradeoffs in our supplementary experiments (\cref{sec:appx:exper:exper}), %wherein we explore
on weighted and unweighted SVM %, (hinge risk), 
%{logistic regression}, and 0-1 loss, %using race-based and gender-based group membership
and logistic regressors with race and gender groups. %, are presented in the appendix.

%TODO overall / avg?

%Individual group behavior is extremely complicated. General trends improve our self groups but not necessarily monotonic for instance other population is very small and never worse than 3rd so it doesn't strongly impact the objective Initial behaviour is mostly the relatively large black Initial behaviour is mostly the relatively large black or as Where subsequently the extremely small but extremely high risk 8 of population is improved

\draftnote{
TODO: Basic Rademacher bound?
TODO: transition.
TODO: LR experiments?
}

\section{Comparative Analysis of Welfare, Malfare, and Inequality Indices} %Functions}
\label{sec:comparisons}

This section serves as an interlude between the concept and axiomatic derivation malfare, and the statistical and machine learning applications of malfare minimization.
%In this section,
Here we examine some of our core decisions, and explore
%what would change, should we have made alternative decisions.
the differences that arise under alternative axioms and other counterfactuals.

In particular, \cref{sec:comparisons:welfare} shows that malfare and welfare functions are not equivalent, and describes salient differences that arise when trying to estimate them from sampled (dis)utility values. \Cref{sec:comparisons:as} then shows that under an alternative axiomatization, i.e., that of additive separability, the concept of uniform sample complexity is generally ill-behaved.
Finally, \cref{sec:comparisons:inequality} explores the relationships between \emph{inequality indices} and welfare or malfare functions, deriving deep connections between the power mean and the Atkinson, Theil, and generalized entropy indices. % but noting alternative Gini ...

\draftnote{TODO PREV SEC INTRO?  All secs?}

%%%%%%%%%%%%%%%%%%%%
%Welfare Comparison:

\subsection{The Non-Equivalence of Welfare and Malfare Functions}
%\section{Contrasting Welfare and Malfare}
\label{sec:comparisons:welfare}

\begin{figure}
\tikzset{
    declare function={
            mean3(\va,\vb,\vc) = ((\va + \vb + \vc) / 3);
            pm2(\va,\vb,\p) = ((\va ^ \p + \vb ^ \p) / 2) ^ (1 / \p);
            as3(\va,\vb,\vc,\pp) = ((\va ^ \pp + \vb ^ \pp + \vc ^ \pp) / 3);
            pm3(\va,\vb,\vc,\pp) = as3(\va,\vb,\vc,\pp) ^ (1 / \pp);
            asw3(\va,\vb,\vc,\wa,\wb,\wc,\pp) = (\wa * \va ^ \pp + \wb * \vb ^ \pp + \wc * \vc ^ \pp);
            pmw3(\va,\vb,\vc,\wa,\wb,\wc,\pp) = asw3(\va,\vb,\vc,\wa,\wb,\wc,\pp) ^ (1 / \pp);
            sgn(\x) = (and(\x<0, 1) * -1) + (and(\x>0, 1) * 1);
        },
}
\tikzset{
    %Main line types:
    mainline/.style={thick},
    %Aux line types:
    boundline/.style={densely dashed,line cap=round,thin},%,opacity=0.5},
    popline/.style={thin,opacity=0.8},
    limitline/.style={dots,thick,opacity=0.9},
    %Bound fill
    linefill/.style={opacity=0.4,pattern=vertical lines},
    solidfill/.style={opacity=0.1},
    allfill/.style={solidfill,postaction={linefill}},
    %Discontinuities:
    discfill/.style={only marks,mark=*,mark options={scale=0.8,thin}},
    dischole/.style={discfill,fill=white},    
  }

\colorlet{darkgreen}{green!50!black}

\centering

\begin{subfigure}[t]{0.495\textwidth}
%\begin{subfigure}[t]{1\textwidth}

\hspace{-0.1cm}\begin{tikzpicture}[blend group=lighten]
\begin{axis}[
    width=1.14\textwidth,
    axis equal image,
    xmin=-3,xmax=5,
    ymin=0,ymax=2,
    domain=-3:5,
    %domain=-4:4,
    samples=16,smooth,%thin,
    no markers,
    xlabel={$\smash{p}$},
    legend pos={south east},
    xtick distance=1,ytick distance=1,
    grid,
    grid style={draw=black,draw opacity=0.25},
    legend style={font=\small,draw=none,fill=none},
    every axis/.append style={font=\small},
  ]

%\beta approaches \infty
\addplot[color=black,ultra thick,samples=2] { 1 };

%\beta=0
\addplot[color=red,thick,domain=-3:0.0001,samples=2] { 0.003 }; %{ 0 };
 \addplot[color=red,thick,domain=0.0001:1] {
    %min(1, sqrt(4 * x)) * 
    %x^(1/4) * 
    (1-(1-x)^4) *
    pm3(0.00001, 1, 1.99999, x)
  };
\addplot[color=red,thick,domain=1:5] { pm3(0, 1, 2, x) };

\foreach \b[evaluate=\b as \bp using {\b^(1.6) / 5 } %{\b*\b}
  ] in {1,2,...,25} {
  %\pgfmathsetmacro\colp{100-\b*10}
  %\addplot[color={red!\colp!blue},mainline] { pm3(1 + x, 2 + x, 3 + x, \b) - x };
  %\pgfmathsetmacro\k{100-10*\b}
  %\addplot[color=red!\k,mainline] { pm3(1 + x, 2 + x, 3 + x, \b) - x };
  %\addplot[color=red!90,mainline] { pm3(1 + x, 2 + x, 3 + x, \b) - x };
  %\pgfmathsetmacro\k{\b*10}
  %\addplot[color=myblue!\k] { pm3(1 + x, 2 + x, 3 + x, \b) - x };
  \pgfmathparse{100-\b*4}
  \edef\temp{\noexpand\addplot[color=red!\pgfmathresult!blue] { pm3(0 + \bp, 1 + \bp, 2 + \bp, x) - \bp };
  %\noexpand\addlegendentry{$\beta=\b$};
  }
  \temp
  %\addplot[color=red!\pgfmathresult!blue] { pm3(1 + x, 2 + x, 3 + x, \p) - x };
  %
  %\addplot[color=blue,mainline] { ;
}

\end{axis}
\end{tikzpicture}

\subcaption{{
$\Mean_{p}(\lv + \beta; \wv) - \beta$ as a function of $p$. \\
\textcolor{red}{Red: $\beta = 0$}; \textcolor{blue}{Blue: increasing $\beta$}; \textcolor{black}{Black: $\lim_{\beta \to \infty}$.}
}}
\label{fig:aggregator-beta}

\end{subfigure}
\begin{subfigure}[t]{0.495\textwidth}
%\begin{subfigure}[t]{1\textwidth}

%\pgfplotsset{major grid style={dotted,green!50!black}}

\hspace{-0.03cm}\begin{tikzpicture}[blend group=lighten]
\begin{axis}[
    width=1.14\textwidth,
    axis equal image,
    xmin=0,xmax=8,
    ymin=0,ymax=2,
    domain=0.0001:8,
    %domain=-4:4,
    samples=16,smooth,%thin,
    no markers,
    xlabel={$\smash{\beta}$},
    legend pos={south east},
    xtick distance=1,ytick distance=1,
    grid,
    grid style={draw=black,draw opacity=0.25},
    legend style={font=\small,draw=none,fill=none},
    every axis/.append style={font=\small},
  ]

%p=0
%\addplot[color=red,thick,domain=0.0001:1,samples=60] { (x * (1 + x) * (2 + x))^(1/3) - x };
%\addplot[color=red,thick,domain=1:16] { (x * (1 + x) * (2 + x))^(1/3) - x };

%p=1
%\addplot[color=red,dashed,samples=2] {1};
\addplot[color=red,very thick,samples=2] {1};

\foreach \p[evaluate=\p as \pp using {1+\p/2}] in {1,2,...,10} {
  \pgfmathparse{100-\p*10}
  \edef\temp{\noexpand\addplot[color=red!\pgfmathresult!blue] {
  %pm3(1 + x, 2 + x, 3 + x, \p) - x 
  pm3(0 + x, 1 + x, 2 + x, \pp) - x 
  };
  %\noexpand\addlegendentry{$p=\p$};
  }
  \temp
}
\foreach \p[evaluate=\p as \pp using {1-\p/2}] in {1,2,...,10} {
  \pgfmathparse{100-\p*10}
  \edef\temp{
  \noexpand\addplot[color=red!\pgfmathresult!darkgreen,domain=0.0001:1,samples=32] { pm3(0 + x, 1 + x, 2 + x, \pp) - x };
  \noexpand\addplot[color=red!\pgfmathresult!darkgreen,domain=1:8] { pm3(0 + x, 1 + x, 2 + x, \pp) - x };
  %\noexpand\addlegendentry{$p=\p$};
  }
  \temp
}

\end{axis}
\end{tikzpicture}

\subcaption{{
$\Mean_{p}(\lv + \beta; \wv) - \beta$ as a function of $\beta$. \\
\textcolor{red}{Red: $p = 1$}; \textcolor{blue}{Blue: increasing $p$}; \textcolor{darkgreen}{Green: decreasing $p$}.
}}
\label{fig:aggregator-beta}

\end{subfigure}

%\\

\caption{
Plots of the affine-transformed $\Mean_{p}(\lv + \beta; \wv) - \beta$ aggregator function, for various values of $\beta$ and $p$.
All plots use $\lv \doteq (0, 1, 2)$ and $\wv = (\frac{1}{3}, \frac{1}{3}, \frac{1}{3})$ (i.e., unweighted power-means).
}
\label{fig:aggregator-affine}

\end{figure}
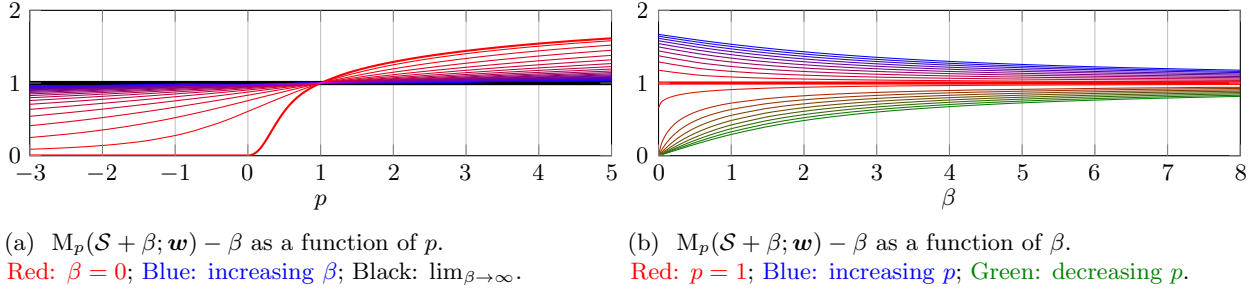

\draftnote{Thm should be: to get equivalence, require particular form of $f_{\lv}$.  Is non-affine, thus can't stat estimate easily?}

We now take a moment to comment on the surprising dissimilarity between welfare and malfare functions.
In particular, we show that intuition from univariate optimization, where maximization and minimization are symmetric, breaks down for welfare maximization and malfare minimization, % functions, 
and furthermore, % we show that, 
from the perspective of estimation, %show in a very strong sense that, 
except for the egalitarian and utilitarian cases, no fair welfare function is equivalent to any fair malfare function.

We would like to show that there \emph{does not exist} some mapping between utility and disutility values, such that under said mapping, welfare and malfare are equivalent.
Furthermore we adopt a weak notion of equivalence, requiring only that they induce the same partial ordering.
In other words, given fair welfare and malfare functions $\Welfare(\cdot; \cdot)$ and $\Malfare(\cdot; \cdot)$, we now study the existence of mappings $f_{\lv}(\cdot)$ and $F_{\Mean}(\cdot)$ such that
\[
F_{\Mean} \circ \Welfare(f_{\lv} \circ \lv; \wv) = \Malfare(\lv; \wv) \enspace, \quad 
  %\text{equivalently???} \ \Welfare(\lv; \wv) = F_{\Mean}^{-1} \circ \Malfare(f_{\lv}^{-1} \circ f; \wv)
  \left( \text{or equivalently}, \ \Welfare(\lv; \wv) = F_{\Mean} \circ \Malfare(f_{\lv} \circ \lv; \wv) \right)
  \enspace,
\]
where, as usual, $\lv$ must be positively connoted in $\Welfare(\cdot, \cdot)$ and negatively connoted in $\Malfare(\cdot, \cdot)$.

Of course, such a function pair exists in general; for $\Welfare_{p}(\cdot; \cdot)$ and $\Malfare_{p}(\cdot; \cdot)$, we may take $f_{\lv}(u) \doteq u^{\nicefrac{p}{q}}$ and $F_{\Mean}(u) \doteq f_{\lv}^{-1}(u) = u^{\nicefrac{q}{p}}$.
However, from the perspective of \emph{estimation} (and thus from the perspective of \emph{machine learning}), this relationship is unsatisfying, as we want to exploit a relationship between $\Expect_{\ProbDist}[\loss]$ and an empirical estimate $\smash{\EExpect_{\bm{x}}[\loss]}$, for some loss function $\loss$, where $\bm{x} \distributed \ProbDist^{m}$.
In general, $\smash{\EExpect_{\bm{x}}[\loss^{\nicefrac{p}{q}}]}$ is a biased estimator of $\smash{\Expect_{\ProbDist}^{\nicefrac{p}{q}}[\loss]}$, as %would be
is any nonlinear function; %, a good estimator for $\Expect_{\ProbDist}[\loss]$ does not necessarily imply a good estimator for $\smash{\Expect_{\ProbDist}^{\nicefrac{p}{q}}[\loss]}$.
we thus restrict our attention to \emph{affine functions}, i.e., we require $f_{\lv}(u) \doteq \beta + \alpha u$.
\draftnote{TODO explain linearity assumption; cxn to reweighting.}

At first glance, this seems promising, as in univariate optimization, we have
\[
\max_{x \in \X} f(x) = -\min_{x \in \X} -f(x) \enspace,
\]
which would seem to suggest we take $F_{\Mean}(u) = -u$ and $f_{\lv}(u) = -u$, yielding
\[
\Welfare(\lv; \wv) = %\lv_{1} = -(-\lv_{1}) = 
F_{\Mean} \circ \Malfare(f_{\lv} \circ \lv; \wv) -\Malfare(-\lv; \wv) \enspace?
\]
Unfortunately, 
%The obvious thing to do is to take utility equals negative risk however 
except in the \emph{egalitarian} ($p = -\infty$ welfare, $p = \infty$ malfare) and \emph{utilitarian} ($p = 1$) cases, or when $\abs{\Population} = 1$, we have the necessary requirements that sentiment values be nonnegative, as otherwise key properties (various cardinal welfare axioms) of the power-mean break down.
Thus the attempt to pattern-match the univariate case has failed; a more sophisticated strategy is required.% we would require taking fractional powers and roots of negative numbers.

\if 0
TODO OLD:

We now close at

Despite their similar similarity and largely symmetric axiomatization, there is in general no equivalence between fair welfare and malfare functions.
In particular, given a risk function we can define a mouth or function and if we wish to convert it to a welfare function we must Define a utility function as risk must be negatively conducted and utility must be positively connoted the obvious thing to do is to take utility equals negative risk however except in the egalitarian and utilitarian cases we have the necessary requirements that both utility and risk be non- non- as otherwise we would require taking fractional powers and roots of negative numbers.
\fi

The sophomoric approach is then to preserve nonnegativity, by taking $f_{\lv}(u) \doteq \beta - u$, and $F_{\Mean}(u) \doteq f_{\lv}^{-1}(u) =  \beta - u$, where we must choose $\beta$ to exceed the maximum utility value.
Of course, the choice of $\beta$ is rather arbitrary, % as any are prime greater than our would also suffice to satisfy the negativity constraint 
and this strategy is fruitless with \emph{unbounded sentiment values} (e.g., the cross entropy loss or square loss).  %any convex regression loss).
Furthermore this strategy fails to ensure fairness, in the sense that the original fairness concept is not preserved, and in particular, the status-quo (utilitarianism) is preserved as $\beta$ is taken to infinity, i.e.,
\[
\forall p \in \R: \lim_{\beta \to \infty} \Mean_{p}(\beta \pm \lv, \wv) - \beta = \pm \Mean_{1}(\lv; \wv) \enspace.
\]
We thus conclude that, in general, there is no way to contort a \emph{loss function} into a \emph{utility function} such that any welfare function of \emph{expected utility} preserves the fairness trade-offs made by some malfare function on the \emph{expected loss} (nor vice versa).

%\bigskip

\paragraph{The Statistical Inestimability of Welfare Functions}

In this work, we focus primarily on fair learning and statistical estimation with malfare functions.
Much of what we accomplish is not possible for fair welfare functions, primarily because $\Welfare_{p}(\cdot; \cdot)$ for $p \in [0, 1)$ are not %necessarily 
Lipschitz continuous. %;  %is non-Lipschitz; 
%e.g.
\if 0
for example, the \emph{Nash social welfare} (a.k.a.\ \emph{unweighted geometric welfare}) $\Welfare_{0}(\lv; \omega \mapsto \mathsmaller{\frac{1}{\NGroups}}) = \smash{\sqrt[\NGroups]{\vphantom{\prod}\smash{\prod_{i=1}^{\NGroups}} \lv_{i}}}$ is unstable to perturbations of each $\lv_{i}$ around $0$.
\fi
Leveraging this idea, we now construct welfare estimation tasks %instances, in 
for which \emph{sample complexity} is \emph{significantly larger} than mean estimation, and may even be \emph{unbounded}. %, as exemplified by \cref{ex:nsw-estimation,ex:wnsw-estimation}. %the following.

We first show that even the \emph{unweighted Nash social welfare} of two groups is surprisingly difficult to estimate.
%In particular, note 
This result is best appreciated in light of the fact that the sample complexity of estimating the bias $p$ of a Bernoulli coin is $\LandauOmega(\smash{\frac{\ln \mathsmaller{\frac{1}{\delta}}}{\varepsilon}})$, which is sharp as $p \to 0$, $\varepsilon \to 0$, $\delta \to 0$, yet we find that the sample-complexity of welfare estimation is substantially larger. %, %this problem 
%we find a substantially larger sample lower-bound. %which is substantially better than the sample complexity of related welfare estimation problems.
%In particular,
Note also that the construction is quite natural, utilizing only two (unweighted) groups, with utility samples of \emph{bounded range}.
\draftnote{Check these carefully.}

\iftrue %Unweighted case: possible but ``more difficult'' than mean estimation
\begin{example}[Estimating Nash Social Welfare]
\label{ex:nsw-estimation}
\if 0
Suppose utility samples for groups $1$ are $\textsc{Bernoulli}(1)$ distributed (i.e., constant).
Now, we construct two scenarios under which the Nash social welfare differs by more than $2\varepsilon$, but the
\fi
Suppose utility samples for groups $1$ and $2$ are $\textsc{Bernoulli}(1)$ and $\textsc{Bernoulli}(p)$ distributed, respectively, for some $p \in [0, 1]$.
Clearly $\Welfare_{0}(\lv; \wv) = \sqrt{p}$, and given a size $m$ sample for group $2$, the probability of observing all $0$ values is $(1-p)^{m}$.
%\delta \doteq 
%, which implies $\delta = (1-p)^{m}$.
As this always occurs for $p = 0$, in this case, we must predict $\Welfare_{0}(\lv; \wv) \leq \varepsilon$.
However, if $\sqrt{p} > 2\varepsilon \Leftrightarrow p > 4\varepsilon^{2}$, we must predict $\Welfare_{0}(\lv; \wv) > \varepsilon$, which is mutually exclusive with the above.
Now let $\delta$ denote the probability of this event, and note that no mean-estimator can disambiguate the above cases, and thus $\delta$ lower-bounds the failure rate of any mean estimator (or welfare estimator).
We %may thus
now conclude that for any $p > 0$, any ($\varepsilon < \frac{\sqrt{p}}{2}$, $\delta \leq (1-p)^{m}$) approximation of $\Welfare_{0}(\lv; \wv)$ requires a necessary sample of size %$\geq m$, where $m$ is given by
\[
m \geq \frac{\ln(\delta)}{\ln(1-p)} > \frac{\ln(\delta)}{\ln(1-4\varepsilon^{2})} \geq \frac{\ln \frac{1}{\delta}}{4\varepsilon^{2}} %\implies m > \log_{4\varepsilon^{2}}(\delta) \enspace.
\enspace.
\]
\end{example}
\fi

We now find that the situation is infinitely worse when we are allowed to weight the welfare function.
The next example shows that the sample complexity of welfare estimation then becomes unbounded.
\begin{example}[Estimating Weighted Nash Social Welfare]
\label{ex:wnsw-estimation}
Suppose as in \cref{ex:nsw-estimation}.
We now consider the \emph{weighted Nash social welfare}, letting $\wv \doteq (1 - w, w)$, for $w \in (0, 1)$.
We then have
\[
\Welfare_{0}(\lv; \wv) = \exp \left( (1 - w) \ln(1) + w \ln(p) \right) = p^w \enspace.
\]

Again, when we observe all $0$ values, we must predict $\Welfare_{0}(\lv; \wv) \leq \varepsilon$, but now if $p^{w} > 2\varepsilon$, we must predict $\Welfare_{0}(\lv; \wv) > \varepsilon$, which are again mutually exclusive predictions.
Now, for any $\varepsilon < \frac{1}{2}$, we may take $p^{w} > 2\varepsilon$, which implies $w > \frac{\ln(2\varepsilon)}{\ln(p)}$, and $p > (2\varepsilon)^{\frac{1}{w}}$.
Thus for any $p > 0$, we require
\[
m \geq \frac{\ln(\delta)}{\ln(1-p)} > \frac{\ln(\delta)}{\ln(1-(2\varepsilon)^{\frac{1}{w}})} \geq \frac{\ln \frac{1}{\delta}}{(2\varepsilon)^{\frac{1}{w}}} \enspace. %\implies m > \log_{4\varepsilon^{2}}(\delta) \enspace.
\]

As $w$ was a free variable (for any $w \in (0, 1)$, the constraint $w > \frac{\ln(2\varepsilon)}{\ln(p)}$ is satisfied for sufficiently small $p$), we may thus conclude that for fixed $\varepsilon$, $\delta$, there exist problem instances (parameterized by $w, p$) in this class for which the sample complexity of welfare estimation is arbitrarily large.
\draftnote{What if g1 has constant $c$ utility?}
\end{example}

These results should be contrasted with \cref{lemma:stat-est,coro:stat-est}, where we show that estimation of \emph{any fair malfare function} is essentially no harder than estimation of risk values.
Thus despite their apparent similarity, we conclude that welfare and malfare functions are not isomorphic, and furthermore they have substantially different properties, where malfare is generally more amenable to statistical estimation.

\subsection{A Comparison with the Additively Separable Form}
\label{sec:comparisons:as}
%\label{sec:pop-mean:as}

\todo{back ref earlier discussion?  Or mix with new pars?}

\begin{figure}

\tikzset{
    declare function={
            mean3(\va,\vb,\vc) = ((\va + \vb + \vc) / 3);
            pm2(\va,\vb,\p) = ((\va ^ \p + \vb ^ \p) / 2) ^ (1 / \p);
            as3(\va,\vb,\vc,\pp) = ((\va ^ \pp + \vb ^ \pp + \vc ^ \pp) / 3);
            pm3(\va,\vb,\vc,\pp) = as3(\va,\vb,\vc,\pp) ^ (1 / \pp);
            asw3(\va,\vb,\vc,\wa,\wb,\wc,\pp) = (\wa * \va ^ \pp + \wb * \vb ^ \pp + \wc * \vc ^ \pp);
            pmw3(\va,\vb,\vc,\wa,\wb,\wc,\pp) = asw3(\va,\vb,\vc,\wa,\wb,\wc,\pp) ^ (1 / \pp);
            sgn(\x) = (and(\x<0, 1) * -1) + (and(\x>0, 1) * 1);
        },
}
\if 0
\tikzset{
        hatch distance/.store in=\hatchdistance,
        hatch distance=10pt,
        hatch thickness/.store in=\hatchthickness,
        hatch thickness=2pt
    }
\fi

\tikzset{
    dot diameter/.store in=\dot@diameter,
    dot diameter=3pt,
    dot spacing/.store in=\dot@spacing,
    dot spacing=4pt,
    dots/.style={
        line width=\dot@diameter,
        line cap=round,
        dash pattern=on 0pt off \dot@spacing
    }
}
\tikzfading[name=fade out,
inner color=transparent!0,
outer color=transparent!100]

\tikzset{
    %Main line types:
    mainline/.style={thick},
    %Aux line types:
    boundline/.style={densely dashed,line cap=round,thin},%,opacity=0.5},
    popline/.style={thin,opacity=0.8},
    limitline/.style={dots,thick,opacity=0.9},
    %Bound fill
    linefill/.style={opacity=0.4,pattern=vertical lines},
    solidfill/.style={opacity=0.1},
    allfill/.style={solidfill,postaction={linefill}},
    %Discontinuities:
    discfill/.style={only marks,mark=*,mark options={scale=0.8,thin}},
    dischole/.style={discfill,fill=white},    
  }

\newcommand{\pinsize}{\scriptsize}

\centering

\begin{subfigure}[t]{0.49\textwidth}
\hspace{-0.1cm}\begin{tikzpicture}[blend group=multiply]
\begin{axis}[
    width=1.12\textwidth,
    axis equal image,
    xmin=-5,xmax=5,
    ymin=-2,ymax=4,
    domain=-8:8,
    %domain=-4:4,
    samples=30,smooth,thick,no markers,
    xlabel={$p$},
    legend pos={south east},
    xtick distance=1,ytick distance=1,
    grid,
    legend style={font=\small,draw=none,fill=none},
    every axis/.append style={font=\small},
  ]

%\addplot {1};
%\addplot {2};
%\addplot {3};

\newcommand{\ev}{0.25}

\if 0
\foreach \xi in {1, 2, 3} {
\addplot[samples=2,color={red!50!black},popline,forget plot=true] { \xi };
\addplot[samples=2,color={red!50!black},popline,forget plot=true,name path={ub\xi}] { \xi + \ev};
\addplot[samples=2,color={red!50!black},popline,forget plot=true,name path={lb\xi}] { \xi - \ev};
\addplot[color={red!50!black},pattern color={red!50!black},allfill,forget plot=true] fill between[of={lb\xi} and {ub\xi}];
}
\fi

\addplot[color=red,mainline] { pm3(1, 2, 3, x) };

\addplot[color=red,boundline,name path=lb,forget plot=true] { pm3(1-\ev, 2-\ev, 3-\ev, x) };
\addplot[color=red,boundline,name path=ub,forget plot=true] { pm3(1+\ev, 2+\ev, 3+\ev, x) };

%\addplot[color=red,linefill,forget plot=true] fill between[of=lb and ub];
%\addplot[color=red,solidfill,forget plot=true] fill between[of=lb and ub];
\addplot[color=red,pattern color=red,allfill,forget plot=true] fill between[of=lb and ub];

%\addplot { pm3(1, 2.5, 3, x) };
%\addplot[color={red!50!black},limitline] { .25 };
%\addplot[color={red!50!black},limitline,forget plot=true] { .75 };

\addlegendentry{$\Mean_{p}(\lv \pm \epsv; \wv)$};
%\addlegendentry{$\Mean_{p}(1, 2.5, 3)$};
%\addlegendentry{$\Mean_{-\infty}(\lv; \wv)$};
%\addlegendentry{$\Mean_{\infty}(\lv; \wv)$};
%\addlegendentry{$\Mean_{\pm\infty}(\lv; \wv)$};

%Plot AS: split into +/-, with discontinuities
\addplot[domain=-8:-0.00001,color=blue,mainline] { -as3(1, 2, 3, x) };

\begin{scope}[blend mode=normal]

%Transparent BG
\draw[color=blue!10!white,,line width=3pt] (axis cs: +0,{(ln((1 - \ev) * (2 - \ev) * (3 - \ev))/3}) -- (axis cs: +0,{(ln((1 + \ev) * (2 + \ev) * (3 + \ev))/3});
\draw[color=blue,thin] (axis cs: +0,{(ln((1 - \ev) * (2 - \ev) * (3 - \ev))/3}) -- (axis cs: +0,{(ln((1 + \ev) * (2 + \ev) * (3 + \ev))/3});

\addplot[color=blue,dischole,thin,forget plot=true] coordinates {(0,{(ln((1 - \ev) * (2 - \ev) * (3 - \ev))/3})};
\addplot[color=blue,dischole,thin,forget plot=true] coordinates {(0,{(ln((1 + \ev) * (2 + \ev) * (3 + \ev))/3})};

%Transparent
\addplot[color=blue,discfill,opacity=0.25,forget plot=true] coordinates {(0,{(ln((1 - \ev) * (2 - \ev) * (3 - \ev))/3})};
\addplot[color=blue,discfill,opacity=0.25,forget plot=true] coordinates {(0,{(ln((1 + \ev) * (2 + \ev) * (3 + \ev))/3})};

%Solid blue
\addplot[color=blue,discfill,forget plot=true] coordinates {(0,{ln(6)/3})} node[pin={[pin distance=.25cm]0:{\pinsize $\ln \! \sqrt[3]{1 \cdot 2 \cdot 3} = \frac{1}{3} \ln 6$}}]{};

%Limits
\addplot[color=blue,dischole,forget plot=true] coordinates {(0,-1)} node[pin={[pin distance=.25cm]225:{\pinsize -$1$}}]{};
\addplot[color=blue,dischole,forget plot=true] coordinates {(0,1)} node[pin={[pin distance=.25cm]225:{\pinsize $1$}}]{};

\end{scope}

\addplot[domain=-0.00001:5,color=blue,mainline,forget plot=true] { as3(1, 2, 3, x) };
\addlegendentry{$\MeanAs_{p}(\lv \pm \epsv; \wv)$};

\if 0
\addplot[color=blue,discfill,opacity=0.25,forget plot=true] coordinates {(0,{(ln((1 - \ev) * (2 - \ev) * (3 - \ev))/3})};
\addplot[color=blue,discfill,opacity=0.25,forget plot=true] coordinates {(0,{(ln((1 + \ev) * (2 + \ev) * (3 + \ev))/3})};
\fi

\addplot[domain=-8:-0.00001,color=blue,boundline,name path=lbas0,forget plot=true] { -as3(1-\ev, 2-\ev, 3-\ev, x) };
\addplot[domain=-0.00001:5,color=blue,boundline,name path=lbas1,forget plot=true] { as3(1-\ev, 2-\ev, 3-\ev, x) };

\addplot[domain=-8:-0.00001,color=blue,boundline,name path=ubas0,forget plot=true] { -as3(1+\ev, 2+\ev, 3+\ev, x) };
\addplot[domain=-0.00001:5,color=blue,boundline,name path=ubas1,forget plot=true] { as3(1+\ev, 2+\ev, 3+\ev, x) };

\addplot[domain=-8:-0.00001,color=blue,pattern color=blue,allfill,forget plot=true] fill between[of=lbas0 and ubas0];

\addplot[domain=-0.00001:5,color=blue,pattern color=blue,allfill,forget plot=true] fill between[of=lbas1 and ubas1];

%Draw axes in screen (under multiply)
\begin{scope}[blend mode=screen]
\addplot[color=black,thick,forget plot=true] { 0 };
\draw[color=black,thick] (axis cs: +0,-2) -- (axis cs: +0,+4);
\end{scope}

\end{axis}
\end{tikzpicture}
\subcaption{%$\Mean_{p}(\lv; \wv)$ with 
$\lv = (1, 2, 3)$ and $\wv = (\frac{1}{3}, \frac{1}{3}, \frac{1}{3})$, and $\epsv = \frac{1}{4}\bm{1}$.
}
\label{fig:pmeans:123}
\end{subfigure}
\begin{subfigure}[t]{0.49\textwidth}
\hspace{-0.1cm}\begin{tikzpicture}[blend group=multiply]
\begin{axis}[
    width=1.12\textwidth,
    axis equal image,
    xmin=-4,xmax=6,
    ymin=-3,ymax=3,
    %ymin=1,ymax=3,
    %domain=-8:8,
    domain=-4:6,
    %domain=-4:4,
    samples=30,smooth,thick,no markers,
    xlabel={$p$},
    legend pos={south east},
    xtick distance=1,ytick distance=1,
    grid,
    legend style={font=\small,draw=none,fill=none},
    every axis/.append style={font=\small},
    %every pin edge/.style={<-,shorten <=1pt,decorate,decoration={snake,pre length=4pt}}],
    %every pin edge/.style={thick,line cap=round}],
  ]

\newcommand{\ev}{0.2}
\addplot[color=red,mainline] { pmw3(0.25, 1, 1.75, .25, .5, .25, x) };

\addplot[color=red,boundline,name path=lb,forget plot=true] { pmw3(0.25-\ev,1-\ev,1.75-\ev, .25,.5,.25, x) };
\addplot[color=red,boundline,name path=ub,forget plot=true] { pmw3(0.25+\ev,1+\ev,1.75+\ev, .25,.5,.25, x) };

\addplot[color=red,pattern color=red,allfill,forget plot=true] fill between[of=lb and ub];

%\addplot { pm3(1, 2.5, 3, x) };
%\addplot[color={red!50!black},limitline] { 1 };
%\addplot[color={red!50!black},limitline,forget plot=true] { 3 };

\addlegendentry{$\Mean_{p}(\lv \pm \epsv; \wv)$};
%\addlegendentry{$\Mean_{\pm\infty}(\lv; \wv)$};

%Plot AS: split into +/-, with discontinuities
\addplot[domain=-2:-0.00001,color=blue,mainline] { -asw3(0.25, 1, 1.75, .25, .5, .25, x) };
%\addplot[domain=-2:-0.00001,color=blue,mainline,forget plot=true] { -asw3(0.25, 1, 1.75, .33333, .5, .16667, x) };
\addplot[domain=-0.00001:5,color=blue,mainline,forget plot=true] { asw3(0.25, 1, 1.75, .25, .5, .25, x) };
%\addplot[domain=-0.00001:5,color=blue,mainline,forget plot=true] { asw3(0.25, 1, 1.75, .33333, .5, .16667, x) };

\begin{scope}[blend mode=normal]

%Transparent BG
\draw[color=blue!10!white,line width=3pt] (axis cs: +0,{ln(.25 - \ev) * .25 + ln(1 - \ev) * .5 + ln(1.75 - \ev) * .25}) -- (axis cs: +0,{ln(.25 + \ev) * .25 + ln(1 + \ev) * .5 + ln(1.75 + \ev) * .25});
\draw[color=blue,thin] (axis cs: +0,{ln(.25 - \ev) * .25 + ln(1 - \ev) * .5 + ln(1.75 - \ev) * .25}) -- (axis cs: +0,{ln(.25 + \ev) * .25 + ln(1 + \ev) * .5 + ln(1.75 + \ev) * .25});

\addplot[color=blue,dischole,thin,forget plot=true] coordinates {(0,{ln(.25 - \ev) * .25 + ln(1 - \ev) * .5 + ln(1.75 - \ev) * .25})};
\addplot[color=blue,dischole,thin,forget plot=true] coordinates {(0,{ln(.25 + \ev) * .25 + ln(1 + \ev) * .5 + ln(1.75 + \ev) * .25})};

%Transparent
\addplot[color=blue,discfill,opacity=0.25,forget plot=true] coordinates {(0,{ln(.25 - \ev) * .25 + ln(1 - \ev) * .5 + ln(1.75 - \ev) * .25})};
\addplot[color=blue,discfill,opacity=0.25,forget plot=true] coordinates {(0,{ln(.25 + \ev) * .25 + ln(1 + \ev) * .5 + ln(1.75 + \ev) * .25})};

%Solid blue
\addplot[color=blue,discfill,forget plot=true] coordinates {(0,{ln(.25) * .25 + ln(1) * .5 + ln(1.75) * .25})} node[pin={[pin distance=.25cm]0:{\pinsize $\frac{1}{4} \ln \! \frac{1}{4} + \frac{1}{2} \ln \! \frac{1}{2} + \frac{1}{4} \ln \! \frac{7}{4}$}}]{};

%Limits
\addplot[color=blue,dischole,forget plot=true] coordinates {(0,-1)} node[pin={[pin distance=.25cm]315:{\pinsize -$1$}}]{};
\addplot[color=blue,dischole,forget plot=true] coordinates {(0,1)} node[pin={[pin distance=.25cm]135:{\pinsize $1$}}]{};

\end{scope}

\addlegendentry{$\MeanAs_{p}(\lv \pm \epsv; \wv)$};

\if 0
\addplot[color=blue,discfill,opacity=0.25,forget plot=true] coordinates {(0,{ln(.25 - \ev) * .25 + ln(1 - \ev) * .5 + ln(1.75 - \ev) * .25})};
\addplot[color=blue,discfill,opacity=0.25,forget plot=true] coordinates {(0,{ln(.25 + \ev) * .25 + ln(1 + \ev) * .5 + ln(1.75 + \ev) * .25})};
\fi

\addplot[domain=-2:-0.00001,color=blue,boundline,name path=lbas0,forget plot=true] { -asw3(0.25-\ev,1-\ev,1.75-\ev, .25,.5,.25, x) };
\addplot[domain=-0.00001:7,color=blue,boundline,name path=lbas1,forget plot=true] { asw3(0.25-\ev,1-\ev,1.75-\ev, .25,.5,.25, x) };

\addplot[domain=-3:-0.00001,color=blue,boundline,name path=ubas0,forget plot=true] { -asw3(0.25+\ev,1+\ev,1.75+\ev, .25,.5,.25, x) };
\addplot[domain=-0.00001:7,color=blue,boundline,name path=ubas1,forget plot=true] { asw3(0.25+\ev,1+\ev,1.75+\ev, .25,.5,.25, x) };

\iftrue

\addplot[domain=-8:-0.00001,color=blue,pattern color=blue,allfill,forget plot=true] fill between[of=lbas0 and ubas0];

\addplot[domain=-0.00001:5,color=blue,pattern color=blue,allfill,forget plot=true] fill between[of=lbas1 and ubas1];
\fi

%Draw axes in screen (under multiply)
\begin{scope}[blend mode=screen]
\addplot[color=black,thick,forget plot=true] { 0 };
\draw[color=black,thick] (axis cs: +0,-3) -- (axis cs: +0,+3);
\end{scope}

\end{axis}
\end{tikzpicture}
\subcaption{%$\Mean_{p}(\lv; \wv)$ with 
$\lv = (\frac{1}{4}, 1, 1 + \frac{3}{4})$, $\wv = (\frac{1}{4}, \frac{1}{2}, \frac{1}{4})$, and $\epsv = \frac{1}{5}\bm{1}$.
}
\label{fig:pmeans:wide}
\end{subfigure}

\if 0
\begin{tikzpicture}

\draw[draw=red,fill=red,fill opacity=0.5,ultra thick] (0, 0) circle (1cm);
\draw[draw=blue,fill=blue,fill opacity=0.5,ultra thick] (0, 1) circle (1cm);

\begin{scope}[xshift=2cm]

\draw[draw=blue,fill=blue,fill opacity=0.5,ultra thick] (0, 1) circle (1cm);
\draw[draw=red,fill=red,fill opacity=0.5,ultra thick] (0, 0) circle (1cm);

\end{scope}

\begin{scope}[blend mode=multiply,xshift=4cm]

\draw[draw=red,fill=red,fill opacity=0.5,ultra thick] (0, 0) circle (1cm);
\draw[draw=blue,fill=blue,fill opacity=0.5,ultra thick] (0, 1) circle (1cm);

\end{scope}

\begin{scope}[blend mode=multiply,xshift=6cm]

\draw[draw=blue,fill=blue,fill opacity=0.5,ultra thick] (0, 1) circle (1cm);
\draw[draw=red,fill=red,fill opacity=0.5,ultra thick] (0, 0) circle (1cm);

\end{scope}

\begin{scope}[blend mode=screen,xshift=8cm]

\draw[draw=red,fill=red,fill opacity=0.5,ultra thick] (0, 0) circle (1cm);
\draw[draw=blue,fill=blue,fill opacity=0.5,ultra thick] (0, 1) circle (1cm);

\end{scope}

\begin{scope}[blend mode=screen,xshift=10cm]

\draw[draw=blue,fill=blue,fill opacity=0.5,ultra thick] (0, 1) circle (1cm);
\draw[draw=red,fill=red,fill opacity=0.5,ultra thick] (0, 0) circle (1cm);

\end{scope}
\end{tikzpicture}
\fi

\caption{
A comparison of the $p$-power-mean and $p$-CAS aggregator function families, both as a function of $p$, for an unweighted (\ref{fig:pmeans:123}) and weighted (\ref{fig:pmeans:wide}) three-member population.
We plot the aggregates themselves, as well as \emph{upper and lower bounds} on means based on an uncertainty interval $\Mean(\lv \pm \varepsilon; \wv)$, shown as \emph{shaded regions}.
%Note that equality holds for the utilitarian ($p=1$) case.
The power mean is continuous in $p$, but the $p$-CAS is discontinuous at $p=0$, where the value is distinct from both the left and right limits. %with distinct value, left, and right limits.
These discontinuities are plotted in the usual manner, with upper and lower bounds at $p=0$ shaded. %, due to the structure of the function be balanced on the limits equals the limits
\todo{pin equality at 1?}
\draftnote{Could use $\varepsilon = \frac{1}{4}$?}
}
\label{fig:pmean-vs-addsep}
\end{figure}
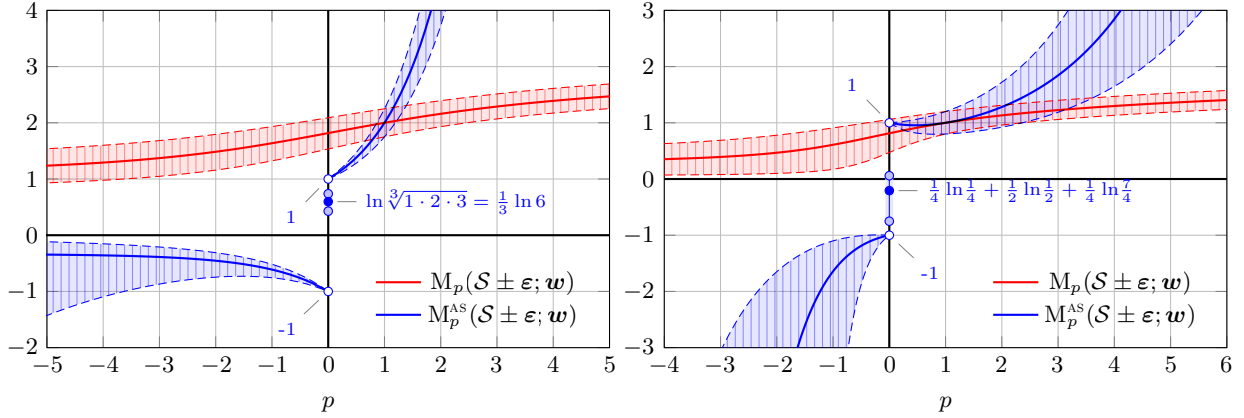

\if 0 %UNWEIGHTED VERSION
For context, we present an additional axiom; that of \emph{additive separability}.
For simplicity, we present it only in the unweighted discrete case, as there is some subtlety to an equivalent measure-theoretic formulation, and we derive no benefit from assuming this axiom, as it is largely incompatible with our assumptions, and presented only for comparison purposes. 
TODO
\todo{continuity}
\todo{Measure theory: \url{https://canvas.harvard.edu/courses/21356/files/3575061/download?verifier=UmCfZb9gSR0o9EPkJND4WQY8ljNfTjX2ss3fZiWi&wrap=1}; cite sth}
\begin{definition}[Additive Separability]
\label{def:axiom:as}
An aggregator function $\Mean(\lv_{1:\NGroups})$ is \emph{additively separable} if there exist functions $f_{1:\NGroups}$ \st{} each $f_{i} \in \RNN \to \RNN$, and $\Mean(\lv_{1:\NGroups})$ may be decomposed as
TODO $\to \R$?
\[
\Mean(\lv_{1}, \lv_{2}, \dots, \lv_{\NGroups}) = \smash{\sum_{i=1}^{\NGroups}} f_{i}(\lv_{i}) \enspace.
\]
\draftnote{
weighted: exists $f: \R^{0+} \times \Population \mapsto \R$ s.t.
\[
\Mean(\lv; \wv) = \int\limits_{\wv} f( \lv(\PopItem); \PopItem ) \, \mathrm{d}(\PopItem) = \Expect_{\PopItem \distributed \wv} \bigl[ f( \lv(\PopItem); \PopItem ) \bigr] \enspace.
\]
TODO: harmonic problem? Output $\R$ not $\R^{0+}$ BC $p < 0$. 
TODO could assume all f are same?
}
\end{definition}
\fi

For context, we present an additional axiom; that of \emph{additive separability}.
We \emph{do not} assume this axiom henceforth; rather we present it for comparison purposes, as it is commonly assumed in welfare %classical
 economics.
\draftnote{TODO
We find that the notion of uniform sample complexity is both futile and foolish for the additively-separable power-mean family, as it is far easier to estimate some additively separable means than others the consistency is the same as the power mean, due to their simple deterministic relationship.
}

%and we derive no benefit from assuming this axiom, as it is largely incompatible with our assumptions, and 

%presented only for comparison purposes. 
\todo{continuity}
\todo{Measure theory: \url{https://canvas.harvard.edu/courses/21356/files/3575061/download?verifier=UmCfZb9gSR0o9EPkJND4WQY8ljNfTjX2ss3fZiWi&wrap=1}; cite sth}
\begin{definition}[Additive Separability]
\label{def:axiom:as}
An aggregator function %$\Mean(\lv; \wv)$
$\Mean(\cdot; \cdot)$ is \emph{additively separable}
%if there exist functions $f_{\PopItem}: \RNN \to \R$ for each $\PopItem \in \Population$ \st{} each $f_{i} \in \RNN \to \R_{0}$, 
if there exists a function $f: \RNN \times \Population \to \RExt$ such that for any $\lv: \Population \to \RNN$ and weights measure $\wv$ over $\Population$, $\Mean(\lv, \wv)$ may be decomposed as
\[
\Mean(\lv; \wv) = \int\limits_{\wv} f( \lv(\PopItem); \PopItem ) \, \mathrm{d}(\PopItem) = \Expect_{\PopItem \distributed \wv} \bigl[ f( \lv(\PopItem); \PopItem ) \bigr] \enspace.
\]
\draftnote{Could assume f is the same across $\Population$?}
\end{definition}

\begin{definition}
Suppose $\lv: \Population \to \RNN$ and weights measure $\wv$.
For any $p \in \R$, we define the $p$-\emph{canonical-additively-separable} ($p$-CAS) aggregator function as
\begin{equation}
\label{eq:paspm}
\MeanAs_{p}(\lv; \wv) \doteq \!\! \lim_{\varepsilon \to 0^{+}} \! \int\limits_{\wv} f_{p}(\lv(\PopItem) + \varepsilon ) \, \mathrm{d}(\PopItem) = \!\!\lim_{\varepsilon \to 0^{+}} \Expect_{\PopItem \distributed \wv}\bigl[f_{p} ( \lv (\PopItem) + \varepsilon ) \bigr]
\enspace,
\ \text{with} \ 
\begin{cases}
\ p = 0 & f_{0}(x) \doteq \ln(x) \\[-0.08cm]
%\ p \neq 0 & f_{p}(x) \doteq \sgn(p)x^{p} \\[-0.1cm]
\ p \neq 0 & f_{p}(x) \doteq \sgn(p)x^{p} \\[-0.1cm]
%p > 0 & f_{p}(x) \doteq x^{p} \\[-0.08cm]
%p = 0 & f_{0}(x) \doteq \ln(x) \\[-0.08cm]
%p < 0 & f_{p}(x) \doteq -x^{p} \\%[-0.1cm]
\end{cases}
\enspace.
\end{equation}
\end{definition}
Here $f_{p}$ is defined as in \cref{thm:pop-mean-prop}~\cref{thm:pop-mean-prop:lin-fact}, and again $\lv: \Population \to \R_{+}$ is extended to $\Population \to \RNN$ via the right limit. %by taking the limit $\lim_{\varepsilon \to 0^{+}} \MeanAs_{p}(\lv + \varepsilon; \wv)$.
These limits are simpler than in the power-mean, and we could equivalently take the limits in $f_{p}$, which % $f_{p}$ by taking
results in $\ln(0^{+}) = -\frac{1}{0^{+}} = -\infty$.
Note that the Debreu-Gorman theorem is often stated in this form; i.e., 
%We note immediately that this form matches that of 
it is \cref{thm:pop-mean-prop}~\cref{thm:pop-mean-prop:lin-fact}, taking $F(x) = x$.\draftnote{$F(u)$ or $F(v)$?}
It is thus closely related to the power-mean, as % for all $p \in \R$, as
\begin{equation}
\MeanAs_{p}(\lv; \wv) \mathop{=}\limits_{p \in \R} f_{p}^{-1} \bigl(\sgn(p)\Mean_{p}(\lv; \wv) \bigr) \mathop{=}\limits_{p \in \R \setminus \{0\}} \Mean_{p}^{p}(\lv; \wv) \enspace.
\end{equation}

If we assume the additive separability axiom, as well as axioms~\ref{def:cardinal-axioms:dgfirst}-\ref{def:cardinal-axioms:dglast}, it then holds that any aggregator function $\Mean(\cdot; \cdot)$ can be expressed as
\[
\Mean(\lv; \wv) = \beta + \alpha \MeanAs_{p}(\lv; \wv) \enspace,
\]
for some $p \in \R$, $\alpha \in \RPos$, $\beta \in \R$, i.e., $\Mean(\cdot; \cdot)$ is a positive affine transform of the CAS family (this identity follows essentially from \cref{thm:pop-mean-prop}~\cref{thm:pop-mean-prop:lin-fact}).
It is a rather subtle matter to axiomatically restrict this family to the CAS family, but the following pair of axioms suffice.
\if 0
but it is quite important to do so, as the arbitrary multiplicative scalar $\alpha$ precludes any \emph{uniform sample complexity} for additively estimating this family.
TODO
since for any Epsilon there exists an alpha for which it is difficult to Alpha estimate.
\fi
%
\if 0 %Details:
The unit-scale axiom is helpful for $p > 0$, but not sufficient, as the additive $\beta$ term gives an extra degree of freedom so we must additionally assume a 0-identity axiom of $\Mean(\PopItem \mapsto 0; \wv) = 0$ (note that for the power-mean, this extra degree of freedom is handled by the multiplicative linearity axiom).
Unfortunately, these axioms only suffice for the $p > 0$ case, and they must be adapted for $p < 0$, as here $\MeanAs_{p}(\PopItem \mapsto 1; \wv) = -1$ and $\MeanAs_{p}(\PopItem \mapsto 0; \wv) = -\infty$.
The contortions continue as we consider the $p=0$ case, where now $\MeanAs_{p}(\PopItem \mapsto 0; \wv) = -\infty$ and $\MeanAs_{p}(\PopItem \mapsto 1; \wv) = 0$.
Indeed, to restrict the additively separable form to its canonical representation, a case analysis indicates the following two axioms are sufficient, however unlike the entirely natural \emph{unit scale} and \emph{multiplicative linearity} axioms, \emph{0-identity} and \emph{1-identity} read as quite arbitrary, and only through a case analysis can it be seen that they just-so-happen to restrict the family appropriately.
\fi
%
\begin{definition}[Canonical Additive Separability Restriction Identity Axioms]
\label{def:axioms:as-cas}
Suppose aggregator function $\Mean(\cdot; \cdot)$ and probability measure $\wv$.
We define the following axioms.
 %Suppose weighted aggregator function $\Mean(\cdot; \wv)$.
%Suppose aggregator function $\Mean(\cdot; \cdot)$, sentiment functions $\lv: \Population \to \RNN$, and probability measure $\wv$. % is a  over %some space
% $\Population$

\begin{multicols}{2}
\begin{enumerate}[wide, labelwidth=0pt, labelindent=0pt]\setlength{\itemsep}{3pt}\setlength{\parskip}{0pt}
\item $0$-Identity: $\Mean(\PopItem \mapsto 0; \wv) \in \{ -\infty, 0 \}$.
\item $1$-Identity: $\Mean(\PopItem \mapsto 1; \wv) \in \{ -1, 0, 1 \}$.
\end{enumerate}
\end{multicols}
\end{definition}
%The two axioms of \cref{as-cas} suffice to perform this restriction, however 
Unlike the entirely natural \emph{unit scale} and \emph{multiplicative linearity} axioms (for the power-mean), the \emph{0-identity} and \emph{1-identity} axioms read as quite arbitrary, and only through a detailed six-way case analysis can it be seen that they just-so-happen to restrict the family appropriately.
Furthermore, %the $1$-identity axiom is
both are incompatible with the \emph{unit scale} axiom (consider $p \leq 0$) and the \emph{identity property} (\cref{thm:pop-mean-prop}~\cref{thm:pop-mean-prop:id}).
Indeed, even the additive-separability axiom itself seems rather heavy-handed, assuming something very specific that is supposedly convenient for the economist, with little justification as to why and how it serves as a fundamental property of \emph{cardinal welfare} itself.

We note that from a classical perspective, there is very little difference between the $p$-power-mean and $p$-CAS families.
They are isomorphic under the comparison operator, thus they defined the same ordering over preferences, and we shall see that it is easy to construct consistent estimators from either from consistent estimators for sentiment values.
However, the same cannot be said for \emph{uniform sample complexity},
therefore our results for FPAC-learnability would be quite different under % which is key to FPAC-learnability, therefore our framework would not be the same under 
this alternative axiomatization.
\Cref{coro:stat-est} describes such a bound for the $p \geq 1$ power-means, and \cref{fig:pmean-vs-addsep} directly contrasts these families, wherein it is clear that the difficulty of estimating the $p$-CAS family varies wildly as a function of $p$.

In closing, we remark that in many ways, the power-mean is more intuitive as a generalization of the mean-concept, and its convenient dimensional-analysis properties, and the potential for direct comparisons between aggregator functions and sentiment values, do not extend to the $p$-CAS family.

\subsection{Relating Power Means and Inequality Indices}
\label{sec:comparisons:inequality}
%\label{sec:pop-mean:inequality}

\todo{Does Atkinson use his index in this context already?  Predates Sen!  TODO: Dalton 1920?}

We now discuss and define %several 
\emph{relative inequality indices} $\IE(\lv; \wv)$, which have been employed in the literature \citep{sen1997economic} to construct \emph{welfare functions} of the form
\[
\Welfare(\lv; \wv) = \Welfare_{1}(\lv; \wv) \bigl(1 - \IE(\lv; \wv)\bigr) \enspace.
\]
This characterization intuitively starts with the \emph{utilitarian welfare}, which measures \emph{overall satisfaction} and then \emph{downweights} based on how unfairly distributed utility is amongst the population.
The ``relative'' in relative inequality indices connotes the fact that they are restricted to domain $[0, 1]$, thus the welfare metric matches the utilitarian under perfect equality, and is $0$ under maximal inequality. %, (although this restriction may hold only for some subset of the permissible parameter space, i.e., $p \leq 1$).\todo{Why we care.}
\todo{Check cite of Sen 1973. \citep{sen1997economic}, not \citep{sen1977weights}?}

We show that a large class of such functions are actually power means, which both gives them axiomatic justification, and shows prior support in the literature for the power mean.
In particular, we first consider the \emph{Atkinson index} (\citeyear{atkinson1970measurement}) relative inequality measure family. %\citep{atkinson1970measurement}.

\begin{definition}[Atkinson Index]
\label{def:atk}
%Take $\varepsilon = 1 - p$ (thus $p = 1 -\varepsilon$).
For all $\varepsilon \in \R$, we define the Atkinson index as
\[
\ATK_{\varepsilon}(\lv; \wv) \doteq 1 - \frac{\Mean_{1-\varepsilon}(\lv; \wv)}{\Mean_{1}(\lv; \wv)}
  \enspace.
\]
\end{definition}

Note that often the Atkinson index is restricted to $\varepsilon \in [0, 1]$; outside this range, it may exceed $1$\cyrus{What?}.
Furthermore, the Atkinson index is generally stated without weights, and in a mathematically equivalent form, in which the resemblance to the power mean is less obvious, but for our purposes the above form is clearer.
From it, we immediately have the following lemma.

\begin{lemma}[Relating Atkinson Indices and Power Means]
\label{lemma:atk}
Suppose some $\varepsilon \in \R$, and take $p = 1 - \varepsilon$.
It then holds that
\[
\Mean_{p}(\lv; \wv) = \Mean_{1}(\lv; \wv)\bigl(1-\ATK_{\varepsilon}(\lv; \wv)\bigr) \enspace.
\]
\end{lemma}
\begin{proof}
This is a direct consequence of \cref{def:atk}, noting $p = 1 - \varepsilon \Leftrightarrow \varepsilon = 1 - p$.
\end{proof}
This result is not particularly surprising in light of the welfare-centric derivation of \citet{atkinson1970measurement}, but nonetheless it yields a valuable alternative way to think about power means and inequality-weighted welfare functions.
In particular, it gives a direct \emph{axiomatic justification} of the welfare function $\Welfare(\lv; \wv) = \Welfare_{1}(\lv; \wv) (1 - \ATK_{\varepsilon}(\lv; \wv))$ (see \cref{thm:pop-mean-prop}), and also gives an alternative intuitive interpretation of power-mean welfare (as inequality-weighted utilitarian welfare).

Furthermore, \cref{lemma:atk} %this relationship
casts light on the relationship between \emph{inequality-index constrained} ($\leq c$) fair learning methods and power-mean welfare (or malfare) optimization.
In particular, if $\lv$ is a function of some parameter $\theta \in \Theta$, assuming \emph{strong duality} holds,\todo{Does it always?} the \emph{Lagrangian dual} yields %, we get
\[
\sup_{\theta \in \Theta: \ATK_{\varepsilon}(\lv(\theta); \wv) \geq c}  \Welfare_{1}(\lv(\theta); \wv) = \inf_{\lambda \geq 0} \sup_{\theta \in \Theta} \Welfare_{1}(\lv(\theta); \wv) - \lambda \bigl( \ATK_{\varepsilon}(\lv(\theta); \wv) - c \bigr) \enspace.
\]
Now, consider that for the power-mean, by \cref{lemma:atk}, we have
\[
\sup_{\theta \in \Theta} \Welfare_{p}(\lv(\theta); \wv) = \sup_{\theta \in \Theta} \Welfare_{1}(\lv(\theta); \wv) - \Welfare_{1}(\lv(\theta); \wv) \ATK_{\varepsilon}(\lv(\theta); \wv) \enspace.
\]
The similarity between these forms is immediately clear,
though they may %might
 make %slightly
 different trade-offs between equality (%as represented by the
Atkinson index) and total utility (utilitarian welfare). However, note that given a sufficiently rich parameter space $\Theta$, %it is always the case that there is some value of
there always exists some $c$ such that the infimum and supremum of the Lagrangian dual are realized by some ($\lambda$, $\theta$), % such that the supremum over $\theta$ % infimum of the lagrangian duel over Lambda is realized by 
such that $\theta$ also realizes the supremum of the power-mean.
In this sense, we may think of maximizing the power mean welfare as maximizing inequality-constrained welfare, while automatically selecting an appropriate value of the constraint $c$ to balance %trade-off between
 utility and equality.
\draftnote{Is the Lagrangian dual content clear (and correct)?}

Another advantage of direct welfare or malfare optimization over inequality-constrained optimization is that it can be quite difficult to accurately estimate inequality indices from a finite sample, e.g., \citet{rongve1997estimation} show asymptotic normality and variance analysis, but not finite-sample guarantees.
%, whereas, \cref{lemma:stat-est} can be used directly to obtain finite sample guarantees on the appropriate power means, which then imply a confidence interval for the inequality index.
However, it's worth noting that even \emph{uniform sample complexity bounds} on both the objective and the constraints do not imply uniform sample complexity for \emph{constrained maximization}.
% any quality indices constrained maximization is in general
This is in general a very difficult problem, and makes the analysis of \emph{Seldonian learners} \citep{thomas2019preventing} quite challenging, as the sample complexity of constrained optimization may be uniformly bounded only for particular choices of constraint and objective.
The following example sharply portrays the issue.

\begin{example}[Unbounded Sample Complexity of Constrained Optimization]
\label{example:co}
Suppose we wish to select between three classifiers with utility vectors ($\lv_{1}$, $\lv_{2}$, $\lv_{3}$), utilitarian welfares ($1$, $2$, $3$), and inequality indices ($0$, $c - \gamma$, $c + \gamma$).
In particular, we wish to select the $\lv_{i}$ to maximize utilitarian welfare under the constraint that some inequality index does not exceed $c$ \emph{over the data distribution}.
Clearly $\lv_{1}$ satisfies the $c$-inequality constraint, but we require a $\gamma$-estimate of the inequality indices to determine whether $\lv_{2}$ and $\lv_{3}$ satisfy the inequality constraints.
Thus for any $\varepsilon < \frac{1}{2}$, the sample complexity of this welfare %or malfare optimization
maximization problem is actually independent of the additive error $\varepsilon$, but depends on $\gamma$, which may be taken arbitrarily close to $0$, yielding unbounded sample complexity.
\end{example}

%\smallskip

Note that similar relationships and impossibility results %properties
 may be shown for %the
 isomorphic inequality measures, including the \emph{Theil indices} (\citeyear{theil1967economics}) %\citep{theil1967economics},
 and \emph{generalized entropy indices} \citep{shorrocks1980class}, although in this context their forms are generally less pleasing.
On the other hand, there exist inequality indices with no relation to the power mean.
For example, many such inequalities measures based on the Lorenz curve, such as the generalized Gini index, can't be expressed as a function of power mean and utilitarian welfare.
Such indices are instead naturally related to other welfare functions, e.g., the generalized Gini social welfare function \citep{weymark1981generalized}.
We don't directly consider such welfare functions (as they necessarily violate one or more axioms of \cref{def:cardinal-axioms}), but many of our results and constructions can be adapted to them with little difficulty.
%When such inequality indices are of interest, ...
\todo{GEI: Hidden for now; do it later.  Could add defs and just state similar for GEI / theil?}

\todo{Highly-relevant read \url{https://www.wgtn.ac.nz/cpf/publications/pdfs/2015/WP11_2014-Interpreting-Inequality-Measures-and-Changes-in-Inequality.pdf}.}
\draftnote{Comparison: Gini section!}

\ifthesis

\cyrus{TODO: CLEAN THIS UP!}

%If 0

\fi %Thesis

\todo{\section{Statistical Estimation and Learning Theory}
\label{sec:est-learn}}

\todo{Add next section}

\todo{Malfare / welfare connection}

%NOTE: THIS DOCUMENT USES A WEAKER PAC CONCEPT (FINITE)

\todo{Restore this:}
\if 0
\section{Malfare in ML}

\todo{completely out of place!}

\todo{Open question (or thm): does FPAC imply egal malfare?  I say yes?  Based on gap between p and $\infty$ malfare?}

\todo{Check cvx concave proofs}
\fi

%\section{Statistical and Computational Efficiency Guarantees for Learning with FPAC-Learning}
\section{Statistical and Computational Learning-Efficiency Guarantees} 
%\section{Statistical and Computational Efficiency Guarantees with FPAC-Learning}
\label{sec:pac}

\if 0
The remainder of this document builds upon and refines the ideas of \cref{sec:est-learn:pac} \cyrus{this was in my proposal, now removed}.
Some definitions are modified.
Whereas \cref{sec:est-learn:pac} was primarily concerned with constructing FPAC learners, the idea here is primarily to relate PAC and FPAC learning, with the understanding that this allows the vast breadth of research of PAC-learning algorithms and necessary / sufficient conditions to be applied to FPAC learning.
In particular, we show a hierarchy of fair-learnability via generic statistical and computational learning theoretic bounds and reductions.
\fi

In this section, we define a formal notion of fair-learnability, termed \emph{fair-PAC (FPAC) learning}, where a loss function and hypothesis class are FPAC-learnable essentially if any distribution can be 
%approximately learned from a polynomially-sized sample (w.h.p.).
learned to \emph{approximate malfare-optimality} from a \emph{finite sample} (\whp). 
We then construct various FPAC learners, % with Whereas \cref{sec:est-learn:pac} was primarily concerned with constructing FPAC learners, the idea here is primarily 
and relate the concept to standard PAC learning~\citep{valiant1984theory}, with the understanding that this allows the vast breadth of research of PAC-learning algorithms, and quite saliently, necessary and sufficient conditions, to be applied to FPAC learning.
In particular, we show a hierarchy of fair-learnability via generic statistical and computational learning theoretic bounds and reductions.

\paragraph{Hypothesis Classes and Sequences}

%\cyrus{Only needed to discuss computational complexity now; to new section?}

We now define \emph{hypothesis class sequences}, %required to discuss nontrivial \emph{computational complexity} in learning.
which allow us to distinguish statistically-easy problems, %of, e.g.,
like learning hyperplanes in finite-dimensional $\R^{\DSeq}$, from statistically-challenging problems, like learning hyperplanes in $\R^{\infty}$. %or learning boolean formulae over, from the.
It is also used to analyze the \emph{computational complexity} of learning algorithms as $\DSeq$ increases.
This definition is adapted from definiton~8.1 of \citet{shalev2014understanding}, which treats only \emph{binary classification}.

\begin{definition}[Hypothesis Class Sequence]
A \emph{hypothesis class} is a family of functions mapping domain $\X$ to codomain $\Y$, and a \emph{hypothesis class sequence} %$\HC = \langle \HC_{1}, \HC_{2}, \dots \rangle$ 
$\HC = \HC_{1}, \HC_{2}, \dots$ is a \emph{concentric} (nondecreasing) sequence of \emph{hypothesis classes}, each mapping $\X \to \Y$.
In other words, $\HC_{1} \subseteq \HC_{2} \subseteq \dots$.% \subseteq (\X \to \Y)$.
\draftnote{Should we have domains $\X_{1}, \X_{2}, \dots$?}
%\todo{Good name: mata-class is confusing, class sequence is confusing.}
\end{definition}

Usually, %$(i \leq j) \implies (\HC_{\DSeq} \subseteq \HC_{j})$, and 
each $\HC_{\DSeq}$ is easily derived from $\HC_{\DSeq-1}$.\todo{Make this part of def}
For instance, \emph{linear classifiers} naturally form a sequence of families using their \emph{dimension}:
\[
\HC_{\DSeq} \doteq \left\{ \vec{x} \mapsto \sgn\bigl(\vec{x} \cdot (\vec{w} \circ \vec{0}) \bigr) \, \middle | \, \vec{w} \in \R^{\DSeq} \right\} \enspace.
\]
Here each $\HC_{\DSeq}$ is defined over domain $\X = \R^{\infty}$, but it is often more natural to discuss each $\HC_{\DSeq}$ as a family over $\X_{\DSeq} = \R^{\DSeq}$.
In such cases, $\X = \lim_{n \to \infty} \X_{\DSeq}$, where the set-theoretic limit always exists (this essentially follows from nondecreasing monotonicity of the sequence $\HC$).
Similarly, unit-scale \emph{univariate polynomial regression} naturally decomposes as
\[
\HC_{\DSeq} \doteq \left\{ x \mapsto (x, x^{2}, \dots, x^{\DSeq}) \cdot \vec{w} \, \middle | \, \vec{w} \in [-1, 1]^{\DSeq} \right\} \enspace.
\]

For context, we first present a generalized notion of PAC-learnability, which we then generalize to FPAC-learnability.
Standard presentations consider only classification under 0-1 loss, but following the generalized learning setting of \citet{vapnik2013nature}, some authors consider generalized notions for other learning problems \citep[see, e.g.,][definition~3.4]{shalev2014understanding} %(see, e.g., definition~3.4 of \citet{shalev2014understanding}).

\todo{Check against 3.4 of UML! Classification is \citet[Definition~3.1]{shalev2014understanding}.  Note: they don't assume polynomial!  Is it the same?}
\begin{definition}[PAC-Learnability]
\label{def:pac}
%An algorithm $\PACAlgo$ is said to be 
Suppose \emph{hypothesis class sequence} $\HC_{1} \subseteq \HC_{2} \subseteq \dots$, % \subseteq \X \to \Y$
all over $\X \to \Y$, and \emph{loss function} $\ell: \Y \times \Y \to \R_{0+}$.\todo{Often target space and codomain don't match: interval estimators, recommender systems, and density estimators.}
We say $\HC$ is \emph{PAC-learnable} \wrt\ $\ell$ if %$\exists$
there exists a (randomized) algorithm $\PACAlgo$, such that %$\forall$:
for all\todo{explain sampling oracle?}\draftnote{plural forall, here and in FPAC?}
\begin{enumerate}[wide, labelwidth=0pt, labelindent=0pt]\setlength{\itemsep}{3pt}\setlength{\parskip}{0pt}
\item sequence indices $\DSeq$;
\item instance distributions $\ProbDist$ over $\X \times \Y$;
\item additive approximation errors $\varepsilon > 0$; and
\item failure probabilities $\delta \in (0, 1)$;
\end{enumerate}
it holds that $\PACAlgo$ can identify a hypothesis $\hat{h} \in \HC$, i.e., $\hat{h} \gets \PACAlgo(\ProbDist, \varepsilon, \delta, \DSeq)$, such that\todo{Call it $\PACAlgo_{\HC}$?}
\begin{enumerate}[wide, labelwidth=0pt, labelindent=0pt]\setlength{\itemsep}{3pt}\setlength{\parskip}{0pt}
\item there exists some %$\SampleComplexity(\frac{1}{\varepsilon}, \frac{1}{\delta}, \DSeq): (\R_{+} \times \R_{+} \times \N) \to \N$
\emph{sample complexity} function $\SampleComplexity(\varepsilon, \delta, \DSeq): \bigl( \R_{+} \times (0, 1) \times \N \bigr) \to \N$ s.t.\ $\PACAlgo(\ProbDist, \varepsilon, \delta, \DSeq)$ consumes no more than $\SampleComplexity(\varepsilon, \delta, \DSeq)$ samples from $\ProbDist$ (i.e., has finite sample complexity); and
%\emph{polynomial} (in $\NGroups$, $\frac{1}{\varepsilon}$, and $\frac{1}{\delta}$) expected time complexity.%, space, entropy, and sample complexity (i.e., requires polynomially bounded computational, memory, randomness, and training instances);
\item with probability at least $1 - \delta$ (over randomness of $\PACAlgo$), $\hat{h}$ obeys
\[
%\Expect_{(x, y) \distributed \ProbDist} \bigl[ \loss(y, \hat{h}(x)) \bigr] \leq \inf_{h^{*} \in \HC} \Expect_{(x, y) \distributed \ProbDist} \bigl[\loss(y, h^{*}(x)) \bigr] + \varepsilon \enspace.
\Risk(\smash{\hat{h}}; \LossFunction, \ProbDist) \leq \inf_{h^{*} \in \HC} \Risk(h^{*}; \LossFunction, \ProbDist) + \varepsilon \enspace.
\]
\todo{$\hat{h}$ is \emph{measurable}?}\todo{Check remark 3.1 of UML}
\end{enumerate}
The class of such learning problems is denoted $\PAC$, thus we write $(\HC, \LossFunction) \in \PAC$ to denote PAC-learnability.

\todo{time complexity commented}
\if 0
\medskip

Additionally, $\HC$ is \emph{polynomial-time-PAC-learnable} if %$\PACAlgo$
\begin{enumerate}[wide, labelwidth=0pt, labelindent=0pt]\setlength{\itemsep}{3pt}\setlength{\parskip}{0pt}
\setcounter{enumi}{3}
\item $\PACAlgo$ runs in $\Poly(\frac{1}{\varepsilon}, \frac{1}{\delta}, \DSeq)$ expected time complexity;

NOTE: thus SC is also poly.

\item the output of $\PACAlgo$ is sufficient to evaluate $\hat{h}(\cdot)$ in $\Poly(\frac{1}{\varepsilon}, \frac{1}{\delta}, \DSeq)$ expected time complexity.
\end{enumerate}
The subclass of polynomial-time learning problems is denoted $\PAC_{\Poly}$, thus we write $(\HC, \LossFunction) \in \PAC_{\Poly}$ to denote poly-time PAC-learnability.

\medskip
\fi

%Finally
Furthermore, if for all $d$, the space of $\ProbDist$ is restricted such that
\[
\exists \, h \in \HC_{\DSeq} \text{\ \ s.t.\ } \Risk(h; \LossFunction, \ProbDist) = 0 \enspace,
\]
then $(\HC, \LossFunction)$ is \emph{realizable-PAC-learnable}, written $(\HC, \LossFunction) \in \PAC^{\Realizable}$. %, otherwise (in general) $(\HC, \LossFunction)$ is \emph{agnostic PAC learnable}, written $(\HC, \LossFunction) \in \PAC^{\Agnostic}$.
\end{definition}

\todo{Add this observation in better form:}
\if 0
\begin{observation}
It is clear from the definition that

\[
\PAC(\LossFunction) \subseteq \PAC^{\Realizable}(\LossFunction)
\]

\[
\PAC_{\Poly}(\LossFunction) \subseteq \PAC_{\Poly}^{\Realizable}(\LossFunction)
\]

\[
\PAC_{\Poly}(\LossFunction) \subseteq \PAC(\LossFunction)
\]

\[
\PAC_{\Poly}^{\Realizable}(\LossFunction) \subseteq \PAC^{\Realizable}(\LossFunction)
\]
\end{observation}
\fi

\todo{
TODO: NOTE ON WHY SEQS; GENERALIZE TO DAG; subjectivity of seq.}

\todo{
note taking all $\HC_{i} = \dots $, or even $\lim_{i \to \infty} \HC_{i} \setminus \HC_{i-1} = \emptyset$, then ignore sequence.
}

\begin{observation}[On Realizable Learning]
Our definition of realizability appears to differ from the standard form, in which $\ProbDist$ is a distribution over only $\X$, and $y$ is simply computed as $h^{*}(x)$, for some $h^{*} \in \HC$.
We instead constrain $\ProbDist$ such that there exists a 0-risk $h^{*} \in \HC$, which is equivalent for any loss function $\LossFunction$ such that $\LossFunction(y, \hat{y}) = 0 \Leftrightarrow y = \hat{y}$, e.g., the 0-1 classification loss, or the absolute or square error regression losses.
With our definition, it is much clearer that realizable learning is a special case of agnostic learning,
%is that the connection to agnostic learning is much clearer and furthermore the generalization to fair learning is also more straightforward. And maybe not really that last part.
and furthermore, we handle a much broader class of problems,
%for which a unique ground truth may not exist, e.g., \emph{recommender systems} in which multiple recommendations are equally satisfactory.
for which there may be some amount of \emph{noise}, or wherein a ground truth may not even exist.

For example, in a \emph{recommender system}, $y$ may represent the \emph{set of items} that $x$ will like, and $h$ may predict a singleton set, and thus we take $\LossFunction(y, \{\hat{y}\}) = \1_{y}(\hat{y})$.
There is no ground-truth here, but rather we seek a \emph{compatible solution} that recommends appropriate items to everyone.
Similarly, in %high-dimensional
\emph{multiclass classification}, %we often treat
often the classifier output $\hat{y}$ is a \emph{ranked list} of predictions, and the \emph{top-$k$ loss} is taken to be $\LossFunction(y, \hat{y}) = \1_{\hat{y}_{1:k}}(y)$.
%If $\Domain$ is sufficiently simple
%Again, no ground truth exists, 
We don't necessarily have $h^{*} \in \HC$, but 0-risk learning is still possible if there is not ``too much'' ambiguity (e.g., foxes and dogs can be confused, as long as they are %separated from
ranked above horses and zebras).
%
 %In particular a top K classifier predicts an ordered list of class estimates in order of likelihood and incurs 0 loss if the true class is within the top K predicted classes. For ambiguous examples we may remain the allies obal so long as no example is to ambiguous for instance a dog and a Fox might be ambiguous as long as they can clearly be distinguished from a horse.
%
Finally, the task of an \emph{interval estimator} is to predict an \emph{interval} $\hat{y}$ for every $x$ in which $y$ must lie, thus again $\LossFunction(y, \hat{y}) = \1_{\hat{y}}(y)$.
Under bounded noise conditions, the \emph{interval estimation} problem can easily be realizable, even if it is impossible to exactly recover the \emph{ground truth} from noisy labels.
\end{observation}

\subsection{Fair Probably Approximately Correct Learning} 
\label{sec:pac:fpac}

\todo{Why we need d seq: must assume X is finite (discretized), eventually dist starts to repeat, then it becomes weighted learning problem?}

We now generalize %this concept 
PAC-learnability to \emph{fair-PAC (FPAC) learnability}. % (FPAC-learnability).
In particular, we replace the \emph{univariate risk-minimization} task with a \emph{multivariate malfare-minimization} task.
Following the theory of \cref{sec:pop-mean:prop}, we do not commit to any particular objective, but instead require that a FPAC-learner is able to minimize \emph{any} fair malfare function satisfying the standard axioms.
As we move from a univariate task to a multivariate (over $\NGroups$ groups) task, problem instances grow not just in problem complexity $\DSeq$, but also in the number of groups $\NGroups$, as it stands to reason that both sample complexity and computational complexity may increase with additional groups.
%Furthermore, here problem instances grow not just in problem complexity $\DSeq$, but also in the number of groups $\NGroups$.
%We allow sampling from the \emph{joint distribution} of \emph{per-group} instances, thus sample complexity requirements remains unchanged, but computational complexity may now depend $\NGroups$.

\begin{definition}[FPAC-Learnability]
\label{def:fair-pac}
%An algorithm $\PACAlgo$ is said to be 
Suppose \emph{hypothesis class sequence} $\HC_{1} \subseteq \HC_{2} \subseteq \dots \subseteq \X \to \Y$, and \emph{loss function} $\ell: \Y \times \Y \to \R_{0+}$.
We say $\HC$ is \emph{fair PAC-learnable} \wrt\ $\ell$ if %$\exists$
there exists a (randomized) algorithm $\PACAlgo$, such that %$\forall$:
for all
\todo{poly computable}
\begin{enumerate}[wide, labelwidth=0pt, labelindent=0pt]\setlength{\itemsep}{3pt}\setlength{\parskip}{0pt}
\item sequence indices $\DSeq$;
\item group counts $\NGroups$;
\item per-group instance distributions $\ProbDist_{1:\NGroups}$ over $(\X \times \Y)^{\NGroups}$;
\item group weights measures $\wv$ over group indices $\{1, \dots, \NGroups\}$;
\item malfare concepts $\Malfare(\cdot; \cdot)$ satisfying axioms~\ref{def:cardinal-axioms:mono}-\ref{def:cardinal-axioms:unit} and \ref{def:cardinal-axioms:apd};
\if 0
\draftnote{%I.e., 
Equivalently, $\Malfare(\dots)$ is %necessarily 
some convex power mean $\Malfare_{p}(\dots)$, thus $p \geq 1$.}
\fi
\item additive approximation errors $\varepsilon > 0$; and
\item failure probabilities $\delta \in (0, 1)$;\draftnote{Call it ``confidence parameter''.}
\end{enumerate}
it holds that $\PACAlgo$ can identify a hypothesis $\hat{h} \in \HC$, i.e., $\hat{h} \gets \PACAlgo(\ProbDist_{1:\NGroups}, \wv, \Malfare, \varepsilon, \delta, \DSeq%, \NGroups
)$, such that
\begin{enumerate}[wide, labelwidth=0pt, labelindent=0pt]\setlength{\itemsep}{3pt}\setlength{\parskip}{0pt}
\item there exists some %$\SampleComplexity(\frac{1}{\varepsilon}, \frac{1}{\delta}, \DSeq): (\R_{+} \times \R_{+} \times \N) \to \N$
\emph{sample complexity} function $\SampleComplexity(\varepsilon, \delta, \DSeq, \NGroups): \bigl( \R_{+} \times (0, 1) \times \N \times \N \bigr) \to \N$
%such that
s.t.\ $\PACAlgo(\ProbDist_{1:\NGroups}, \wv, \Malfare, \varepsilon, \delta, \DSeq)$ consumes no more than $\SampleComplexity(\varepsilon, \delta, \DSeq, \NGroups)$ samples\todo{from what? any $\ProbDist_{i}$ (total)? }
(%i.e., has 
finite sample complexity); and
\if 0
%\emph{polynomial} (in $\NGroups$, $\frac{1}{\varepsilon}$, and $\frac{1}{\delta}$) expected time complexity.%, space, entropy, and sample complexity (i.e., requires polynomially bounded computational, memory, randomness, and training instances);
\item $\PACAlgo$ has $\Poly(\frac{1}{\varepsilon}, \frac{1}{\delta}, \DSeq)$ expected sample complexity; and\todo{Do I want $\NGroups$ in the poly?  Or not needed?  Prove it either way?} %\emph{polynomial} (in $\frac{1}{\varepsilon}$, and $\frac{1}{\delta}$) expected time complexity.%, space, entropy, and sample complexity (i.e., requires polynomially bounded computational, memory, randomness, and training instances);
\fi
\item with probability at least $1 - \delta$ (over randomness of $\PACAlgo$), $\hat{h}$ obeys
\[
%\Malfare\left(i \mapsto \biggl( \Expect_{(x, y) \distributed \ProbDist_{i}} \bigl[ \loss(y, \hat{h}(x)) \bigr] \biggr); \wv \right) \leq \argmin_{h \in \HC} \Malfare \left(i \mapsto \biggl(\Expect_{(x, y) \distributed \ProbDist_{i}} \bigl[\loss(y, h(x)) \bigr] \biggr); \wv \right) + \varepsilon \enspace.
\Malfare\left(i \mapsto \Risk(\smash{\hat{h}}; \LossFunction, \ProbDist_{i}); \wv \right) \leq \inf_{h^{*} \in \HC} \Malfare \left(i \mapsto \Risk(h^{*}; \LossFunction, \ProbDist_{i}); \wv \right) + \varepsilon \enspace.
\]
\todo{$\hat{h}$ is \emph{measurable}?}
\end{enumerate}

The class of such fair-learning problems is denoted $\FPAC$, thus we write $(\HC, \LossFunction) \in \FPAC$ to denote fair-PAC-learnability.\todo{comment on time.}

\if 0
\medskip

Additionally, $\HC$ is \emph{polynomial-time-FPAC-learnable} if %$\PACAlgo$
\begin{enumerate}[wide, labelwidth=0pt, labelindent=0pt]\setlength{\itemsep}{3pt}\setlength{\parskip}{0pt}
\setcounter{enumi}{3}
\item $\PACAlgo$ runs in $\Poly(\frac{1}{\varepsilon}, \frac{1}{\delta}, \DSeq, \NGroups)$ expected time complexity;
\item the output of $\PACAlgo$ is sufficient to evaluate $\hat{h}(\cdot)$ in $\Poly(\frac{1}{\varepsilon}, \frac{1}{\delta}, \DSeq)$ expected time complexity.\todo{Do I need / want $\NGroups$ as part of the evaluation TC?}
\end{enumerate}

The subclass of polynomial-time fair learning problems is denoted $\FPAC_{\Poly}$, thus we write $(\HC, \LossFunction) \in \FPAC_{\Poly}$ to denote poly-time FPAC-learnability.
\fi

\medskip

Finally, %\emph{realizable PAC learning} 
if for all $d$, the space of $\ProbDist$ is restricted such that
\[
\exists \, h \in \HC_{\DSeq} \text{\ \ s.t.\ } \max_{i \in 1, \dots, \NGroups} \Risk(h; \LossFunction, \ProbDist_{i}) = 0 \enspace,
\]
then $(\HC, \LossFunction)$ is \emph{realizable-FPAC-learnable}, written $(\HC, \LossFunction) \in \FPAC^{\Realizable}$.\todo{Could rewrite as $\Malfare_{\infty}$ condition?}
\end{definition}
%\draftnote{Should these be expected or worst-case complexities?}

We now observe that a few special cases are familiar learning problems, though we argue that all cases are of interest, and simply represent different ideals of fairness, which may be situationally appropriate. %depending on the target fairness-concept.
\begin{observation}[Malfare Functions and Special Cases]
By assumption, $\Malfare(\cdot; \cdot)$ must be $\Malfare_{p}(\cdot; \cdot)$ for some $p \in [1, \infty)$.\todo{use this in proofs, cite thm?}
Taking $\NGroups=1$ implies $\bm{w} = (1)$, and $\Malfare_{p}(\lv; \wv) = \lv_{1}$, thus reducing the problem to standard PAC-learning (risk minimization).\todo{Agnostic / realizable assumption too}
Similarly, taking $p=1$ converts the problem to \emph{weighted risk minimization} (weights determined by $\bm{w}$), and $p=\infty$ yields a \emph{minimax optimization problem}, where the maximum is over groups, as commonly encountered in adversarial and robust learning settings.\todo{cite Mohri, Chebyshev, ALL papers?}
%  $p=\infty$.}
%
\end{observation}

\todo{Price of sharing and constructing malfare optimal learners / classes?}

\paragraph{An Aside: The Flexibility of FPAC-Learning}

Note that the generalized definition of (fair) PAC-learnability is sufficiently broad so as to include many \emph{supervised}, \emph{semi-supervised}, and \emph{unsupervised} learning problems.
While this is not immediately apparent, consider that, for instance, %learning a 
\emph{$k$-means clustering} can be expressed as a \emph{learning problem}, where the task is to identify a set of $k$ cluster centers, each of which are vectors in $\R^{\DSeq}$.
In particular, the hypothesis class is isomorphic to $\R^{k \times \DSeq}$, it operates by mapping a given vector $\vec{x}$ onto the nearest cluster center, and the loss function is the \emph{square distance} to said cluster center.\todo{Convert to computational hardness reduction?} 
This is %actually 
a surprisingly natural fairness issue %; for instance clustering is somewhat related to computing a Steiner tree Which 
when cast as a \emph{resource allocation problem}.
%if, for instance, 
For example, if each cluster center represents a cellphone tower, %it is then a fairness issue 
then we seek to place towers to serve \emph{all groups}, and to avoid serving one or more groups particularly well at the expense of the others. %one group well while neglecting others.
\todo{ALSO MEAN ESTIMATION; PARAMETRIC (AND NONPARAM) DISTRIBUTION ESTIMATION (see uml 13.7 exercise 2).}
\todo{Active and online settings.}

\paragraph{On Computational Efficiency}
Some authors consider not just the \emph{statistical} but also the \emph{computational} performance of learners, generally requiring that $\PACAlgo$ have \emph{polynomial time complexity} (thus implicitly polynomial sample complexity).
In other words, they require that $\PACAlgo(\ProbDist, \varepsilon, \delta, \DSeq)$ terminates in $\SampleComplexity(\varepsilon, \delta, \DSeq) \in \Poly(\frac{1}{\varepsilon}, \frac{1}{\delta}, \DSeq)$ steps.
%The main focus of this manuscript is sample complexity, but for completeness, we note that a 
A similar concept of \emph{polynomial-time FPAC-learnability} is equally interesting, where here we assume $\PACAlgo(\ProbDist_{1:\NGroups}, \wv, \Malfare, \varepsilon, \delta, \DSeq)$ may be computed by a Turing machine (with access to \emph{sampling} and \emph{entropy} oracles) in $\SampleComplexity(\varepsilon, \delta, \DSeq, \NGroups) \in \Poly(\frac{1}{\varepsilon}, \frac{1}{\delta}, \DSeq, \NGroups)$ steps. %, and thus implicitly the same bound on sample complexity.
We denote these concepts $\smash{\PAC_{\Poly}^{\Agnostic}}$, $\smash{\PAC_{\Poly}^{\Realizable}}$, $\smash{\FPAC_{\Poly}^{\Agnostic}}$, and $\smash{\FPAC_{\Poly}^{\Realizable}}$.
\draftnote{Also produces an efficient computable function as output} %entirely reasonable, \todo{expound.}

\todo{Can we show equivalence via grid search without considering computation?  Exponential time cost?}
%\end{observation}

%\paragraph{Perspective on the realizable case}
%TODO

\todo{Restore figure?}
\if 0
\subsection{OLD FIGURE}

TODO: GENERALIZE SPLIT SUBGROUP:

weakly or $\Malfare$-PAC learnable: specific function (e.g., minimize -egal), relax cdtns.

strongly FAIR-PL: poly in $\NGroups$, properties of $\HC$, or properties / classes of loss?

generalize m-estimator

realizable / agnostic.

data / dist dependent.

statistical: glivenko-cantelli / donsker

efficient: time poly.

\begin{figure}
\begin{tikzpicture}[
    scale=3,
    axs/.style={rectangle,draw,thick,rounded corners=0.5ex,fill=blue!10!white},
    props/.style={rectangle,draw,thick,rounded corners=0.5ex,fill=red!10!white},
    arrs/.style={line width=0.5mm, -{Stealth[length=4mm, open]}},
  ]

\node[axs] (a1) at (-1, -1) {PAC-Learnable};
\node[axs] (a1) at (-1, 0) {Realizable};
\node[axs] (a1) at (-1, 1) {Agnostic};

\node[axs] (a1) at (0, -1) {Sample-Efficient};
\node[axs] (a1) at (1, -1) {Computation-Efficient};

\end{tikzpicture}

Each quadrant has subclasses:

$n=1$: full
strong: poly in $\NGroups$ / other params of $\HC$
uniform over $\loss$ family

\end{figure}
\fi

\todo{Reweighting classes}

\todo{Antithm: reweighting.}

\todo{Non-reduction: minimax or discrete optimization? submodularity? Want PAC $\not \implies$ Fair PAC}

\todo{Can weighting just be factored into the loss?  Don't think so.}

%\subsection{Result 0: Trivialities}

\paragraph{Some trivial reductions}

\if 0
\begin{observation}[Agnosticism and Realizability]
\label{obs:ag-real}
$(\HC, \LossFunction) \in \FPAC_{\Poly} \implies(\HC, \LossFunction) \in \FPAC_{\Poly}^{\Realizable}$.

\todo{Explain notation}

\end{observation}
\fi

\if 0
\begin{observation}[Group Count Hierarchy]
$\NGroups$ FPAC implies $\NGroups' \leq \NGroups$ FPAC. 
\end{observation}
\todo{describe / write.}
\fi

\if 0
\begin{observation}
\label{obs:pac2fpac}
\[
(\HC, \LossFunction) \in \FPAC^{y}_{x} \implies(\HC, \LossFunction) \in \PAC^{y}_{x} \enspace,
\]
where $x$ denotes either poly-time or unconstrained, and $y$ denotes either realizable or agnostic.
%(Efficient) FPAC $\implies$ (efficient) PAC.
\end{observation}
\fi
We first observe (immediately from \cref{def:pac,def:fair-pac}) that PAC-learning is a special case of FPAC-learning.
In particular, taking $\NGroups=1$ implies $\Mean_{p}(\lv) = \Mean_{1}(\lv) = \lv_{1}$, thus \emph{malfare-minimization} coincides with \emph{risk minimization}.
%This is because all power means coincide with the identity function in this case thus it reduces to a risk bound.
The more interesting question, which we seek to answer in the remainder of this document, is \emph{when} and \emph{whether} the \emph{converse} holds.
%In particular
Furthermore, when possible, we would like to show practical, sample-and-compute-efficient \emph{constructive reductions}.

\paragraph{Realizability}
We first show that in the \emph{realizable case}, \emph{PAC-learnability} implies \emph{FPAC-learnability}.
In particular, we employ a simple and practical \emph{constructive} polynomial-time reduction.
\todo{More efficient if can change loss function; e.g., loss is avg or worst-case over a per-group sample.  But this may require a different assumption.}
Our reduction simply takes a sufficiently %(polynomially large) 
number of samples from the \emph{uniform mixture distribution} over all $\NGroups$ groups, and PAC-learns on this distribution.
More efficient reductions are possible for particular values of $p$, $\NGroups$, and $\wv$, but our polynomial reduction suffices to show the desideratum.
As the reduction is constructive (and polynomial), %-efficiency-preserving),
this gives us generic algorithms for (polynomial-time) realizable FPAC-learning in terms of algorithms for (polynomial-time) realizable PAC-learning.

\todo{Call it a property?}
\begin{theorem}[Realizable Reductions]
\label{thm:realizable-pac2fpac}
Suppose loss function $\LossFunction$ and hypothesis class $\HC$.
Then
\begin{enumerate}[wide, labelwidth=0pt, labelindent=0pt]\setlength{\itemsep}{3pt}\setlength{\parskip}{0pt}
\item $(\HC, \LossFunction) \in \PAC^{\Realizable} \implies (\HC, \LossFunction) \in \FPAC^{\Realizable}$; and
\item $(\HC, \LossFunction) \in \PAC_{\Poly}^{\Realizable} \implies (\HC, \LossFunction) \in \FPAC_{\Poly}^{\Realizable}$.
\end{enumerate}
In particular, we construct a (polynomial-time) FPAC-learner for $(\HC, \LossFunction)$ by noting that there exists some $\PACAlgo'$ with sample-complexity $\SampleComplexity_{\PACAlgo'}(\varepsilon, \delta, \DSeq)$ and time complexity $\TimeComplexity_{\PACAlgo'}(\varepsilon, \delta, \DSeq)$ % \in \Poly(\frac{1}{\varepsilon}, \frac{1}{\delta}, \DSeq, \NGroups)$
to PAC-learn $(\HC, \LossFunction)$, and taking $%\PACAlgo'(p, \wv, \ProbDist_{1:\NGroups}, \varepsilon, \delta, \DSeq)
\PACAlgo(\ProbDist_{1:\NGroups}, \wv, \Malfare, \varepsilon, \delta, \DSeq) \doteq \PACAlgo'(\mix(\ProbDist_{1:\NGroups}),\frac{\varepsilon}{\NGroups}, \delta, \DSeq)$.\todo{Explain mix (uniform)}
Then $\PACAlgo$ FPAC-learns $(\HC, \LossFunction)$, with sample-complexity $\SampleComplexity_{\PACAlgo}(\varepsilon, \delta, \DSeq, \NGroups) = \SampleComplexity_{\PACAlgo'}(\frac{\varepsilon}{\NGroups}, \delta, \DSeq)$, % \in \Poly((\frac{1}{\varepsilon}, \frac{1}{\delta}, \DSeq, \NGroups)$, 
and time-complexity $\TimeComplexity_{\PACAlgo}(\varepsilon, \delta, \DSeq, \NGroups) = \TimeComplexity_{\PACAlgo'}(\frac{\varepsilon}{\NGroups}, \delta, \DSeq)$. % \in \Poly((\frac{1}{\varepsilon}, \frac{1}{\delta}, \DSeq, \NGroups)$.
\end{theorem}
\begin{proof}
We first show the \emph{correctness} of $\PACAlgo$.
%Suppose $\hat{h}$ returned by $\PACAlgo$'.
Suppose $\hat{h} \gets \PACAlgo'(p, \wv, \ProbDist_{1:\NGroups}, \varepsilon, \delta, \DSeq)$.
Then, with probability at least $1 - \delta$ (by the guarantee of $\PACAlgo$), we have
\begin{align*}
\Malfare_{p}( i \mapsto \Risk(h; \LossFunction, \ProbDist_{i})), \wv)
 &\leq \Malfare_{\infty} \bigl( i \mapsto \Risk(h; \LossFunction, \ProbDist_{i}), i \mapsto \mathsmaller{\frac{1}{\NGroups}}\bigr) & \\
 %&= \Malfare_{\infty} \bigl( i \mapsto \Risk(h; \LossFunction, \ProbDist_{i}), i \mapsto \mathsmaller{\frac{1}{\NGroups}}\bigr) & \\
 &\leq \NGroups \Malfare_{1} \bigl( i \mapsto \Risk(h; \LossFunction, \ProbDist_{i}), i \mapsto \mathsmaller{\frac{1}{\NGroups}}\bigr) & \\
 &= \NGroups \Risk \bigl(h; \LossFunction, \mix(\ProbDist_{1:\NGroups}) \bigr) %& \\
 %&
 \leq \NGroups \smash{\mathsmaller{\frac{\varepsilon}{\NGroups}}} = \varepsilon \enspace. & \\[-0.25cm]
\end{align*}
%\vspace{-0.25cm}
We thus may conclude that $(\HC, \LossFunction)$ is \emph{realizable-PAC-learnable} by $\PACAlgo'$, with \emph{sample complexity} $\SampleComplexity_{\PACAlgo}(\varepsilon, \delta, \DSeq, \NGroups) = \SampleComplexity_{\PACAlgo'}(\frac{\varepsilon}{\NGroups}, \delta, \DSeq)$, which by the nature of $\PACAlgo'$, is finite.
Similarly, if $\PACAlgo$ has polynomial runtime, then so too does $\PACAlgo'$, thus we may also conclude efficiency.\todo{Justifications}
\end{proof}

While mathematically correct, if somewhat trivial, unfortunately, this argument does not extend to the agnostic case, essentially because it is not in general possible to %mutually
 simultaneously satisfy all groups.
Some authors \citep[e.g.,][]{krasanakis2018adaptive,jiang2020identifying} have addressed related fair-learning problems by optimizing the \emph{risk} of a mixture over groups, \emph{iteratively reweighting} the mixture during training.
% however, this seems not to lead to an FPAC learner, the malfare of ERM solutions are highly unstable in the reweighting.
This strategy generalizes our algorithm for the realizable case, wherein we begin with the uniform mixture, and terminate at an $\varepsilon$-$\delta$ optimum before executing a single reweighting.
%This strategy
It is tempting to think it could be adapted to FPAC-learn in the agnostic setting, % can not however result in an FPAC learner, as the 
however the following example shows this is not the case.
%To see this, 
Suppose $\Y \doteq \{a, b\}$, group $A$ always wants $a$, and group $B$ always wants $b$, with symmetric preferences, and we wish to optimize egalitarian malfare.
For any reweighting, the utility-optimal solution is always to produce all $a$ or all $b$, except when $\wv_{1} = \wv_{2} = \frac{1}{2}$, %with equal weights, when 
in which case all solutions are equally good.
%Thus
In this example, for no reweighting do all reweighted-risk solutions even \emph{approximate} the egalitarian-optimal solution (which is evenly split between $a$ and $b$). % even an approximately optimal uniquely results in an even approximately mouth are optimal solution.
We thus conclude that simple constructive reductions using PAC-learners as subroutines are not likely to solve the FPAC-learning problem.%general problem.
\todo{Something might work with Lipschitz loss / covering of reweighting space; future work.}
\draftnote{Clean up example!}

In addition to the argument being inextensible to the agnostic case, we note that, philosophically speaking, realizable FPAC learning is rather uninteresting, essentially because in a world where all parties may be satisfied completely, the obvious solution is to do so (and this solution is in fact an equilibrium).
Thus unfairness and bias issues logically only arise in a world of \emph{conflict} (e.g., in zero-sum settings, or under limited resources constraints, which foster \emph{competition} between groups).
We henceforth focus our efforts on the more interesting agnostic-learning setting. %Thus while the realizable case is straightforward and solvable, we argue that in practice, the agnostic learning setting is far more relevant.

\todo{Observe?}
\if 0
\begin{observation}
per-group realizability is much stronger, and (realizable) PAC doesn't imply it.
\end{observation}
\fi

\todo{Note that the reduction can't really be improved for $p=\infty$; since in realizable, SC $\propto 1/\varepsilon' = \NGroups/\varepsilon$, and obviously need factor $\NGroups$, since if all but one group is trivially satisfiable, only last group matters?}

%TODO: Idea loosely based on uml 13.7 EX 1 (not really): if 0 error on all groups, then w.h.p., roughly $\frac{\delta}{\NGroups}$, per-group.

\todo{If 0-neighborhoods of $\LossFunction'$ contain 0-neighborhoods of $\LossFunction$ (w.r.t., $\hat{y}$), then PAC on $\LossFunction$ implies PAC on $\LossFunction'$ in realizable case. Agnostic form?}

\todo{GRID REDUCTION: fair-PAC-learnable, but not by ERM; is it unstable?}

\section{Characterizing Fair Statistical Learnability with FPAC-Learners}
\label{sec:ftfsl}

We first consider only questions of \emph{statistical learning}.
In other words, we ignore computation for now, and show only that \emph{there exist} FPAC-learning algorithms. % of unbounded runtime.
In particular, we show a generalization of the \emph{fundamental theorem of statistical learning} to fair learning problems.
The aforementioned result relates \emph{uniform convergence} and \emph{PAC-learnability}, and is generally stated for binary classification only.
We define a natural generalization of uniform convergence to arbitrary learning problems within our framework, and then show conditions under which a generalized fundamental theorem of (fair) statistical learning holds.
In particular, we show that, neglecting computational concerns, PAC-learnability and FPAC-learnability are equivalent for learning problems where PAC-learnability implies uniform convergence (e.g., binary classification).
% but for other problems
For problems where this relationship does not hold, it remains an open question whether $(\HC, \loss) \in \PAC \implies (\HC, \loss) \in \FPAC$. %with a particular \emph{no-free-lunch} guarantee.
%TODO NFL

\subsection{A Generalized Concept of Uniform Convergence}
\label{sec:ftfsl:uc}

We now define a \emph{generalized notion} of \emph{uniform convergence}.
In particular, our definition applies to \emph{any bounded loss function},\footnote{Boundedness should not be strictly necessary for learnability even uniform convergence, but vastly simplifies all aspects of the analysis.  In many cases, it can be relaxed to moment-conditions, such as \emph{sub-Gaussian} or \emph{sub-exponential} assumptions.} thus greatly generalizing the standard notion for binary classification \citep[see, e.g.,][]{shalev2014understanding}.\todo{which as noted by \citep{blumer1989learnability}, PAC-learnability and finite VC-dimension \citep{vapnik1968uniform} are essentially equivalent (subject to basic regularity conditions)}
\cyrus{risk definition?  962 of content, commented.  Before def 4.3!}

\begin{definition}[Uniform Convergence]
\label{def:uc}
Suppose $\LossFunction: \Y \times \Y \to [0, \frange] \subseteq \R$ and hypothesis class $\HC \subseteq \X \to \Y$.\todo{Is boundedness needed?}
%Let $\F \doteq \LossFunction \circ \HC$.
We say $(\HC, \LossFunction) \in \UC$ %or $\F \in \UC$ 
if\todo{What about $\HC \in \UC$?}
\[
\lim_{m \to \infty} \sup_{\ProbDist \text{ over } \X \times \Y} \Expect_{\bm{z} \distributed \ProbDist^{m}}\left[ \sup_{h \in \HC} \abs{ \ERisk(h; \LossFunction, \bm{z}) - \Risk(h; \LossFunction, \ProbDist) } \right] = 0 \enspace.
\]

%\todo{Under boundedness: equivalent? to:}

\end{definition}
\todo{Define $\SampleComplexity_{\UC}$ as well.}
%\begin{observation}
We stress that this definition is both uniform over $\LossFunction$ composed with the \emph{hypothesis class} $\HC$ and uniform over \emph{all possible distributions} $\ProbDist$.
The classical definition of \emph{uniform convergence in probability} applies to a singular $\ProbDist$, however it is standard in PAC-learning and VC theory to assume uniformity over $\ProbDist$, so we adopt this latter convention.
%
%Our definition generalizes 
Standard uniform convergence definitions also consider only the convergence of \emph{empirical frequencies} of events to their \emph{true frequencies}, whereas we generalize to consider uniform convergence of the \emph{empirical means} of functions to their \emph{expected values}.
%\end{observation}

In discussing uniform convergence, it is often necessary to consider not the {loss function} or {hypothesis class} \emph{in isolation}, but rather their \emph{composition}, defined as
\[
\forall h \in \HC: \ (\loss \circ h)(x, y) \doteq \LossFunction(y, h(x)) \ \ \ \& \ \ \ \loss \circ \HC \doteq \{ \loss \circ h \, | \, h \in \HC \} \enspace.
\]
It is also helpful to consider the \emph{sample complexity} of $\varepsilon$-$\delta$ {uniform-convergence}\draftnote{Define $\varepsilon$-$\delta$ uniform-convergence}, where we take
\[
\SampleComplexity_{\UC}(\LossFunction \circ \HC, \varepsilon, \delta) \doteq \argmin \left\{ m \ \middle| \ \sup_{\ProbDist \text{ over } \X \times \Y} \Prob\left( \smash{\sup_{h \in \HC}} \abs{ \Expect_{\ProbDist}[\loss \circ h] - \smash{\EExpect_{\bm{z} \distributed \ProbDist^{m}}}[\loss \circ h]} > \varepsilon\right) \leq \delta\right\} \enspace, 
\]
i.e., %to be (an upper-bound on) 
the minimum sufficient sample size to ensure $\varepsilon$-$\delta$ uniform-convergence over the \emph{loss family} $\LossFunction \circ \HC$.\todo{Is this clear?}

\todo{Dead conjecture: conditions for it holding?}
\if 0 
\begin{conjecture}
Suppose $\F$ is uniformly convergent, and there exists some $\ProbDist$ such that $\exists f \in \F$ with $\Var_{\ProbDist}[f] > 0$.\draftnote{This might require a discrete-valued loss function?}
Then
\[
\sup_{\ProbDist \text{ over } \X \times \Y} \Expect_{\bm{z} \distributed \ProbDist^{m}}\left[ \sup_{h \in \HC} \abs{  \ERisk(h; \LossFunction, \bm{z}) - \Risk(h; \LossFunction, \ProbDist) } \right] \in \sup_{\ProbDist \text{ over } \X \times \Y} [\mathsmaller{\frac{1}{2}}, 2] \Rade_{m}\bigl( \LossFunction \circ \HC - \Expect_{\ProbDist}[\LossFunction \circ \HC], \ProbDist \bigr) \enspace,
\]
thus \todo{either 0, constant, slow rate.}
\[
\sup_{\ProbDist \text{ over } \X \times \Y} \Expect_{\bm{z} \distributed \ProbDist^{m}}\left[ \sup_{h \in \HC} \abs{  \ERisk(h; \LossFunction, \bm{z}) - \Risk(h; \LossFunction, \ProbDist) } \right] \in \LandauTheta \frac{1}{\sqrt{m}} \enspace.
\]
\todo{Prove rate can't be anything else?  Need particular rate for poly SC results.}

\todo{Why this form?  Fast learning does not exist (uniformly).  Is that true for all loss functions?}

\draftnote{False: in regression example, it's $\frac{1}{2^{\LandauOmega(m)}}$?}

\end{conjecture}

\draftnote{This should further imply an equality with \emph{uniform covering numbers}, via the majorizing measures argument (assuming boundedness)?}
\fi

\todo{Note that UC quantified by many dimension quantities for classification, regression, etc.}

It is in general true %\todo{is it?} 
%true for binary classification 
%that \emph{polynomial-rate uniform convergence}
that \emph{uniform convergence} implies \emph{PAC-learnability}; this is well-known for binary classification, but we show the generalized result for completeness.
The converse is true for some learning problems, but not for others, which we shall use in the consequent subsection as a powerful tool to characterize when PAC-learnability implies FPAC-learnability.\todo{Not sure, now that poly is assumed?  Or does it work, because we had PAC =/=> UC, which implies PAC =/=> poly UC?}

\subsection{The Fundamental Theorem of (Fair) Statistical Learning}
\label{sec:ftfsl:ftfsl}

The following result, generally termed the \emph{fundamental theorem of statistical learning}, relates \emph{uniform convergence}, \emph{combinatorial dimensions} and \emph{PAC-learnability}.
It is often stated for \emph{binary classification} \citep[theorem~6.2]{shalev2014understanding}, wherein the relevant combinatorial dimension is the \emph{Vapnik-Chervonenkis} dimension, though we state the multi-class variant \citep[theorem~29.3]{shalev2014understanding}, in terms of the \emph{Natarajan} dimension.

%We first state the fundamental theorem of statistical learning (for classification) \citep[theorems 6.2~and~29.3]{shalev2014understanding}.\todo{cite \citep{vapnik1968uniform}, natarajan?}

\begin{theorem}[Fundamental Theorem of Statistical Learning {[Classification]}]
Suppose $\LossFunction$ is the 0-1 loss for $k$-class classification, where $k < \infty$.
%Suppose also $k < \infty$, and the Natarajan dimension $d$ of $\HC$ is also finite.
Then the following are equivalent.
\begin{enumerate}[wide, labelwidth=0pt, labelindent=0pt]\setlength{\itemsep}{3pt}\setlength{\parskip}{0pt}
\item $\forall \DSeq \in \N$: $\HC_{\DSeq}$ has finite Natarajan-dimension (= VC dimension for $k=2$ classes).
%\item $\HC$ has the uniform convergence property.
%\cyrus{Should be: ``$\forall \DSeq \in \N$: $\HC_{\DSeq}$ has the uniform convergence property?''}
\item $\forall \DSeq \in \N$: $(\LossFunction, \HC_{\DSeq})$ has the uniform convergence property.
%$\LossFunction \circ \HC_{\DSeq}$ has the uniform convergence property.
\todo{loss circ present in original statement?}
\item Any ERM rule is a successful agnostic-PAC learner for $\HC$.
\item $\HC$ is agnostic-PAC learnable.
\item Any ERM rule is a successful realizable-PAC learner for $\HC$.
\item $\HC$ is realizable-PAC learnable.
\end{enumerate}
%\draftnote{Should I generalize to arbitrary learning problems (arbitrary $\LossFunction$)?}
\end{theorem}
%\begin{proof}
%\end{proof}

It is somewhat subtle to generalize this result to arbitrary learning problems.
In particular, there are PAC-learnable problems for which uniform convergence \emph{does not hold}.\todo{I'm not sure about this: Similarly there are cases where uniform convergence holds but not any ERM rule is capable of agnostic PAC learning? Infinite classes?}
However, \citet{alon1997scale} %See 6.7 of UML
show similar results for various regression problems, with %the \emph{pseudodimension}, or 
the (scale-sensitive) $\gamma$-\emph{fat-shattering dimension}\todo{cite Pollard} playing the role of the Vapnik-Chervonenkis or Natarajan dimensions in classification.\todo{What assumptions do alon need?}
We now show that essentially the same result holds for \emph{fair statistical learning}, i.e., malfare minimization.

\todo{Discuss NFL}
\if 0
In order to obtain a similar result, we require sufficient conditions for a \emph{no-free lunch theorem} for general learning (statistical estimation) problems.
The following suffices, adapted from the classification version given by \citep[thm.~5.1]{shalev2014understanding}.
\begin{theorem}[No Free Lunch Theorem]

Suppose ...

\todo{Check remark 7.2.}
\todo{write me.}

Corollary:

$(\HC, \LossFunction) \in \PAC \implies (\HC, \LossFunction) \in \UC$

\end{theorem}

\draftnote{Actually just need that $\exists \gamma$ s.t. $\VC( \{ x \mapsto \1_[0,\gamma] \LossFunction(h(x), y) \, | \, h \in \HC \} ) < \infty$?  Might need relationship between $\gamma$, $\varepsilon$ for polynomial? Anyway, that gives us VC infinite $\implies$ not agnostic PAC?  Thus ag PAC $\implies$ VC finite.  But does this imply UC?  No, it does not.  Need similar for all $\gamma$-cover:, $\VC( \{ x \mapsto \1_[i\gamma,(i+1)\gamma] \LossFunction(h(x), y) \, | \, h \in \HC, i \in \N \} ) < \infty$; looking a bit like PDIM or fat-shattering?  Does this do anything?  How does NFL reduction look?}
\fi

\todo{Discuss Natarajan, graph, pseudo?}

\begin{theorem}[Fundamental Theorem of Fair Statistical Learning]
\label{thm:ftfsl}
\if 0

PAC implies Fair PAC

Non-constructive proof.

May need stronger assumption: something like PAC by ERM implies fair PAC by EMM?

Adapt FTSL (classification, natarajan)?  Regression?

Suppose boundedness (needed for UC to be needed?).
\fi
%
Suppose $\LossFunction$ such that $\forall \HC: \, (\HC, \LossFunction) \in \PAC^{\Realizable} \implies (\HC, \LossFunction) \in \UC$.
Then, for any hypothesis class sequence $\HC$, the following are equivalent:
\begin{enumerate}[wide, labelwidth=0pt, labelindent=0pt]\setlength{\itemsep}{3pt}\setlength{\parskip}{0pt}
%\item \label{ftfsl:uc} $\HC$ has the uniform convergence property.\todo{poly?}\todo{Need Lipschitz loss?}
%\cyrus{Should be: ``$\forall d \in \N$: $\HC_{\DSeq}$ has the (generalized) uniform convergence property?''}
\item \label{ftfsl:uc} $\forall d \in \N$: %$\LossFunction \circ \HC_{\DSeq}$
 $(\LossFunction, \HC_{\DSeq})$ has the (generalized) uniform convergence property.
\item \label{ftfsl:ag-emm} Any EMM rule is a successful agnostic-FPAC learner for $(\LossFunction, \HC)$.
\item \label{ftfsl:ag} $(\LossFunction, \HC)$ is agnostic-FPAC learnable.
\item \label{ftfsl:re-emm} Any EMM rule is a successful realizable-FPAC learner for $(\LossFunction, \HC)$.
\item \label{ftfsl:re} $(\LossFunction, \HC)$ is realizable-FPAC learnable.
%\item $\HC$ has finite VC-dimension.
\end{enumerate}
\todo{What changes if we instead assume $\forall \HC: \, (\HC, \LossFunction) \in \PAC^{\Agnostic} \implies (\HC, \LossFunction) \in \UC$?  I think we just lose 3 and 4?  Or can we show that the conditions are identical? Update 1: I think it's the same.  Update 2: I think we lose everything, since we used NFL for agnostic learning.}
\end{theorem}
\begin{proof}
First note that \ref{ftfsl:uc} $\implies$ \ref{ftfsl:ag-emm} is a rather straightforward consequence of the definition of uniform convergence and the %Lipschitz properties
contraction property of fair malfare functions (\cref{thm:pow-mean-prop}~item~\ref{thm:pow-mean-prop:contraction}).
%\draftnote{More detail required.}
In particular, take $m \doteq \SampleComplexity_{\UC}(\LossFunction \circ \HC_{\DSeq}, {\frac{\varepsilon}{2}}, {\frac{\delta}{\NGroups}})$.
By union bound, this implies that with probability at least $1 - \delta$, taking samples $\bm{z}_{1:\NGroups,1:m} \distributed \ProbDist_{1}^{m} \times \dots \times \ProbDist_{\NGroups}^{m}$, we have
\[
\forall i \in \{1, \dots, \NGroups\}: \ \sup_{h \in \HC_{\DSeq}} \abs{\Risk(h; \LossFunction, \ProbDist_{i}) - \ERisk(h; \LossFunction, \bm{z}_{i})} \leq \frac{\varepsilon}{2} \enspace.
\]
Consequently, as $\Malfare(\cdot; \wv)$ is $1$-$\norm{\cdot}_{\infty}$-$\abs{\cdot}$-Lipschitz in risk (see~\cref{lemma:stat-est}), %empirical risks on $\NGroups$ samples of the given sizes from , and by (Lipschitz)
it holds with probability at least $1 - \delta$ that
\[
\forall h \in \HC_{\DSeq}: \  \abs{\Malfare \bigl( i \mapsto \ERisk(h; \LossFunction, \bm{z}_{i}); \wv \bigr) - \Malfare \bigl( i \mapsto \Risk(h; \LossFunction, \ProbDist_{i}); \wv \bigr)} \leq \frac{\varepsilon}{2} \enspace.
\]

Now, for EMM-optimal $\hat{h}$, and malfare-optimal $h^{*}$, we apply this result twice to get
\begin{align*}
\Malfare \bigl( i \mapsto \Risk(\hat{h}; \LossFunction, \ProbDist_{i}); \wv \bigr)
  &\leq \Malfare \bigl( i \mapsto \ERisk(\hat{h}; \LossFunction, \bm{z}_{i}); \wv \bigr) + \mathsmaller{\frac{\varepsilon}{2}} & \\
  &\leq \Malfare \bigl( i \mapsto \ERisk(h^{*}; \LossFunction, \bm{z}_{i}); \wv \bigr) + \mathsmaller{\frac{\varepsilon}{2}} & \\
  &\leq \Malfare \bigl( i \mapsto \Risk(h^{*}; \LossFunction, \ProbDist_{i}); \wv \bigr) + \varepsilon \enspace. & \\
%\hat{h} \leq \hat{M} \hat{h} \leq \dots \sup_{h \in \HC_{\DSeq}} \abs{\Risk(h; \LossFunction, \ProbDist_{i}) - \ERisk(h; \LossFunction, \bm{z}_{i,:})} \leq \varepsilon \enspace.
\end{align*}
Therefore, under uniform convergence, the EMM algorithm agnostic FPAC learns $(\HC, \LossFunction)$ with finite sample complexity $\SampleComplexity_{\PACAlgo}(\varepsilon, \delta, \DSeq, \NGroups) = \NGroups \cdot \SampleComplexity_{\UC}(\LossFunction \circ \HC_{\DSeq}, {\frac{\varepsilon}{2}}, {\frac{\delta}{\NGroups}})$, % \in \Poly(\frac{1}{\varepsilon}, \frac{1}{\delta}, \DSeq, \NGroups)$, 
completing \ref{ftfsl:uc} $\implies$ \ref{ftfsl:ag-emm}.

Now, observe that \ref{ftfsl:ag-emm} $\implies$ \ref{ftfsl:ag} and \ref{ftfsl:re-emm} $\implies$ \ref{ftfsl:re} are almost tautological: the existence of (agnostic / realizable) FPAC learning algorithms imply (agnostic / realizable) FPAC learnability.

Now, %note that
 \ref{ftfsl:ag-emm} $\implies$ \ref{ftfsl:re-emm} and \ref{ftfsl:ag} $\implies$ \ref{ftfsl:re} hold, as realizable learning is a special case of agnostic learning. % (see \cref{obs:ag-real}).

As \ref{ftfsl:uc} implies 2-4, which in turn each imply \ref{ftfsl:re}, it remains only to show that \ref{ftfsl:re} $\implies$ \ref{ftfsl:uc}, 
%The difficult part is to show that \ref{ftfsl:re} $\implies$ \ref{ftfsl:uc}, 
i.e., if $\HC$ is realizable FPAC learnable, then $\HC$ has the uniform convergence property.
In general, the question is rather subtle, but here the assumption ``suppose $\LossFunction$ such that $(\HC, \LossFunction) \in \PAC^{\Realizable} \implies (\HC, \LossFunction) \in \UC$'' does most of the work.
In particular, as PAC-learning is a special case of FPAC-learning, %by \cref{obs:pac2fpac}
we have \todo{that}
\[
(\HC, \LossFunction) \in \FPAC^{\Agnostic} \implies (\HC, \LossFunction) \in \PAC^{\Agnostic} \enspace,
\]
then applying the assumption yields $(\HC, \LossFunction) \in \UC$.
\todo{New:
If $\abs{\X}$ is finite, then ?
\ref{ftfsl:re}: $\HC$ is realizable FPAC learnable implies $\HC$ is realizable PAC learnable.
}
\end{proof}
The reductions and equivalences that compose this result are graphically depicted in \cref{fig:ftfsl}.
\cyrus{Note: should be concept of $\HC$ being uniformly convergent, respected by Lipschitz loss?}

\begin{observation}[The Gap between Uniform Convergence and (Fair) PAC-Learnability]
Note that the assumption ``suppose $\LossFunction$ such that $(\HC, \LossFunction) \in \PAC^{\Realizable} \implies (\HC, \LossFunction) \in \UC$'' does not in general hold.
In many cases of interest, it is known to hold, e.g., finite-class classification under 0-1 loss, and bounded regression under square and absolute loss \citep{alon1997scale}.
%It is known to not hold in some more obscure cases.
\todo{I think this is true?  ``However, there are well-known and very natural cases where it does not hold; for instance, \emph{unregularized} Gaussian-Kernel SVM with \emph{ramp loss} on bounded finite-dimensional sets are not uniformly convergent, \emph{however} may still be PAC-learned via \emph{regularized} learning algorithms.''  Maybe easier to show for $\ell_{2}$ regularized logistic regression with $\X = (\pm 1)^{\infty}$, or easier-still, $\X = \{ \1_{i} : i \in \N \}$.}
In general, verifying this condition is a rather subtle task %, and until more general techniques are developed, 
that must be repeated for each learning problem (loss function).
We fully characterize the relationship between PAC and FPAC learnability when they are equivalent to uniform convergence, but in the remaining cases, while clearly FPAC implies PAC, it remains an open question whether PAC implies FPAC.
%Regardless, this lies beyond the scope of this manuscript, as we have reduced the question of FPAC learnability to one of (generic) PAC-learnability.
\todo{Can we better characterize this dichotomy?  What about density estimation and recommender systems?}
\todo{Furthermore, the gap for PAC-learnability implies a similar gap for FPAC-learnability. \draftnote{Reduction goes here.}}
\end{observation}

\todo{\cyrus{Show gap with ``dominated choices:'' some $h \in \HC$ are not uniformly convergent, but are dominated by uniformly convergent $h$.  $\X = \X_{1} \times [0, 1]$, $\HC = \HC_{1} \times ([0, 1] \to [0, 1])$.  $\LossFunction(h(x_{1}, x_{2}), (y_{1}, y_{2}) = \LossFunction_{1}(h_{1}(x_{1}), y_{1}) + h_{2}(x_{2})y_{2}$. }
\cyrus{Even simpler gap: $\LossFunction(y, \hat{y}) = \hat{y}$ or $y\hat{y}$.  $\HC = \X \to [0, 1]$.  Easy: pick $0$ function, but no UC here!}
\cyrus{More realistic gap: Gaussian-Kernel SVM (with regularization) for logistic / ramp loss.  Can overfit, but regularization based on sample size prevents this?}}

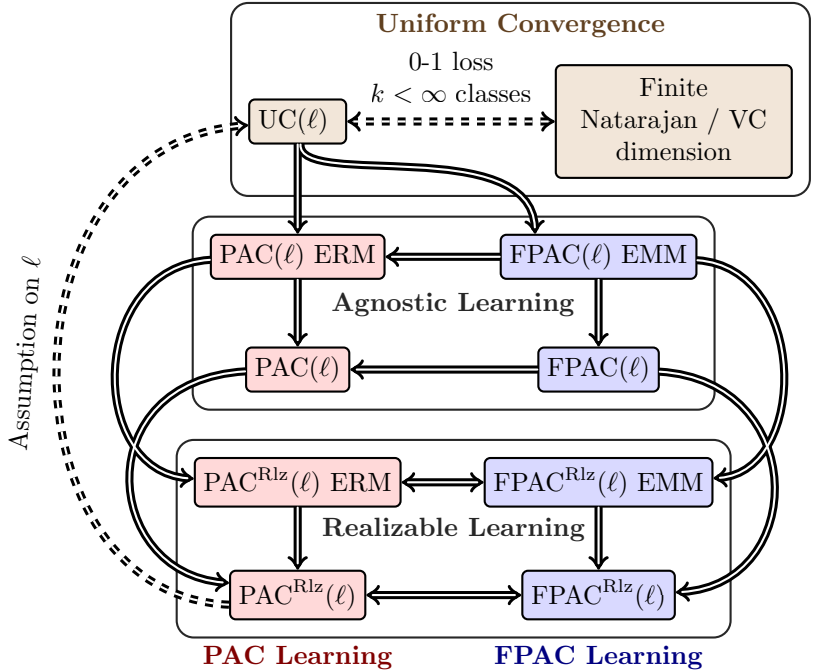
\begin{figure}
%\begin{SCfigure}[0.72]

\colorlet{pcol}{red}
\colorlet{fpcol}{blue}
\colorlet{uccol}{brown}

%\null\hspace{-0.25cm}
\begin{centering}
\begin{tikzpicture}[
    xscale=2.0,yscale=1.5,
    class/.style={rectangle,draw,thick,rounded corners=0.5ex},
    grouping/.style={rectangle,draw,opacity=0.85,thick,rounded corners=1.5ex,inner sep=1.5ex},
    %imparrowhead/.style={Classical TikZ Rightarrow[length=1.8mm,line width=0.4mm]}, %-{Implies[length=1.8mm,width=1.7mm]}
    imp/.style={line width=0.4mm,double equal sign distance,line cap=round,-{Classical TikZ Rightarrow[length=1.8mm,width=2.7mm,line width=0.4mm]}},
    equiv/.style={imp,{Classical TikZ Rightarrow[length=1.8mm,width=2.7mm,line width=0.4mm]}-{Classical TikZ Rightarrow[length=1.8mm,width=2.7mm,line width=0.4mm]}},
    ucclass/.style={class,fill=uccol!20!white},
    pclass/.style={class,fill=pcol!15!white},
    fpclass/.style={class,fill=fpcol!15!white},
  ]

%UC:
\begin{scope}[yshift={0.2cm}]
%\node (ucspace) at (3,3) {};
\node[ucclass] (uc) at
  %(1, 2)
  (0, 2)
  { $\UC(\LossFunction)$ \vspace{1cm} \null};
\node[ucclass] (dim) at
  %(4, 2)
  (2.5, 2)
  { \begin{tabular}{c} Finite \\ Natarajan / VC \\ dimension \\ \end{tabular} };

%TODO: opacity, sth like fill=uccol!5!white, need to draw box below everything.
\node[grouping,fit={(uc) (dim) (3,2.9)},draw=black] (ucgroup) {};
\node[above = -0.6cm of ucgroup] (uctitle) {\bf \color{uccol!50!black} Uniform Convergence};

\end{scope}

%Agnostic:
\node[pclass] (pac) at (0, 0) { $\PAC(\LossFunction)$ };
\node[fpclass] (fpac) at (2, 0) { $\FPAC(\LossFunction)$ };

\node[pclass] (pac-erm) at (0, 1) { $\PAC(\LossFunction)$ ERM };
\node[fpclass] (fpac-emm) at (2, 1) { $\FPAC(\LossFunction)$ EMM };

\node[grouping,fit=(pac) (fpac) (pac-erm) (fpac-emm),draw=black] {\bf Agnostic Learning};

%Realizable:
\node[pclass] (rpac) at (0, -2) { $\PAC^{\Realizable}(\LossFunction)$ };
\node[fpclass] (rfpac) at (2, -2) { $\FPAC^{\Realizable}(\LossFunction)$ };

\node[pclass] (rpac-erm) at (0, -1) { $\PAC^{\Realizable}(\LossFunction)$ ERM };
\node[fpclass] (rfpac-emm) at (2, -1) { $\FPAC^{\Realizable}(\LossFunction)$ EMM };

\node[grouping,fit=(rpac) (rfpac) (rpac-erm) (rfpac-emm),draw=black] {\bf Realizable Learning};

%More Labels:
%\node at (0, -2.55) {\bf \color{pcol!50!black} $\bm{\PAC}$ Learning};
%\node at (2, -2.55) {\bf \color{fpcol!50!black} $\bm{\FPAC}$ Learning};

\node at (0, -2.55) {\bf \color{pcol!50!black} PAC Learning};
\node at (2, -2.55) {\bf \color{fpcol!50!black} FPAC Learning};

%TODO nodes over implications?

%Implication edges (Organized by source):

\begin{scope}[blend mode=screen]

%UC:

\draw[imp] (uc.277) .. controls (0, 1.6) and (1.55, 2.1) .. (fpac-emm.160);
\draw[imp] (uc) -- (pac-erm);

\draw[equiv,dashed,line cap=butt] (dim) -- node[above] { \begin{tabular}{c} 0-1 loss \\ $k < \infty$ classes \\ \end{tabular} } (uc);

%FTSL
%\draw[imp] (uc) -- (rfpac-emm);
%\draw[imp] (uc) -- (rpac-erm);

%Agnostic:

\draw[imp] (fpac) -- (pac);
\draw[imp] (fpac-emm) -- (pac-erm);

\draw[imp] (fpac-emm) -- (fpac);
\draw[imp] (pac-erm) -- (pac);

\draw[imp] (fpac-emm.358) .. controls (3.4, 0.9) and (3.4, -0.9) .. (rfpac-emm.2);
\draw[imp] (pac-erm.182) .. controls (-1.4, 0.9) and (-1.4, -0.9) .. (rpac-erm.178);

\draw[imp] (fpac.358) .. controls (3.4, -0.1) and (3.4, -1.9) .. (rfpac.2);
\draw[imp] (pac.182) .. controls (-1.4, -0.1) and (-1.3, -1.8) .. (rpac.170);

%Realizable:

%\draw[imp] (rfpac) -- (rpac);
%\draw[imp] (rfpac-emm) -- (rpac-erm);
\draw[equiv] (rfpac) -- (rpac);
\draw[equiv] (rfpac-emm) -- (rpac-erm);

%TODO: poly time reverse arrows.

\draw[imp] (rfpac-emm) -- (rfpac);
\draw[imp] (rpac-erm) -- (rpac);

\draw[imp,dashed,line cap=butt] (rpac.188) .. controls  (-2, -2) and (-2, 2) .. node[left of=-0.5cm,xshift=0.6cm,yshift=0.2cm,rotate=86] { Assumption on $\LossFunction$ }
  %{ \begin{tabular}{c} $\LossFunction$ Satisfies \\ ??? \\ \end{tabular} }
  (uc);

\end{scope}

\end{tikzpicture} \\
\end{centering}
%\vspace{-0.5cm}

\caption{Implications between membership in %various
PAC and FPAC classes.
In particular, for arbitrary %but 
fixed $\LossFunction$, implication denotes \emph{implication of membership} of some $\HC$ (i.e., containment); see \cref{thm:realizable-pac2fpac} and \ref{thm:ftfsl}. %\cref{thm:realizable-pac2fpac,,thm:ftfsl} \cref{thm:ftfsl,,thm:realizable-pac2fpac} TODO FIXME.
Dashed implication arrows hold conditionally on $\LossFunction$.
Note that when the assumption on $\LossFunction$ (see \cref{thm:ftfsl}) holds, the hierarchy collapses, and in general, under realizability, some classes are known to coincide.
}
\todo{bounded for UC?}
\todo{Is finite $k$ Natarajan required for UC or something else in FTFSL?}
%Fundamental Theorem of Fair Statistical Learning
%\caption{Fundamental Theorem of Fair Statistical Learning}
\label{fig:ftfsl}
%\vspace{-0.5cm}
\todo{Grouping to classic document.}
\todo{$(\HC, \LossFunction) \in$ notation}
%\end{SCfigure}
\end{figure}

\section{Characterizing Computational Fair-Learnability}
%\section{Characterizing Computational Learnability with Efficient FPAC Learners}
\label{sec:ccl}

\todo{Example: linear classifier / regressor; both coverable efficiently for fixed $d$, but expo in $d$.}

In this section, we consider the more granular question of whether FPAC learning is computationally harder than PAC learning.
In other words, where previously we showed conditions under which $\PAC = \FPAC$, here we focus on the subset of models with polynomial time training efficiency guarantees, i.e., we ask the question, when does $\PAC_{\Poly} = \FPAC_{\Poly}$ hold?
\Cref{thm:realizable-pac2fpac} has already characterized the computational complexity of \emph{realizable} FPAC-learning, so we now focus on the \emph{agnostic} case.
%We first conclude affirmatively in the \emph{realizable case}, though we argue this case is not particularly interesting.  %also far less interesting than the agnostic case.
%The situation is more complicated in the agnostic case.
%In the {agnostic case}, 
Here we show neither a generic reduction or non-constructive proof that $\PAC_{\Poly} = \FPAC_{\Poly}$, nor do we show a counterexample; rather we leave this question for future work.
We do, however, show that under conditions commonly leveraged as sufficient for polynomial-time PAC-learning, so too is polynomial-time FPAC-learning possible.
In particular, \cref{sec:ccl:co} provides an efficient constructive reduction (i.e., an algorithm) for efficient FPAC-learning under standard \emph{convex optimization} settings, and \cref{sec:ccl:uc} shows the same when $\HC$ may be approximated by a small cover, and said cover may be efficiently enumerated.
The computation-theoretic results of this section are summarized graphically in \cref{fig:efpac}.

In both the convex optimization and efficient enumeration settings, the proofs take the same general form: we show that $\varepsilon$-\emph{approximate} EMM on $m$ total samples is computationally efficient (in $\Poly(m, \varepsilon, \DSeq)$ time), and then argue that so long as \emph{sample complexity} $\SampleComplexity_{\UC}(\LossFunction \circ \HC_{\DSeq}, \varepsilon, \delta) \in \Poly({\frac{1}{\varepsilon}}, {\frac{1}{\delta}}, \DSeq)$ of \emph{uniform convergence} is polynomial, i.e., in $\Poly(\frac{1}{\varepsilon}, \frac{1}{\delta}, \DSeq)$,
%of FPAC-learning via EMM is polynomial,
then we may construct an FPAC-learner using $\varepsilon$-approximate EMM with polynomial \emph{time complexity}.
%It is then a straightforward matter of accounting
In particular, the proofs simply account for \emph{optimization} and \emph{sampling} error, and in both cases construct polynomial-time FPAC-learners.
Furthermore, as our training meta-algorithms can be applied to various hypothesis classes, we discuss specific instantiations for well-known machine learning models %that our algorithms are capable of optimizing 
throughout, and these and others are summarized in \cref{table:emm-models}.
\cyrus{Cleaner presentation: Lemma here: $\varepsilon$-EMM in poly time implies POLY FPAC-learnability.
Need to define composition!
Suppose $\LossFunction \circ \HC \in \UC_{\Poly}$, and exists $\varepsilon$-EMM $\PACAlgo'$ in ? time.  Then ...??? 
}

\draftnote{Add this in?}
\if 0
\paragraph{EMM, Bound Efficiency, and Active Learning}
TODO below

Similarly, sharper sample-complexity bounds for particular values of $p$ are possible\todo{Lipschitz norm changes}; again our aim here is merely to characterize FPAC-learnability, which is insensitive to constant-factor and polynomial terms.

In particular the bound leverages ..., instead of \cref{lemma:stat-est} directly, which suggests that a we can use the sharper convergence guarantee why which is only useful for $p < \infty$ and be it would be better to optimize a sort of regularized objective Q although this in itself can be greatly improved

TODO 2g or g? and elsewhere? Future work:
%
In particular, for population, perhaps should minimize
\[
%\hat{h} \doteq \argmin (\ERisk - \dots, \wv 
\hat{h} \doteq \argmin_{h \in \HC} \Malfare\left(i \mapsto 0 \vee \ERisk(h; \LossFunction, \bm{z}_{i}) - 2\ERade_{m}(\loss \circ \HC, \ProbDist_{i}) - 3\frange\mathsmaller{\sqrt{\frac{\ln \frac{\NGroups}{\delta}}{2m}}}; \wv \right) \enspace.
\]
TODO: Variance (Maurer / Bousquet), SRM. loc / centr ERADE.
\fi

\todo{
Here we find that for many learning problems, the conditions for FPAC-learnability are identical to those for PAC-learnability.
We will find that under realizability, the same equivalence holds, however we conjecture that in general, agnostic FPAC-learning and agnostic PAC-learning are not poly-time equivalent.}

%\if 0
\iftrue
%\begin{SCfigure}[0.5]
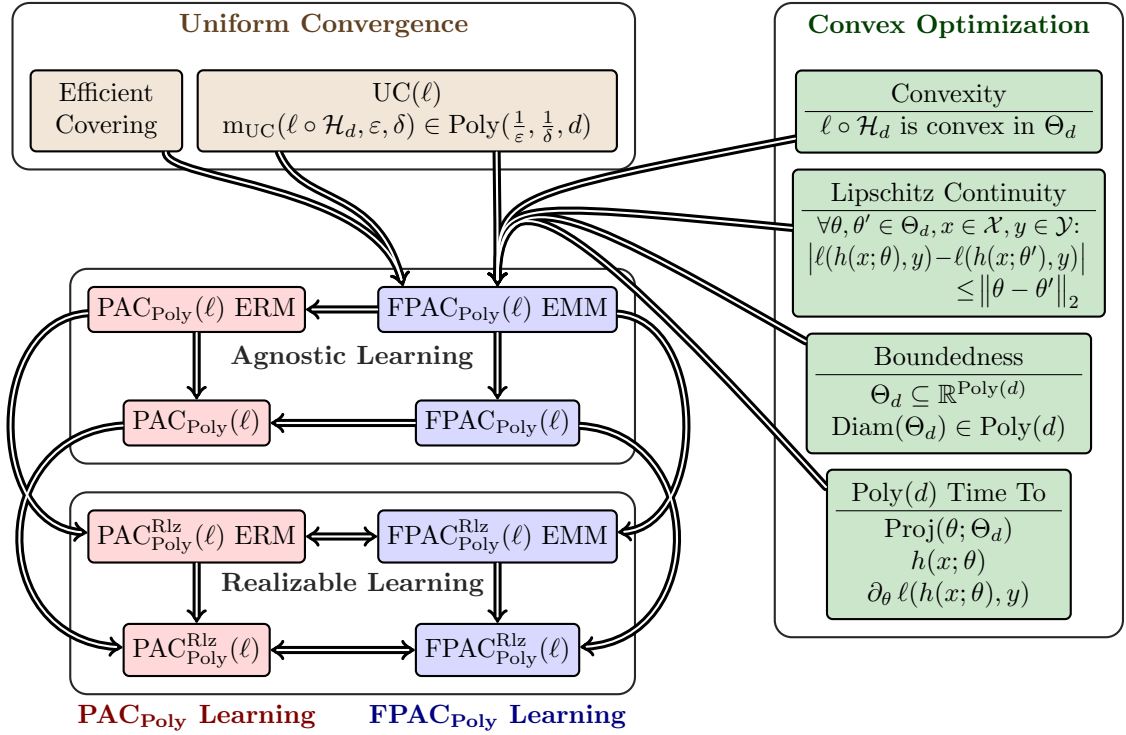
\begin{figure}

\colorlet{uccol}{brown}
\colorlet{cocol}{green!50!black}

\colorlet{pcol}{red}
\colorlet{fpcol}{blue}

\begin{centering}
%\null\hspace{-0.25cm}
\begin{tikzpicture}[
    xscale=2.0,yscale=1.51,
    class/.style={rectangle,draw,thick,rounded corners=0.5ex},
    grouping/.style={rectangle,draw,opacity=0.85,thick,rounded corners=1.5ex,inner sep=1.5ex},
%imparrowhead/.style={Classical TikZ Rightarrow[length=1.8mm,line width=0.4mm]}, %-{Implies[length=1.8mm,width=1.7mm]}
    imp/.style={line width=0.4mm,double equal sign distance,line cap=round,-{Classical TikZ Rightarrow[length=1.8mm,width=2.7mm,line width=0.4mm]}},
    equiv/.style={imp,{Classical TikZ Rightarrow[length=1.8mm,width=2.7mm,line width=0.4mm]}-{Classical TikZ Rightarrow[length=1.8mm,width=2.7mm,line width=0.4mm]}},
    ucclass/.style={class,fill=uccol!20!white},
    coclass/.style={class,fill=cocol!20!white},
    pclass/.style={class,fill=pcol!15!white},
    fpclass/.style={class,fill=fpcol!15!white},
  ]

%Implications first (drawn under other things):

%TODO

\begin{scope}[xshift={-0.6cm},yshift={-0.25cm}]
%UC:
%\node (ucspace) at (3,3) {};
\node[ucclass] (uc) at
  %(1, 2)
  (2, 3)
  { \begin{tabular}{c}
    $\UC(\LossFunction)$ \\
    %$\UC_{\Poly}(\LossFunction)$ \\
    %$\forall \DSeq \in \mathbb{N}$ \\ % \ $ \\ %(\LossFunction, \HC_{\DSeq})$ \\ % \in \UC$ \\ % \N: $ \\
    %$\forall \DSeq: \ (\LossFunction, \HC_{\DSeq}) \in \UC$ \\
    %$\forall \DSeq \in \N:$ \\
    %$\forall \in \N: \ (\LossFunction, ) \in $ \\
    %$\forall \DSeq \in \N$:$\ (\LossFunction, \HC) \in \UC $ \\
    $\SampleComplexity_{\UC}(\LossFunction \circ \HC_{\DSeq}, \varepsilon, \delta) \in \Poly({\frac{1}{\varepsilon}}, {\frac{1}{\delta}}, \DSeq)$ \\
  \end{tabular} };
\if 0
\node[ucclass] (dim) at
  %(4, 2)
  (0, 4)
  { \begin{tabular}{c} Finite \\ Natarajan / VC \\ dimension \\ \end{tabular} };
\fi
\node[ucclass] (efficient-covering) at
  %(4, 2)
  (0, 3)
  { \begin{tabular}{c} Efficient \\ Covering \\ \end{tabular} };

%TODO: opacity, sth like fill=uccol!5!white, need to draw box below everything.
\node[grouping,fit={(uc) %(dim)
  (efficient-covering) (2.5,3.8)},draw=black] (ucgroup) {};
\node[above = -0.6cm of ucgroup] (uctitle) {\bf \color{uccol!50!black} Uniform Convergence};

\end{scope}

%Convexity

\begin{scope}[xshift={1cm},yshift={-0.25cm}]
\if 0
\node[coclass] (co) at
  %(1, 2)
  (4, 3)
  { Convex Optimization }; % $\UC(\LossFunction)$ \vspace{1cm} \null};
\fi

\node[coclass] (co-psg-cvx) at
  %(4, 2)
  (4, 2.994)
  { \begin{tabular}{c}
    %CLB Parameter Space: \\
    %\hline
    Convexity \\
    \hline
    %\small
    $\LossFunction \circ \HC_{\DSeq}$ is convex in $\Theta_{\DSeq}$ \\
    \end{tabular} };
\node[coclass] (co-psg-lpz) at
  %(4, 2)
  (4, 1.835)
  { \begin{tabular}{c}
    %CLB Parameter Space: \\
    %\hline
    Lipschitz Continuity \\
    \hline
    \small
    \if 0
    $\forall x \in \X, y \in \Y$: \\
    %$\forall x \in \X$: 
    $\lambda_{\loss}$-$\norm{\cdot}_{2}$-$\norm{\cdot}_{\Y}$ Lpz.\ $h(x; \cdot)$ \\
    %$\forall y \in \Y$: 
    $\lambda_{\HC}$-$\norm{\cdot}_{\Y}$-$\abs{\cdot}$ Lpz.\ $\LossFunction(y, \cdot)$ \\
    %$\lambda_{\loss}\lambda_{\HC}$-$\norm{\cdot}_{2}$-$\abs{\cdot}$ Lpz.\ $\loss \circ \HC_{\DSeq}$ \\
    %$\forall \theta,\theta' \in \Theta_{\DSeq}, x \in \X, y \in \Y$ TODO $ \loss(h(x; \theta), y) - \loss(h(x; \theta'), y) \leq ?$ \\
    $\Downarrow$ \\
    %$\forall \theta,\theta' \in \Theta_{\DSeq}$: \\ %, x \in \X, y \in \Y$: \\
    \fi
    \scalebox{0.97}[1]{$\forall \theta,\theta' \in \Theta_{\DSeq}, x \in \X, y \in \Y$}: \\
    %$\abs{(\loss \circ h(\cdot; \theta) - \loss \circ h(\cdot; \theta')(z))}$ \\
    \clap{\scalebox{0.89}[1]{$\abs{\loss(h(x; \theta), y) \! - \! \loss(h(x; \theta'), y)}$}} \\
    \hfill\scalebox{0.95}[1]{$\null \leq \norm{\theta - \theta'}_{2}$} \\
    \end{tabular} };
\node[coclass] (co-psg-bd) at
  %(4, 2)
  (4, 0.51)
  { \begin{tabular}{c}
    Boundedness \\
    \hline
    \\[-0.36cm]
    %\small
    $\Theta_{\DSeq} \subseteq \R^{\Poly(\DSeq)}$ \\
    %Bounded 
    $\Diam(\Theta_{\DSeq}) \in \Poly(\DSeq)$ \\
    \end{tabular} };

\if 0
\node[coclass] (co-psg-clb) at
  %(4, 2)
  (4, 2.65)
  { \begin{tabular}{c}
    %CLB Parameter Space: \\
    %\hline
    Convex $\LossFunction \circ \HC_{\DSeq}$ \\
    Lipschitz $\LossFunction$, $\HC_{\DSeq}$ \\
    %Low-Dimensional
    $\Theta_{\DSeq} \subseteq \R^{\Poly(\DSeq)}$ \\
    %Bounded 
    $\Diam(\Theta_{\DSeq}) \in \Poly(\DSeq)$ \\
    \end{tabular} };
\fi
\node[coclass] (co-psg-time) at
  %(4, 2)
  (4, -0.82)
  { \begin{tabular}{c}
    %$\Time(\dots) \in \Poly(\DSeq)$ \\
    $\Poly(\DSeq)$ Time To \\
    \hline
    \\[-0.37cm]
    %\small
    %$\hline
    %Projection
    %$\Time(\Proj(\Theta_{\DSeq})) \in \Poly(\DSeq)$ \\
    %$\Proj(\tilde{\theta}, \Theta_{\DSeq})$ \\
    $\Proj(\theta; \Theta_{\DSeq})$ \\
    %Evaluable 
    %$\Time(h(\cdot; \cdot)) \in \Poly(\DSeq)$ \\
    $h(x; \theta)$ \\
    %$\grad_{\mathrm{sub} \theta} \loss(h(x; \theta), y)$ \\
    $\partial_{\theta} \, \loss(h(x; \theta), y)$ \\
    %Self-Concordancy? \\ 
    %\todo{Projection?}
    \end{tabular} };

%TODO: opacity, sth like fill=uccol!5!white, need to draw box below everything.
\node[grouping,fit={%(co) 
  %(co-psg-clb)
  (co-psg-cvx) (co-psg-lpz) (co-psg-bd)
  (co-psg-time) (4.5,3.8)},draw=black] (cogroup) {};
\node[above = -0.6cm of cogroup] (cotitle) {\bf \color{cocol!50!black} Convex Optimization};

\end{scope}

%Agnostic:
\node[pclass] (pac) at (0, 0) { $\PAC_{\Poly}(\LossFunction)$ };
\node[fpclass] (fpac) at (2, 0) { $\FPAC_{\Poly}(\LossFunction)$ };

\node[pclass] (pac-erm) at (0, 1) { $\PAC_{\Poly}(\LossFunction)$ ERM };
\node[fpclass] (fpac-emm) at (2, 1) { $\FPAC_{\Poly}(\LossFunction)$ EMM };

\node[grouping,fit=(pac) (fpac) (pac-erm) (fpac-emm),draw=black] {\bf Agnostic Learning};

%Realizable:
\node[pclass] (rpac) at (0, -2) { $\PAC_{\Poly}^{\Realizable}(\LossFunction)$ };
\node[fpclass] (rfpac) at (2, -2) { $\FPAC_{\Poly}^{\Realizable}(\LossFunction)$ };

\node[pclass] (rpac-erm) at (0, -1) { $\PAC_{\Poly}^{\Realizable}(\LossFunction)$ ERM };
\node[fpclass] (rfpac-emm) at (2, -1) { $\FPAC_{\Poly}^{\Realizable}(\LossFunction)$ EMM };

\node[grouping,fit=(rpac) (rfpac) (rpac-erm) (rfpac-emm),draw=black] {\bf Realizable Learning};

%More Labels:
%\node at (0, -2.6) {\bf \color{pcol!50!black} $\bm{\PAC}_{\bm{\Poly}}}$ Learning};
%\node at (2, -2.6) {\bf \color{fpcol!50!black} $\bm{\FPAC}_{\bm{\Poly}}}$ Learning};

\node at (0, -2.6) {\bf \color{pcol!50!black} $\bm{\mathrm{PAC}}_{\bm{\mathrm{Poly}}}$ Learning};
\node at (2, -2.6) {\bf \color{fpcol!50!black} $\bm{\mathrm{FPAC}}_{\bm{\mathrm{Poly}}}$ Learning};

%TODO nodes over implications?

%Implication edges (Organized by source):
\begin{scope}[blend mode=screen]

%UC:

%\draw[imp] (uc) -- (fpac-emm);
%\draw[imp] (uc) -- (pac-erm);

\draw[imp] (uc.198) .. controls (0.7, 1.9) and (1.2, 2.1) .. (fpac-emm.166);
%\draw[imp] (uc) to[out=300] controls (0.5, 2.25) and (1.5, 2) .. (fpac-emm);

\draw[imp] (efficient-covering.325) .. controls +(270:0.25) and (1.2, 2.1) .. (fpac-emm.166);

\draw[imp] (uc.335) .. controls +(270:0.1) and (2, 2) .. (fpac-emm);

%Convex Optimization:
%\draw[imp] (co-psg-clb) .. controls (2, 2) .. (fpac-emm);
\draw[imp] (co-psg-cvx.190) .. controls (2, 2) .. (fpac-emm);
\draw[imp] (co-psg-lpz) .. controls (2, 1.97) .. (fpac-emm);
%\draw[imp] (co-psg-bd) .. controls (2, 12) .. (fpac-emm);
\draw[imp] (co-psg-bd.160) .. controls (2.7, 1.73) and (2, 2.3) .. (fpac-emm);
%\draw[imp] (co-psg-time) .. controls (2, 2) .. (fpac-emm);
\draw[imp] (co-psg-time.155) .. controls (2.7, 1.6) and (2, 2.4) .. (fpac-emm);

%\draw[equiv,dashed] (dim) -- node[above] { \begin{tabular}{c} 0-1 loss \\ $k < \infty$ classes \\ \end{tabular} } (uc);

%FTSL
%\draw[imp] (uc) -- (rfpac-emm);
%\draw[imp] (uc) -- (rpac-erm);

%Agnostic:

\draw[imp] (fpac) -- (pac);
\draw[imp] (fpac-emm) -- (pac-erm);

\draw[imp] (fpac-emm) -- (fpac);
\draw[imp] (pac-erm) -- (pac);

%Old lines
\if 0
\draw[imp] (fpac-emm) .. controls (3, 0) .. (rfpac-emm);
\draw[imp] (pac-erm) .. controls (-1, 0) .. (rpac-erm);

\draw[imp] (fpac) .. controls (3, -1) .. (rfpac);
\draw[imp] (pac) .. controls (-1, -1) .. (rpac);
\fi

\draw[imp] (fpac-emm.358) .. controls (3.4, 0.9) and (3.4, -0.9) .. (rfpac-emm.2);
\draw[imp] (pac-erm.182) .. controls (-1.4, 0.9) and (-1.4, -0.9) .. (rpac-erm.178);

\draw[imp] (fpac.358) .. controls (3.4, -0.1) and (3.4, -1.9) .. (rfpac.2);
\draw[imp] (pac.182) .. controls (-1.4, -0.1) and (-1.4, -1.9) .. (rpac.178);

%Realizable:

%\draw[imp] (rfpac) -- (rpac);
%\draw[imp] (rfpac-emm) -- (rpac-erm);
\draw[equiv] (rfpac) -- (rpac);
\draw[equiv] (rfpac-emm) -- (rpac-erm);

%TODO: poly time reverse arrows.

\draw[imp] (rfpac-emm) -- (rfpac);
\draw[imp] (rpac-erm) -- (rpac);

\end{scope}

\if 0 %TODO: not sure about this.
\draw[imp,dashed] (rpac) .. controls  (-2, -2) and (-2, 2) .. node[left of=-0.5cm,xshift=0.6cm,yshift=0.2cm,rotate=86] { Assumption on $\LossFunction$ }
  %{ \begin{tabular}{c} $\LossFunction$ Satisfies \\ ??? \\ \end{tabular} }
  (uc);
\fi

\end{tikzpicture} \\
\end{centering}

\caption{
Implications between membership in various poly-time PAC and FPAC classes.  In particular, for arbitrary but fixed $\LossFunction$, implication denotes \emph{implication of membership} of some $\HC$ (i.e., containment).
See \cref{thm:aemm-covering,thm:aemm-psg}. \todo{see ...} %See \cref{thm:ftfsl,thm:ftfsl}, as well as \cref{thm:realizable-pac2fpac}.
\todo{UC: l circ HC?}
%Dashed implication arrows hold conditionally on $\LossFunction$.
%Note that when the assumption on $\LossFunction$ (see \cref{thm:ftfsl}) holds, the hierarchy collapses, and in general, under realizability, some classes are known to coincide.
}
%\label{fig:pac-fpac-poly}
\label{fig:efpac}
%\todo{Grouping to classic document.}
\todo{$(\HC, \LossFunction) \in$ notation?}
\todo{$\UC(\HC)$ vs $\UC(\LossFunction \circ \HC)$?}
%\end{SCfigure}
\end{figure}
\fi

\iftrue

\begin{table}
\caption{
A Menagerie of Malfare-Minimizing Model-Classes
}
\label{table:emm-models}

\smallskip

{%\small
\centering
\begin{tabular}{lcccc|rrc}
\multicolumn{5}{c|}{Model Class} & \multicolumn{3}{c}{Training Details} \\
Model Name & %Domain
  $\X_{\DSeq}$ & %Model Space
  $\Theta_{\DSeq}$ & %Codomain
  $\Y$ & $\loss$
    & Sample Complexity & Learner\ \  %Training Method %Train Algo. %Algorithm
    & \hspace{-0.3cm}${\PAC_{\Poly}}$ \hspace{-0.3cm} \\
\hline
\hline
%TODO Best Expert & \dots & \\
%Weighted Panel of Experts & \dots \\
%\hline
\if 0
$\lambda$-$\norm{\cdot}_{2}$ Linear Regressor & $\Ball_{2}^{\DSeq}$ & $\Ball_{2}^{\DSeq}$ & %$[-1, 1]$
$\mathclap{\Ball_{1}^{1}}$
 & $\loss^{2}_{2}$%\S
  & ? & %\multirow{6}{*}{\ref{alg:aemm-psg}}
  \hspace{-0.25cm}\small\Cref{alg:aemm-psg} & \cmark \\
Unit Polynomial Regr. & $\Ball_{1}^{1}$ & $\Ball_{\in
fty}^{1+\DSeq}$ & $\Ball_{1}^{1}$ %$[-1, 1]$
  & $\loss^{2}_{2}$\S & %$\sup_{y \in \Y} \loss(?,?)$ diam?
  \tiny$\LandauO(\frac{\NGroups(2 + \DSeq)^{4}}{\varepsilon^{2}})$ TODO 6th pow? & & \\
\fi
$\lambda$-$\norm{\cdot}_{1}$ Linear SVM\textsuperscript{\ensuremath\diamondsuit} & $\Ball_{\infty}^{\DSeq}$ & $\lambda\Ball_{1}^{\DSeq}$ & $\R$ & $\loss_{\mathrm{hinge}}$ & $\LandauO(\frac{\NGroups \lambda^{2} \log \frac{\DSeq\NGroups}{\delta}}{\varepsilon^{2}})$ & \hspace{-0.25cm}\small\Cref{alg:aemm-psg} & \cmark \\ % \cref{alg:aemm-psg} \\
$\lambda$-$\norm{\cdot}_{2}$ Linear SVM\textsuperscript{\ensuremath\diamondsuit} & $\Ball_{2}^{\DSeq}$ & $\lambda\Ball_{2}^{\DSeq}$ & $\R$ & $\loss_{\mathrm{hinge}}$ & $\LandauO(\frac{\NGroups \lambda^{2} \log \frac{\NGroups}{\delta}}{\varepsilon^{2}})$ & & \\
$\lambda$-$\norm{\cdot}_{2}$ Logistic Regr.\textsuperscript{\ensuremath\diamondsuit} & $\Ball_{2}^{\DSeq}$ & $\lambda\Ball_{2}^{\DSeq}$ & $\R$ & $\loss_{\mathrm{H}}$ & $\LandauO(\frac{\NGroups \lambda^{2} \log \frac{\NGroups}{\delta}}{\varepsilon^{2}})$ & & \\
\iftrue
%Kernelized Logistic Regression / SVM
$\lambda$-$\norm{\cdot}_{2}$-$\Phi$ SVM / LR\textsuperscript{\ensuremath\diamondsuit} & $\Ball_{2}^{\DSeq}$ & $\lambda\Phi(\Ball_{2}^{\DSeq})$ & $\R$ & $\loss_{\mathrm{hinge}}$/$\loss_{\mathrm{H}}$ &  $\LandauO($\scalebox{0.8}[1]{$\frac{\NGroups\lambda^{2} \Diam^{2}(\Theta_{\DSeq}) \log \! \frac{\NGroups}{\delta}}{\varepsilon^{2}}$}$)$ & \!\!\! \scalebox{0.8}[1]{\scriptsize with kernel trick\textsuperscript{\ensuremath\clubsuit}} \!\!\!\! & \\
\fi
\hline
%Univariate Threshold
Decision Stump\textsuperscript{\ensuremath\heartsuit} & $\R^{\DSeq}$ & $T_{\DSeq}$ & $\Y$ & $\loss$ & $\LandauO(\frac{\NGroups\norm{\loss}_{\infty}^{2}\log\frac{\DSeq\NGroups}{\delta}}{\varepsilon^{2}})$ & %\multirow{3}{*}{\ref{alg:aemm-covering}}
\hspace{-0.25cm}\small\Cref{alg:aemm-covering} & \cmark \\
Depth-$k$ Decision Tree\textsuperscript{\ensuremath\heartsuit}\hspace{-0.2cm} & $\R^{\DSeq}$ & \hspace{-0.4cm} \scriptsize $T_{\DSeq}^{2^k-1} \times \Y^{2^k}$ \hspace{-0.4cm} & $\Y$ & $\loss$ & $\LandauO(\frac{\NGroups k 2^{k} \norm{\loss}_{\infty}^{2} \log\frac{\DSeq\NGroups}{\delta}}{\varepsilon^{2}})$ & & \cmark \\
%Depth-$\log_{2}(L)$ Decision Trees & $\R^{\DSeq}$ & $(\{1, \dots, \DSeq\} \times \pm \times \R)^{L-1} \times \Y^{L}$ & \begin{tabular}{cc} $k$-class & ? \\ $k$-reg & ? \\ \end{tabular}  & \cref{alg:aemm-covering} TODO: rephrase as depth $\log_{2}(L)$? \\
%Linear Classifier (Hyperplane)
Hyperplane Classifier & $\R^{\DSeq}$ & $\R^{\DSeq+1}$ & $\mathclap{\pm 1}$ & $\loss_{0\mathrm{-}1}$ & $\LandauO( \frac{\NGroups \DSeq \log \frac{\NGroups}{\delta}}{\varepsilon^{2}} )$ & & \xmark \\
\iftrue
McCulloch-Pitts %Network
  NN\textsuperscript{\ensuremath\spadesuit} & $\R^{\DSeq}$ & \hspace{-0.2cm} $\R^{\DSeq \! \times \! h} \!\! \times \! \R^{h \! \times \! k}$ \hspace{-0.2cm} %\tiny $\R^{(1+\DSeq) \times H} \times \R^{(1+H) \times k}$ $\emptyset$ TODO $\Theta_{\DSeq}$
  & \hspace{-0.25cm} \scalebox{0.8}[1]{$1, \dots, k$} \hspace{-0.25cm} & $\loss_{0\mathrm{-}1}$ & $\LandauOTilde(\frac{\NGroups h(\DSeq+k)}{\varepsilon^{2}})$ %\tiny$\LandauO(\frac{H(\DSeq+k)\ln(k)\ln(H(\DSeq+k)\ln(\frac{\NGroups}{\delta})}{\varepsilon^{2}}$
  & & \xmark \\
\fi
\hline
\\
\end{tabular}
 \\[-0.13cm]
}\todo{$k$-dim linear models? boolean formulae? reco systems? Q: if we limit user prefs, is everything poly, since most objects don't matter? experts}

{
\small
%In the above,
Here $\smash{\Ball_{q}^{\DSeq} \doteq \{ x \in \R^{\DSeq} \, | \,  \norm{x}_{q} \leq 1 \}}$ denotes the $\smash{\ell_{q}}$-unit ball in $\R^{\DSeq}$, and $T_{\DSeq} \doteq (\{1, \dots, \DSeq\} \times \pm 1 \times \R)$ denotes a univariate threshold function, which consists of a \emph{feature index}, a \emph{direction}, and a \emph{threshold value}.
Furthermore, %$\loss_{2}^{2}(\cdot, \cdot)$ denotes the \emph{square loss}, $\loss_{\mathrm{hinge}}(\cdot, \cdot)$ the \emph{hinge loss}, 
$\loss_{\mathrm{hinge}}(\cdot, \cdot)$ denotes the \emph{hinge loss}, $\loss_{\mathrm{H}}(\cdot, \cdot)$ the \emph{cross entropy loss}, and $\loss_{0\mathrm{-}1}(\cdot, \cdot)$ the \emph{0-1 loss}.

\smallskip

{\textsuperscript{\ensuremath\diamondsuit} \footnotesize Sample complexity bounds via standard Rademacher average bounds for linear families
 %linear Rademacher average bounds
% Rademacher averages
 \citep[see, e.g.,][Chapter~26]{shalev2014understanding}, leveraging the \emph{boundedness} and \emph{Lipschitz continuity} of this construction.}

\iftrue
{\textsuperscript{\ensuremath\clubsuit} \footnotesize 
%Polynomial sample complexity holds so long as the projected space $\Phi(\Ball_{2}{^\DSeq})$ remains bounded, but 
Training efficiency via the kernel trick requires additional assumptions on the projection $\Phi(\cdot)$ and a compatible kernel $K(\cdot,\cdot)$.
%TODO -$K(\cdot,\cdot)$
}
\fi

{\textsuperscript{\ensuremath\heartsuit} \footnotesize VC-theoretic sample complexity bounds for decision trees and stumps are as derived by \citet{leboeuf2020decision}.} %TODO: pseudodimension? range of \loss?

\iftrue
{\textsuperscript{\ensuremath\spadesuit} \footnotesize The %3-layer
McCulloch-Pitts (\citeyear{mcculloch1943logical}) neural network uses %\emph{threshold activations}, 
the \emph{threshold activation function}.
We analyze a 3-layer model, with hidden layer width $h$, for which the Natarajan dimension is 
%and has Natarajan dimension
 $\smash{\LandauOTilde\bigl(H(\DSeq+k)\bigr)}$.%$\LandauO\smash{\bigl(H(\DSeq+k)\ln(H(\DSeq+k)\bigr)}$. %, which yields the sample complexity result. %TODO: can we drop the ln k?
}
\fi
\todo{Restore missing items: Linreg. square loss.  CHECK EVERYTHING!}
%TODO {\SS}
% The same bound and alogrithm applies to any convex loss function with codomain $[0, 1]$, e.g., \emph{absolute loss}. %Various options, bounded convex loss 
%
\draftnote{Add linear / polynomial regressor}
\todo{Better landau O notation}
\draftnote{TODO table 1!}
}

\todo{loss function?}
\end{table}
\fi

\draftnote{Convex optimization, efficiently enumerable uniform convergence, and stable optimization with grid search?}

\todo{
A Negative Result:
Show that $k$-means clustering is statistically easy, but computationally hard?  On the $\ell_{2}$ ball $\mathcal{B}_{\DSeq}$ with $k=d$.  Covering + minima argument for rade avg (SC).  STRUCTURE: SBWOC efficient covering routine exists.  Then exists polytime alg to optimize.  Contradicts P=NP!
But wait..., cover sizes are exponential in $d$ if log covers are poly?
}

\subsection{Efficient FPAC Learning with Convex Optimization}
\label{sec:ccl:co}
\if 0
Unlike in the realizable case, it is rather difficult to generally reduce fair-agnostic-PAC learning to agnostic-PAC learning.\todo{Counterexample?}
In this section, we address a slightly weaker question: do standard conditions under sufficient for agnostic-PAC learnability imply agnostic-FPAC learnability?
In particular, 
\fi
%Here we show concretely and constructively the existence of FPAC-learners under standard convex optimization assumptions via the \emph{subgradient method} \citep{shor2012minimization}, with constants fully derived.
Here we present \cref{alg:aemm-psg}, which constructs a polynomial-time FPAC-learner under standard convex-optimization assumptions via the \emph{subgradient method}\footnote{
The \emph{subgradient} $\partial_{\theta} \, f(\theta)$ generalizes the \emph{gradient} $\grad_{\theta} f(\theta)$ of a function $f$ evaluated at $\theta$, and the two are coincident for \emph{differentiable convex functions}, i.e., $\partial_{\theta} \, f(\theta) = \grad_{\!\theta} f(\theta)$.
We adopt this setting since the subgradient method yields optimization convergence guarantees even for \emph{nondifferentiable} convex functions, and we assume throughout that a subgradient $\partial_{\theta} \, \loss(h(x; \theta), y)$ may be evaluated in $\Poly(\DSeq)$ time.
}
\citep{shor2012minimization}, with constants fully derived.
Sharper analyses are of course possible, and potential improvements are discussed subsequently, but our result is immediately practical, and can be applied verbatim to problems like \emph{generalized linear models} \citep{nelder1972generalized} and many \emph{kernel methods}\todo{cite KM? other cvx models?} with little analytical effort (under appropriate regularity conditions).
%Indeed, the result immediately applies that many classic convex models, such as \emph{linear regression}, \emph{linear SVM}, \emph{logistic regression}, and many kernel methods are amenable to malfare minimization (under appropriate regularity conditions).
Further details on several such models are presented in \cref{table:emm-models}.

\if 0

\subsection{Convex Optimization}

Convex, Lipschitz, bounded, finite dimension, constant time evaluable gradient approx (implied by ctime value approx)

Is that enough?

pmean preserves all of these?  $\infty$ breaks smoothness though?

Self-concordancy: \url{https://en.wikipedia.org/wiki/Self-concordant_function}.  See \url{https://www.wolframalpha.com/input/?i=d%5E3%2Fdx%5E3+%28x%5Ep+%2B+1%29%5E%281%2Fp%29}  Use this: \url{https://www.wolframalpha.com/input/?i=d%2Fdx+%28x%5Ep+%2B+1%29%5E%281%2Fp%29}?

Composition of convex $\Malfare$ w/ convex $\theta$?  Tricky because multivariate.

See also

ELEG/CISC 867: Advanced Machine Learning
Spring 2019
Lecture 9: Convex Learning Problems
Lecturer: Xiugang Wu

CLB is good.
$\beta$-smooth (gradient is $\beta$-Lipschitz)

Strongly convex: stronger than strict \url{https://en.wikipedia.org/wiki/Convex_function#Strongly_convex_functions}

\fi

\todo{Does diminishing step size $\frac{1}{i}$ help: %\url{https://www.wolframalpha.com/input/?i=%28sum+i%3D1+to+n+of+1%2Fi%5E2%29+%2F+%28sum+i%3D1+to+n+of+1%2Fi%29} Does that diminish too quickly?  What about $\frac{1}{\sqrt{i}}$: %\url{https://www.wolframalpha.com/input/?i=%28sum+i%3D1+to+n+of+1%2Fi%29+%2F+%28sum+i%3D1+to+n+of+1%2Fsqrt%28i%29%29}
}

%\begin{algorithm}
%\begin{algorithmic}

\todo{KNN: model size depends on training size; no PSG required.}

\begin{algorithm}%[htbp]
\caption{Approximate Empirical Malfare Minimization via the Subgradient Method}
\label{alg:aemm-psg}
\algrenewcommand\algorithmicindent{1.1em}
%\scalebox{0.9}[0.75]{
%\scalebox{0.96}[0.95]{
%\begin{minipage}{1.04\textwidth}
%\scalebox{1}[1]{
%\begin{minipage}{1\textwidth}
\begin{algorithmic}[1]

\Procedure{$\PACAlgo_{\mathrm{PSG}}$}{$\LossFunction, \HC, \theta_{0}, \SampleComplexity_{\UC}(\cdot, \cdot), \ProbDist_{1:\NGroups}, \wv, \Malfare(\cdot; \cdot), \varepsilon, \delta$}

\State {%\small
 \Input $\lambda_{\LossFunction}$-Lipschitz loss function
 $\LossFunction$, $\lambda_{\HC}$-Lipschitz hypothesis class $\HC$ with parameter space $\Theta$ s.t.\ $\loss \circ \HC$ is convex,
initial guess $\theta_{0} \in \Theta$,
%$\varepsilon$-$\delta$ 
uniform-convergence sample-complexity bound $\SampleComplexity_{\UC}(\cdot, \cdot)$, %(i.e., $\LossFunction \circ \HC_{\DSeq}$ exhibits $\varepsilon$-$\delta$ \emph{uniform convergence} with sample complexity $\SampleComplexity(\varepsilon, \delta, \DSeq) \in \Poly({\frac{1}{\varepsilon}}, {\frac{1}{\delta}}, \DSeq)$.
%per-
group distributions $\ProbDist_{1:\NGroups}$, % weighted by $\wv$,
group weights $\wv$,
malfare function 
$\Malfare(\cdot; \cdot)$,
%probabilistic additive %-error guarantee
%malfare-optimality
%error
and optimality guarantee $\varepsilon$-$\delta$\todo{Could fold weights into $\Malfare$}}

%Take $\PACAlgo(\ProbDist_{1:\NGroups}, \wv, \Malfare, \varepsilon, \delta, \DSeq)$ to run the subgradient algorithm from arbitrary initial $\theta_{0} \in \Theta_{\DSeq}$ on objective $f(\cdot)$ for $n \doteq \ceil*{ \frac{\Diam(\Theta_{\DSeq})^{4} + \lambda_{\LossFunction}^{4}\lambda_{\HC}^{4}}{\varepsilon^{2}} }$ (uniform) steps of size $\alpha \doteq n^{-\frac{1}{2}}$.

\State {%\small
 \Output $\varepsilon$-$\delta$-$\Malfare(\cdot; \cdot)$-optimal $\hat{h} \in \HC$ %(under the conditions of \cref{thm:aemm-psg}).
 }

\State $\SampleComplexity_{\PACAlgo} \gets \SampleComplexity_{\UC}({\frac{\varepsilon}{3}}, {\frac{\delta}{\NGroups}})$ \Comment{Determine sufficient sample size}

\State $\bm{z}_{1:g,1:\SampleComplexity_{\PACAlgo}} \distributed \ProbDist_{1}^{\SampleComplexity_{\PACAlgo}} \times \dots \times \ProbDist_{\NGroups}^{\SampleComplexity_{\PACAlgo}}$ \Comment{Draw training sample for each group}

\State $n \gets \ceil*{\left(\frac{3\Diam(\Theta)\lambda_{\LossFunction}\lambda_{\HC}}{\varepsilon}\right)^{2}}$ \Comment{Iteration count}

\State $\alpha \gets \frac{\Diam(\Theta)}{\lambda_{\LossFunction}\lambda_{\HC}\sqrt{n}}$ %\draftnote{Ignoring $\ceil{\cdot}$ this is $\gets \frac{\varepsilon}{2\lambda_{\LossFunction}\lambda_{\HC}}$}
  \Comment{Learning rate ($\approx \frac{\varepsilon}{3\lambda_{\LossFunction}^{2}\lambda_{\HC}^{2}}$)}

\State $f(\theta): \Theta \mapsto \R_{0+} \doteq \Malfare \bigl( i \mapsto \ERisk(h(\cdot; \theta); \LossFunction, \bm{z}_{i}); \wv \bigr)$ \Comment{Define empirical malfare objective}

\State  $\hat{\theta} \gets \textsc{ProjectedSubgradient}(f, \Theta, \theta_{0}, n, \alpha)$ \Comment{Run PSG algorithm on empirical malfare} % objective from $\theta_{0}$}
%\State $\theta_{1:n} \gets \textsc{PSG}(f, \Theta, \theta_{0}, n, \alpha)$
%\State \Return $\argmin_{\theta \in \theta_{0:n}} f$

%\For{$i \in 1, \dots, n$}
%\EndFor

\State \Return $h(\cdot; \hat{\theta})$ \Comment{Return $\varepsilon$-$\delta$ optimal model} %\Comment{Return near-optimally parameterized model}

\EndProcedure
\end{algorithmic}
\end{algorithm}

\todo{uc to fpac emm sum prf}

\begin{restatable}[Efficient FPAC Learning via Convex Optimization]{theorem}{thmaemmpsg}
\label{thm:aemm-psg}
Suppose each hypothesis space $\HC_{\DSeq} \in \HC$ is indexed by %some set
$\Theta_{\DSeq} \subseteq \R^{\Poly(\DSeq)}$, i.e., $\HC_{\DSeq} = \{ h(\cdot; \theta) \mid \theta \in \Theta_{\DSeq} \}$, %such that
s.t.\ (Euclidean) $\Diam(\Theta_{\DSeq}) \in \Poly(\DSeq)$, and $\forall x \in \X, \theta \in \Theta_{\DSeq}$, $h(x; \theta)$ can be evaluated in $\Poly(\DSeq)$ time, and $\tilde{\theta} \in \R^{\Poly(\DSeq)}$ can be Euclidean-projected onto $\Theta_{\DSeq}$ in $\Poly(\DSeq)$ time.
Suppose also $\LossFunction$ such that $\forall x \in \X, y \in \Y: \ \theta \mapsto \LossFunction(y, h(x; \theta))$ is a \emph{convex function},
and suppose Lipschitz constants $\lambda_{\LossFunction}, \lambda_{\HC} \in \Poly(\DSeq)$ and some norm $\norm{\cdot}_{\Y}$ over $\Y$ s.t.\ $\LossFunction$ is $\lambda_{\LossFunction}$-$\norm{\cdot}_{\Y}$-$\abs{\cdot}$-Lipschitz in $\hat{y}$, i.e.,\todo{Perhaps it should be $\lambda_{\Theta_{\DSeq}}$, not $\lambda_{\HC}$?}

\[
\forall y, \hat{y}, \hat{y}' \in \Y: \ \abs{\LossFunction(y, \hat{y}) - \LossFunction(y, \hat{y}')} \leq \lambda_{\LossFunction}\norm{\hat{y} - \hat{y}'}_{\Y} \enspace,
\]
and also that each $\HC_{\DSeq}$ is $\lambda_{\HC}$-$\norm{\cdot}_{2}$-$\norm{\cdot}_{\Y}$-Lipschitz in $\theta$, i.e.,
\[
\forall x \in \X, \theta, \theta' \in \Theta_{\DSeq}: \ \norm{h(x; \theta) - h(x; \theta')}_{\Y} \leq \lambda_{\HC}\norm{\theta - \theta'}_{2} \enspace.
\]
Finally, assume $\LossFunction \circ \HC_{\DSeq}$ exhibits $\varepsilon$-$\delta$ \emph{uniform convergence} with sample complexity $\SampleComplexity_{\UC}(\varepsilon, \delta, \DSeq) \in \Poly({\frac{1}{\varepsilon}}, {\frac{1}{\delta}}, \DSeq)$.
\cyrus{Could assume instead $\HC$ is uniformly convergent, with $\SampleComplexity_{\UC}({\frac{\varepsilon}{2\lambda_{\LossFunction}}}, \frac{\delta}{\NGroups}, \DSeq)$ samples}

\if 0
Now, taking $\SampleComplexity_{\PACAlgo} = \SampleComplexity({\frac{\varepsilon}{2}}, {\frac{\delta}{\NGroups}}, \DSeq)$, and training sample $\bm{z} \distributed \ProbDist_{1}^{\SampleComplexity_{\PACAlgo}} \times \dots \times \ProbDist_{\NGroups}^{\SampleComplexity_{\PACAlgo}}$, define the empirical malfare objective
\[
f(\theta): \Theta_{\DSeq} \mapsto \R_{0+} \doteq \Malfare_{p} \bigl( i \mapsto \ERisk(h(\cdot; \theta); \LossFunction, \bm{z}_{i}); \wv \bigr) \enspace.
\]
Take $\PACAlgo(\ProbDist_{1:\NGroups}, \wv, \Malfare, \varepsilon, \delta, \DSeq)$ to run the subgradient algorithm from arbitrary initial $\theta_{0} \in \Theta_{\DSeq}$ on objective $f(\cdot)$ for $n \doteq \ceil*{ \frac{\Diam(\Theta_{\DSeq})^{4} + \lambda_{\LossFunction}^{4}\lambda_{\HC}^{4}}{\varepsilon^{2}} }$ (uniform) steps of size $\alpha \doteq n^{-\frac{1}{2}}$.
\fi

It then holds that, for arbitrary initial guess $\theta_{0} \in \Theta_{\DSeq}$, given any group distributions $\ProbDist_{1:\NGroups}$, group weights $\wv$, fair malfare function $\Malfare(\cdot; \cdot)$, $\varepsilon$, $\delta$, and $\DSeq$,
the algorithm (see~\cref{alg:aemm-psg})
\[
\PACAlgo(\ProbDist_{1:\NGroups}, \wv, \Malfare(\cdot; \cdot), \varepsilon, \delta, \DSeq) \doteq \PACAlgo_{\mathrm{PSG}}\bigl(\LossFunction, \HC_{\DSeq}, \theta_{0}, \SampleComplexity_{\UC}(\cdot, \cdot, \DSeq), \ProbDist_{1:\NGroups}, \wv, \Malfare(\cdot; \cdot), \varepsilon, \delta\bigr)
\]
%It then holds that $\PACAlgo$
FPAC-learns $(\HC, \LossFunction)$ with sample complexity $\SampleComplexity(\varepsilon, \delta, \DSeq, \NGroups) = \NGroups \cdot \SampleComplexity_{\UC}({\frac{\varepsilon}{3}}, {\frac{\delta}{\NGroups}}, \DSeq)$, and (training) time-complexity $\in \Poly({\frac{1}{\varepsilon}}, {\frac{1}{\delta}}, \DSeq, \NGroups)$, thus $(\HC, \LossFunction) \in \smash{\FPAC_{\Poly}^{\Agnostic}}$.
\end{restatable}
\cyrus{TODO: Pf sketch}

It is of course possible to show similar guarantees under relaxed conditions, and with sharper sample complexity and time complexity bounds; %\footnote{In particular, }
%The above result
\cref{thm:aemm-psg} merely characterizes a simple and standard convex optimization setting under which standard convex-optimization guarantees for \emph{risk minimization} readily translate to \emph{malfare minimization}.
In particular,
%As for improving the result, 
we note that the Lipschitz assumptions can also be weakened without sacrificing (polynomial) time guarantees, and that more sophisticated optimization methods may yield (polynomially) more efficient optimization routines.
Furthermore, in risk minimization, stronger conditions like \emph{strong convexity} and \emph{self-concordancy} yield substantial improvements to optimization time complexity; future work shall determine whether and when such properties are preserved in composition with power-mean malfare functions, and thus whether the relevant highly-efficient specialized optimization methods are applicable.

Indeed we remark now that for $p \approx 1$, and when per-group samples have similar empirical risk values for all models $\bm{h}$ encountered in the traversal through parameter space, then
\[
\forall h \in \bm{h}: \ \Malfare_{p}(i \mapsto \ERisk(h; \loss, \bm{z}_{i}); \wv) \approx \Malfare_{1}(i \mapsto \ERisk(h; \loss, \bm{z}_{i}); \wv) \enspace,
\]
thus the optimization aspects of the problem mimic a standard (weighted) loss minimization problem.
In contrast, as $p \to \infty$, the task becomes a minimax optimization problem (see, e.g., the adversarial learning setting of \citep{mazzetto2021adversarial}), so more specific methods for such tasks, such as the \emph{mirror-prox} algorithm of \citet{juditsky2011solving}, as employed to great effect in a similar minimax setting by \citet{cortes2020agnostic}, may exhibit better (smoother, less oscillatory) behavior when multiple groups are near-tied for maximal empirical risk.

\todo{Contrast sample complexity with \cref{ftfsl} \cref{ftfsl:uc}}

\draftnote{Contrast with submodularity?}

\subsection{Uniform Convergence and Efficient Covering}
\label{sec:ccl:uc}

\begin{algorithm}%[htbp]
\caption{Approximate Empirical Malfare Minimization via Empirical Cover Enumeration}
\label{alg:aemm-covering}
\algrenewcommand\algorithmicindent{1.1em}
%\scalebox{0.9}[0.75]{
%\scalebox{0.96}[0.95]{
%\begin{minipage}{1.04\textwidth}
%\scalebox{1}[1]{
%\begin{minipage}{1\textwidth}
\begin{algorithmic}[1]

\Procedure{$\PACAlgo_{\ECover}$}{$\LossFunction, \HC, \ECover(\cdot, \cdot), \Covering(\cdot, \cdot), \ProbDist_{1:\NGroups}, \wv, \Malfare(\cdot; \cdot), \varepsilon, \delta$} %Could call it $\PACAlgo_{\mathrm{Cover}}$

\State {%\small
 \Input Loss function
 $\LossFunction$, hypothesis class $\HC$, empirical covering routine $\ECover(\cdot, \cdot)$, uniform covering number bound $\Covering(\cdot, \cdot)$,
%per-
group distributions $\ProbDist_{1:\NGroups}$, % weighted by $\wv$,
group weights $\wv$,
malfare function 
$\Malfare(\cdot; \cdot)$,
%probabilistic additive %-error guarantee
%malfare-optimality
%error
solution optimality guarantee $\varepsilon$-$\delta$.\todo{Could fold weights into $\Malfare$}}

\State {%\small
 \Output $\varepsilon$-$\delta$-$\Malfare(\cdot; \cdot)$-optimal $\hat{h} \in \HC$}
 
\State $%\displaystyle
  \SampleComplexity_{\UC}(\varepsilon, \delta) \doteq \ceil*{\frac{8\,\norm{\loss}_{\infty}^{2}\ln \left(\smash{\sqrt[4]{\frac{2\NGroups}{\delta}}} \Covering(\LossFunction \circ \HC, \frac{\varepsilon}{4} ) \right)}{\varepsilon^{2}}}$ \Comment{Bound sample complexity (see \cref{lemma:group-cover}~\cref{lemma:group-cover:sc})}

\State $\SampleComplexity_{\PACAlgo} \gets { \SampleComplexity_{\UC}({\frac{\varepsilon}{3}}, \delta) }$ \Comment{Determine sufficient training sample size}

\State $\bm{z}_{1:g,1:\SampleComplexity_{\PACAlgo}} \distributed \ProbDist_{1}^{\SampleComplexity_{\PACAlgo}} \times \dots \times \ProbDist_{\NGroups}^{\SampleComplexity_{\PACAlgo}}$ \Comment{Draw training sample for each group}

\State $\gamma \doteq \frac{\varepsilon}{3\sqrt{\NGroups}}$ \Comment{Select cover resolution}

\State $\HC_{\gamma} \gets \ECover(\HC, \bigcirc_{i=1}^{\NGroups}\bm{z}_{i}, \gamma)$ \Comment{Enumerate empirical cover of concatenated samples}

\State $\displaystyle \hat{h} \gets \argmin_{h_{\gamma} \in \HC_{\gamma}} \Malfare \bigl( i \mapsto \ERisk(h_{\gamma}; \LossFunction, \bm{z}_{i}); \wv \bigr)$ \Comment{Perform EMM over $\HC_{\gamma}$}

\State \Return $\hat{h}$ \Comment{Return $\varepsilon$-$\delta$-$\Malfare(\cdot; \cdot)$ optimal model} %\Comment{Return near-optimally parameterized model}

\EndProcedure

\end{algorithmic}
\end{algorithm}

As we have seen in \cref{sec:ftfsl}, uniform convergence implies, and is often equivalent to, (fair) PAC-learnability.
However, these results all consider only \emph{statistical learning}, and to analyze \emph{computational learning} questions, we must introduce a strengthening of uniform convergence that considers \emph{computation}.
We now show sufficient conditions for polynomial-time FPAC-learnability via \emph{covering numbers}, which we use both to show uniform convergence and to construct an efficient training algorithm.
%\cyrus{I might like this subsection better before convex optimization?}
In particular, %in the case where PAC-learnability and UC are equivalent, 
%Given this, 
we show that if a \emph{polynomially-large cover} of each $\LossFunction \circ \HC_{\DSeq}$ %$\F$ 
exists, and can be efficiently enumerated, then $(\HC, \LossFunction) \in \FPAC_{\Poly}$.

In what follows, an $\ell_{2}$-$\gamma$-empirical-cover of loss family $(\LossFunction \circ \HC_{\DSeq}) \subseteq \X \mapsto \R_{0+}$ on a sample $\bm{z} \in (\X \times \Y)^{m}$ is any $\HC_{\DSeq,\gamma}$ such that
\[
%\Cover(\HC_{\DSeq}, \bm{z}, \gamma) \text{ s.t.\ } 
\forall h \in \HC_{\DSeq}: \min_{h_{\gamma} \in \HC_{\DSeq,\gamma}}  \sqrt{\frac{1}{m} \sum_{i=1}^{m} \bigl((\loss \circ h)(\bm{z}_{i}) - (\loss \circ h_{\gamma})(\bm{z}_{i})\bigr)^{2}} \leq \gamma \enspace.
\]
We take $\Cover(\LossFunction \circ \HC_{\DSeq}, \bm{z}, \gamma)$ to denote such a cover, and $\Cover^{*}(\LossFunction \circ \HC_{\DSeq}, \bm{z}, \gamma)$ to denote such a cover \emph{of minimum cardinality}.
Finally, we define the \emph{uniform covering numbers}\todo{$\HC \ vs \ \HC_{\DSeq}$}\todo{Naming: there's a better term for it, but ``uniform'' by analogy with \emph{uniform entropy number}, see ``Bennett-type Generalization Bounds: Large-deviation Case and
Faster Rate of Convergence.''}
\[
\Covering(\LossFunction \circ \HC_{\DSeq}, m, \gamma) \doteq \sup_{\bm{z} \in (\X \times \Y)^{m}} \abs{ \Cover^{*}(\LossFunction \circ \HC_{\DSeq}, \bm{z}, \gamma) } \ \ \ \& \ \ \ \Covering(\LossFunction \circ \HC_{\DSeq}, \gamma) \doteq \sup_{m \in \N} \Covering(\LossFunction \circ \HC_{\DSeq}, m, \gamma)\enspace.
\]
This concept is crucial %, as it appears in 
to both our
 \emph{uniform convergence} and \emph{optimization efficiency} guarantees.
In particular, our construction ensures that $\Covering(\LossFunction \circ \HC_{\DSeq}, \gamma)$ is sufficiently small so as to ensure \emph{polynomial training time} on a \emph{polynomially-large} training sample is sufficient to FPAC-learn $(\LossFunction, \HC)$.
%both \emph{uniform convergence} and  as our FPAC learner operates by enumerating an $\ell_{2}$-$\gamma$-empirical-cover not polynomially larger than this, and it appears also in the \emph{uniform convergence} bounds used to ensure polynomial \emph{sample complexity} is sufficient.\todo{connect to UC better}
\cyrus{TODO: cite example of UCN?  Haussler VC bound?}

\draftnote{Is $\HC_{\DSeq}$ confusing?  Better to say $\HC$?}

%With this exposition complete, the FPAC-learning algorithm we present is simply EMM on a cover $\ECover(\LossFunction \circ \HC_{\DSeq}, \bm{z}, \gamma)$ that is not superpolynomially larger than $\Covering(\LossFunction \circ \HC_{\DSeq}, m, \gamma)$.
With this exposition complete, we present \cref{alg:aemm-covering}, which performs EMM on an empirical cover $\ECover(\LossFunction \circ \HC_{\DSeq}, \bm{z}, \gamma)$.
We now show that, under appropriate conditions, such a cover exists, is not superpolynomially larger than $\Covering(\LossFunction \circ \HC_{\DSeq}, \gamma)$, and may be efficiently enumerated.
Furthermore, we show that \cref{alg:aemm-covering} requires only a polynomially-large training sample, and thus is an FPAC-learner.
%We now state the result, noting that full derivation with exact constants appears in the appendix.

%\paragraph{Sufficient conditions for FPAC-learnability}
\todo{OLD:
\todo{necessary conds? PAC is obvious?}
\todo{More of a meta-thm? why not thm then coros for vc / loss functions / covering concepts}
As noted by \citep{blumer1989learnability}, PAC-learnability and finite VC-dimension \citep{vapnik1968uniform} are essentially equivalent (subject to basic regularity conditions).\todo{Cite VC}
We now extend this analogy to FPAC-learning.
\todo{consult blumer for computation details?}
\todo{cite VC?  Describe earlier in thesis?}
}
\begin{restatable}[Efficient FPAC-Learning %with Vapnik-Chervonenkis Classes and Covering Numbers
by Covering]{theorem}{thmaeemcovering}
\label{thm:aemm-covering}
Suppose %$\lambda_{\LossFunction}$-Lipschitz 
loss function $\LossFunction$ of bounded codomain (i.e., $\norm{\LossFunction}_{\infty}$ is bounded), %(i.e., $\LossFunction: (\Y \times \Y) \to [0, \norm{\LossFunction}_{\infty}]$), 
and hypothesis class sequence $\HC$, s.t.\ %such that %for any 
$\forall m, \DSeq \in \N$, $\bm{z} \in (\X \times \Y)^{m}$, there exist
\begin{enumerate}[wide, labelwidth=0pt, labelindent=0pt]\setlength{\itemsep}{3pt}\setlength{\parskip}{0pt}
\item a $\gamma$-$\ell_{2}$ cover $\Cover^{*}(\LossFunction \circ \HC_{\DSeq}, \bm{z}, \gamma)$, where $\abs{\Cover^{*}(\LossFunction \circ \HC_{\DSeq}, \bm{z}, \gamma)} \leq \Covering(\LossFunction \circ \HC_{\DSeq}, %m,
 \gamma) \in \Poly(%m, 
 \frac{1}{\gamma}, \DSeq)$; and %where $\Covering(\HC_{\DSeq}, m, \gamma)$ denotes the $\ell_{2}$ covering number\todo{specifically: uniform covering number} of $\HC_{\DSeq}$; and
\item %a $\gamma$-$\ell_{2}$ cover $\ECover(\HC_{\DSeq}, \bm{z}, \gamma)$ of size $\Poly \Covering(\HC_{\DSeq}, m, \gamma)$, and an algorithm to enumerate it in $\Poly(m, \frac{1}{\gamma},\DSeq)$ time.\todo{exists algo...}\todo{hat vs star}\todo{assume computable in $\Poly(\DSeq)$ time?}
an algorithm to enumerate a $\gamma$-$\ell_{2}$ cover $\ECover(\LossFunction \circ \HC_{\DSeq}, \bm{z}, \gamma)$ of size $\Poly \Covering(\LossFunction \circ \HC_{\DSeq}, m, \gamma)$ %and an  it
in $\Poly(m, \frac{1}{\gamma},\DSeq)$ time.%\todo{exists algo...}\todo{hat vs star}\todo{assume computable in $\Poly(\DSeq)$ time?}
\end{enumerate}
%
%
\if 0
has finite VC-dimension $d$, and the \emph{projection} of $\HC$ onto $\bm{z} \in (\X \times \Y)^{m}$ can be enumerated in $\Poly(m)$ time (wherein $d$ is a constant).
Then $\HC$ is FPAC-learnable.

\todo{integrate:}

Furthermore, suppose generic $\HC_{\DSeq} \subseteq \X \to \Y$, $\loss \in \Y \times \Y \to [0, 1]$, and assume that\todo{Scaling?}
\begin{enumerate}[wide, labelwidth=0pt, labelindent=0pt]\setlength{\itemsep}{3pt}\setlength{\parskip}{0pt}
\item $\Y$ is a set of diameter $D_{\Y}$, and $\loss$ is $\lambda_{\loss}$-Lipschitz continuous;
\item the $\ell_{2}$ covering number%
\footnote{The reader is invited to consult \citet{anthony2009neural} for an encyclopedic overview of various covering numbers and their applications to statistical learning theory.}
$\Covering(\HC_{\DSeq}, \bm{z}, \gamma) \in \Poly(m, \frac{1}{\gamma}, \DSeq)$;
\item a \emph{$\gamma$-precision cover} $\Cover(\HC_{\DSeq}, \bm{z}, \gamma)$ of size $\LandauO(\Covering(\HC_{\DSeq}, \bm{z}, \gamma))$ may be enumerated in $\Poly(m, \frac{1}{\gamma})$ time.
\end{enumerate}
\fi

It then holds that, given any group distributions $\ProbDist_{1:\NGroups}$, group weights $\wv$, fair malfare function $\Malfare(\cdot; \cdot)$, $\varepsilon$, $\delta$, and $\DSeq$,
the algorithm (see~\cref{alg:aemm-covering})
%Then %, taking $\PACAlgo_{\ECover}(\cdots)$ to be as defined in \cref{alg:aemm-covering}, 
%the algorithm (see~\cref{alg:aemm-covering})
\[
\PACAlgo(\ProbDist_{1:\NGroups}, \wv, \Malfare(\cdot; \cdot), \varepsilon, \delta, \DSeq) \doteq \PACAlgo_{\ECover}(\LossFunction, \HC_{\DSeq}, \ECover(\cdot, \cdot), \Covering(\loss \circ \HC_{\DSeq}, \cdot), \ProbDist_{1:\NGroups}, \wv, \Malfare(\cdot; \cdot), \varepsilon, \delta) %\enspace,
\]
FPAC-learns $(\LossFunction, \HC)$ in polynomial time. %thus $(\LossFunction, \HC) \in \FPAC_{\Poly}$, %$\HC$ is efficiently FPAC-learnable.
In particular, (1) is sufficient to show that $(\LossFunction, \HC)$ is FPAC learnable with polynomial \emph{sample complexity}, and (2) is required only to show polynomial \emph{training time complexity}.
\todo{Could start with cover of loss family?}
\end{restatable}
\cyrus{PROOF SKETCH}

This immediately implies that fixed $\HC$ that are \emph{finite}, or of bounded \emph{VC-dimension}, \emph{Natarajan dimension}, \emph{pseudodimension}, or $\gamma$-\emph{fat-shattering dimension}\footnote{The reader is invited to consult \citep{anthony2009neural} for an encyclopedic overview of various combinatorial dimensions, associated covering-number and shattering-coefficient concepts, and their applications to statistical learning theory.}
 are FPAC-learnable. % \citep[sec 4.2]{shalev2014understanding}
For instance, this includes \emph{classifiers} such as \emph{all possible languages} of \emph{Boolean formulae} over (constant) $d$ variables, or \emph{halfspaces} (i.e., linear hard classifiers $\HC \doteq \{ \vec{x} \mapsto \sgn(\vec{x} \cdot \vec{w}) \mid \vec{w} \in \R^{d} \}$), as well as GLM, subject to regularity constraints to appropriately control the loss function.
However, it is perhaps not as powerful as it appears; it applies to \emph{fixed hypothesis classes}, thus each of the above linear models over $\R^{d}$ is polynomial-time FPAC learnable, but it says nothing about their performance as $d \to \infty$.

This is essentially because the \emph{statistical analysis} to show polynomial \emph{sample complexity} requires only that $\smash{\ln} \Covering(\LossFunction \circ \smash{\HC_{\DSeq}}, \gamma) \in \Poly(\smash{\frac{1}{\gamma}}, \DSeq)$, whereas our \emph{training algorithm} must actually \emph{enumerate} an empirical cover, which yields the (exponentially) stronger requirement that $\Covering(\LossFunction \circ \HC_{\DSeq}, \gamma) \in \Poly(\frac{1}{\gamma}, \DSeq)$ for polynomial \emph{time complexity}.
Indeed, we see that while the covering numbers we assume imply uniform convergence with \emph{sample complexity} polynomial in $\DSeq$, when covering numbers grow exponentially in $\DSeq$, then our algorithm yields only \emph{exponential} time complexity in $\DSeq$.
Consequently, the result only implies polynomial-time algorithms w.r.t.\ sequences that grow slowly in complexity; e.g., sequences of linear classifiers that grow only \emph{logarithmically} in dimension, i.e., $\HC_{\DSeq} \doteq  \{ \vec{x} \mapsto \sgn(\vec{x} \cdot \vec{w}) \mid \vec{w} \in \R^{\floor{\ln d}} \}$. %in a weak sense, i.e., 
Further details on when optimizing such models via covering is computationally efficient are presented in \cref{table:emm-models}.

\draftnote{could discuss thresholds, dtrees, MP NNs?}

Note also that \cref{thm:aemm-covering} leverages \emph{covering arguments} in both their \emph{statistical} and \emph{computational} capacities.
Statistical bounds based on covering are generally well-regarded, particularly when strong analytical bounds on covering numbers are available, although sharper results are possible (e.g., through the \emph{entropy integral} or \emph{majorizing measures}\todo{cite}).
Furthermore, while we do construct a \emph{polynomial time} training algorithm, in many cases, specific optimization methods (e.g., stochastic gradient descent or Newton's method) exist to perform EMM \emph{more efficiently} and with \emph{higher accuracy}.\todo{CONNECT EARLIER}
Worse yet, \emph{efficient enumerability} of a cover may be non-trivial in some cases;
while most covering arguments in the wild\todo{examples} are either constructive, or compositional to the point where each component can easily be constructed, it may hold for some problems that computing or enumerating a cover is computationally prohibitive.\todo{coding theory cxn.}
\todo{It remains an open problem whether problems with polynomial size cover but no efficient enumerability exist.}

\todo{Repeated: Despite these limitations, this theorem characterizes a large class in which fair learning is tractable (both statistically and computationally).
In particular, it contains VC, bounded pseudodimension, and bounded fat-shattering dimension classes in which $\varepsilon$-ERM is efficient.\todo{citations / neural covering papers}}
\todo{Reference Hu paper}

\draftnote{Say: The details are largely mechanical and cumbersome, however it is worth %pointing out
mentioning that the main novelty in this result is that we must consider covers not of \emph{a single samples}, but rather covers that are sufficiently rich so as approximate the behavior of $\LossFunction \circ \HC$ \emph{across multiple groups}.
%a notion of covering \emph{across multiple groups} hint difficulty in joint coverings!
}

\draftnote{In practice this algorithm is not particularly efficient. However it is straightforward to modify the search for many particular problems to improve its efficiency.n in particular it is generally not necessary to enumerate the entire cover but rather how to use a branch and bound type search to enumerate locally with a resolution. It is also reasonable to use convex relaxations initially up my contacts relaxation of the target objective in order to restrict the search space
Finally we note that we don't require a deterministic algorithm to guarantee in gamma enumeration research space of the six space a randomized ie sampling based approach EG Bayesian optimization is sufficient so long as it can guarantee that with high probability it will draw at least one gamma optimal point adaptive.
For example it is easy 2 sample uniformly from the space of equivalence classes of projections of Piper planes onto the training set
Which can reduce the space complexity increase the paralyzed ability of and in some cases where do pecans come from reduce the time complexity
}

\todo{include this?}
\todo{Note that a vaguely similar strategy appears in \citep{hu2020fair}, where the projection of all linear classifiers onto $\bm{z}$ is enumerated...}
\todo{more efficient with localization, cover enumeration jumping?}

\paragraph{On Compositionality and Coverability Conditions}
%Note that while t
The covers and covering numbers discussed above are of course properties of each $\LossFunction \circ \HC_{\DSeq}$, rather than $\LossFunction$ and each $\HC_{\DSeq}$ individually.
This creates proof obligation for each loss function of interest, in contrast to \cref{thm:aemm-psg}, wherein only Lipschitz continuity of $\LossFunction$ is assumed, and the remaining analysis is on $\HC$.
Fortunately, in many cases it is still possible to analyze covers of each $\HC_{\DSeq}$ in isolation, and then draw conclusions across a broad family of $\LossFunction$ composed with each $\HC_{\DSeq}$.
In particular, via standard properties of covering numbers, if $\LossFunction$ is Lipschitz continuous \wrt\ %with respect to 
some pseudonorm $\norm{\cdot}_{\Y}$ over $\Y$, and $\gamma$-$\ell_{2}$ covering numbers of each $\HC_{\DSeq}$ \wrt\ $\norm{\cdot}_{\Y}$ are well-behaved, it can be shown that the conditions of \cref{thm:aemm-covering} are met.
This is useful as, for example, regression losses like \emph{square error}, \emph{absolute error}, and \emph{Huber loss} are all Lipschitz continuous on bounded domains, and thus analysis on each $\HC_{\DSeq}$ alone is sufficient to apply \cref{thm:aemm-covering} with each such loss function.\cyrus{Is this even interesting?  Is there enough detail?}

\todo{Stability is cut.  Restore it for general algo that scales badly with $\NGroups$?  Based on covering reweightings?}

\if 0
\paragraph{NEW STABILITY}

\todo{Cut me?}

Stability implies \cref{cor:uc-ftsl} condition $\UC(\LossFunction) = \PAC(\LossFunction)$???  does it make sense?

Theoretical consequence, but uses inefficient enumeration ctrn; must enumerate $\HC$ (unknown how).

Alternative ctrn in terms of $\PACAlgo$ (also inefficient in $\NGroups$ though).

\paragraph{The Non-Equivalence Conjecture}

\begin{conjecture}
There exist $(\HC, \LossFunction)$ such that evaluating
\[
\argmin_{\hat{h} \in \HC_{\DSeq}} \ERisk(\hat{h}; \LossFunction, \bm{z})
\]
is NP-hard. \draftnote{Hardness is in $\frac{1}{\varepsilon}, \DSeq$.

\todo{Need something like \whp{} over $\bm{z}$?}

\end{conjecture}
\begin{consequence}
There exist learning problems characterized by $\LossFunction$ such that PAC-learnability does not imply FPAC learnability.
Thus we have
\[
\FPAC_{\Poly}(\LossFunction) \subset \PAC_{\Poly}(\LossFunction) \enspace,
\]
i.e., the \emph{FPAC-learnable classes} are a \emph{strict subset} of the \emph{PAC-learnable classes}.
\end{consequence}
\begin{proof}
Use fundamental theorem of statistical learning + the conjecture?
\end{proof}
\begin{observation}
This contrasts ???, which shows equivalence under \emph{realizability}.
\end{observation}

\fi

\todo{Good stuff below.}

\todo{Miscellanea}

\draftnote{Comparison writeup below. GGf and others}

\section{Conclusion}
\label{sec:conc}

%The central thrust of
This work introduces \emph{malfare minimization} as a fair learning task, and shows relationships between the \emph{statistical} and \emph{computational} issues of malfare and risk minimization.
In particular, we argue that our method is more in line with welfare-centric machine learning theory than demographic-parity theory, however in \cref{sec:comparisons:inequality} we do show deep connections between welfare or malfare optimization and % our malfare-minimization concept of fair learning has deep connections to 
inequality-constrained loss minimization, %as described in 
which to some extent bridge this divide. % between our methods and fairness-constraint based learning.
We also find that malfare is better aligned to address machine learning tasks cast as loss minimization problems than is welfare, both due to convenient statistical properties, and the greater simplicity of such constructions.
\ifthesis
As such, we build upon the malfare concept of \cref{part:fairness:malfare}, which defines and motivates malfare minimization, by studying the problem from the statistical and computational learning theoretic perspectives.
\else
As such, %[, sections \cref{sec:pop-mean,sec:welfare-ml,sec:stat-est} the first half] 
the first half of this manuscript is dedicated to deriving and motivating malfare minimization, while %[sections \cref{}  the remainder]
the latter half defines the \emph{fair-PAC learning} formalism, and studies the problem from statistical and computational learning theoretic perspectives.
\fi

\draftnote{new par below.  New discussion section?}
Before further detailing our contributions in these areas, we reiterate that malfare itself, as well as its axiomatic characterization as the $p \geq 1$ power-mean family, is indeed the main contribution of this work.
The remainder of the paper explores the \emph{consequences} of this axiomatic definition, some rather simplistic, and others more sophisticated, but we stress that the natural parallels between the statistical and computational aspects of risk minimization and malfare minimization stem from this key definitional decision.

We see this as a measure of the appropriateness of the malfare definition and its use as a fair learning objective, as indeed, other fair-learning formalizations would not behave as such.
What may seem straightforward in hindsight was not, in a sense, predestined to be so; for instance had we adopted the \emph{additive separability} axiom instead of multiplicative linearity (as discussed in \cref{sec:comparisons:as}), malfare would be characteristically $\MeanAs_{p}(\lv; \wv)$ ($= \Malfare_{p}^{p}(\lv; \wv)$ for
%$p \geq 1$)
$p > 0$), rather than $\Malfare_{p}(\lv; \wv)$.
The FPAC-learnability definition, which requires uniform sample complexity \emph{over all fair malfare functions} (all $p \geq 1$) would then be fundamentally flawed, as risk values above $1$ would explode, % (yielding statistical intractability), 
while risk values below $1$ would vanish, as $p \to \infty$ (see~\cref{fig:pmean-vs-addsep}). % and surrounding discussion in \cref{sec:comparisons:as}). %, and our definition is specifically uniform over $p$.
Similarly, \cref{sec:comparisons:welfare} outlines the difficulties that arise should we instead seek to \emph{maximize} any fair \emph{welfare} function.
For a third example, the Seldonian learner \citep{thomas2019preventing} framework, which treats arbitrary constrained nonlinear objectives, also %doesn't
 seems not to be amenable to uniform sample complexity analysis, %this sort of analysis, % (in this case 
due to the
%, in general,
%potentially
generally unbounded sample complexity of determining whether even very simple constraints are satisfied (as discussed in \cref{sec:comparisons:inequality}).%mutually satisfiable. %).
\draftnote{
TODO: sth about diameter of feasible set.
}

\subsection{Contrasting Malfare and Welfare}
\label{sec:conc:contrast}

\draftnote{TODO reference contributions of \cref{sec:comparisons:inequality}.}

With our framework now fully laid out and initial results presented, we now contrast our malfare-minimization framework with traditional welfare-maximization approaches in greater detail.
We do not claim that malfare is a better or more useful concept than welfare; but rather we argue only that it is \emph{significantly different} (with surprising non-equivalence results between power-mean welfare and malfare functions), %, except in the \emph{utilitarian} and \emph{egalitarian} cases), 
stands on an equal axiomatic footing, and it stands to reason that the right tool (malfare) should be used for the %right job
taks at hand (fair risk-minimization).

With this said, we acknowledge that
some learning tasks, e.g., bandit problems and reinforcement learning tasks, are more naturally phrased as \emph{maximizing} utility or (discounted)
 reward. % of taking certain actions possibly contact dependent .
However, with a few exceptions, e.g., the \emph{spherical scoring rule} from decision theory, most supervised learning problems are naturally cast as minimizing nonnegative \emph{loss functions} (arguably via cross-entropy or KL-divergence minimization through maximum-likelihood, either as explicitly intended~\citep{nelder1972generalized}, or \emph{ex-post-facto} through subsequent analysis~\citep{cousins2019cadet}). %[say cross entropy / KL divergence?]

\if 0
Similarly, as mentioned previously, 0-1 loss and 1-0 gain are isomorphic, and there's no clear benefit to optimizing one over another.
In the decision theory setting, considering bounded strictly proper scoring rules, the \emph{spherical rule} $s(y, \hat{y}) \doteq \smash{ \frac{ \hat{y}_{y} }{ \norm{\hat{y}}_{2} } }$ is naturally a gain function, whereas the \emph{Brier rule} $B(y, \hat{y}) \doteq \smash{ \norm{\! \1_{y} - \hat{y}}_{2}^{2} }$ is clearly a loss function.
%In the bounded classification domain it's something of a toss up the spherical scoring role is more naturally phrased as a loss.
%I suppose have code blue in 01 we can of course convert between them via subtraction  but again we argue  this results in changes to  the welfare  endless  this change is not  inconsequential .
In all cases, we may take utility as 1 - loss, but it seems that directly interpreting spherical scores as utilities and maximizing welfare is more clean, elegant, and natural, but on the other hand, so to is using the Brier score as loss, and minimizing malfare. %and clean and similarly  considering Briar rather than 1 - Briar with lost and malfare is also cleaner .
\fi
\cyrus{limiting cxn to robust learning: cite ALL / Mohri?}

%It thus seems that at least
%By these core philosophical arguments, neither malfare nor welfare is inherently superior, however we do note that
We are highly interested in exploring a parallel theory of fair welfare optimization, however some key malfare properties %we have shown and leveraged for malfare 
do not %necessarily 
hold for welfare.
In particular, %we know that
 fair welfare functions $\Welfare_{p}(\cdot; \cdot)$ for $p \in [0, 1)$ are not %necessarily 
Lipschitz continuous;  %is non-Lipschitz; 
%e.g.
for example, the \emph{Nash social welfare} (a.k.a.\ \emph{unweighted geometric welfare}) $\Welfare_{0}(\lv; \omega \mapsto \mathsmaller{\frac{1}{\NGroups}}) = \smash{\sqrt[\NGroups]{\vphantom{\prod}\smash{\prod_{i=1}^{\NGroups}} \lv_{i}}}$ is unstable to perturbations of each $\lv_{i}$ around $0$, which causes %great 
difficulty in both the \emph{statistical} and \emph{computational} aspects of learning. % estimation, computational complexity of optimization seat ferromex and computational complexity.
In \cref{sec:comparisons:welfare}, we leverage this fact to construct seemingly trivial welfare estimation problems that actually exhibit \emph{unbounded sample complexity}.
In particular, in these problems, we must only estimate a single group's \emph{Bernoulli-distributed} utility, which is quite straightforward, but welfare estimation remains intractable.

\if 0
Leveraging this fact, it is trivial to construct welfare-maximization learning problems %instances, in 
for which \emph{sample complexity} is \emph{unbounded};
for example, if utility values are $\textsc{Bernoulli}(q)$-distributed for some group, the sample complexity of $\varepsilon$-$\delta$ %additively
estimating $\Welfare_{0}(\lv; \wv)$ grows unboundedly as $q \to 0$.
\fi

This impossibility result makes straightforward translation of our FPAC framework into a welfare setting rather vacuous, % and uninteresting, 
except in contrived, trivial, or degenerate cases.
%Essentially, this is because while 
This difference between malfare and welfare stems from the fact that although \cref{lemma:stat-est} holds for both welfare and malfare, % functions, 
it does not imply \emph{uniform sample-complexity} bounds, whereas, such bounds are trivial for fair malfare (see~\cref{coro:stat-est}\draftnote{ref rade here.}), % functions, 
due to the \emph{contraction property} (\cref{thm:pow-mean-prop} %item~
\cref{thm:pow-mean-prop:contraction}). %, as employed by % of fair malfare functions (c.f.\ \cref{lemma:stat-est}).
%\cref{lemma:stat-est}.
It thus seems that such a theory of welfare optimization would need either to either impose additional assumptions to avoid non-Lipschitz behavior (e.g., artificially limit the permitted range of $p$), or otherwise provide weaker (non-uniform) learning guarantees.

\subsection{FPAC Learning: Contributions and Open Questions}
\label{sec:conc:contrib}
After motivating the malfare-minimization machine learning task, we introduce fair-PAC-learning to study the statistical and computational difficulty of malfare minimization.
As a generalization of PAC-learning, known hardness results (e.g., lower-bounds on computational and sample complexity of loss minimization) immediately apply, thus, coarsely speaking, the interesting question is whether, for some tasks, malfare minimization is harder than risk minimization.
\Cref{thm:realizable-pac2fpac} answers this question in the negative \emph{under realizability},
as does \cref{thm:ftfsl} for \emph{sample complexity}, under appropriate conditions on the loss function.
However, as far as sample complexity goes, it remains an open question whether agnostic FPAC-learning and PAC-learning are equivalent for loss functions where \emph{uniform convergence} and \emph{PAC-learnability} are not equivalent. %, although we see no reason to think they may not be.
Furthermore, the question of their computational equivalence in the agnostic setting is also open, % (outside of realizability), % The remaining cases are left open, 
although \cref{sec:ccl} at least shows that many conditions sufficient for PAC-learnability are also sufficient for FPAC-learnability.

We are optimistic that our FPAC-learning definitions will motivate the community to further pursue the deep connections between various PAC and FPAC learning settings, %as well as promote cross-pollination between computational learning theorists and fair machine learning researchers and advocates.
as well as promote cross-pollination between computational learning theory and fair machine learning research.
We believe that deeper inquiry into these questions will lead to both a better understanding of what is and is not FPAC-learnable, as well as more practical and efficient reductions and FPAC-learning algorithms.% with improved sample and computational complexity.
%To stimulate such inquiry, we note that several of our reduction AR intentionally presented in highly simplified form which are and are not intended to produce completely Optimal Solutions but rather only Preserve finite sample complexity or polynomial time complexity.

\todo{Add future work!}
\pagebreak[1]

{
\small
\bibliography{paper/bibliography,paper/cyrus}
\vfill
}

%\if 0
\iftrue
\pagebreak[4]

%\backmatter
\appendix

\section{A Compendium of Missing Proofs}
\label{appx:proof}

Here we present all missing proofs of results stated in the main text.

\subsection{Welfare and Malfare}
\label{sec:appx:proof:wm}

We now show \cref{thm:pow-mean-prop}.

\thmpowmeanprop*
\begin{proof}

\todo{Allow finitely-many region-breaks in the loss function (separate covers)}

We omit proof of \cref{thm:pow-mean-prop:mono}, as this is a standard property of power-means, generally termed the \emph{power mean inequality} \citep[Chapter~3]{bullen2013handbook}.

%Lipschitz-Continuity: H\"older's inequality? Discrete formulation:

We first show \cref{thm:pow-mean-prop:subadditivity}.
By the triangle inequality (for $p \geq 1$), we have
\[
\Mean_{p}(\lv + \lv'; \wv) \leq \Mean_{p}(\lv; \wv) + \Mean_{p}(\lv'; \wv) \enspace.
\]
\if 0 %OLD
\[
\Mean_{p}(\lv + \lv'; \wv) = \sqrt[p]{\sum_{i=1}^{\NGroups} \wv_{i}(\lv_{i} + \lv'_{i})^{p}} = \norm{ \wv \otimes (\lv + \lv') }_{p} \leq 
%\begin{array}{c} \norm{\wv}_{p}\norm{\lv + \lv'}_{p}????? \\ 
\norm{\wv \otimes \lv}_{p} + \norm{\wv \otimes \lv'}_{p} = \Mean_{p}(\lv; \wv) + \Mean_{p}(\lv'; \wv) %\end{array}
\]
\fi

We now show \cref{thm:pow-mean-prop:contraction}
First take $\epsv \doteq \lv - \lv'$, and let $\epsv_{+} \doteq \bm{0} \vee \epsv$, where $\bm{a} \vee \bm{b}$ denotes the (elementwise) maximum.
Now consider
\begin{align*}
\Mean_{p}(\lv; \w) &= \Mean_{p}(\lv' + \epsv; \w) & \textsc{Definition of $\epsv$} \\
 &\leq \Mean_{p}(\lv' + \epsv_{+}; \w) & \textsc{Monotonicity} \\
 &\leq \Mean_{p}(\lv'; \wv) + \Mean_{p}(\epsv_{+}; \wv) & \textsc{\Cref{thm:pow-mean-prop:subadditivity}} \\
 &\leq \Mean_{p}(\lv'; \wv) + \Mean_{p}(\abs{\lv - \lv'}; \wv) \enspace, & \textsc{Monotonicity} \\[-0.25cm]
\end{align*}
where here \textsc{Monotonicity} refers to monotonicity of $\Mean_{p}(\lv; \wv)$ in each $\lv(\PopItem)$.
By symmetry, we then have $\Mean_{p}(\lv', \wv) \leq \Mean_{p}(\lv, \wv) + \Mean_{p}(\abs{\lv - \lv'}; \wv)$, which implies the result.

\if 0
\todo{
For $p \leq -1$; we have
\[
\sqrt[p]{\Expect_{\PopItem \distributed \wv}[\lv^{p}(\PopItem)]} 
\]
}
\fi

We now show \cref{thm:pow-mean-prop:curvature}.
First note the special cases of $p \in \pm \infty$ follow by convexity of the maximum ($p=\infty$) and concavity of the minimum ($p=-\infty$).

Now, note that for $p \geq 1$, by concavity of $\sqrt[p]{\cdot}$, Jensen's inequality gives us 
\[
\Mean_{1}(\lv; \wv) = \Expect_{\PopItem \distributed \wv}[\lv(\PopItem)] = \underbrace{\Expect_{\PopItem \distributed \wv}\left[\sqrt[p]{\lv^{p}(\PopItem)}\right] \leq \sqrt[p]{\Expect_{\PopItem \distributed \wv}[\lv^{p}(\PopItem)]}}_{\textsc{Definition of Convexity}} = \Mean_{p}(\lv; \wv) \enspace,
\]
i.e., convexity, and similarly, for $p \leq 1$, $p \neq 0$, by convexity of $\sqrt[p]{\cdot}$, we have
\[
\Mean_{1}(\lv; \wv) = \Expect_{\PopItem \distributed \wv}[\lv(\PopItem)] = \underbrace{\Expect_{\PopItem \distributed \wv}\left[\sqrt[p]{\lv^{p}(\PopItem)}\right] \geq \sqrt[p]{\Expect_{\PopItem \distributed \wv}[\lv^{p}(\PopItem)]}}_{\textsc{Definition of Concavity}} = \Mean_{p}(\lv; \wv) \enspace.
\]
Similar reasoning, now by convexity of $\ln(\cdot)$, shows the case of $p=0$. %, which completes the proof.%\todo{show more detail?}
\end{proof}

We now show \cref{thm:pop-mean-prop}.

\thmpopmeanprop*
\begin{proof}

%1
\Cref{thm:pop-mean-prop:id} is an immediate consequence of axioms
\ref{def:cardinal-axioms:mult} \& \ref{def:cardinal-axioms:unit} (multiplicative linearity and unit scale).
%\todo{proof detail?}

%Old 2
%To see $\Mean(\lv) \leq \Mean_{\infty}(\lv)$, note that $\Mean(\bm{0}) = 0$ (by item 1), thus by axiom 8, $\Mean(\lv) \leq \norm{\lv}_{\infty} = \max \lv = \Mean_{\infty}(\lv)$.
%\todo{use min instead of 0 to extend to nonnegs}

%2
We now note that \cref{thm:pop-mean-prop:lin-fact} is the celebrated Debreu-Gorman theorem \citep{debreu1959topological,gorman1968structure}, extended by continuity and measurability of $\lv$ to the weighted case, and \cref{thm:pop-mean-prop:lin-fact-f} is a simple corollary thereof.\todo{see also for more on Debreu-Gorman: Wakker, P.P., 1989. Additive Representations of Preferences; may need less than continuity?}\draftnote{also \citep{moulin2004fair} pp 66-69 for $f$-mean case.}

%3
We now show \cref{thm:pop-mean-prop:pmean}.\todo{Properly show the split; w/ and w/out scale.}
This result is essentially a corollary of \cref{thm:pop-mean-prop:lin-fact}, hence the dependence on axioms 1-4.
\todo{
We now assume axiom \ref{def:cardinal-axioms:mult}, and the task is to prove $F(u) = \alpha\sqrt[p]{\sgn(p)u}$. %%$F(u) = \alpha\sqrt[p]{\sgn(p)u + \beta^{p}} - \beta$?
}%
Suppose $\lv(\cdot) = 1$.  %, and take $\Mean(\lv; \wv) = c$.
By \cref{thm:pop-mean-prop:id}, %axioms \ref{def:cardinal-axioms:mult} \& \ref{def:cardinal-axioms:unit},
for all $p \neq 0$, we have
\[
\alpha = \alpha\Mean(\lv; \wv)
  = \Mean(\alpha\lv; \wv)
  = F\left(\Expect_{\PopItem \distributed \wv}\bigl[f_{p}(\alpha\lv(\PopItem))\bigr]\right)
  = F\left(\Expect_{\PopItem \distributed \wv}\bigl[f_{p}(\alpha)\bigr]\right)
  %= \begin{cases} p > 0 & \displaystyle F(\Expect_{\PopItem \distributed \wv}\bigl[(\alpha\lv(\PopItem))^{p}\bigr]) = F(\alpha^{p}) \\ p < 0 & \displaystyle F(\Expect_{\PopItem \distributed \wv}\bigl[-(\alpha\lv(\PopItem))^{p}\bigr]) = F(-\alpha^{p}) \\ \end{cases} \enspace.
  %= F\left(\Expect_{\PopItem \distributed \wv}\bigl[\sgn(p)\alpha^{p}\bigr]\right) 
  = F\bigl(\sgn(p)\alpha^{p}\bigr) \enspace.
\]
%From here, it is clear that for $p > 0$, $\alpha = F(\alpha^{p})$, thus $F^{-1}(u) = u^{p}$ and consequently, $\forall v \geq 0: \ F(v) = \sqrt[p]{v}$, and similarly, for $p < 0$, $\alpha = F(-\alpha^{p})$, thus $F^{-1}(u) = -u^{p}$, and consequently, $\forall v \leq 0: \ F(v) = \sqrt[p]{-v}$.
From here, we have $\alpha = F\bigl(\sgn(p)\alpha^{p}\bigr)$, thus $F^{-1}(u) = \sgn(p)u^{p}$, and consequently, %$\forall v \geq 0: \ F(v) = \sqrt[p]{v}$
$F(v) = \sqrt[p]{\sgn(p)v}$.

%Older; long-winded proof
\if 0
\[
1 = \Mean(\lv; \wv) = F(\Expect_{\PopItem \distributed \wv}[f_{p}(\lv(\PopItem))] = F(1)
\]

\[
\alpha = \alpha\Mean(\lv; \wv) = \Mean(\alpha\lv; \wv) = F(\Expect_{\PopItem \distributed \wv}[f_{p}(\alpha\lv(\PopItem))]) = F(\alpha^{p})
\]

\todo{Isn't it obvious from 2nd form?}

Now, we apply $F^{-1}$ to both sides of both equations and divide \todo{check / 0}:

\[
\frac{F^{-1}(\alpha)}{F^{-1}(1)} = \alpha^{p}
\]

Now, note that $F(1) = 1$, thus $F^{-1}(1) = 1$, which leaves
\[
F^{-1}(\alpha) = \alpha^{p} \implies F(\alpha) = \sqrt[p]{\alpha} \enspace.
\]

\todo{Show what happens without assuming unit scale (commented) and other cases of $p$.  It's just $F(\alpha) = \sqrt[p]{\alpha / F^{-1}(c)}$; $F^{-1}(c)$ is just some nonneg constant.}
\fi

Taking $p = 0$ gets us
\if 0
\[
1 = \Mean(\lv; \wv) = F(\Expect_{\PopItem \distributed \wv}[\ln(\lv(\PopItem))]) = F(0)
\]
\fi

\[
\alpha = \alpha\Mean(\lv; \wv) = F\left(\Expect_{\PopItem \distributed \wv}
\bigl[\ln(\alpha\lv(\PopItem))\bigr]\right) = F(\ln a) \enspace,
\]
from which it is clear that $F^{-1}(u) = \ln(u) \implies F(v) = \exp(v)$.

For all values of $p \in \R$, substituting the values of $f_{p}$ and %$F_{p}$
$F(\cdot)$ %TODO F_{p}
into %\cref{thm:pop-mean-prop}~
\cref{thm:pop-mean-prop:lin-fact} yields $\Mean(\lv; \wv) = \Mean_{p}(\lv; \wv)$ by definition.

\if 0
Finally, for the case of $p < 0$, taking $G(u) = F(-u)$, we get
\[
\alpha = \alpha\Mean(\lv; \wv) = \Mean(\alpha\lv; \wv) = F(\Expect_{\PopItem \distributed \wv}[f_{p}(\alpha\lv(\PopItem))] = G(-\alpha^{p}) = F(\alpha^{p}) \enspace,
\]
and the rest proceeds as in the case of $p > 0$.
\fi

%Very old, very wrong proof
%Axiom 5 then implies $\Mean(\lv) \propto \Mean_{p}(\lv)$\todo{prove it?}.
%Finally, axiom \cref{def:cardinal-axioms:unit} implies $\alpha = 1$.

%4&5
We now show \ref{thm:pop-mean-prop:pmean-fair} and \ref{thm:pop-mean-prop:pmean-unfair}.
These properties follow directly from \ref{thm:pop-mean-prop:lin-fact}, wherein $f_{p}$ are defined, and Jensen's inequality. %, wherein we see that  %TODO give more detail?
\todo{prove this in detail.}
%PD transfer implies $\Mean_{1}(\lv) \leq \Mean(\lv)$, because $\Mean_{1}(\lv) = \Mean_{1}(\mu) \leq_{PD} \Mean(\lv)$.
%
\end{proof}

We now show \cref{coro:stat-est}.
\corostatest*
\begin{proof}
This result is a corollary of \cref{lemma:stat-est}, applied to $\epsv$, where we note that for $p \geq 1$, by \cref{thm:pow-mean-prop}~item~\ref{thm:pow-mean-prop:contraction} (contraction)
it holds that
\[
\Mean_{p}(\hat{\lv} + \epsv; \wv) \leq \Mean_{p}(\hat{\lv}; \wv) + \norm{\epsv}_{\infty} \ \& \ \Mean_{p}(\bm{0} \vee (\hat{\lv} - \epsv); \wv) \leq \Mean_{p}(\hat{\lv}; \wv) - \norm{\epsv}_{\infty} \enspace.
\]

Now, for the first bound, note that we take $\epsv_ {i} \doteq \frange\sqrt{\frac{\ln \frac{2\NGroups}{\delta}}{2m}}$, and by Hoeffding's inequality and the union bound, for $\Population = \{1, \dots, n\}$, we have $\forall \PopItem: \, \lv'(\PopItem) - \epsv(\PopItem) \leq \lv(\PopItem) \leq \lv'(\PopItem) + \epsv(\PopItem)$ with probability at least $1 - \delta$.
The result then follows via the \emph{power-mean contraction} (\cref{thm:pow-mean-prop} item \ref{thm:pow-mean-prop:contraction}) property.

Similarly, for the second bound, note that we take $\epsv_ {i} \doteq \frac{\frange\ln \frac{2\NGroups}{\delta}}{3m} + %\sup_{i \in 1, \dots, \NGroups}
\sqrt{\frac{2\Var_{\ProbDist_{i}}[\LossFunction] \ln \frac{2\NGroups}{\delta}}{m}}$, which this time follows via Bennett's inequality and the union bound.
Now, we again apply \cref{lemma:stat-est}, noting that $\Mean(\epsv) \leq \Mean_{\infty}(\epsv) = \norm{\epsv}_{\infty}$ (by \emph{power-mean monotonicity}, \cref{thm:pow-mean-prop} \cref{thm:pow-mean-prop:mono}), 
and the rest follows as in the Hoeffding case.
% and the rest is as with the Hoeffding case. % (note that here thus the \rhs{} is $\norm{\epsv}_{\infty}$).
%\todo{improvement: 1 side is direct by concavity convextity: $2n$ could be $n+1$?}
\end{proof}
%The above bound is a straightforward consequence of Hoeffding's inequality; similar variance-sensitive bounds can be shown with Bennett's inequality, and uniform convergence bounds with Rademacher averages.

\subsection{Efficient FPAC-Learning}
\label{sec:appx:proof:efpac}

We now show \cref{thm:aemm-psg}.
\thmaemmpsg*
\begin{proof}
We now show that this subgradient-method construction of $\PACAlgo$ requires $\Poly(\frac{1}{\varepsilon}, \frac{1}{\delta}, \DSeq, \NGroups)$ time to identify an $\varepsilon$-$\delta$-$\Malfare_{p}(\cdot; \cdot)$-optimal $\tilde{\theta} \in \Theta_{\DSeq}$, and thus fair-PAC-learns $(\HC, \LossFunction)$.
This essentially boils down to showing that (1) the empirical malfare objective is \emph{convex} and \emph{Lipschitz continuous},\todo{bounded unneeded} %, and \emph{bounded},
 and (2) that \cref{alg:aemm-psg} runs sufficiently many subgradient-update steps, with appropriate step size, on a sufficiently large training set, to yield the appropriate guarantees, and that each step of the subgradient method, of which there are polynomially many, itself requires polynomial time.

First, note that by \cref{thm:pop-mean-prop} \cref{thm:pop-mean-prop:pmean,thm:pop-mean-prop:pmean-unfair}, we may assume that $\Malfare(\cdot; \cdot)$ can be expressed as a $p$-power mean with $p \geq 1$; thus henceforth we refer to it as $\Malfare_{p}(\cdot; \cdot)$.
Now, recall that the empirical malfare objective (given $\theta \in \Theta_{\DSeq}$ and training sets $\bm{z}_{1:\NGroups}$) is defined as
\[
\Malfare_{p} \bigl( i \mapsto \ERisk(h(\cdot; \theta); \LossFunction, \bm{z}_{i}); \wv \bigr) \enspace.
\]

%\paragraph{Convexity}
We first show that empirical malfare is convex in $\Theta_{\DSeq}$.
By assumption and positive linear closure, $\ERisk(h(\cdot; \theta'); \LossFunction, \bm{z}_{i})$ is convex in $\theta \in \Theta_{\DSeq}$.
The objective of interest is the composition of $\Malfare_{p}(\cdot; \wv)$ with this quantity evaluated on each of $g$ training sets.
By \cref{thm:pow-mean-prop} %item~
\cref{thm:pow-mean-prop:curvature}, $\Malfare_{p}(\cdot; \wv)$ is convex $\forall p \in [1, \infty]$ in $\R_{0+}^{\NGroups}$, and by the monotonicity axiom, it is monotonically increasing.
Composition of a monotonically increasing convex function on $\R_{0+}^{\NGroups}$ with convex functions on $\Theta_{d}$ yields a convex function, thus we conclude the empirical malfare objective is convex in $\Theta_{d}$.\todo{Better proof!}

We now show that empirical malfare is Lipschitz continuous.
Now, note that for any $p \geq 1$, $\wv$,
\[
\forall \lv, \lv': \ \abs{\Malfare_{p}(\lv; \wv) - \Malfare_{p}(\lv'; \wv)}  \leq 1 \norm{\lv - \lv'}_{\infty} \enspace,
\]
i.e., $\Malfare_{p}(\cdot; \wv)$ is $1$-$\norm{\cdot}_{\infty}$-$\abs{\cdot}$-Lipschitz in \emph{empirical risks} (see \cref{thm:pow-mean-prop} \cref{thm:pow-mean-prop:contraction}), and thus by Lipschitz composition, we have Lipschitz property %obeys the Lipschitz condition
\[
\forall \theta, \theta' \in \Theta_{\DSeq}: \ \abs{ \Malfare_{p} \bigl( i \mapsto \ERisk(h(\cdot; \theta); \LossFunction, \bm{z}_{i}); \wv \bigr) - \Malfare_{p} \bigl( i \mapsto \ERisk(h(\cdot; \theta'); \LossFunction, \bm{z}_{i}); \wv \bigr) } \leq \lambda_{\LossFunction}\lambda_{\HC}\norm{\theta - \theta'}_{2} \enspace.
\]

We now show that \cref{alg:aemm-psg} FPAC-learns $(\HC, \LossFunction)$.
As above, take $m \doteq \SampleComplexity_{\UC}({\frac{\varepsilon}{3}}, \frac{\delta}{\NGroups}, \DSeq)$.
Our algorithm shall operate on a training sample $\bm{z}_{1:\NGroups,1:m} \distributed \ProbDist_{1}^{m} \times \dots \times \ProbDist_{\NGroups}^{m}$.

First note that evaluating a subgradient (via forward finite-difference estimation\cyrus{This may be a (fixable) source of error not accounted for?} or automated subdifferentiation) requires $(\dim(\Theta_{\DSeq}) + 1) m$ evaluations of $h(\cdot; \cdot)$, which by assumption is possible in $\Poly(\DSeq, m) = \Poly(\frac{1}{\varepsilon}, \frac{1}{\delta}, \DSeq, \NGroups)$ time.

The subgradient method produces $\tilde{\theta}$ approximating the empirically-optimal $\hat{\theta}$ such that \citep[see][]{shor2012minimization}
\[
f(\tilde{\theta}) \leq f(\hat{\theta}) + \frac{\norm{\theta_{0} - \smash{\hat{\theta}}}^{2}_{2} + \Lambda^{2}\alpha^{2}n}{2 \alpha n} \leq \frac{\Diam^{2}(\Theta_{\DSeq}) + \Lambda^{2}\alpha^{2}n}{2 \alpha n} \enspace,
\]
for $\Lambda$-$\norm{\cdot}_{2}$-$\abs{\cdot}$-Lipschitz objective $f$, thus taking %$\alpha \doteq \frac{\norm{\theta_{0}, \theta^{*}}_{2}}{\Lambda\sqrt{n}}$
$\alpha \doteq \frac{\Diam(\Theta_{\DSeq})}{\Lambda\sqrt{n}}$ 
yields
\[
f(\tilde{\theta}) - f(\hat{\theta}) \leq \frac{\Diam(\Theta_{\DSeq})\Lambda}{\sqrt{n}} \enspace.
\]

\if 0
thus taking $\alpha = n^{-1/2}$ yields
\[
\leq f(\theta^{*}) + \frac{\norm{\theta_{0}, \theta^{*}}^{2}_{2} + \Lambda^{2}}{2\sqrt{n}} \enspace.
\]
\fi

As shown above, $\Lambda = \lambda_{\LossFunction} \lambda_{\HC}$, thus we may guarantee \emph{optimization error}
\[
\varepsilon_{\mathrm{opt}} \doteq f(\hat{\theta}) - f(\theta^{*}) \leq \frac{\varepsilon}{3}
\]
if we take iteration count
\if 0 %Using bad alpha:
\[
n \geq \frac{\Diam(\Theta_{\DSeq})^{4} + \lambda_{\LossFunction}^{4}\lambda_{\HC}^{4}}{\varepsilon^{2}} \in \Poly(\mathsmaller{\frac{1}{\varepsilon}}, \DSeq) \enspace.
\]
\fi
\[
n \geq \frac{9\Diam^{2}(\Theta_{\DSeq})\lambda_{\LossFunction}^{2}\lambda_{\HC}^{2}}{\varepsilon^{2}} = \left(\frac{3\Diam(\Theta_{\DSeq})\lambda_{\LossFunction}\lambda_{\HC}}{\varepsilon}\right)^{2} \in \Poly(\mathsmaller{\frac{1}{\varepsilon}}, \DSeq) \enspace.
\]

As each iteration requires $m \cdot \Poly(\DSeq) \subseteq \Poly(\frac{1}{\varepsilon}, \frac{1}{\delta}, \DSeq, \NGroups)$ time, %the total time complexity of running the subgradient method for $n$ iterations is 
%It thus holds that, via
the subgradient method identifies an ${\frac{\varepsilon}{3}}$-empirical-malfare-optimal $\tilde{\theta} \in \Theta_{\DSeq}$ in\todo{restore detail}
\if 0
\[
n \in \Poly(\frac{1}{\varepsilon}) \Poly(\DSeq) \subseteq \Poly(\frac{1}{\varepsilon}, \frac{1}{\delta}, \DSeq, \NGroups)
\]
subgradient-method steps, each requiring $m \cdot \Poly(\DSeq) \subseteq \Poly(\frac{1}{\varepsilon}, \frac{1}{\delta}, \DSeq, \NGroups)$ time, thus totaling 
\fi
$\Poly(\frac{1}{\varepsilon}, \frac{1}{\delta}, \DSeq, \NGroups)$ time.

As $m$ was selected to ensure $\frac{\epsilon}{3}$-$\frac{\delta}{\NGroups}$ uniform convergence, we thus have that by uniform convergence, and union bound (over $\NGroups$ groups), with probability at least $1 - \delta$ over choice of $\bm{z}_{1:\NGroups}$, we have
\[
\forall i \in \{1, \dots, \NGroups\}, \theta \in \Theta_{\DSeq}: \ \abs{\Malfare_{p}(i \mapsto \ERisk(h(\cdot; {\theta}); \loss, \bm{z}_{i}); \wv) - \Malfare_{p}(i \mapsto \Risk(h(\cdot; {\theta}); \loss, \ProbDist_{i}); \wv)} \leq {\frac{\varepsilon}{3}} \enspace.
\]

Combining \emph{estimation} and \emph{optimization} errors, we get that with probability at least $1 - \delta$, the approximate-EMM-optimal $h(\cdot; \tilde{\theta})$ obeys
\begin{align*}
\Malfare_{p}(i \mapsto \Risk(h(\cdot; \tilde{\theta}); \loss, \ProbDist_{i}); \wv) 
 &\leq \Malfare_{p}(i \mapsto \ERisk(h(\cdot; \tilde{\theta}); \loss, \bm{z}_{i}); \wv) + \mathsmaller{\frac{\varepsilon}{3}} & \\
 &\leq \Malfare_{p}(i \mapsto \ERisk(h(\cdot; \smash{\hat{\theta}}); \loss, \bm{z}_{i}); \wv) + \mathsmaller{\frac{2\varepsilon}{3}} & \\
 &\leq \Malfare_{p}(i \mapsto \ERisk(h(\cdot; \theta^{*}); \loss, \bm{z}_{i}); \wv) + \mathsmaller{\frac{2\varepsilon}{3}} & \\
 &\leq \Malfare_{p}(i \mapsto \Risk(h(\cdot; \theta^{*}); \loss, \ProbDist_{i}); \wv) + \varepsilon \enspace. & \\
\end{align*}

We may thus conclude that $\PACAlgo$ fair-PAC learns $\HC$ with sample complexity $\NGroups m = \NGroups \cdot \SampleComplexity_{\UC}({\frac{\varepsilon}{3}}, \frac{\delta}{\NGroups}, \DSeq)$.
Furthermore, as the entire operation requires polynomial time, we have $(\HC, \ell) \in \PAC^{\Agnostic}_{\Poly}$.
\end{proof}

We now work towards proof of \cref{thm:aemm-covering}.
We begin with a technical lemma deriving relevant properties of the cover employed in the main result.
\todo{Split 2+3 to their own result?}

\begin{lemma}[Group Cover Properties]
\label{lemma:group-cover}
Suppose loss function $\LossFunction$ of bounded codomain (i.e., $\norm{\LossFunction}_{\infty}$ is bounded), %(i.e., $\LossFunction: (\Y \times \Y) \to [0, \norm{\LossFunction}_{\infty}]$), 
hypothesis class $\HC \subseteq \X \to \Y$, and per-group samples $\bm{z}_{1:\NGroups,1:m} \in (\X \times \Y)^{\NGroups \times m}$, letting $\bigcirc_{i=1}^{\NGroups} \bm{z}_{i}$ denote their \emph{concatenation}.
Now define
\[
\ECover_{\cup(1:\NGroups)} \doteq \bigcup_{i=1}^{\NGroups} \ECover \left( \LossFunction \circ \HC, \bm{z}_{i}, \gamma \right)
 \ \ \ \& \ \ \ 
 \ECover_{\circ(1:\NGroups)} \doteq \ECover \left( \LossFunction \circ \HC, \bigcirc_{i=1}^{\NGroups} \bm{z}_{i}, \mathsmaller{\frac{\gamma}{\sqrt{\NGroups}}} \right)
 \enspace.
\]
Then, letting $\ECover_{}$ refer generically to either %$\{\ECover_{\circ}, \ECover_{(1:\NGroups)}\}$, 
$\ECover_{\cup(1:\NGroups)}$ or $\ECover_{\circ(1:\NGroups)}$, the following hold.

%TODO better notation!

%TODO no DSEQ here!

\begin{enumerate}
\item \label{lemma:group-cover:counting} If $\ECover_{}$ is of \emph{minimal cardinality}, then %$\abs{\ECover_{}}$ is $\Poly(\dots TODO)$ minimal, and ?, then $\abs{\ECover_{}} \in \Poly(\dots)$??
%
%In particular, 
\[
\abs{\ECover_{\cup(1:\NGroups)}} \leq \NGroups \Covering(\loss \circ \HC, m, \gamma) \ \ \ \& \ \ \ \abs{\ECover_{\circ(1:\NGroups)}} \leq \Covering(\loss \circ \HC, \NGroups m, \mathsmaller{\frac{\gamma}{\sqrt{\NGroups}}}) \enspace.
\]
%TODO: visual math, in poly?

\item \label{lemma:group-cover:uc} $\displaystyle \sup_{\ProbDist \text{ over } \X \times \Y} \Rade_{m}(\LossFunction \circ \HC, \ProbDist) 
  \leq
    \inf_{\gamma \geq 0} \gamma + \norm{\loss}_{\infty}\sqrt{\frac{\ln \Covering(\LossFunction \circ \HC, %m, 
      \gamma ) }{2m}}$.
    %\underbrace{\gamma}_{\textsc{Discretization}} + 
    %\underbrace{\norm{\loss}_{\infty}\sqrt{\frac{\ln \Covering(\LossFunction \circ \HC, m, \gamma ) }{2m}}}_{\textsc{Entropy}} \in \norm{\loss}_{\infty}\sqrt{\frac{\LandauO(1)\ln(? TODO )}{2m}}$
%TODO: MASSART ENTROPY DISCRETIZATION?

\item \label{lemma:group-cover:sc}
%If $\ln \Covering( \LossFunction \circ \HC, m, \gamma ) \in \Poly( m, \frac{1}{\gamma} )$, 
Suppose $\ln \Covering( \LossFunction \circ \HC, \gamma ) \in \Poly( \frac{1}{\gamma} )$.
Then the uniform-convergence sample-complexity of $\loss \circ \HC$ over $\NGroups$ groups obeys
\begin{align*}
\SampleComplexity_{\UC}(\loss \circ \HC, \varepsilon, \delta, \NGroups) &\doteq \argmin \left\{ m \ \middle| \ \sup_{\ProbDist_{1:\NGroups} \text{ over } (\X \times \Y)^{\NGroups}} \Prob\left(\smash{\max_{i \in 1, \dots, \NGroups}} \smash{\sup_{h \in \HC}} \abs{ \Expect_{\ProbDist_{i}}[\loss \circ h] - \smash{\EExpect_{\bm{z}_{i} \distributed \ProbDist_{i}^{m}}}[\loss \circ h]} > \varepsilon\right) \leq \delta \right\} & \\
  &\leq \ceil*{ \frac{8\norm{\loss}_{\infty}^{2}\ln \left(\smash{\sqrt[4]{\frac{2\NGroups}{\delta}}} \Covering(\LossFunction \circ \HC, \frac{\varepsilon}{4} ) \right)}{\varepsilon^{2}} } & \\
  &\in \LandauO \left( \frac{%\norm{\loss}_{\infty}^{2} 
  \ln \frac{\NGroups \Covering(\LossFunction \circ \HC, \varepsilon) }{\delta}}{\varepsilon^{2}} \right) %& \\
  %&
  \subset \Poly\left( %\norm{\loss}_{\infty}, 
    \frac{1}{\varepsilon}, \exp \frac{1}{\delta}, \exp \NGroups \right)
  \enspace.
  %&\subseteq \Poly( \mathsmaller{\frac{1}{\varepsilon}}, \mathsmaller{\frac{1}{\delta}}, \norm{\loss}_{\infty}, \NGroups) \enspace.
\end{align*}

\todo{Adding $\NGroups$ is new concept.}
\todo{Eq or doteq on UC SC?}
\todo{assume dists?}
\draftnote{note in paper: usually $\exp(\NGroups)$, $\exp(\frac{1}{\delta})$ is OK?}

\item \label{lemma:group-cover:individual} For the sample $\bm{z}_{i}$ associated with each group $i \in 1, \dots, \NGroups$, $\ECover_{}$ is a $\gamma$-uniform-approximation of \emph{empirical risk} $\ERisk(h; \LossFunction, \bm{z}_{i})$, and a $\gamma$-$\ell_{2}$ cover of the \emph{loss family} $\loss \circ \HC$, as
\[
\max_{i \in 1, \dots, \NGroups} \min_{h_{\gamma} \in \ECover_{}} \abs{ \ERisk(h; \LossFunction, \bm{z}_{i}) - \ERisk(h_{\gamma}; \LossFunction, \bm{z}_{i}) }
  \leq \max_{i \in 1, \dots, \NGroups} \min_{h_{\gamma} \in \ECover_{}} \sqrt{\frac{1}{m} \sum_{j=1}^{m} \bigl((\loss \circ h)(\bm{z}_{i,j}) - (\loss \circ h_{\gamma})(\bm{z}_{i,j})\bigr)^{2}}
  \leq %\sqrt{\NGroups}
  \gamma
  \enspace. %; \ \&
\]

\item \label{lemma:group-cover:simultaneous} $\ECover_{\circ(1:\NGroups)}$, but not necessarily $\ECover_{\cup(1:\NGroups)}$, \emph{simultaneously} (across all groups) $%\sqrt{\NGroups}
\gamma$-uniformly-approximates \emph{empirical risk}, and is a $%\sqrt{\NGroups}
\gamma$-$\ell_{2}$ cover of the \emph{loss family} $\loss \circ \HC$, as
\[
\min_{h_{\gamma} \in \ECover_{\circ(1:\NGroups)}} \!\! \max_{i \in 1, \dots, \NGroups} \abs{ \ERisk(h; \LossFunction, \bm{z}_{i}) - \ERisk(h_{\gamma}; \LossFunction, \bm{z}_{i}) }
  \leq  \!\!\! \!\! \min_{h_{\gamma} \in \ECover_{\circ(1:\NGroups)}}  \!\!  \max_{i \in 1, \dots, \NGroups} \! \sqrt{ \! \frac{1}{m} \! \sum_{j=1}^{m} \bigl((\loss \circ h)(\bm{z}_{i,j}) - (\loss \circ h_{\gamma})(\bm{z}_{i,j})\bigr)^{2}}
  \leq %\sqrt{\NGroups}
  \gamma
  \enspace.
\]
\end{enumerate}

\end{lemma}
\begin{proof}
We first show \cref{lemma:group-cover:counting,lemma:group-cover:uc,lemma:group-cover:sc}, followed by a key intermediary relating risk values and $\ell_{2}$ distances, and close by showing \cref{lemma:group-cover:individual,lemma:group-cover:simultaneous}.

We begin with \cref{lemma:group-cover:counting}.
Both bounds follow directly from the definition of uniform covering numbers. %$\Covering(
\if 0
TODO

%
TODO To further reinforce this point, note that ... whose minimal size could not exceed requiring only a cover of size  $\abs{\dots} \in [1, \NGroups]\Covering(\LossFunction \circ \HC, m, \gamma)$, i.e., no more than \emph{linear growth} in $\NGroups$, whereas, ... has size ..., whose growth depends on ?, which by assumption is only \emph{polynomial} in $\NGroups$ (TODO: in practice, covering numbers are ??? in $\gamma$?
%
\fi

We now show \cref{lemma:group-cover:uc}.
This result follows via a standard sequence of operations over the Rademacher average.
In particular, observe
\begin{align*}
\sup_{\ProbDist \text{ over } \X \times \Y} \Rade_{m}(\LossFunction \circ \HC, \ProbDist) \hspace{-2cm} & \hspace{2cm} = \sup_{\ProbDist \text{ over } \X \times \Y} \Expect_{\bm{z} \distributed \ProbDist^{m}} \left[ \ERade_{m}(\LossFunction \circ \HC, \bm{z}) \right] & \textsc{Definition of $\Rade$} \\
 &\leq \sup_{\ProbDist \text{ over } \X \times \Y} \Expect_{\bm{z} \distributed \ProbDist^{m}} \left[ \inf_{\gamma \geq 0} \gamma + \ERade_{m}(\Cover^{*}(\LossFunction \circ \HC, \bm{z}, \gamma), \bm{z}) \right] & \textsc{Discretization} \\
 &\leq \inf_{\gamma \geq 0} \gamma + \sup_{\ProbDist \text{ over } \X \times \Y} \Expect_{\bm{z} \distributed \ProbDist^{m}} \left[ \norm{\loss}_{\infty}\sqrt{\frac{\ln \abs{\Cover^{*}(\LossFunction \circ \HC, \bm{z}, \gamma ) }}{2m}} \right] & \textsc{Massart's Inequality} \\
 %&\leq \inf_{\gamma \geq 0} \gamma + \norm{\loss}_{\infty}\sqrt{\frac{\ln \Covering(\LossFunction \circ \HC, m, \gamma ) }{2m}} \enspace. & \hspace{-2cm} \textsc{Definition of Uniform Covering Number} \\ %With m
 &\leq \inf_{\gamma \geq 0} \gamma + \norm{\loss}_{\infty}\sqrt{\frac{\ln \Covering(\LossFunction \circ \HC, \gamma ) }{2m}} \enspace, & \hspace{-2cm} \textsc{Definition of $\Covering$%Uniform Covering Number
    } \\ %Without m
\end{align*}
where the \textsc{Massart's Inequality} step follows via \emph{Massart's finite class inequality} \citep[lemma 1]{massart2000some}\todo{check number?}, and the \textsc{Discretization} step via \emph{Dudley's discretization argument}.\todo{cite Dudley}

\bigskip

We now %note that
show \cref{lemma:group-cover:sc}. % follows directly from \cref{lemma:group-cover:uc}, paired with the assumption that $\ln \Covering(\LossFunction \circ \HC, m, \gamma ) \in \Poly(\dots)$, and TODO MCD!
%In particular,
By the \emph{symmetrization inequality}\todo{cite}, and a 2-tailed application of McDiarmid's bounded difference inequality \citep{mcdiarmid1989method}, where changing any $\bm{z}_{i,j}$ has bounded difference $\frac{\norm{\loss}_{\infty}}{m}$, we have that
\[
\forall i: \ \Prob\left( \sup_{h \in \HC} \abs{ \Expect_{\ProbDist_{i}}[\loss \circ h] - \EExpect_{\bm{z}_{i} \distributed \ProbDist_{i}^{m}}[\loss \circ h]} > 2\Rade_{m}(\loss \circ \HC, \ProbDist_{i}) + {\norm{\loss}_{\infty}\sqrt{\frac{\ln \frac{2}{\delta}}{2m}}}%_{\mathclap{\textsc{SD / $\Malfare(\cdot; \cdot)$ Tail Bound}}}
 \right) \leq \delta
\]
thus by union bound over $\NGroups$ groups,
%and applying \cref{lemma:group-cover:uc} to bound Rademacher averages, 
we have
\[
\Prob\left( \max_{i \in 1, \dots, \NGroups} \sup_{h \in \HC} \abs{ \Expect_{\ProbDist_{i}}[\loss \circ h] - \EExpect_{\bm{z}_{i} \distributed \ProbDist_{i}^{m}}[\loss \circ h]} > \sup_{\ProbDist \text{ over } \X \times \Y} 2\Rade_{m}(\loss \circ \HC, \ProbDist) + \norm{\loss}_{\infty}\sqrt{\frac{\ln \frac{2\NGroups}{\delta}}{2m}} \right) \leq \delta \enspace.
\]

Now, let \emph{estimation error} bound $\displaystyle \epsilon_{\mathrm{est}} \doteq \, \smash{\sup_{\mathclap{\ProbDist \text{ over } \X \times \Y}}} \ \ 2\Rade_{m}(\loss \circ \HC, \ProbDist) + \norm{\loss}_{\infty}\mathsmaller{\sqrt{\frac{\ln \frac{2\NGroups}{\delta}}{2m}}}$, and observe that via \cref{lemma:group-cover:uc},
\[
\epsilon_{\mathrm{est}} \leq \inf_{\gamma \geq 0} 2\gamma + 2\norm{\loss}_{\infty}\sqrt{\frac{\ln \Covering(\LossFunction \circ \HC, \gamma ) }{2m}} + \norm{\loss}_{\infty}\sqrt{\frac{\ln \frac{2\NGroups}{\delta}}{2m}} \enspace.
\]
\if 0
\begin{align*}
 %\implies\\
  \varepsilon
  &\leq \inf_{\gamma \geq 0} 2\gamma + 2\norm{\loss}_{\infty}\sqrt{\frac{\ln \Covering(\LossFunction \circ \HC, m, \gamma ) }{2m}} + \norm{\loss}_{\infty}\sqrt{\frac{\ln \frac{2\NGroups}{\delta}}{2m}}%}_{\mathclap{\textsc{SD / $\Malfare(\cdot; \cdot)$ Tail Bound}}} 
  & \textsc{%\Cref{lemma:group-cover}~
    \Cref{lemma:group-cover:uc}} \\ %\textsc{See Above} \\
  &\leq \inf_{\gamma \geq 0} 2\gamma + 2\norm{\loss}_{\infty}\sqrt{\frac{\ln \Covering(\LossFunction \circ \HC, \gamma ) }{2m}} + \norm{\loss}_{\infty}\sqrt{\frac{\ln \frac{2\NGroups}{\delta}}{2m}} \enspace. & \Covering(\LossFunction \circ \HC, m, \gamma ) \leq \Covering(\LossFunction \circ \HC, \gamma ) \\
\end{align*}
\fi
%
From here, we solve for an upper-bound on sample-size $m$ to get
\begin{align*}
\SampleComplexity_{\UC}(\loss \circ \HC, \varepsilon, \delta, \NGroups)
  &\leq \ceil*{ \inf_{\gamma \geq 0} \frac{4\norm{\loss}^{2}_{\infty} \ln \Covering(\LossFunction \circ \HC, \gamma ) + \norm{\loss}^{2}_{\infty}\ln \frac{2\NGroups}{\delta}}{2(\varepsilon - 2\gamma)^{2}} } %& TODO combine w/ next? \\
  %&=
  = \ceil*{ \frac{2\norm{\loss}_{\infty}^{2}\ln \left(\smash{\sqrt[4]{\frac{2\NGroups}{\delta}}} \Covering(\LossFunction \circ \HC, \gamma) \right)}{(\varepsilon - 2\gamma)^{2}} } \hspace{-3cm} & \\
  &\leq \ceil*{ \frac{8\norm{\loss}_{\infty}^{2}\ln \left( \smash{\sqrt[4]{\frac{2\NGroups}{\delta}}} \Covering(\LossFunction \circ \HC, \frac{\varepsilon}{4} ) \right)}{\varepsilon^{2}} } & \textsc{Set $\gamma = \frac{\varepsilon}{4}$} \\
  &\in \LandauO \left( \frac{%\norm{\loss}_{\infty}^{2} 
    \ln \frac{\NGroups \Covering(\LossFunction \circ \HC, \varepsilon) }{\delta}}{\varepsilon^{2}} \right) %& \\
  %&
  \subset \Poly\left( %\norm{\loss}_{\infty}, 
  \frac{1}{\varepsilon}, \exp \frac{1}{\delta}, \exp \NGroups \right) \enspace. & \Covering(\LossFunction \circ \HC, \varepsilon \bigr) \in \Poly \frac{1}{\varepsilon} \\
\end{align*}
%Now, substitute in $\gamma %\propto \varepsilon 
%= \frac{\varepsilon}{3}$ to get the result.

\bigskip

We now show an intermediary which immediately implies the left inequalities of both \cref{lemma:group-cover:individual,lemma:group-cover:simultaneous}.
In particular, we may relate these empirical risk gaps to (size-normalized) $\ell_{2}$ distance, as (for each $i$) we have $\forall h \in \HC, h_{\gamma} \in \ECover_{}$ that
\if 0
\begin{align*}
\sqrt{\frac{1}{m}\sum_{j=1}^{m} \bigl( (\loss \circ h) (\bm{z}_{i,j}) - (\loss \circ h_{\gamma}) (\bm{z}_{i,j}) \bigr)^{2} }
 &\geq \frac{1}{m}\sum_{j=1}^{m} \abs{ (\loss \circ h) (\bm{z}_{i,j}) - (\loss \circ h_{\gamma}) (\bm{z}_{i,j}) } & \textsc{?} \\
 &\geq \abs{ \ERisk(h; \LossFunction, \bm{z}_{i}) - \ERisk(h_{\gamma}; \LossFunction, \bm{z}_{i}) } & \\
 &>\sqrt{\NGroups} \gamma \enspace. & \textsc{Assumption} \\
\end{align*}
TODO
\fi
\[
\scalebox{0.985}{$\displaystyle
\abs{ \ERisk(h; \LossFunction, \bm{z}_{i}) \! - \! \ERisk(h_{\gamma}; \LossFunction, \bm{z}_{i}) } \! \leq \! \frac{1}{m} \! \sum_{j=1}^{m} \abs{ (\loss \circ h) (\bm{z}_{i,j}) \! - \! (\loss \circ h_{\gamma}) (\bm{z}_{i,j}) } \! \leq \! \sqrt{ \! \frac{1}{m} \! \sum_{j=1}^{m} \bigl( (\loss \circ h) (\bm{z}_{i,j}) \! - \! (\loss \circ h_{\gamma}) (\bm{z}_{i,j}) \bigr)^{\!2} \!} \! \enspace.
$}
\]
Here the last inequality holds since we divide $m$ \emph{inside the $\sqrt{\cdot}$}.
The opposite inequality holds for standard $\ell_{1}$ and Euclidean distance, where $m$ is not divided, essentially because the $\ell_{1}$ and $\ell_{2}$ distances differ by up to a factor $\sqrt{m}$, but this form may be familiar as the relationship between the \emph{mean} and \emph{root mean square} errors. %direction should be familiar from the \emph.
The unconvinced reader may note that this size-normalized $\ell_{2}$ distance is in fact the (unweighted) $p=2$ power-mean, and thus this step follows via \cref{thm:pow-mean-prop}~\cref{thm:pow-mean-prop:mono}.

We now show the right inequality of \cref{lemma:group-cover:individual}.
Note that for the case of $\ECover_{\cup(1:\NGroups)}$, the result is almost tautological, as it holds per group by the union-based construction of $\ECover_{\cup(1:\NGroups)}$.
The case of $\ECover_{\circ(1:\NGroups)}$ is more subtle, but we defer its proof to the final item, as it then follows as an immediate consequence of the \emph{max-min inequality}, i.e., $\forall h \in \HC$,\todo{Check scaling in thesis}%
\[
\scalebox{0.99}{$\displaystyle
\max_{i \! \in \! 1, \dots, \NGroups} \! \min_{h_{\gamma} \! \in \! \ECover_{\!\circ\!(1:\NGroups)}} \!\!\!\! \sqrt{ \! \frac{1}{m} \!\! \sum_{j=1}^{m} \bigl((\loss \circ h)(\bm{z}_{i,j}) \! - \! (\loss \circ h_{\gamma})(\bm{z}_{i,j})\bigr)^{\!2}} \leq \!\!\!\!  \min_{h_{\gamma} \! \in \! \ECover_{\!\circ\!(1:\NGroups)}} \!\! \max_{i \in 1, \dots, \NGroups} \!\! \sqrt{ \! \frac{1}{m} \!\! \sum_{j=1}^{m} \bigl((\loss \circ h)(\bm{z}_{i,j}) \! - \! (\loss \circ h_{\gamma})(\bm{z}_{i,j})\bigr)^{\!2}} \enspace.
$}
\]

\bigskip

We now show \cref{lemma:group-cover:simultaneous}.
%\todo{argmin or something that denotes breaking ties?}
\if 0
We first show that
%Now, note that
$\ECover_{\circ(1:\NGroups)}$ is also a $\sqrt{\NGroups}\gamma$-$\ell_{2}$ cover of each $\bm{z}_{i}$ (simultaneously). %(observe this via the contrapositive; if it were not so, $\ECover$ would not be a $\gamma$-$\ell_{2}$ cover of the concatenation $\bm{z}$).
%
\fi
Suppose (by way of contradiction) %$\exists$ some $i \in \{1, \dots, \NGroups\}$, 
that there exists some $h \in \HC$ such that
\[
%\Cover(\LossFunction \circ \HC, \bm{z}, \gamma) \text{ s.t.\ } 
%\exists h \in \HC: 
\min_{h_{\gamma} \in \ECover_{\circ(1:\NGroups)}} \max_{i \in 1, \dots, \NGroups} \sqrt{\frac{1}{m} \sum_{j=1}^{m} \bigl((\loss \circ h)(\bm{z}_{i,j}) - (\loss \circ h_{\gamma})(\bm{z}_{i,j})\bigr)^{2}} > \gamma \enspace.
\]
One then need only consider the summands associated with a maximal $i$ to observe that this implies %It then holds that
\[
\min_{h_{\gamma} \in \ECover_{\circ(1:\NGroups)}}  \sqrt{\frac{1}{m\NGroups} \sum_{i=1}^{\NGroups} \sum_{j=1}^{m} \bigl((\loss \circ h)(\bm{z}_{i,j}) - (\loss \circ h_{\gamma})(\bm{z}_{i,j})\bigr)^{2}} > \frac{\gamma}{\sqrt{\NGroups}} \enspace,
\]
thus $\ECover_{\circ(1:\NGroups)}$ %$\LossFunction \circ \HC_{\DSeq,\gamma}$
 is not a $\frac{\gamma}{\sqrt{\NGroups}}$-$\ell_{2}$ cover of $\bigcirc_{i=1}^{\NGroups} \bm{z}_{i}$, which contradicts its very definition.
We thus conclude %may conclude that
\[
%\Cover(\LossFunction \circ \HC, \bm{z}, \gamma) \text{ s.t.\ } 
\forall h \in \HC: 
\min_{h_{\gamma} \in \ECover_{\circ(1:\NGroups)}} \max_{i \in 1, \dots, \NGroups} \sqrt{\frac{1}{m} \sum_{j=1}^{m} \bigl((\loss \circ h)(\bm{z}_{i,j}) - (\loss \circ h_{\gamma})(\bm{z}_{i,j})\bigr)^{2}} \leq \gamma \enspace.
\]
\end{proof}

With \cref{lemma:group-cover} in hand, we are now ready to show \cref{thm:aemm-covering}.
\thmaeemcovering*
\begin{proof}
We now constructively show the existence of a fair-PAC-learner $\PACAlgo$ for $(\LossFunction, \HC)$ over domain $\X$ and codomain $\Y$.
As in \cref{thm:aemm-psg}, we first note that by \cref{thm:pop-mean-prop} \cref{thm:pop-mean-prop:pmean,thm:pop-mean-prop:pmean-unfair}, under the conditions of FPAC learning, this reduces to showing that we can learn any malfare concept $\Malfare_{p}(\cdot; \cdot)$ that is a $p$-power mean with $p \geq 1$.

We first assume a training sample $\bm{z}_{1:\NGroups,1:m} \distributed \ProbDist_{1}^{m} \times \dots \times \ProbDist_{\NGroups}^{m}$, i.e., a collection of $m$ draws from each of the $\NGroups$ groups.
In particular, we shall select $m$ to guarantee that the \emph{estimation error} for the malfare %over each group
 does not exceed $\frac{\varepsilon}{3}$ with probability at least $1 - \delta$, i.e., we require that with said probability,
\[
\epsilon_{\mathrm{est}} \doteq \abs{\Malfare_{p}(i \mapsto \ERisk(h; \loss, \bm{z}_{i}); \wv) - \Malfare_{p}(i \mapsto \Risk(h; \loss, \ProbDist_{i}); \wv)} \leq \frac{\varepsilon}{3} \enspace.
\]

Now, note that by \cref{thm:pow-mean-prop}~\cref{thm:pow-mean-prop:contraction} (contraction), we have
\[
\abs{\Malfare_{p}(i \mapsto \ERisk(h; \loss, \bm{z}_{i}); \wv) - \Malfare_{p}(i \mapsto \Risk(h; \loss, \ProbDist_{i}); \wv)} 
  %\leq \max_{i \in 1, \dots, \NGroups} \  \abs{ \Risk(\hat{h}; \loss, \ProbDist_{i}); \wv) - \ERisk(\hat{h}; \loss, \bm{z}_{i}); \wv) } \enspace,
  \leq \max_{i \in 1, \dots, \NGroups} \sup_{h \in \HC_{\DSeq}} \abs{ \Risk({h}; \loss, \ProbDist_{i}); \wv) - \ERisk({h}; \loss, \bm{z}_{i}); \wv) }
  \enspace,
\]
and by \cref{lemma:group-cover}~\cref{lemma:group-cover:sc}, 
%If $\ln \Covering( \LossFunction \circ \HC, m, \gamma ) \in \Poly( m, \frac{1}{\gamma} )$, 
%Suppose $\ln \Covering( \LossFunction \circ \HC, \gamma ) \in \Poly( \frac{1}{\gamma} )$.
%Then
%the uniform-convergence sample-complexity of $\loss \circ \HC$ obeys
a sample of size
\[
m = \left\lceil \frac{81\norm{\loss}_{\infty}^{2}\ln \left(\smash{\sqrt[4]{\frac{2\NGroups}{\delta}}} \Covering(\LossFunction \circ \HC, \frac{\varepsilon}{12} ) \right)}{\varepsilon^{2}} \right\rceil %& \\
  \in \LandauO \left( \frac{%\norm{\loss}_{\infty}^{2} 
    \ln \frac{\NGroups \Covering(\LossFunction \circ \HC, \varepsilon) }{\delta}}{\varepsilon^{2}} \right) %& \\
  %&
  \subset \Poly\left( %\norm{\loss}_{\infty}, 
  \frac{1}{\varepsilon}, \exp \frac{1}{\delta}, \exp \NGroups \right) 
  %\enspace.
\]
\if 0
\begin{align*}
\SampleComplexity_{\UC}(\loss \circ \HC, \varepsilon, \delta, \NGroups) &\doteq \argmin \left\{ m \middle| TODO SUP \Prob\left(\max_{i \in 1, \dots, \NGroups} \sup_{h \in \HC} \abs{ \Expect_{\ProbDist_{i}}[\loss \circ h] - \EExpect_{\bm{z}_{i} \distributed \ProbDist_{i}^{m}}[\loss \circ h]} > \varepsilon\right) \leq \delta\right\} & \\
  &\leq \frac{8\norm{\loss}_{\infty}^{2}\ln \left(\sqrt[4]{\frac{2\NGroups}{\delta}} \Covering(\LossFunction \circ \HC, \frac{\varepsilon}{4} ) \right)}{\varepsilon^{2}} & \\
  &\in \LandauO \left( \frac{\norm{\loss}_{\infty}^{2} \ln \frac{\NGroups \Covering(\LossFunction \circ \HC, \varepsilon) }{\delta}}{\varepsilon^{2}} \right) %& \\
  %&
  \subset \Poly\left(\norm{\loss}_{\infty}, \frac{1}{\varepsilon}, %\exp 
    \frac{1}{\delta}, %\exp
    \DSeq, 
    \NGroups \right) \enspace.
  %&\subset \Poly( \mathsmaller{\frac{1}{\varepsilon}}, \mathsmaller{\frac{1}{\delta}}, \norm{\loss}_{\infty}, \NGroups) \enspace.
\end{align*}
\fi
suffices to ensure that % with probability at least $1 - \delta$,
\[
\Prob\left(\max_{i \in 1, \dots, \NGroups} \sup_{h \in \HC_{\DSeq}} \abs{ \Risk({h}; \loss, \ProbDist_{i}); \wv) - \ERisk({h}; \loss, \bm{z}_{i}); \wv) } > \frac{\varepsilon}{3} \right) \leq \delta \enspace,
\]
thus guaranteeing the stated estimation error bound.

\bigskip

With our sample size and \emph{estimation error} guarantee, we now define the \emph{learning algorithm} and bound its \emph{optimization error}.
Take \emph{cover precision} $\gamma \doteq \frac{\varepsilon}{3\sqrt{\NGroups}}$.
%It will be convenient to cover the \emph{concatenation} of all $\NGroups m$ samples, so we %take $\bm{z}^{C} \in \X^{\NGroups m}$ to refer to this concatenation.
%TODO just abuse notation!
By assumption, for each $\DSeq \in \N$, we may enumerate a $\gamma$-cover $\ECover(\LossFunction \circ \HC_{\DSeq}, \bigcirc_{i=1}^{\NGroups} \bm{z}_{i}, \gamma)$, where $\bigcirc_{i=1}^{\NGroups} \bm{z}_{i}$ denotes the \emph{concatenation} of each $\bm{z}_{i}$, in $\Poly(\NGroups m, \frac{1}{\gamma}, \DSeq) %= \Poly(\NGroups m, \frac{6\NGroups}{\varepsilon}, \DSeq) 
= \Poly(\frac{1}{\varepsilon}, \frac{1}{\delta}, \DSeq, \NGroups)$ time.
%Here we abuse notation by letting $\bm{z}$ to refer to the \emph{concatenation} of all $\NGroups m$ (per-group) samples, and 
For the remainder of this proof, we refer to this cover as $\ECover$.
 % (we shall later show that $m \in \Poly(\frac{1}{\varepsilon}, \frac{1}{\delta}, \DSeq, \NGroups)$ to complete the desideratum).
%, which is $\Poly(m, \frac{1}{\varepsilon})$, as $3D_{\Y}$ is constant.
%\todo{suffices to cover w.h.p.?}

%\todo{Detail: need to cover $\cup_{i=1}^{\NGroups} \bm{z}_{i}$, with $\gamma$ precision each, which means sth like $\gamma/\NGroups$ precision?  Alternatively, take a \emph{union} over per-group covers?} 
Now, we take the learning algorithm to be \emph{empirical malfare minimization} over $\ECover$.
Let
\[
\hat{h} \doteq \argmin_{h_{\gamma} \in \ECover} \Malfare_{p}(i \mapsto \ERisk(h_{\gamma}; \LossFunction, \bm{z}_{i}); \wv) %$. \& 
\ \ \ \& \ \ \ \tilde{h} \doteq \argmin_{h \in \HC_{\DSeq}} \Malfare_{p}(i \mapsto \ERisk(h; \LossFunction, \bm{z}_{i}); \wv) \enspace,%$. \&
\]
where ties may be broken arbitrarily.
%Similarly, we may conclude the stronger statement,
Note that via standard covering properties, %we thus bound the
that the \emph{optimization error} is bounded as
\begin{align*}
\varepsilon_{\mathrm{opt}}
  &\doteq \Malfare_{p}\bigl( i \mapsto \ERisk(\hat{h}; \LossFunction, \bm{z}_{i}); \wv \bigr) - \Malfare_{p}\bigl( i \mapsto \ERisk(\tilde{h}; \LossFunction, \bm{z}_{i}); \wv \bigr) & \textsc{Definition} \\
  &= \sup_{{h} \in \HC_{\DSeq}} \min_{h_{\gamma} \in \ECover} \Malfare_{p}\bigl( i \mapsto \ERisk(h_{\gamma}; \LossFunction, \bm{z}_{i}); \wv \bigr) - \Malfare_{p}\bigl( i \mapsto \ERisk({h}; \LossFunction, \bm{z}_{i}); \wv \bigr) \hspace{-0.2cm} & \begin{tabular}{r} \textsc{Properties of Suprema} \\ \textsc{Definition of} $\hat{h}, \tilde{h}$ \\ \end{tabular} \!\!\! \\
  %\leq \NGroups \gamma  = \frac{\varepsilon}{6}
  &\leq \sup_{h \in \HC_{\DSeq}} \min_{h_{\gamma} \in \ECover} \max_{i \in \{1, \dots, \NGroups\}} \abs{ \ERisk(h_{\gamma}; \LossFunction, \bm{z}_{i}) - \ERisk(h; \LossFunction, \bm{z}_{i}) } & \hspace{-1.7cm} \textsc{\Cref{thm:pow-mean-prop} \Cref{thm:pow-mean-prop:contraction} (Contraction)} \\ %\hspace{-2cm}\begin{tabular}{r} \textsc{\Cref{lemma:stat-est}} \\ \textsc{\Cref{thm:pow-mean-prop} \Cref{thm:pow-mean-prop:contraction} (Contraction)} \\ \end{tabular} \!\!\! \\
  &\leq \sqrt{\NGroups}\gamma = \frac{\varepsilon}{3} \enspace. & \textsc{\Cref{lemma:group-cover} \Cref{lemma:group-cover:simultaneous}} \\
\end{align*}

This controls for \emph{optimization error} between the true and approximate EMM solutions $\tilde{h}$ and %approximate EMM solution 
$\hat{h}$.
\todo{Use $\{$ on 1,\dots,g? Inconsistent now.}

\bigskip

We now combine the optimization and estimation error inequalities, letting
\[
h^{*} \doteq \argmin_{h \in \HC_{\DSeq}} \Malfare_{p}(i \mapsto \Risk(h; \LossFunction, \ProbDist_{i}); \wv) \enspace,
\]
denote the true \emph{malfare optimal} solution, over \emph{distributions} rather than \emph{samples}, breaking ties arbitrarily.
We then derive % take $\epsilon_{m} \doteq \Malfare_{p}(i \mapsto \Risk(h^{*}; \loss, \ProbDist_{i}); \wv) - \Malfare_{p}(i \mapsto \Risk(\hat{h}; \loss, \ProbDist_{i}); \wv)$, and note
\begin{align*}
  %\epsilon_{m}
  \Malfare_{p}(i \mapsto \Risk(h^{*}; \loss, \ProbDist_{i}); \wv) - \Malfare_{p}(i \mapsto \Risk(\hat{h}; \loss, \ProbDist_{i}); \wv)
  \hspace{-6cm} & \hspace{6cm} \\
    &= \left(\Malfare_{p}(i \mapsto \Risk(h^{*}; \loss, \ProbDist_{i}); \wv) -  \Malfare_{p}(i \mapsto \ERisk(h^{*}; \loss, \bm{z}_{i}); \wv)\right) & \leq \epsilon_{\mathrm{est}} \\
    %& \ \ + \left(\Malfare_{p}(i \mapsto \ERisk(h^{*}; \loss, \bm{z}_{i}); \wv) -  \Malfare_{p}(i \mapsto \ERisk(h^{*}; \loss, \bm{z}_{i}); \wv)\right) & \leq ? \\
    & \ \ + \left(\Malfare_{p}(i \mapsto \ERisk(h^{*}; \loss, \bm{z}_{i}); \wv) -  \Malfare_{p}(i \mapsto \ERisk(\tilde{h}; \loss, \bm{z}_{i}); \wv)\right) & \leq 0 \\
    & \ \ + \left(\Malfare_{p}(i \mapsto \ERisk(\tilde{h}; \loss, \bm{z}_{i}); \wv) - \Malfare_{p}(i \mapsto \ERisk(\hat{h}; \loss, \bm{z}_{i}); \wv) \right) & \leq \epsilon_{\mathrm{opt}} \\
    & \ \ + \left(\Malfare_{p}(i \mapsto \ERisk(\hat{h}; \loss, \bm{z}_{i}); \wv) - \Malfare_{p}(i \mapsto \Risk(\hat{h}; \loss, \ProbDist_{i}); \wv)\right) & \leq \epsilon_{\mathrm{est}} \\
  &\leq \epsilon_{\mathrm{opt}} + 2\epsilon_{\mathrm{est}} = \frac{\varepsilon}{3} + \frac{2\varepsilon}{3} = \varepsilon \enspace. & \textsc{See Above} \\
  %&\leq \frac{5}{6}\varepsilon + 4\norm{\loss}_{\infty}\sqrt{\frac{\ln \bigl( 2\Covering(\LossFunction \circ \HC_{\DSeq}, m, \gamma) \bigr)}{2m}} + 4\norm{\loss}_{\infty}\sqrt{\frac{\ln \frac{\NGroups}{\delta}}{2m}} \enspace. & \\
  %&\leq \frac{5}{6} \varepsilon + 8\norm{\loss}_{\infty} \sqrt{\frac{\ln \Covering(\LossFunction \circ \HC_{\DSeq}, m, \gamma)\ln \frac{\NGroups}{\delta}}{2m}} \enspace. & \forall a, b \geq 1: \ \sqrt{a} + \sqrt{b} \leq 2\sqrt{ab} \\
\end{align*}
%
%\todo{Last step assumes $\delta$ is small.}
\if 0
Now, select the smallest $m \in \N$ such that %the \rhs{} is no greater than $\varepsilon$
$\epsilon_{m} \leq \varepsilon$, and note that
\[
m \leq \left\lceil 1152 \norm{\loss}^{2}_{\infty} \ln \left( \frac{2\NGroups\Covering(\LossFunction \circ \HC_{\DSeq}, m, \gamma)}{\delta} \right) \right\rceil
  %\in \LandauO\left( \right) \subseteq 
  \in \Poly\bigl( \mathsmaller{\frac{1}{\varepsilon}}, \ln \mathsmaller{\frac{\NGroups}{\delta}}, \ln \Covering(\LossFunction \circ \HC_{\DSeq}, m, \gamma) \bigr)
  \subset \Poly( \mathsmaller{\frac{1}{\varepsilon}}, \mathsmaller{\frac{1}{\delta}}, \NGroups, \DSeq) \enspace,
\]
TODO: no last step!  Or is it OK BC of log?  Usually only need cover to grow in either $m$ or $\frac{1}{\gamma}$...

essentially because, by assumption, $\ln \Covering(\LossFunction \circ \HC_{\DSeq}, m, \gamma) \in \ln \Poly(m, \frac{1}{\gamma}, \DSeq)$\todo{Better notation for $\log \circ \Poly$?  Just $\LandauO \log(\dots)$?}, and $\norm{\loss}_{\infty}$ is constant in this context. % \subseteq  \Poly(\dots)$, 
%and the McDiarmid tail bound terms.
%Now pick $m$ as smallest $m$ s.t. $\sqrt{\dots} = \eps/3$, wh.??? $m \in \Poly(?)$.
We run the learning algorithm, taking per-group sample size $m$, which as advertised is polynomial in all required parameters.
\fi
%
We thus conclude that this algorithm produces an $\varepsilon$-$\Malfare_{p}(\cdot; \cdot)$ optimal solution with probability at least $1 - \delta$, and furthermore both the sample complexity and time complexity of this algorithm are $\Poly(\frac{1}{\varepsilon}, \frac{1}{\delta}, \NGroups, \DSeq)$.
Hence, as we have constructed a polynomial-time fair-PAC learner for $(\HC, \LossFunction)$, we may conclude $(\HC, \LossFunction) \in \FPAC_{\Poly}$.
%
%Dudley discretization argument (essentially Massart's lemma).
\todo{Are conditions applied by UC always via Sudakov?}
\todo{assume computability to evaluate}
%\fi
%
\end{proof}

\pagebreak[3]

\section{Experimental Setup and Extensions}
\label{appx:exper}

%\section{Additional Experiments and Setup}
\label{sec:appx:exper}

\subsection{Data, Preprocessing, and Experimental Setup}
\label{sec:appx:exper:setup}

All experiments are conducted on the \texttt{adult} dataset, derived from the 1994 US Census database, and obtained from the \href{https://archive.ics.uci.edu/ml/datasets/adult}{UCI repository} \citep{Dua:2021}, where it was donated by Ronny Kohavi and Barry Becker. %, and is free to use for all research purposes.
This dataset has $m=48842$ instances, and we used a $90\%:10\%$ training:test split.
The task has binary target variable \texttt{income}, $6$ numeric features, and $8$ categorical features, including \texttt{race} split into $5$ ethnoracial groups, and \texttt{gender} split into $2$ gender groups.
In each experiment, the target and protected group are omitted from the feature set, the remaining categorical features are 1-hot encoded, and all $\DSeq$ features are $z$-score normalized.

All experiments are with $\lambda$-$\ell_{2}$-norm constrained linear predictors, i.e., the hypothesis class is
\[
\HC \doteq \left\{ h(\vec{x}; \vec{\theta}) = \vec{x} \cdot \vec{\theta}  \ \middle| \ \vec{\theta} \in \R^{\DSeq}, \ \norm{\vphantom{\theta}\smash{\vec{\theta}}}_{2} \leq \lambda \right\} \enspace.
\]

The output of this hypothesis class is real-valued, but for this binary classification task, we take $\Y = \pm 1$, so the loss function is selected to reify this value with a semantic classification interpretation.
The 0-1 loss (for hard classification) is defined as
\[
\loss_{\mathrm{01}}(y, h(\vec{x}; \vec{\theta}))
 %= \loss_{\mathrm{hinge}}(y, \hat{y}) 
 = 1 - y\sgn(\vec{x} \cdot \vec{\theta}) \enspace,
\]
which is readily interpreted in a decision-theoretic sense, but is generally computationally intractable to optimize.
The SVM objective is generally stated in terms of the \emph{hinge loss}, which acts as a convex relaxation of the 0-1 loss.  %i.e., 
The hinge-loss is defined as
\[
\loss_{\mathrm{hinge}}(y, h(\vec{x}; \vec{\theta}))
 %= \loss_{\mathrm{hinge}}(y, \hat{y})
 %= \max(0, 1 - y\hat{y})
 = \max(0, 1 - y (\vec{x} \cdot \vec{\theta})) \enspace,
\]
which is of course \emph{convex} in $\vec{\theta}$, and obeys $\loss_{\mathrm{01}}(y, h(\vec{x}; \vec{\theta})) \leq \loss_{\mathrm{hinge}}(y, h(\vec{x}; \vec{\theta}))$.
Finally, the logistic-regression cross-entropy loss (measured in nats) is %defined as
%The logistic regression objective is 
(see, e.g., ch.~9.3 of \citep{shalev2014understanding})
\[
\loss_{\mathrm{LRCE}}(y, h(\vec{x}; \vec{\theta})) = \ln \bigl( 1 + \exp \bigl( - y (\vec{x} \cdot \smash{\vec{\theta}} ) \bigr) \bigr) \enspace,
\]
which interprets the model output as a \emph{probabilistic classification} $\Prob(y = 1 | \hat{y}) = \smash{\frac{1}{1 + \exp(-\hat{y})}}$.
Note that for $\hat{y} \not\approx 0$, $\loss_{\mathrm{LRCE}}(y, \hat{y}) \approx \loss_{\mathrm{hinge}}(y, \hat{y})$, and logistic regression may also be viewed as a \emph{convex relaxation} of hard classification, as $\loss_{\mathrm{01}}(y, \hat{y}) \leq \smash{\frac{1}{\ln(2)}}\loss_{\mathrm{LRCE}}$ (perhaps more naturally, the $\smash{\frac{1}{\ln(2)}}$ constant vanishes if we measure cross entropy in \emph{bits} rather than \emph{nats}).
%TODO bm or vec?

In all experiments with \emph{weighted risk values}, we use regularity constraint $\lambda = 4$, and in the experiments with \emph{unweighted risk values}, we take $\lambda = 10$.

\paragraph{Implementation and Computational Resources}

Computation was not a concern on these simple convex linear models; all experiments were run on a %cheap
 low-end %commodity
 laptop with no GPU acceleration.

\Cref{thm:aemm-psg} analytically quantifies the computational complexity %sufficient to train our models to $\varepsilon$-optimality
of $\varepsilon$-EMM, but in our experiments, we simply used standard out-of-the-box first-order methods (adaptive projected gradient descent and SLSQP), as well as derivative-free methods (COBYLA) to train all models.

%\subsection{Experiments in Unweighted Risk}
\subsection{Supplementary Experiments}
\label{sec:appx:exper:exper}

\begin{SCfigure}[1]
%\hspace{-0.25cm}\includegraphics[width=0.334\textwidth,trim={1.7cm 1.0cm 2.0cm 0.8cm}, clip]{paper/figs/race-w01} %&

%\null\hspace{-0.4cm}
\includegraphics[width=0.5\textwidth,trim={0.49cm 0.2cm 1.5cm 0.64cm},clip
]{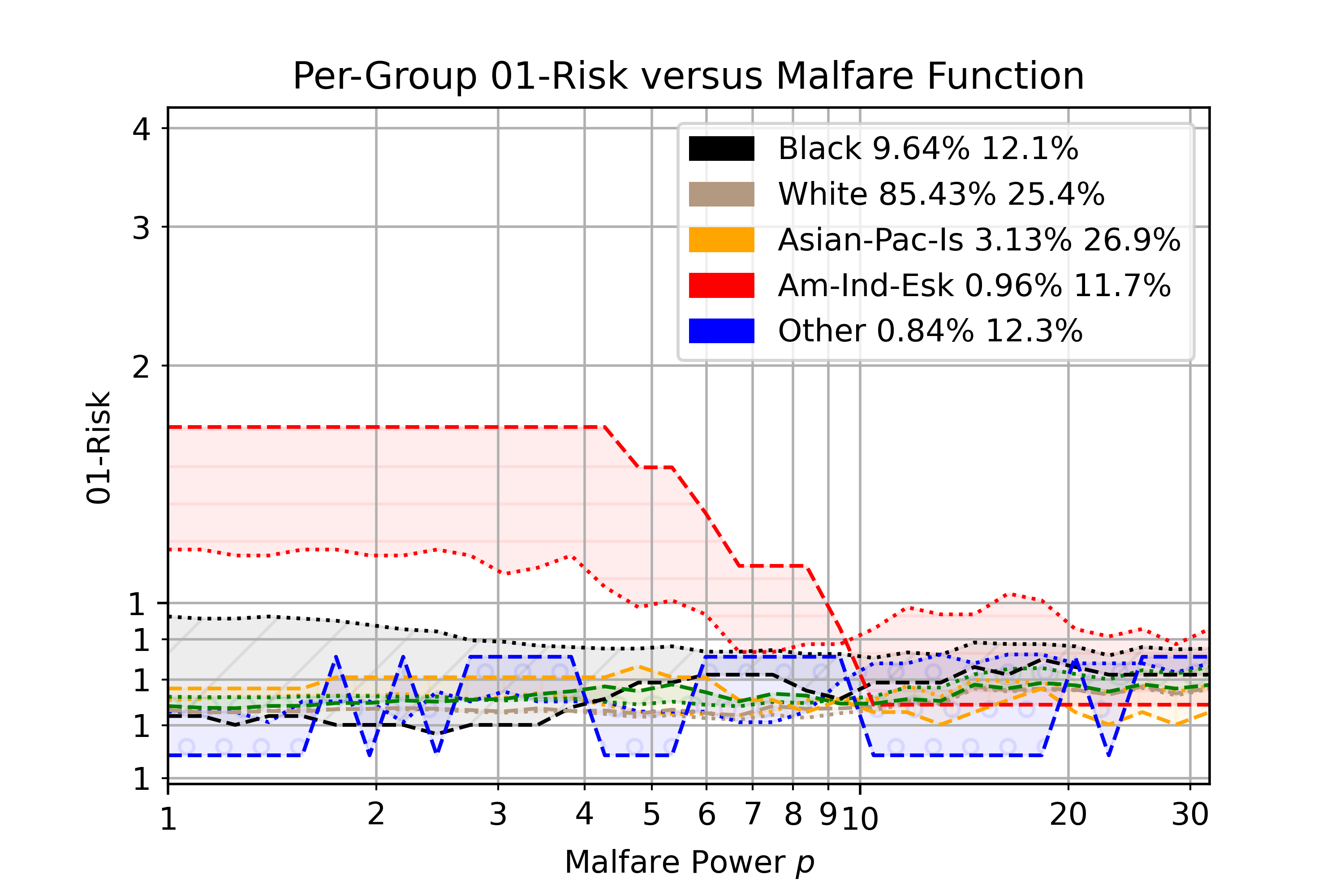}

\vspace{-0.5cm}

\caption{Training and test 0-1 risk, per-group and malfare (green), of \texttt{adult} experiment, as a function of malfare power $p$.
The model is optimized for weighted malfare of weighted hinge-risk, and is thus identical to that of \cref{fig:exp-adult}.
Here hinge-risk is optimized as a \emph{proxy} for the 0-1 risk, so only the reported risk function changes in this figure.
}
\label{fig:exp-adult-w01r}
\end{SCfigure}

\paragraph{0-1 Risk of Weighted SVM}
\Cref{fig:exp-adult-w01r} complements \cref{fig:exp-adult}, reporting the same per-group and malfare statistics, except now on the (similarly weighted) 0-1 risk, rather than the weighted hinge risk.
Here, the interpretation is that the hinge risk is a \emph{convex proxy} for the 0-1 risk, as it would be computationally intractable to optimize the 0-1 risk directly.
Because we optimize hinge risk, but report 0-1 risk, we don't expect to see monotonicity in malfare, and the discontinuity of the 0-1 risk is  %clearly visible
manifest as noise in risk values. % with increasing $p$.
%However,
Nevertheless, if %the 
hinge risk is %an effective
a good proxy for %the
 0-1 risk, we should still see a \emph{general trend} of the classifier becoming fairer (improving high-risk group performance) %and better performing for high-risk groups 
\wrt\ 0-1 risk as it becomes fairer \wrt\ hinge risk, and we do in fact observe this with increasing $p$.
% report and compute risk values 

\begin{figure}

%\vspace{-1cm}

%3 col formatting:
\if 0
\begin{tabular}{ccc}
\hspace{-0.25cm}\includegraphics[width=0.334\textwidth,trim={1.7cm 1.0cm 2.0cm 0.8cm}, clip]{paper/figs/race-hinge} &
\hspace{-0.32cm}\includegraphics[width=0.334\textwidth,trim={1.7cm 1.0cm 2.0cm 0.8cm}, clip]{paper/figs/race-hinge} &
\hspace{-0.32cm}\includegraphics[width=0.334\textwidth,trim={1.7cm 1.0cm 2.0cm 0.8cm}, clip]{paper/figs/race-hinge} \\
\end{tabular}
\fi

%2 col formatting:
\begin{centering}
\!\!\!\!\!\scalebox{1}[0.88]{
\begin{tabular}{ccc}
%\hspace{-0.25cm}
\includegraphics[width=0.48\textwidth,trim={1.7cm 1.0cm 2.0cm 0.8cm}, clip]{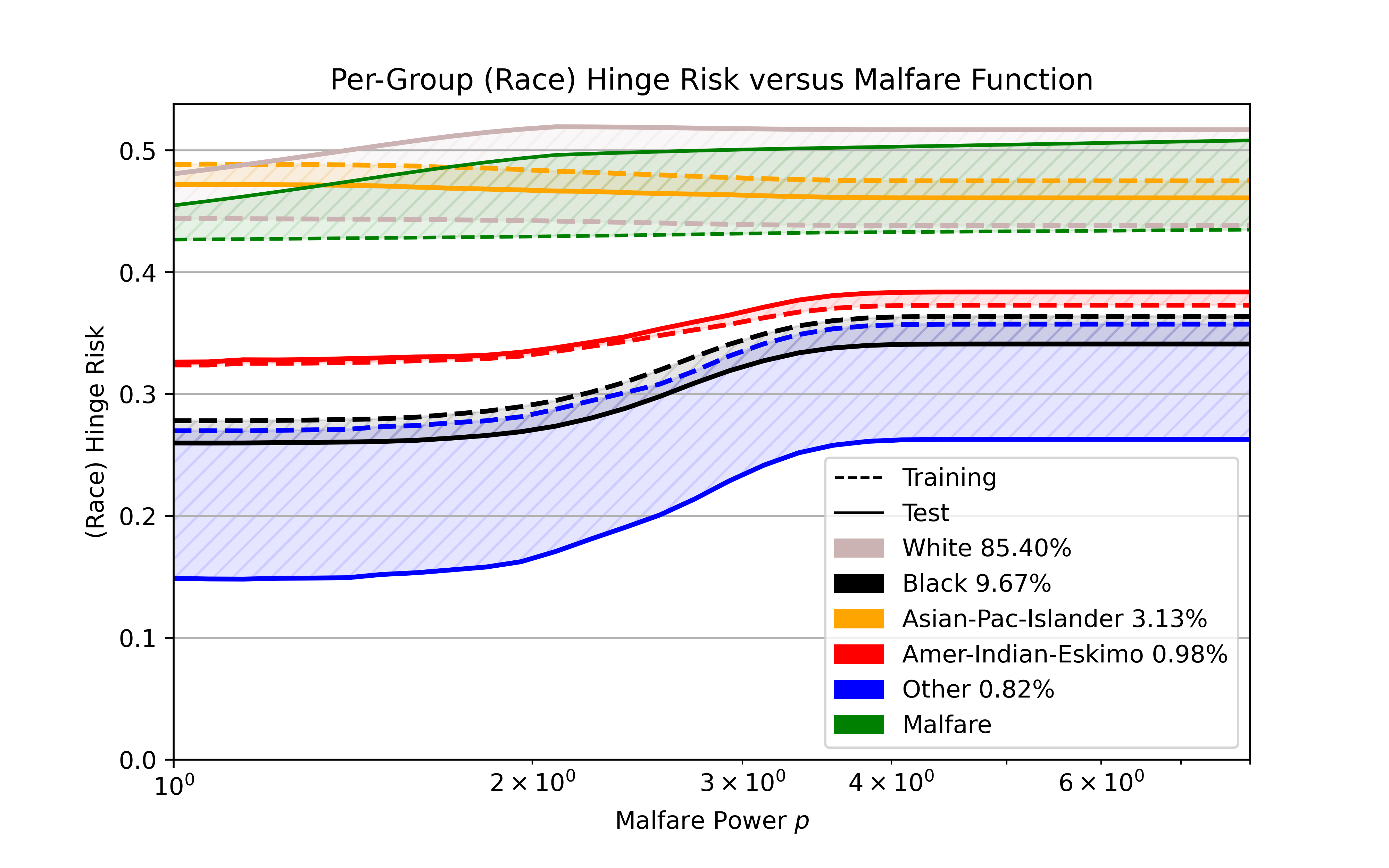} &
%\hspace{-0.32cm}
\includegraphics[width=0.48\textwidth,trim={1.7cm 1.0cm 2.0cm 0.8cm}, clip]{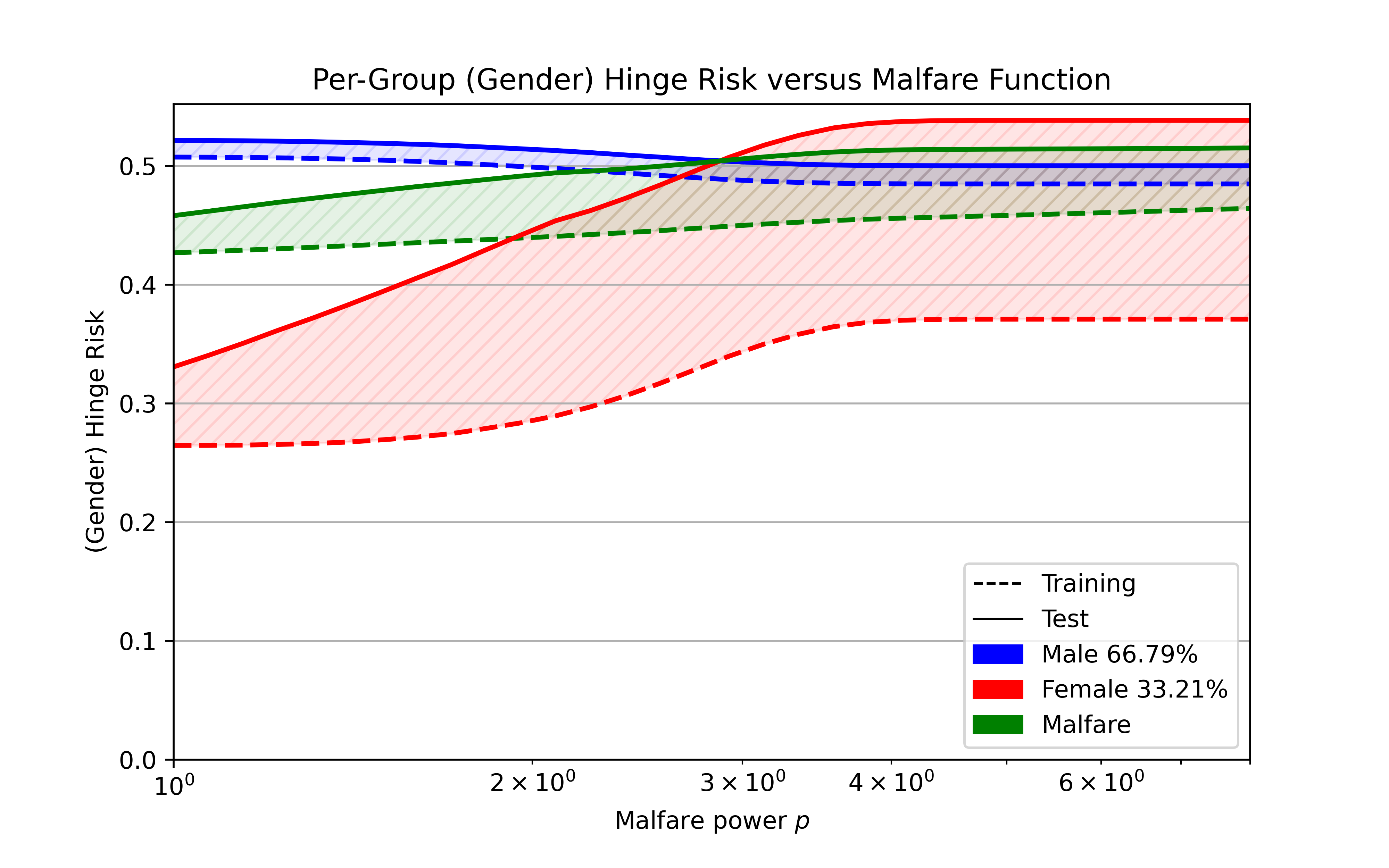} \\
%\hspace{-0.32cm}
%\hspace{-0.25cm}
\includegraphics[width=0.48\textwidth,trim={1.25cm 1.0cm 2.0cm 0.8cm}, clip]{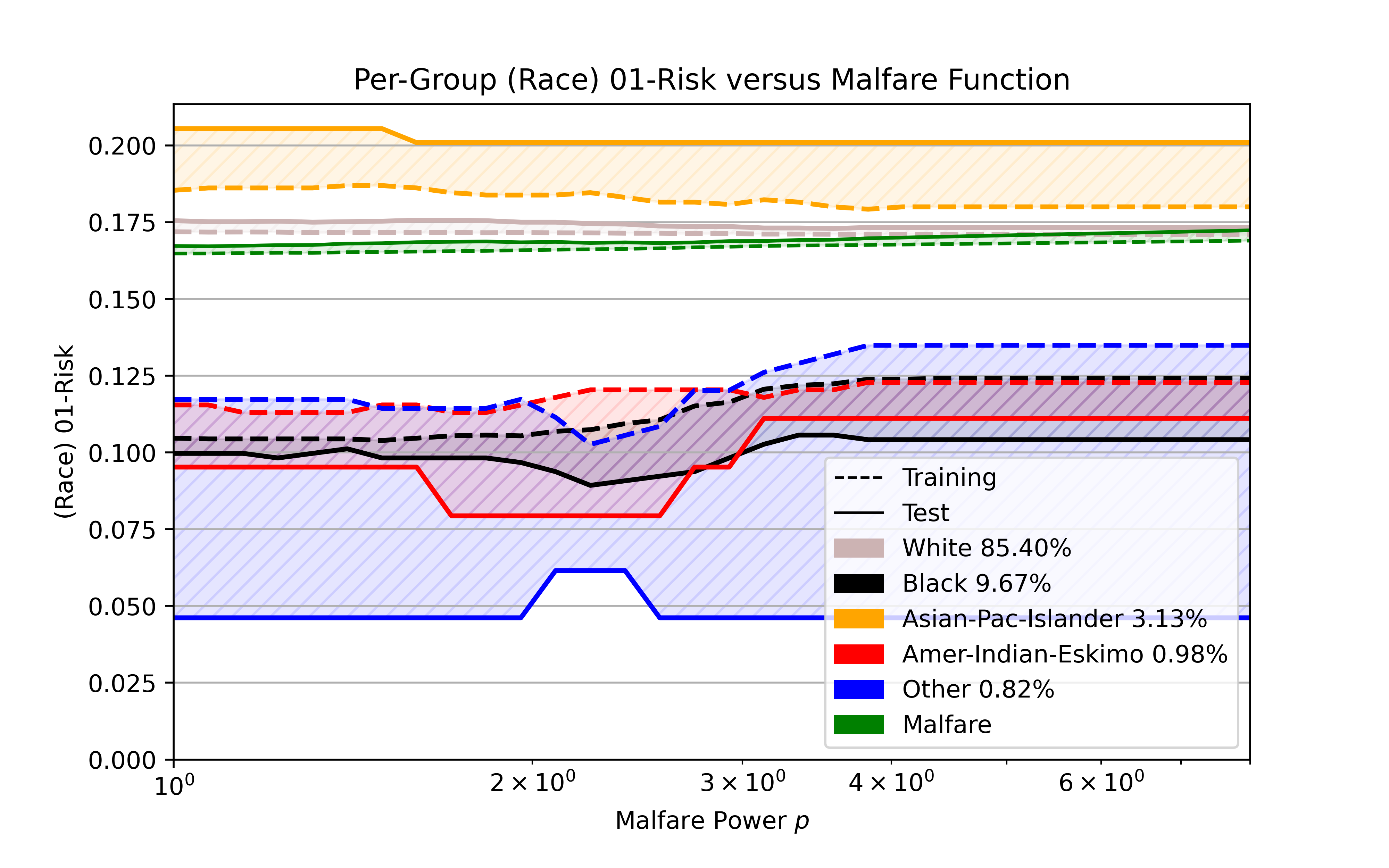} &
%\hspace{-0.32cm}
\includegraphics[width=0.48\textwidth,trim={1.25cm 1.0cm 2.0cm 0.8cm}, clip]{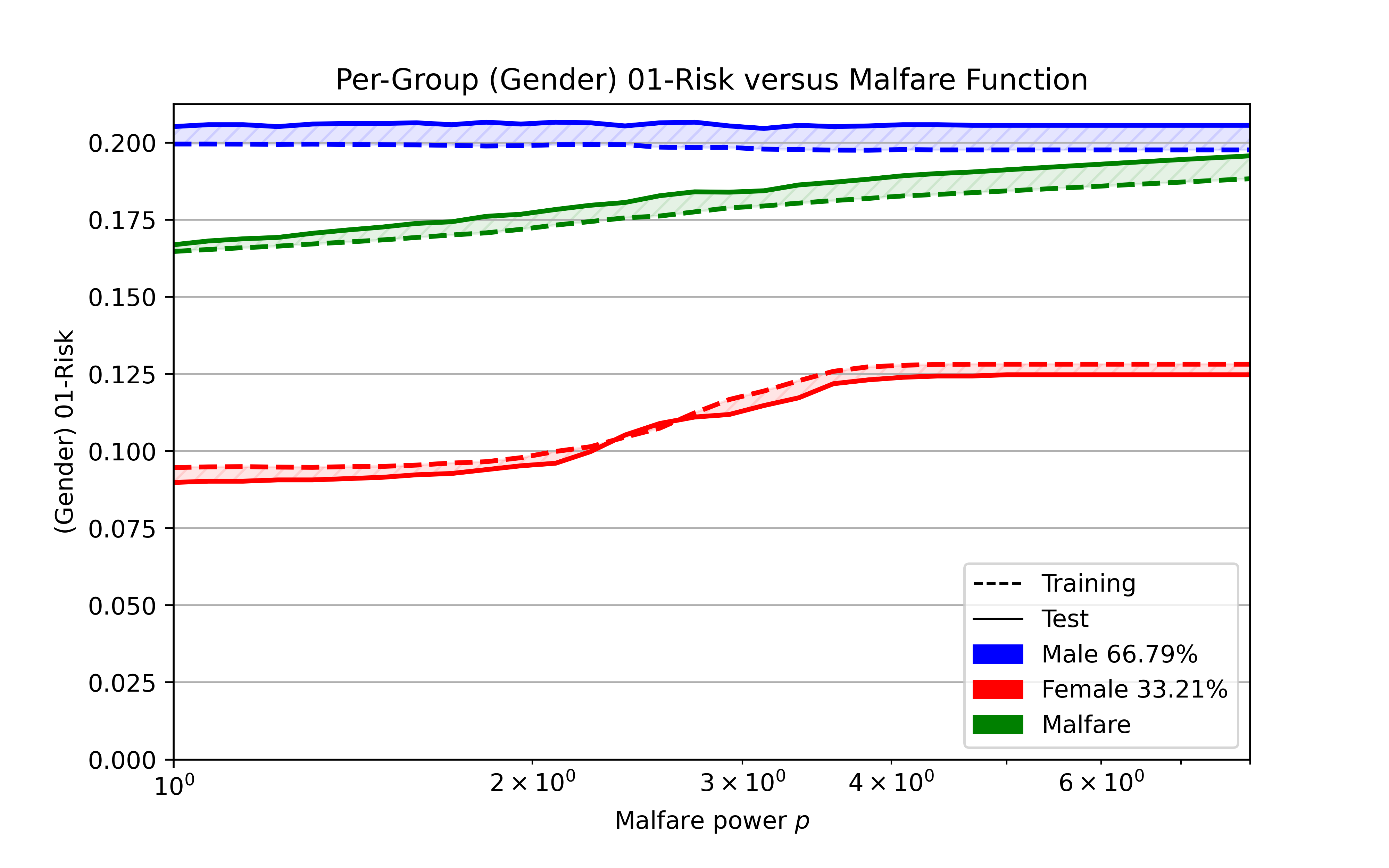} &
%\hspace{-0.32cm}
\end{tabular}} \\
\end{centering}
%\vspace{-0.4cm}
\caption{\small{}  %In this experiment on the lauded \texttt{adult} dataset, the task is to predict whether income is above or below \$50,000/year.
Unweighted linear SVM experiments on \texttt{adult} dataset, with groups split by race (left) and gender (right), malfare and risk plotted against $p$.
\if 0
use all continuous features, and 1-hot encodings of discrete features, % except gender and race 
to predict the target.
\fi
%\\
The upper row depicts hinge-risks and malfare on hinge-risks, and the lower row depicts the 0-1 risks and malfare on 0-1 risks (of the models trained on hinge-risk).
All plots show training and test per-group risk and malfare, %, as well as weighted average loss (utilitarian malfare) and $\Malfare_{p}(\cdot; \wv)$, 
as a function of $p$, with shaded regions depicting train-test gaps.
%We regret that the archaic dataset does not contain finer-grained ethnoracial and gender information.
}
\label{fig:exp-adult-uw}
\if 0
We then plot training and test accuracy within each group As a function of the objective which is the power mean $p$, with weights matching empirical frequencies in the dataset.
%
Note: data in $\R^{104}$, reg constr $\lambda = 10$,  z-score normalization, thus $\ERade_{m}(\F, \bm{x}) \leq \sqrt{\frac{\norm{x}_{2}^{2} \lambda}{m}} \approx \sqrt{\frac{}{m}} \approx 0.22$
\fi
\end{figure}

\paragraph{Unweighted SVM}
These experiments are quite similar to those of \cref{fig:exp-adult} and \cref{fig:exp-adult-w01r}, except here we optimize the malfare of, and report the values of, the \emph{unweighted} hinge risk.
In these experiments, we also take regularity constraint $\norm{\vphantom{\theta}\smash{\vec{\theta}}}_{2} \leq \lambda = 10$, and report the hinge and 0-1 risks and malfares, using \emph{race} and \emph{gender} groups.
As such, the objective is to minimize the $\Malfare_{p}(\cdot; \wv)$ malfare of per-group hinge-loss, using per-group-frequencies as malfare weights, i.e., %thus the objective here is
%\\
%$\displaystyle
%\ \ \ \ \ 
\[
\hat{h} \doteq \smash{\argmin_{h \in \HC}} \Malfare_{p}\left( i \mapsto \ERisk(h; \LossFunction_{\mathrm{hinge}}, \bm{z}_{i}); \wv \right) \enspace.
\]
%\vspace{-0.15cm}
%$

With both gender and race, we see significantly variations in model performance between groups.
We stress that group size and affluence are not directly correlated with model accuracy; for instance, here we see that model performance on the (generally affluent)
\emph{Male}, \emph{white}, and \emph{Asian} populations is relatively poor, due to greater income homogeneity within these groups (in direct contrast to the \emph{weighted} experiments).
%This is due to their  %are significantly less accurate than for the other groups. 
%This occurs despite the male and Asian white populations being in both cases the majority of the training set data set. In fact it is the absence of these populations that makes them difficult to predict as the group are more heterogeneous with respect to income.

In all cases, we see that increasing $p$ improves the \emph{training set performance} of the model on the high-risk (inaccurate) groups (male, white, and Asian), %(male Asian and white) 
at the cost of significant performance degradation for the more accurate groups.  %remaining groups.
However, the trend does not always hold in \emph{test set} performance, since raising $p$ increases the relative importance of \emph{high-risk subpopulations} in training, which leads to increased overfitting. %upon examination of test set performance police Reserve some interesting trends. 
This highlights the phenomenon of \emph{overfitting to fairness}, as we see that improved training set malfare does not necessarily translate to the test set. % appears to be fair this parent is fails to generalize to the training set.
%In the next section, we develop a general theory of f
%TODO continue!
%Despite this note now 
%Note that the generalization properties of SVM are well understood (see X), and overfitting occurs in this model, largely because populations are very small so the Rademacher diverges involved how large

%\begin{SCfigure}[0.5]
\begin{figure}
\centering
\begin{subfigure}[b]{0.482\textwidth}{}
%\null\hspace{-0.4cm}\
\includegraphics[width=1\textwidth,trim={1.08cm 0.2cm 1.5cm 0.64cm},clip
]{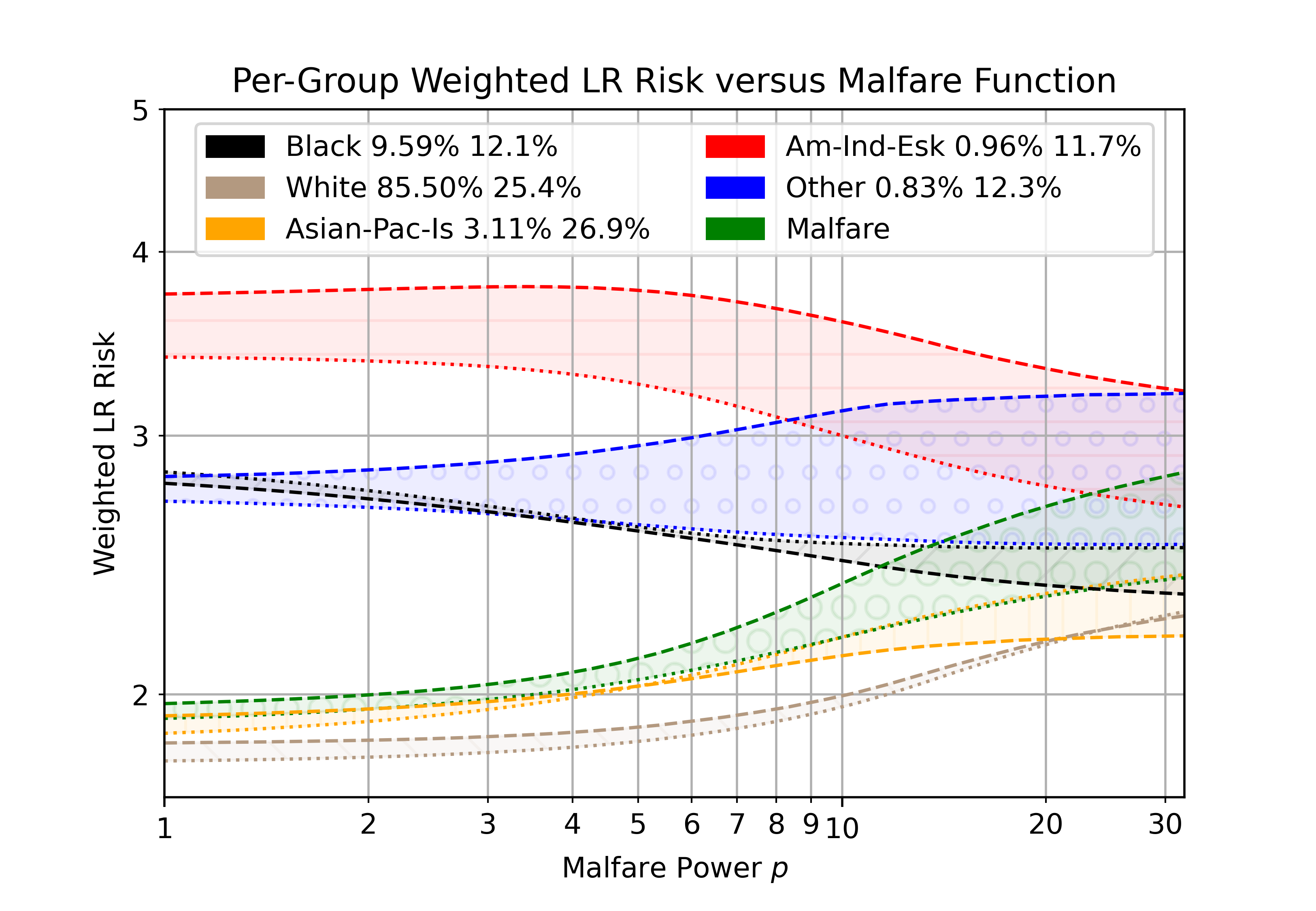}
%\caption{Weighted CE Risk (Logistic Regression)}
%\label{fig:adult-wlr:lr}
\end{subfigure}
\begin{subfigure}[b]{0.49\textwidth}
\includegraphics[width=1\textwidth,trim={0.72cm 0.2cm 1.5cm 0.64cm},clip]{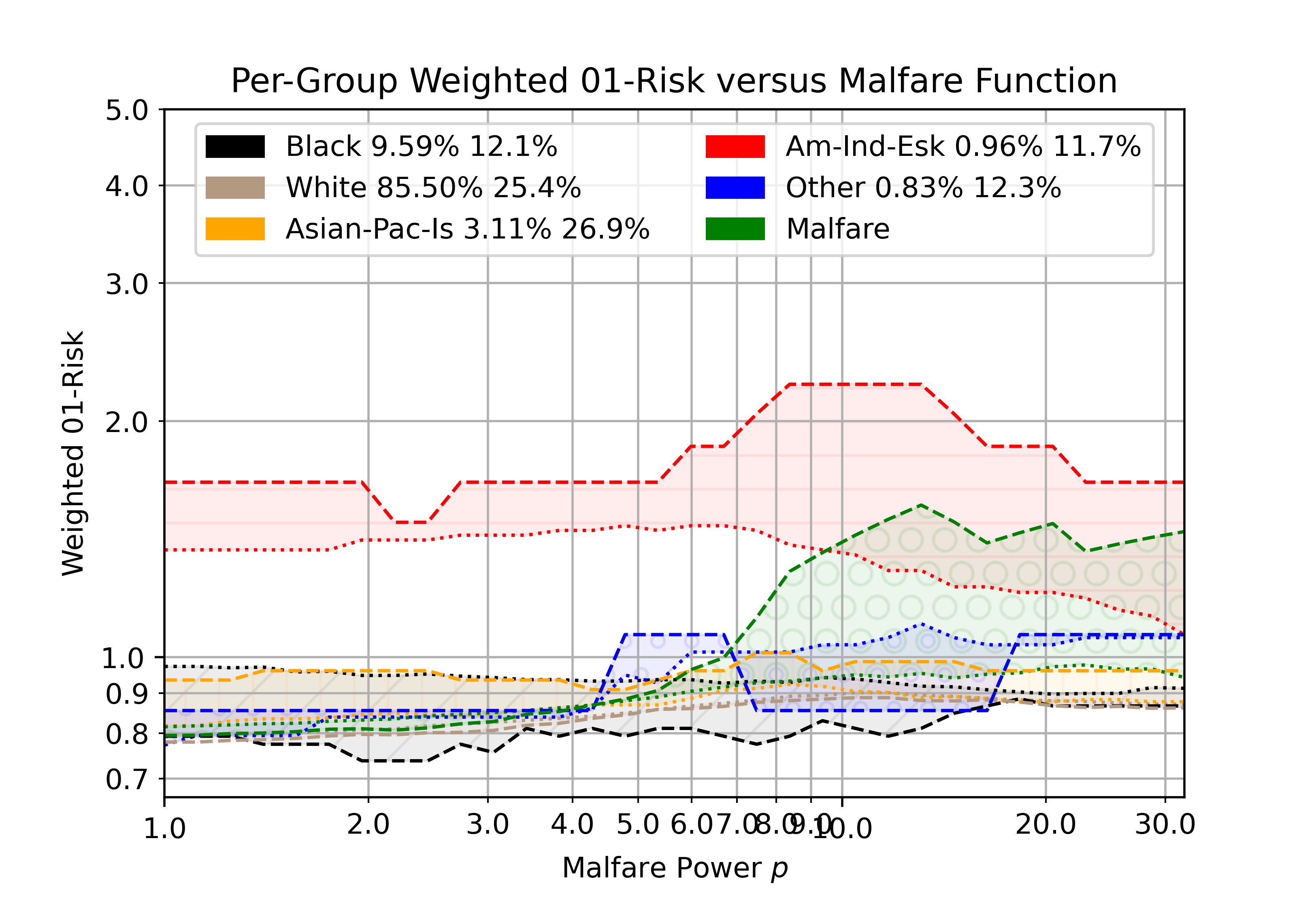}
%\caption{Weighted 0-1 Risk (Logistic Regression)}
%\label{fig:adult-wlr:01}
\end{subfigure}

\caption{Experiments on \texttt{adult} dataset on race groups, with \emph{weighted logistic regression} malfare objective.
%TODO explain
}
\label{fig:exp-adult-wlr}
\end{figure}

\paragraph{Logistic Regression Experiments}
\Cref{fig:exp-adult-wlr} complements the previous experiments, where now we optimize malfare of (weighted) \emph{cross entropy risk} of logit predictors, where weights are chosen as in \cref{fig:exp-adult}, i.e., we optimize
\[
\hat{h} \doteq \smash{\argmin_{h \in \HC}} \Malfare_{p}\left( i \mapsto \mathsmaller{\frac{1}{\bm{b}_{i}}} \ERisk(h; \LossFunction_{\mathrm{LRCE}}, \bm{z}_{i}); \wv \right) \enspace.
\]
We draw essentially the same conclusions as with the hinge risk: malfare minimization yields to better training performance of the model for high-risk (Black, native American, and other) groups, and better test-performance, except in the \emph{other} group, which is tiny and badly overfit.

\fi

\end{document}